\documentclass{article}
\usepackage{graphicx} 

\usepackage[utf8]{inputenc}
\usepackage{amsmath,amsfonts,amssymb}
\usepackage{mathtools}
\usepackage{stackrel}

\usepackage{xspace}

\usepackage{enumitem}
\usepackage[numbers,sort&compress]{natbib}

\usepackage{url}
\usepackage{amsthm}
\usepackage{mathrsfs}
\usepackage{bbm}
\usepackage{dsfont}
\usepackage{color}
\usepackage{comment}
\usepackage[top=1in,bottom=1in,left=1in,right=1in]{geometry}

\newtheorem{theorem}{Theorem}[section]

\newtheorem{proposition}[theorem]{Proposition}
\newtheorem{lemma}[theorem]{Lemma}

\newtheorem{condition}{Condition}

\newtheorem*{remark}{Remark}

\newtheorem{definition}{Definition}

\usepackage[normalem]{ulem}
\usepackage{hyperref}
\usepackage{cleveref}[2012/02/15]
\usepackage{caption}
\usepackage{nicefrac}

\newcommand{\cT}{\mathcal{T}}

\newcommand{\cB}{\mathcal{B}}
\newcommand{\cA}{\mathcal{A}}
\newcommand{\cE}{\mathcal{E}}
\newcommand{\cF}{\mathcal{F}}
\newcommand{\bbE}{\mathbb{E}}
\newcommand{\bbP}{\mathbb{P}}
\newcommand{\bbR}{\mathbb{R}}
\newcommand{\pa}[1]{\left(#1\right)}
\newcommand{\ac}[1]{\left\{#1\right\}}
\newcommand{\cro}[1]{\left[#1\right]}

\newcommand{\condsignal}{\hyperref[as:signal]{\texttt{C-Signal}}\xspace}
\newcommand{\condinvariance}{\texttt{Independent-Sampling}\xspace}
\newcommand{\condmoment}{\hyperref[as:moment]{\texttt{C-Moment}}\xspace}
\newcommand{\condvariance}{\hyperref[as:gen]{\texttt{C-Variance}}\xspace}
\newcommand{\condmomentWR}{\hyperref[as:moment:without:replacement]{\texttt{C-Variance-Permutation}}\xspace}
\newcommand{\condinvarianceWR}{\texttt{Permutation Sampling}\xspace}
\newcommand{\HidSub}{\hyperref[mod:HS]{(HS)\xspace}}
\newcommand{\HidSubI}{\hyperref[mod:HS-I]{(HS-I)\xspace}}
\newcommand{\HidSubP}{\hyperref[mod:HS-P]{(HS-P)\xspace}}

\newcommand{\StoBlo}{\hyperref[mod:SBM]{(SBM)\xspace}}
\newcommand{\StoBloI}{\hyperref[mod:SBM-I]{(SBM-I)\xspace}}
\newcommand{\StoBloP}{\hyperref[mod:SBM-P]{(SBM-P)\xspace}}

\newcommand{\ToeSer}{\hyperref[mod:TS]{(TS)\xspace}}
\newcommand{\ToeSerI}{\hyperref[mod:TS-I]{(TS-I)\xspace}}
\newcommand{\ToeSerP}{\hyperref[mod:TS-P]{(TS-P)\xspace}}

\usepackage{xcolor} 

\newcommand{\1}{\mathbf{1}}

\title{Low-degree lower bounds via almost orthonormal bases}

\author{Alexandra Carpentier\footnote{Institut für Mathematik~---~Universität Potsdam, Potsdam, Germany. Alexandra.Carpentier@uni-potsdam.de},  \ Simone Maria Giancola\footnote{Laboratoire de Math\'ematiques d'Orsay, Universit\'e Paris~-~Saclay, CNRS, France. simonegiancola09@gmail.com }, \  Christophe Giraud\footnote{Laboratoire de Math\'ematiques d'Orsay, Universit\'e Paris~-~Saclay, CNRS, France. Christophe.Giraud@universite-paris-saclay.fr} \ 
and Nicolas Verzelen\footnote{INRAE, Institut Agro, MISTEA,
Univ. Montpellier, France. Nicolas.Verzelen@inrae.fr} }

\date{}
\begin{document}

\maketitle
\begin{abstract}
Low-degree polynomials have emerged as a powerful paradigm for providing evidence of statistical-computational gaps across a variety of high-dimensional statistical models~\cite{SurveyWein2025}. For detection problems, where the goal is to test a planted distribution $\mathbb{P}'$ against a null distribution $\mathbb{P}$ with independent components, the standard approach is to bound the advantage using an $L^2(\mathbb{P})$-orthonormal family of polynomials. However, this method breaks down for estimation tasks or more complex testing problems where $\mathbb{P}$ has some planted structure, so that no simple $L^2(\mathbb{P})$-orthogonal polynomial family is available. To address this challenge, several technical workarounds have been proposed~\cite{SchrammWein22,SohnWein25}, though their implementation can be delicate.

In this work, we propose a more direct proof strategy. Focusing on random graph models, we construct a basis of polynomials that is almost orthonormal under $\mathbb{P}$, in precisely those regimes where statistical-computational gaps arise. This almost orthonormal basis not only yields a direct route to establishing low-degree lower bounds, but also allows us to explicitly identify the polynomials that optimize the low-degree criterion. This, in turn, provides insights into the design of optimal polynomial-time algorithms.
We illustrate the effectiveness of our approach by recovering known low-degree lower bounds, and establishing new ones for problems such as hidden subcliques, stochastic block models, and seriation models.
\end{abstract}

\section{Introduction}\label{sec:introduction}
In high-dimensional statistics, a central objective is to design computationally efficient estimation  or test  procedures that achieve the best possible statistical performance. However, in many fundamental problems, such as sparse PCA, planted clique, or clustering, the best known polynomial-time algorithms fail to attain the performance that is provably achievable by the optimal estimators. This gap between the information-theoretic optimum and the best polynomial-time performance, known as a statistical-computational gap, has been conjectured to occur broadly.
From this perspective, the performance of an efficient algorithm should be compared not to the information-theoretic  optimum, but to the best achievable by any polynomial-time method, leading naturally to the problem of proving lower bounds for polynomial-time algorithms.
Because the instances of statistical problems are random, worst-case complexity theory (classes P, NP, etc.) is not the appropriate tool to quantify their difficulty. Evidence of computational hardness is rather obtained by restricting attention to specific classes of algorithms, e.g.\ the sum-of-squares (SoS) hierarchy~\cite{HopkinsFOCS17,Barak19}, statistical query algorithms~\cite{kearns1998efficient,brennan2020statistical}, or low-degree polynomials~\cite{hopkins2018statistical,KuniskyWeinBandeira,SchrammWein22}; or by exhibiting geometric obstructions such as the overlap gap property~\cite{gamarnik2021overlap}; or by reductions from problems that are believed to be hard~\cite{brennan2020reducibility,pmlr-v30-Berthet13,brennan2018reducibility}.

Among these approaches, low-degree polynomial (LD) lower bounds have become particularly popular, as they currently provide the sharpest known computational barriers for many detection problems (for instance, sparse PCA~\cite{ding2024subexponential}, spiked tensor models~\cite{Hopkins17,KuniskyWeinBandeira}, or community detection~\cite{Hopkins17}), and estimation problems (for instance, planted coloring~\cite{kothari2023planted}, submatrix estimation~\cite{SchrammWein22}, recovery of dense cycles~\cite{mao2023detection}, or estimation in stochastic block models and graphon models~\cite{luo2023computational,SohnWein25}), see~\cite{SurveyWein2025} for a recent overview. In the LD framework, the only estimators and test statistics under consideration are polynomials of degree at most $D$ in the entries of the data. It is conjectured that, for a broad range of problems, polynomials of degree $O(\log n)$ perform as well as the best polynomial-time algorithms. Hence, showing that all polynomials of degree $O(\log n)$ fail is regarded~\cite{KuniskyWeinBandeira} as strong evidence that no polynomial-time algorithm succeeds.
Moreover, LD lower bounds are closely related to other notions of computational hardness, including statistical query lower bounds~\cite{brennan2020statistical}, the Franz--Parisi criterion, which originates from the study of free energy landscapes in statistical physics~\cite{bandeira2022franz}, and the performance of approximate message passing algorithms~\cite{montanari2025equivalence}.

In this work, we consider testing and estimation problems on random graph models with latent structure.
We observe an undirected graph with $n$ nodes, encoded in the adjacency matrix $Y^{\star}= (Y^{\star}_{ij})_{1\leq i < j\leq n}\in \bbR^{n(n-1)/2}$, where $Y^\star_{ij}$ equals 1 when there is an edge between $i$ and $j$, and 0 otherwise. We consider a latent structure model, where for some $q\in (0,1/2]$, some symmetric matrix $\Theta\in \mathbb{R}^{n\times n}$, and some unobserved latent assignment $z\in [n]^{n}$, the $Y^\star_{ij}$s are sampled independently conditionally on $z$, with conditional distribution
\begin{equation}\label{eq:definition:sample:graph}
 \mathbb{P}\left[Y^{\star}_{ij} = 1 | z \right]= q+ \Theta_{z_i z_j}\  ; \quad \quad  
 \mathbb{P}\left[Y^{\star}_{ij} = 0 | z \right]= 1- q-  \Theta_{z_iz_j}\enspace .
\end{equation}
It is more standard to consider the matrix $\Theta^{\star}$ defined  by $\Theta^\star_{ij}=q+ \Theta_{ij}$, but the parametrization with $\Theta$ will be more convenient for our purpose.  We consider two different sampling schemes for this latent assignment vector $z$:

\begin{condition}[\texttt{Independent sampling}]\label{as:perinv}
For $i=1,\ldots, n$, the $z_i$'s are sampled independently and uniformly on $[n]$. 
\end{condition}

\begin{condition}[\texttt{Permutation sampling}]\label{as:perinv:without:replacement}
The vector $z=(z_i)_{i=1,\ldots,n}$ is distributed as the uniform permutation over $[n]$.
\end{condition}

Under \condinvarianceWR, $\mathbb{E}[Y^{\star}|z]$ is distributed as a random permutation of $\Theta^\star$ whereas, under  \condinvariance, $\mathbb{E}[Y^{\star}|z]$ corresponds to some sampling with replacement of $\Theta^\star$. 
Importantly, the distribution of $Y^{\star}$ is permutation-invariant in both cases.
Throughout the manuscript, we assume that $q\in (0,1/2]$, thereby leaving out the case where the graph is very dense ($q\in (1/2,1)$). Extension to denser graphs is possible at the price of different formulas. As we are mostly interested in the sparse regime, this assumption is mild.
This model encompasses three classical random graph models that depend on some parameter $\lambda\in [0,1-q]$ and some integer $k\in [n]$.

\begin{itemize}
\item[(HS)]\phantomsection\label{mod:HS} {\bf Hidden subclique}.
 Set    $\Theta_{ij}= \lambda \1\{i\leq k\}\1\{j\leq k\}$. When two nodes $i$ and $j$ belong to the hidden subclique (that is, $z_i\leq k$ and $z_j\leq k$), then the connection probability equals $\lambda+q$. Under \condinvariance, each node belongs to the hidden subclique with probability $k/n$, whereas, under \condinvarianceWR, the size of the hidden subclique is exactly $k$. We refer to the former model as (HS-I)\phantomsection\label{mod:HS-I}, and to the latter as (HS-P)\phantomsection\label{mod:HS-P}.
\item[(SBM)]\phantomsection\label{mod:SBM} {\bf Stochastic block model}. Assume that $n/k$ is an integer.  Then, we set   $\Theta_{ij}= \lambda \1\{\lceil \tfrac{i}{k}\rceil = \lceil \tfrac{j}{k}\rceil  \}$, where  $\lceil x\rceil$ stands for the upper integer part. Under \condinvariance  (SBM-I)\phantomsection\label{mod:SBM-I}, $Y^\star$ is sampled as an SBM with $K=n/k$ groups of random sizes. Under \condinvarianceWR (SBM-P)\phantomsection\label{mod:SBM-P},
$Y^\star$ is sampled as an SBM with $K=n/k$ groups of size exactly $k$.
\item[(TS)]\phantomsection\label{mod:TS} {\bf T\oe plitz seriation}. For simplicity, assume that $k$ is even. We have $\Theta_{ij}= \lambda\1\{|i-j|\leq k/2\}$, where the label $z\in [n]^{n}$ is either sampled uniformly at random in $[n]^{n}$ (TS-I)\phantomsection\label{mod:TS-I}, or is sampled uniformly in the set of permutations (TS-P)\phantomsection\label{mod:TS-P}. 
\end{itemize}

\subsubsection{Our contributions.} 
\begin{enumerate}
    \item We present a novel approach for proving low-degree lower bounds for testing and estimation in random graph models with planted structure. Our method relies on constructing a new basis of low-degree polynomials invariant under vertex relabeling, which is \emph{almost orthonormal} when the planted structure (i.e.\ $\Theta$) is small. Typically, this property holds as long as the signal is weak enough to prevent non-trivial recovery using degree-$\log(n)$ polynomials. The technique offers a simpler systematic framework for proving low-degree bounds, particularly effective when the latent vector $z$ is not i.i.d., thereby opening the door to addressing previously unsolved and challenging settings. An additional advantage of this framework is that it allows us to explicitly identify the polynomials that optimize the low-degree criterion, providing insights for the design of optimal polynomial-time algorithms, see Section~\ref{subsec:insights} and open problem \#6 in~\cite{SurveyWein2025}.
\item We establish two new low-degree lower bounds for testing and estimation in the general model~\eqref{eq:definition:sample:graph}, covering the three models \HidSub, \StoBlo, and \ToeSer. These bounds yield new results for testing between different planted structures in these three models, as well as new results for estimation in the \ToeSerI, \ToeSerP, \HidSubP\ and \StoBloP\ models. We also recover several known results for \HidSubI\ and \StoBloI, up to logarithmic factors. Note that, in this paper, we assume throughout that $|\Theta|_{\infty}$ is of smaller order than $q$ up to a polynomial in $D$, which is not necessary, but which significantly simplifies our analysis, as we want to have a generic analysis for all models.
\end{enumerate}

\subsubsection{A glimpse at our technique for deriving LD bounds.}
We now give a brief overview of our approach for deriving LD bounds; full details appear in Section~\ref{s:nearort}.
To simplify the forthcoming analysis, from now on, we work  with the centered adjacency matrix $Y$ defined by 
\begin{equation}\label{eq:definition:Y}
Y_{ij} = Y^{\star}_{ij}- q \enspace , \text{ for any }1\leq i < j \leq n\enspace .
\end{equation}
Let $\bbP$ denote the distribution under the null hypothesis $H_{0}$ in the testing setting, or the distribution of the data in the estimation setting. Proving LD lower bounds amounts to establishing an upper bound of the form (see Section~\ref{s:set} for further details)
\begin{equation}\label{eq:intro:LD1}
     \sup_{f: \mathrm{deg}(f)\leq D}\frac{\bbE\cro{xf(Y)}^2}{{\bbE\cro{f(Y)^2}}} \leq  \bbE\cro{x}^2(1+o(1))\enspace ,
\end{equation}
where $x$ is the likelihood ratio $x={d\bbP_{H_{1}}\over d\bbP}(Y)$ in testing $\bbP$ against $\bbP_{H_1}$, or the target quantity in estimation.  The supremum ranges over all polynomial functions $f$ of degree at most $D$, and  $\bbE\cro{x}^2=1$ in testing problems. The value $\bbE\cro{x}^2$ corresponds to the supremum for $D=0$, i.e., when restricting to trivial constant polynomials.

In a simple detection setting with $\Theta = 0$ under $H_{0}$, the supremum in~\eqref{eq:intro:LD1} can be evaluated explicitly.  
Indeed, the monomials
  $\ac{\phi_{S}(Y):=\prod_{(i,j)\in S} \frac{Y_{ij}}{\sqrt{q(1-q)}}:S\in\mathcal{S}_{\leq D}}$, indexed by ~\\ $\mathcal{S}_{\leq D}=\ac{S\subset \ac{(i,j):1\leq i<j\leq n} : |S|\leq D}$, form an $L^2(\bbP)$-orthonormal basis  for degree-$D$ polynomials. Defining $\hat x_{S}:= \bbE\cro{x\phi_{S}(Y)}$, we obtain
\begin{equation}\label{eq:intro:LD2}
 \sup_{f: \mathrm{deg}(f)\leq D}\frac{\bbE\cro{xf(Y)}^2}{{\bbE\cro{f(Y)^2}}} = \sup_{(\alpha_{S})_{S\in\mathcal{S}_{\leq D}}} \frac{\pa{\sum_{S\in\mathcal{S}_{\leq D}}\alpha_{S} \hat x_{S}}^2}{\sum_{S\in\mathcal{S}_{\leq D}}\alpha_{S}^2}=\|(\hat x_{S})_{S\in\mathcal{S}_{\leq D}}\|^2\enspace ,
\end{equation}
so the problem reduces to comparing $\|(\hat x_{S})_{S\in\mathcal{S}_{\leq D}}\|^2$ with~$\bbE\cro{x}^2$.
This is the classical approach first derived in \cite{HopkinsFOCS17,Hopkins17,KuniskyWeinBandeira}.
However, the convenient simplification fails for estimation problems or more complex testing settings with $\Theta \neq 0$ under $H_{0}$, where no simple explicit $L^2(\bbP)$-orthonormal basis for low-degree polynomials is available.  
Two strategies have been proposed to address this issue:
\begin{enumerate}
\item The approach of~\cite{SchrammWein22} applies an affine transformation to $Y$, and then uses a partial Jensen inequality, integrating over the latent variable inside the square, i.e., schematically
\begin{equation}\label{eq:intro:LD3}
\bbE\cro{f(Y)^2} \geq \bbE\cro{\pa{\bbE_{z}\cro{f(Y)}}^2} =: \|M^{\mathsf{T}} \alpha\|^2 \enspace ,
\end{equation}
yielding an upper triangular matrix $M$ that can be simply inverted.  
The supremum~\eqref{eq:intro:LD1} is then bounded above by $\|M^{-1} \hat{x}\|^2$,  which can be evaluated thanks to the explicit inversion of $M$.
This method has been successfully applied to certain estimation problems in stochastic block models, graphons~\cite{luo2023computational}, and dense cycle recovery~\cite{mao2023detection}.  
However, the integration over $z$ within the square can cause cancellations between symmetric terms, significantly shrinking the $L^2$-norm and leading to suboptimal bounds, see e.g.~\cite{Even25a}.

\item The more powerful method of~\cite{SohnWein25} bypasses the construction of an $L^2(\bbP)$-orthonormal basis by instead building one in the extended space $L^2(\bbP^{W})$, where $\bbP^{W}$ is the distribution of $W = (Y, z)$.  
The task then reduces to finding a minimal norm solution $u$ of an overcomplete system $Mu = \hat{x}$.  
This approach has yielded tight bounds in a variety of problems~\cite{SohnWein25,pmlr-v291-chin25a}, but its applicability can be limited in complex settings, as it requires identifying special solutions of a large overcomplete system.
\end{enumerate}

\medskip
We propose a simpler and more direct method for evaluating the supremum~\eqref{eq:intro:LD1}.  
While constructing an explicit $L^2(\bbP)$-orthonormal basis seems infeasible beyond basic detection problems, we relax the requirement to almost orthonormality.  
Our method is based on two key ideas:

\begin{enumerate}
\item Restrict attention to polynomials $f$ invariant under permutations of the vertex labels, i.e.,
$f(Y_{\sigma}) = f(Y)$ for any permutation $\sigma$ of $[n]$, where $[Y_{\sigma}]_{ij} = Y_{\sigma(i),\sigma(j)}$.  
Indeed, the supremum in~\eqref{eq:intro:LD1} is achieved for $f$ invariant under permutations.
Such symmetry properties have been exploited in previous works~\cite{semerjian2024matrix,kunisky2024tensor,montanari2025equivalence} where the authors leverage some invariance under permutations or orthogonal transformations.

\item Construct a basis of invariant low-degree polynomials that is almost $L^2(\bbP)$-orthonormal in the weak signal regime, in the sense that
\begin{equation}\label{eq:intro:LD4}
\bbE\cro{\pa{\sum_{S\in \mathcal{S}_{\leq D}}\alpha_{S}\phi_{S}(Y)}^2} = \|(\alpha_{S})_{S\in\mathcal{S}_{\leq D}}\|^2(1+o(1))\enspace ,
\end{equation}
typically when $|\Theta|_{\infty}=\lambda$ is small.
\end{enumerate}

To achieve the key property~\eqref{eq:intro:LD4}, we start from the basis $\{\phi_{S} : S \in \mathcal{S}_{\leq D}\}$, adjust it to ensure $\bbE\cro{\phi_{S}(Y)\phi_{S'}(Y)} = 0$ for many (but not all) distinct $S, S'$, and then average over permutations of the labels to enforce invariance.  
A central result is that the resulting basis is almost orthonormal for weak signals. Compared with~\cite{SchrammWein22}, our method avoids the potentially suboptimal Jensen step~\eqref{eq:intro:LD3}. Compared with~\cite{SohnWein25}, computations are simpler, which may facilitate its application to more intricate problems.
For example, problems where the latent vector $z$ is a permutation can be treated more directly, whereas earlier analyses were considerably more involved~\cite{Even25b}. 
Another important advantage of our direct approach is that it allows us to identify the dominant polynomials in~\eqref{eq:intro:LD1}, thereby yielding optimal algorithms for the underlying testing or estimation task.

\subsection{Related literature}\label{subsec:related literature}

\subsubsection{Low-degree polynomials in hypothesis testing and estimation.}
Historically, the low-degree method originated from the study of the sum-of-squares (SoS) semidefinite programming hierarchy~\cite{Barak19}.
The idea of capturing polynomial-time complexity via low-degree polynomials emerged in a sequence of works~\cite{HopkinsFOCS17,Hopkins17,hopkins2018statistical,KuniskyWeinBandeira} on detection problems, namely hypothesis testing under a simple null distribution (typically with independent entries).
The core strategy is to expand the likelihood ratio in a basis orthonormal under the null distribution, and then solve the resulting optimization problem explicitly.
This approach has been successfully applied to a broad range of models, including community detection~\cite{Hopkins17}, spiked tensor models~\cite{Hopkins17,KuniskyWeinBandeira}, sparse PCA~\cite{ding2024subexponential}, and planted subgraph problems~\cite{elimelech2025detecting}, among many others.

In contrast, the literature on complex testing problems is relatively sparse.
Two notable exceptions are~\cite{rush2022easier} and~\cite{kothari2023planted}, which study testing between two different “planted” distributions, each with a distinct type of hidden structure, for example, testing between stochastic block models with different numbers of communities, or between $q$-colorable and $(q+\ell)$-colorable random graphs.
Their proofs adapt techniques from~\cite{SchrammWein22} originally developed for estimation.

Theory for estimation (or “recovery”) has been primarily developed in~\cite{SchrammWein22} and~\cite{SohnWein25}.
The framework of~\cite{SchrammWein22} has been successfully applied to submatrix estimation~\cite{SchrammWein22}, stochastic block models and graphons~\cite{luo2023computational}, and Gaussian mixture models~\cite{Even24}.
It has also been extended to more complex latent variable models by exploiting conditional independence~\cite{Even25a} and weighted dependency graph theory~\cite{Even25b}, yielding lower bounds for challenging settings such as sparse clustering, biclustering, and multiple feature matching.
The more recent work~\cite{SohnWein25} develops a technically more involved but sharper theory, providing exact constants for thresholds and establishing lower bounds for polynomials of degree $D$ as large as fractional powers of $n$.
This approach has been applied to planted submatrix, planted subclique, spiked Wigner, and stochastic block models~\cite{SohnWein25,pmlr-v291-chin25a}.
When there is no detection–recovery gap, i.e., when recovery is as easy as detection, recovery lower bounds can be directly derived from detection bounds.
For problems exhibiting a gap, more sophisticated detection-to-recovery reductions have recently been proposed~\cite{li2025algorithmic,ding2025low}.

The ideas of leveraging symmetries in the data generating distribution and constructing a nearly orthogonal basis first appeared in~\cite{montanari2025equivalence} for the rank-one matrix estimation problem, where the basis is derived from Hermite polynomials. This strategy was further developed in~\cite{kunisky2024tensor} for tensor models such as the spiked tensor model, introducing the tensor cumulant basis of rotationally invariant polynomials, which is nearly orthogonal under the tensor Wigner distribution.  Beyond the difference between the statistical models, \cite{kunisky2024tensor,montanari2025equivalence} only establish the near orthogonality\footnote{Here, near orthogonality means that the eigenvalues of the corresponding Gram matrix are bounded away from $0$ and from $\infty$, whereas almost orthogonality ensures that its eigenvalues are asymptotically close to $1$.} of their basis in specific regimes or asymptotics: the tensor Wigner distribution being a ``pure noise'' model,~\cite{kunisky2024tensor} need to rely on a Jensen-type argument reminiscent of~\cite{SchrammWein22} to consider estimation problems. \cite{montanari2025equivalence} also considered a specific asymptotic regime for their rank-one matrix estimation problem. In both~\cite{montanari2025equivalence,kunisky2024tensor}, the spectrum of the associated Gram matrix is bounded away from zero and infinity, but does not approach 1 as the problem size grows, a key distinction from our setting. In comparison to those two works, we prove that our basis construction is a versatile and simple tool to establish near-optimal LD lower and upper bounds. 
We further elaborate on the connection between our basis construction and~\cite{kunisky2024tensor},  \cite{montanari2025equivalence} in Section~\ref{s:nearort}.

Finally, beyond predicting computational thresholds for polynomial-time algorithms, low-degree polynomials can also provide insight into time complexity in the hard regime.
The low-degree conjecture~\cite{hopkins2018statistical} posits that degree-$D$ polynomials can serve as a proxy for algorithms with runtime approximately $n^{D}$.
Extensions of this framework address optimization problems~\cite{gamarnik2024hardness} and refutation tasks~\cite{kothari2023planted}.
For a comprehensive overview of the low-degree method, its connections to other hardness frameworks, and a broader set of references, we refer to the recent survey~\cite{SurveyWein2025}.

\subsubsection{Hidden subclique.}
The planted clique problem, corresponding to $q = 1/2$ and $p:= \lambda +q = 1$, is a canonical example of a problem exhibiting an information–computation gap.
While the existence of a hidden clique can be detected as soon as $k > 2\log_{2}(n)$ by exhaustively scanning all possible cliques, all known polynomial-time tests fail when $k = o(\sqrt{n})$.
Low-degree hardness for detection in this regime was proven in~\cite{hopkins2018statistical}, adapting arguments from~\cite{Barak19}, and the corresponding hardness of estimation was established in~\cite{SchrammWein22}.

For the planted subclique problem (i.e., $p < 1$),~\cite{SchrammWein22} showed low-degree hardness for recovery when
\begin{equation}\frac{\lambda}{\sqrt{q}}\left( 1 \vee \frac{k}{\sqrt{n}} \right) \leq \log(n)^{-2}\enspace .\end{equation}
This result was refined in~\cite{SohnWein25}, which proved low-degree hardness of recovery for
\begin{equation}\frac{\lambda k}{\sqrt{q(1-q)n}} < e^{-1/2}\enspace .\end{equation}
By analogy with the planted submatrix problem, when $k\gg \sqrt{n}$, recovery is conjectured to be possible above this precise threshold using approximate message passing~\cite{DM15}.

For detection,~\cite{Dhawan25} showed  low-degree failure roughly when $p = o\big( \sqrt{q} n/k^2 \big)$ for $k \geq \sqrt{n}$, and when $p = o(q^{\log_{n}(k)})$ for $k = o(\sqrt{n})$.
Finally,~\cite{elimelech2025detecting} studied the more general case where the hidden subclique is replaced by an arbitrary hidden subgraph.
They found contrasting behaviors depending on the subgraph density: a statistical-computational gap appears only for dense subgraphs, specifically when the subgraph density exceeds the logarithm of its number of nodes.

\subsubsection{Stochastic block model.}
The stochastic block model with connection probabilities $p,q$ scaling as $1/n$ has attracted significant attention since the seminal paper of~\cite{Decelle2011}, which, using tools from statistical physics, conjectured computational hardness of recovery below the Kesten–Stigum (KS) threshold
\begin{equation}\label{eq:KS}
{\lambda k\over \sqrt{\lambda k + nq}} <1\enspace .
\end{equation}
Non-trivial recovery above this threshold was established in~\cite{massoulie2014community,bordenave2015non,AbbeSandon2015a,pmlr-v291-chin25a}.
Low-degree hardness of detection below the KS threshold~\eqref{eq:KS} was proven in~\cite{Hopkins17}; see also~\cite{bandeira2021spectral,kunisky2024low}.

For recovery,~\cite{SohnWein25,pmlr-v291-chin25a,ding2025low} proved low-degree hardness below the KS level~\eqref{eq:KS} when $k \gg \sqrt{n}$ and the polynomial degree $D$ is a fractional power of $n$.
This result was extended to the denser regime with $1/n \ll p,q \ll 1$ and $k \gg \sqrt{n}$ in~\cite{luo2023computational,pmlr-v291-chin25a}.
For $p,q$ of constant order, the same conclusion holds at the modified KS threshold
\begin{equation}{\lambda k\over \sqrt{\lambda k(1-p-q)+nq(1-q)}} <1\enspace .\end{equation}
When $k \leq \sqrt{n}$,~\cite{luo2023computational} established computational hardness for $\lambda = O(\sqrt{q} \log(n)^{-2})$, although this bound is believed to be suboptimal~\cite{pmlr-v291-chin25a}.

\subsubsection{T\oe plitz seriation.}
Optimal statistical rates for various loss functions have been derived in~\cite{flammarion2019optimal,cai2023matrix,berenfeld2024seriation}.
However, the best known polynomial-time algorithms  achieve significantly slower rates~\cite{cai2023matrix,berenfeld2024seriation}.
For this reason, statistical-computational gaps have been conjectured, e.g., in~\cite{cai2023matrix,berenfeld2024seriation}.
In particular,~\cite{berenfeld2024seriation} and~\cite{Even25b} proved a low-degree lower bound for a Gaussian version of the \ToeSerI\ and \ToeSerP\ models, showing that low-degree polynomials fail when ${k\lambda \over \sqrt{n}}\vee \lambda \leq 1$ up to poly-logarithmic factors.

\subsection{Organization of the manuscript}

In Section~\ref{s:set}, we introduce the two statistical problems studied in this paper, namely the problem of estimating an entry of $\Theta$, and a specific composite-composite testing problem, where we want to test a small alteration of our structure. We introduce our invariant basis for LD polynomials in Section~\ref{s:nearort}, and we establish its almost orthonormality
for all our models, when the signal is weak enough. In Section~\ref{s:mainres}, we then rely on these almost orthonormal polynomials to establish LD lower bounds for the estimation and testing problems, and we explain in Section~\ref{subsec:insights} how they provide insights for the design of optimal polynomial-time procedures.
Additional definitions,  important for the proof, are introduced in  Section~\ref{sec:graph:definition}.
While the proof of the almost orthonormality property in the general case is technical, as we simultaneously  handle different models, it becomes much simpler when instantiated to a specific model, like the hidden subclique model \HidSubI. 
To provide insights, we convey  in Section~\ref{s:core}  the core ideas by detailing the proof for the specific hidden subclique model \HidSubI.  Finally, we present in Appendix~\ref{sec:main_orthonormalite} general conditions under which our basis is almost orthonormal, these conditions being satisfied for all models under consideration in this work.

\section{Setting}\label{s:set}

Recall the six statistical  models \HidSubI, \StoBloI, \ToeSerI, \HidSubP, \StoBloP, \ToeSerP\ described in the introduction. Henceforth, we write $\mathbb P$ and $\mathbb E$ for the probability and expectation of $Y$.

In this manuscript, we tackle  two statistical tasks: {\bf estimation} and {\bf complex testing}. In estimation, the goal is to recover $\mathbb{E}[Y|z]=(\Theta_{z_iz_j})_{1\leq i < j \leq n}$. In complex testing, the goal is to test some structural properties of the matrix $(\Theta_{z_iz_j})_{1\leq i < j \leq n}$, this is in sharp contrast with signal detection problems~\cite{KuniskyWeinBandeira} which test the nullity of $\Theta$. 

\subsection{Estimation}

As is standard for LD lower bounds in estimation problems~\cite{SchrammWein22}, we focus on estimating the functional 
\begin{equation}x = \mathbf 1\{\Theta_{z_1,z_2} \neq 0\}\enspace .\end{equation}
Note that $\Theta_{z_1,z_2}= \mathbb{E}[Y_{1,2}|z] \in \{0,\lambda\}$ for all six models \HidSubI, \StoBloI, \ToeSerI, \HidSubP, \StoBloP, and \ToeSerP, so that proving an LD lower bound for estimating $x$ readily allows us, by linearity, to establish an LD lower bound for estimating the matrix $(\Theta_{z_iz_j})$ in Frobenius norm.
We recall that
\begin{equation} \inf_{f: \mathrm{deg}(f)\leq D} \mathbb{E}\cro{(f-x)^2}=\mathbb{E}[x^2]- \mathrm{Corr}^2_{\leq D}\enspace ,\end{equation}
where
\begin{equation}\label{eq:definition_corr}
\mathrm{Corr}_{\leq D} = \sup_{f: \mathrm{deg}(f)\leq D}\frac{\mathbb{E}[fx]}{\sqrt{\mathbb{E}[f^2]}} = \sup_{(\alpha_S)_{S \in \mathcal S_{\leq D}}} \frac{\mathbb E\left[ x\sum_{S \in \mathcal S_{\leq D}} \alpha_S \phi_S(Y)\right]}{\sqrt{\mathbb E\left[\left[\sum_{S \in \mathcal S_{\leq D}} \alpha_S \phi_S(Y)\right]^2\right]}}\enspace 
\end{equation}
is the {\bf low-degree correlation criterion} introduced in~\cite{SchrammWein22}.
As explained in the introduction (see~\eqref{eq:intro:LD1}) proving that $\mathrm{Corr}^2_{\leq D}$ is no larger than $\mathbb{E}[x]^2(1+o(1))$ for $D$ of the order of $\log(n)$ is a strong indication of the computational hardness of the estimation problem. Our aim is therefore to characterize the regimes of $(k,n,\lambda)$ such that $\mathrm{Corr}^2_{\leq D}\leq \mathbb{E}[x]^2(1+o(1))$.

\subsection{Complex testing}

Fix $\epsilon\in (0,1]$. For each of our six models, we define an alteration:
\begin{itemize}
    \item[1] {\bf Alteration of \HidSub.} First sample $(\Theta_{z_iz_j})$ from \HidSubI\ (resp. \HidSubP). For all $i$ such that $z_i\leq k$, we set the $i$-th row and $i$-th column of $(\Theta_{z_iz_j})$ to zero with probability $\epsilon$. In plain words, under the alteration of \HidSubI, the size of the hidden subclique is  distributed as $Bin(n,k(1-\epsilon)/n)$ instead of $Bin(n,k/n)$, whereas under the alteration of \HidSubP, its size is distributed as $Bin(k,1-\epsilon)$ instead of being equal to $k$. 
    \item[2] {\bf Alteration of \StoBlo.} First sample $(\Theta_{z_iz_j})$ from \StoBloI\ (resp. \StoBloP) and sample uniformly a group $\hat{l}\in [n/k]$. Then, for all $i$ such that $z_i\in [(\hat{l}-1)k+1;\hat{l}k]$, we set the $i$-th row and $i$-th column of $(\Theta_{z_iz_j})$ to zero with probability $\epsilon$. In this alteration, we decrease the size of one of the  $K=n/k$ groups of the SBM and we create a new group of size $\epsilon k$ (in expectation) whose probability of connection is always equal to $q$. 
    \item[3] {\bf Alteration of \ToeSer.} First sample $(\Theta_{z_iz_j})$ from \ToeSerI\ (resp. \ToeSerP) and sample uniformly a position $\hat{l}\in [n]$. Then, for all $i$ such that $z_i\in [\hat{l}-k/2;\hat{l}+k/2]$, we set the $i$-th row and $i$-th column of $(\Theta_{z_iz_j})$ to zero with probability $\epsilon$. 
    In the alteration of \ToeSerP, this amounts to erasing some of the entries of the T\oe plitz matrix. 
\end{itemize}
We have defined these alterations as illustrative and unified examples of complex testing problems. We could adapt the methodology to other structural tests (e.g., the number of groups in the SBM), as the main difficulty in establishing the LD lower bounds is to introduce a candidate basis and establish its almost orthonormality. 

Henceforth, we write $\mathbb{P}_{H_1}$ and $\mathbb{E}_{H_1}$ for the probability and expectation in the altered model.
For each of the models, we consider the testing problem 
\begin{equation}H_0: Y\sim \mathbb{P}\quad \text{  against  }~~~H_1: Y \sim \mathbb{P}_{H_1}\enspace . \end{equation}
In the low-degree framework~\cite{hopkins2018statistical,KuniskyWeinBandeira,SurveyWein2025}, the difficulty of the testing problem is characterized by 
\begin{equation}\label{def:Adv}
\mathrm{Adv}_{\leq D} = \sup_{f: \mathrm{deg}(f)\leq D}\frac{\mathbb{E}_{H_1}[f]}{\sqrt{\mathbb{E}[f^2]}} = \sup_{(\alpha_S)_{S \in \mathcal S_{\leq D}}} \frac{\mathbb E_{H_1}\left[ \sum_{S \in \mathcal S_{\leq D}} \alpha_S \phi_S(Y)\right]}{\sqrt{\mathbb E\left[\left[\sum_{S \in \mathcal S_{\leq D}} \alpha_S \phi_S(Y)\right]^2\right]}}\enspace . 
\end{equation}
As explained in~\eqref{eq:intro:LD1}, $\mathrm{Adv}_{\leq D}\leq 1+o(1)$ for $D$ of the order of $\log(n)$ is a strong indication of the hardness of testing $\mathbb{P}$ against $\mathbb{P}_{H_1}$.

\medskip 

In order to control both $\mathrm{Adv}_{\leq D}$ and $\mathrm{Corr}_{\leq D}$, we introduce in the next section a basis of invariant polynomials. Once its almost orthonormality under $\mathbb{P}$ is established, tight bounds for $\mathrm{Adv}_{\leq D}$ and $\mathrm{Corr}_{\leq D}$ will easily follow.

\section{Almost orthonormal invariant polynomials}\label{s:nearort}

In this section, we construct a specific basis of node-permutation invariant polynomials. As the construction for the testing problem is slightly simpler than for the estimation problem, we start with a dedicated basis for bounding $\mathrm{Adv}_{\leq D}$.

\subsection{Basis for the complex testing problem}

First, we exploit the permutation invariance of the distribution $\mathbb{P}$ to reduce the space of polynomials. A function $f: \mathbb{R}^{n\times n}\to \mathbb{R}$ is said to be invariant under permutations if, for any matrix $Y$, and  any bijection $\sigma: [n]\to [n]$, we have $f(Y)= f(Y_{\sigma})$ where $Y_{\sigma}= (Y_{\sigma(i),\sigma(j)})$. 

\begin{lemma}\label{lem:reduction:permutation}
Fix any degree $D>0$. If both $\mathbb{P}$ and $\mathbb{P}_{H_1}$ are permutation invariant, then the low-degree advantage  $\mathrm{Adv}_{\leq D}$ is achieved by a permutation invariant polynomial.
\end{lemma}

This reduction was already done in~\cite{semerjian2024matrix,kunisky2024tensor,montanari2025equivalence}.
To introduce our basis of invariant polynomials, we consider simple undirected graphs $G= (V,E)$ where $V=\{v_1,\ldots, v_{r}\}$ is the set of nodes and where $E$ is the set of edges. We write $\#\mathrm{CC}_G$ for the number of connected components of $G$. 

\begin{definition}[Collection $\mathcal{G}_{\leq D}$]
Let $\mathcal{G}_{\leq D}$ be any maximum collection of graphs $G=(V,E)$ such that (i) $G$ does not contain any isolated node,  (ii) $1\leq |E|\leq D$, and (iii) no two graphs in $\mathcal{G}_{\leq D}$ are isomorphic.
\end{definition}
In fact, $\mathcal{G}_{\leq D}$ corresponds to the collection of equivalence classes (with respect to isomorphism) of all graphs with at most $D$ edges, and without isolated nodes. Henceforth, we refer to $\mathcal{G}_{\leq D}$ as the collection of \emph{templates}. Consider a template $G= (V,E)\in \mathcal{G}_{\leq D}$. We define $\Pi_V$ as the set of injective mappings from $V$ to $[n]$. An element $\pi\in \Pi_V$ corresponds to a labeling of the generic nodes in $V$ by elements in $[n]$. For $\pi\in \Pi_V$, we define the polynomials
\begin{equation}\label{eq:definition:P_G}
P_{G,\pi}(Y)= \prod_{(u,v)\in E} Y_{\pi(u),\pi(v)} ; \quad \text{and}\quad P_G = \sum_{\pi\in \Pi_V} P_{G,\pi}\enspace .
\end{equation}
 For short, we sometimes write $P_G$ for $P_G(Y)$, when there is no ambiguity. For the invariant polynomials $P_G$, we say that $G$ is the {\bf template} (graph) that indexes the polynomial. The idea of indexing the invariant polynomials by templates is borrowed from~\cite{kunisky2024tensor,montanari2025equivalence}, although their bases are different to account for normal distributions. 

Let us denote by $\mathcal{P}^{\mathrm{inv}}_{\leq D}$ the subspace of permutation  invariant polynomials $f$  with  degree at most $D$. The next lemma states that, as expected, any permutation invariant polynomial can be expressed using polynomials $P_G$ indexed by $G \in \mathcal G_{\leq D}$. 
\begin{lemma}\label{lem:invariant:graph}
Assume that  $2D\leq n$.  For any $f$ in $\mathcal{P}^{\mathrm{inv}}_{\leq D}$, there exist unique numerical values $\alpha_{\emptyset}$ and $(\alpha_{G})_{G\in\mathcal{G}_{\leq D}}$ such that $f(Y)=\alpha_{\emptyset} + \sum_{G \in \mathcal G_{\leq D}} \alpha_G P_G(Y)$.
\end{lemma}

\subsubsection{Correction of the monomials.} The family $(P_G)_{G \in \mathcal G_{\leq D}}$ is orthogonal  under the distribution with null signal, namely $\Theta=0$. However, it is far from being the case when $\Theta \neq 0$, and we have to adjust the basis. 

The main ingredient is to tweak the polynomials $P_{G,\pi}$ involved in $P_G$. Consider a template $G\in \mathcal{G}_{\leq D}$ with $c$ connected components $(G_1,G_2,\ldots, G_c)$. Then, we define
\begin{equation}\label{eq:cross} 
\overline{P}_{G} := \sum_{\pi\in \Pi_V} \overline{P}_{G,\pi} \ ; \quad\text{with}\quad  \overline{P}_{G,\pi}:=\prod_{l=1}^c \left[P_{G_l,\pi} - \mathbb{E}[P_{G_l,\pi} ]\right] \enspace . 
\end{equation}
 Note that $\mathbb{E}[P_{G_l,\pi} ]$ does not depend on the choice of $\pi$. This correction centers the polynomial associated with each connected component of the template graph. 

 \begin{remark}
 This correction, already implemented in~\cite{montanari2025equivalence}, is instrumental in achieving near and almost orthogonality properties, see the comment on their difference in the literature review, Section~\ref{subsec:related literature}. To see that, let us consider the hidden subclique \HidSubI\ model. Given two templates $G^{(1)}$, $G^{(2)}$ and two labelings $\pi^{(1)}\in \Pi_{V^{(1)}}$, $\pi^{(2)}\in \Pi_{V^{(2)}}$,  write $\pi^{(1)}[G^{(1)}]$, $\pi^{(2)}[G^{(2)}]$ for the graphs $G^{(1)},G^{(2)}$ with labeled nodes $\pi^{(1)}(V^{(1)}), \pi^{(2)}(V^{(2)})$.  In the model \HidSubI,  $P_{G^{(1)},\pi^{(1)}}$ and $P_{G^{(2)},\pi^{(2)}}$ are independent as long as $\pi^{(1)}[G^{(1)}]$ and $\pi^{(2)}[G^{(2)}]$ do not intersect. Then, by definition of $\overline{P}_{G,\pi}$, one can check that $\mathbb E[\overline{P}_{G^{(1)},\pi^{(1)}} \overline{P}_{G^{(2)},\pi^{(2)}}]=0$ as soon as one connected component of $\pi^{(1)}[G^{(1)}]$ does not intersect $\pi^{(2)}[G^{(2)}]$ or vice versa. 
 As a consequence, the correlation between $\overline{P}_{G^{(1)}}$ and $\overline{P}_{G^{(2)}}$ will be quite small. 
 \end{remark}

\subsubsection{Renormalization of the polynomials.} It remains to normalize the polynomials $\overline{P}_{G}$. For this purpose, we need to compute the order of magnitude of $\mathbb{E}\left[\overline{P}_{G}^2\right]= \sum_{\pi^{(1)},\pi^{(2)}}\mathbb{E}\left[\overline{P}_{G,\pi^{(1)}}\overline{P}_{G,\pi^{(2)}}\right]$. Thanks to the previous correction, most terms $\mathbb{E}\left[\overline{P}_{{G},\pi^{(1)}}\overline{P}_{G,\pi^{(2)}}\right]$ are small, and the dominant term is achieved for $\pi^{(1)}$ and $\pi^{(2)}$ such that $\pi^{(1)}[G]= \pi^{(2)}[G]$. There are $|\Pi_V||\mathrm{Aut}(G)|$ such couples $(\pi^{(1)},\pi^{(2)})$, where $\mathrm{Aut}(G)$ stands here for the automorphism group of $G$. All these $|\Pi_V||\mathrm{Aut}(G)|$ terms are identical. Also, it turns out that 
such terms $\mathbb{E}\left[\overline{P}^2_{G,\pi}\right]$ are of the order of $\mathbb{E}\left[P^2_{G,\pi}\right]$. If the matrix $\Theta$ had been equal to zero, we would readily get  $\mathbb{E}\left[P^2_{G,\pi}\right]= \overline{q}^{|E|}$ where  $\overline{q} :=  q(1-q)$. This approximation turns out to be sufficient for our purpose. In light of the above discussion,
for any $G\in \mathcal{G}_{\leq D}$, we define the variance proxy for $P_{G}$ by  
\begin{equation}\label{eqn:variance of graph}
\mathbb V(G) = \frac{n!}{(n-|V|)!} |\mathrm{Aut}(G)|  \overline{q}^{|E|}\enspace , \quad\text{with}\quad \overline{q} = q(1-q)\enspace .
\end{equation}
Finally, we define the normalized polynomial
\begin{equation}\label{eq:definition:P*G}
\Psi_{G} = \frac{\overline{P}_G}{\sqrt{\mathbb V(G)}}\enspace .
\end{equation}
Since $(1,(\Psi_G)_{G\in \mathcal{G}_{\leq D}})$ span the same space as $(1, (P_G)_{G\in \mathcal{G}_{\leq D}})$, we deduce from Lemmas~\ref{lem:reduction:permutation} and~\ref{lem:invariant:graph} the following result.
\begin{lemma}\label{lem:reduction:degree}
If both $\mathbb{P}$ and $\mathbb{P}_{H_1}$ are permutation invariant, then we have 
\begin{align}\label{eq:Adv}
    \mathrm{Adv}_{\leq D} &= \sup_{(\alpha_{\emptyset}, (\alpha_G)_{G\in \mathcal G_{\leq D}})} \frac{\mathbb E_{H_1}\left[ \alpha_{\emptyset}+ \sum_{G \in \mathcal G_{\leq D}} \alpha_G \Psi_{G}\right]}{\sqrt{\mathbb E\left[\left[\alpha_{\emptyset}+ \sum_{G\in \mathcal G_{\leq D}} \alpha_G \Psi_{G}\right]^2\right]}}\enspace . 
\end{align}
\end{lemma}

Our main result is given in the next theorem. It states that, for all our six models, the basis $(1,(\Psi_G)_{G \in \mathcal G_{\leq D}})$ is almost orthonormal as long as $\lambda$ is not too large. 

\begin{theorem}\label{thm:isorefo}
    There exist positive numerical constants $c_0$ and $c$, such that the following holds for all  $D \geq 2$ and all six models  \HidSubI, \StoBloI, \ToeSerI, \HidSubP, \StoBloP, and \ToeSerP. If we assume that 
\begin{equation} \left(\frac{k}{n} \right) \lor \left(\frac{\lambda k}{\sqrt{n q}} \right) \lor \left(\frac{\lambda}{q}\right)   \leq D^{-c_0}\enspace ,\end{equation}
     then, for any vector $\alpha= (\alpha_{\emptyset},(\alpha_G)_{G \in \mathcal G_{\leq D}})$ in $\mathbb{R}^{|\mathcal{G}_{\leq D}|+1}$, we have 
    \begin{equation}(1-cD^{-2})\|\alpha\|^2_2 \leq \mathbb E \left[\left(\alpha_{\emptyset}+ \sum_{G \in \mathcal G_{\leq D}} \alpha_G\Psi_G\right)^2\right] \leq (1+cD^{-2})\|\alpha\|^2_2\enspace .\end{equation}
\end{theorem}

In fact, this theorem is a straightforward consequence of the more general results (Theorems~\ref{thm:iso} and~\ref{thm:iso-WR} and Proposition~\ref{prp:model:conditions}) stated  in the appendix. 
Note that under the assumptions of the theorem, we readily get the  following upper bound for the advantage
\begin{equation}\label{eq:Adv:ortho}
 \mathrm{Adv}^2_{\leq D} \leq  \left(1 - cD^{-2}\right)^{-1}\left[1 + \sum_{G \in \mathcal G_{\leq D}}\mathbb{E}_{H_1}[\Psi_{G}]^2 \right]\enspace . 
\end{equation}
So,  we only need to bound the first moment of the basis elements under the alternative hypothesis to control the advantage, and establish an LD lower bound. This is done in the next section.

To the best of our knowledge, this is the first time that for general structured distributions, and under mild conditions on the parameters $(k,\lambda,q)$, such an almost orthonormal basis has been constructed, although~\cite{montanari2025equivalence} established a similar result with a different constant in a BBP-type asymptotic regime, for  the rank-one matrix estimation problem.

 The conditions of Theorem~\ref{thm:isorefo} are indeed rather mild, except the last one:
\begin{itemize}
    \item The first condition  $\frac{k}{n}  \leq D^{-c_0}$ does not exclude interesting regimes. Indeed, when $k$ is of the order of $n$, no significant statistical-computational gaps arise in our models.    
    \item As further discussed in Section~\ref{s:mainres}, when the condition $\lambda k/\sqrt{n q}   \leq D^{-c_0}$ is not fulfilled (up to a polynomial in $D$), in most interesting regimes, it is possible to reconstruct the signal matrix $(\Theta_{z_iz_j})$ with an LD polynomial estimator, so that the problem is computationally solvable.
    \item The last condition $\lambda \leq qD^{-c_0}$ is more restrictive and is in fact not intrinsic: it entails that we only deal with the regimes where the two probabilities $q$ and  $p=q+\lambda$ are of the same order. Relaxing this condition to $\lambda = o(\sqrt{q})$ is technical but doable at least for the classical instances of the  hidden subclique model \HidSubI\ and stochastic block model \StoBloI. However, this is beyond the scope of this paper, as our aim is to establish simple yet versatile results. In the regime where $\lambda \gg \sqrt{q}$, the normalization $\mathbb V(G)$ defined in~\eqref{eqn:variance of graph} ceases to be a good approximation of the second moment $\mathbb{E}[P^2_{G,\pi}]$ and one has to resort to a different normalization, see the subsequent work~\cite{carpentier2025phase}.  
\end{itemize}

\subsection{Basis for the estimation problem}

We now turn to the basis for the estimation  of $x= \mathbf 1\{\Theta_{z_1,z_2} \neq 0\}$. If the distribution $\mathbb{P}$ is invariant under permutation, then the distribution of $(x,Y)$ is invariant under permutations of the nodes in $\ac{3,\ldots,n}$, as the first two nodes play a specific role. As a consequence, we have to slightly adapt the definition of the $\Psi_G$'s.

Consider a template $G= (V,E)$ with $V=\{v_1,v_2,\ldots, v_r\}$, without isolated nodes (except possibly $v_1,v_2$),  with $|V| \geq 2$ and at least one edge. Let  $\Pi_V^{(1,2)}$ be the set of injective mappings $\pi$ from $V$ to $[n]$ such that $\pi(v_1)=1, \pi(v_2)=2$. We then define the polynomial
\begin{equation*}
P^{(1,2)}_G = \sum_{\pi\in \Pi_V^{(1,2)}} P_{G,\pi}\enspace .
\end{equation*}
For short, we sometimes write $P^{(1,2)}_G$ for $P^{(1,2)}_G(Y)$ when there is no ambiguity.

Let $G^{(1)}=(V^{(1)},E^{(1)})$ and $G^{(2)}= (V^{(2)},E^{(2)})$ be two templates. We say $G^{(1)}$ and $G^{(2)}$ are equivalent if there exists a bijection $\sigma: V^{(1)} \to V^{(2)}$ such that $\sigma(v^{(1)}_1)= v_1^{(2)}$, $\sigma(v^{(1)}_2)= v_2^{(2)}$, and $\sigma$ preserves the edges. In other words, the graphs $G^{(1)}$ and $G^{(2)}$ are isomorphic with the additional constraint that the corresponding bijection maps $v^{(1)}_1$ to $v^{(2)}_1$ and $v^{(1)}_2$ to $v^{(2)}_2$.  Then, we define $\mathcal{G}_{\leq D}^{(1,2)}$ as a maximum collection of non-equivalent templates with at most $D$ edges and at least one edge. 
Consider a template $G\in \mathcal{G}^{(1,2)}_{\leq D}$ with $c$ non-trivial connected components $(G_1,G_2,\ldots, G_c)$. We define
\begin{equation*} 
\overline{P}^{(1,2)}_{G} := \sum_{\pi\in \Pi^{(1,2)}_V} \overline{P}^{(1,2)}_{G,\pi} \ ; \quad\text{with}\quad  \overline{P}^{(1,2)}_{G,\pi}:=\prod_{l=1}^c \left[P_{G_l,\pi} - \mathbb{E}[P_{G_l,\pi} ]\right] \enspace . 
\end{equation*}
Recall that  $\overline{q}=  q(1-q)$. Define the variance proxy 
\begin{equation}\label{eqn:variance of graph:estimation}
\mathbb V^{(1,2)}(G) = \frac{(n-2)!}{(n-|V|)!} |\mathrm{Aut}^{(1,2)}(G)|  \overline{q}^{|E|}\enspace ,
\end{equation}
where $\mathrm{Aut}^{(1,2)}(G)$ is the set of automorphisms of the graph $G$ that leave  $v_1$ and $v_2$ fixed. Finally, we define the polynomials 
\begin{equation}\Psi^{(1,2)}_{G} = \frac{\overline{P}^{(1,2)}_G}{\sqrt{\mathbb V^{(1,2)}(G)}}\enspace .\end{equation}

The following result is the counterpart of Lemma~\ref{lem:reduction:degree} for the estimation problem. 
\begin{lemma}\label{lem:reduction:degree2}
As long as $\mathbb{P}$ is permutation invariant, we have 
\begin{align*}
    \mathrm{Corr}_{\leq D} &= \sup_{\alpha_{\emptyset},(\alpha_G)_{G\in \mathcal G^{(1,2)}_{\leq D}}}\ \frac{\mathbb E\left[ x(\alpha_{\emptyset}+\sum_{G \in \mathcal G^{(1,2)}_{\leq D}} \alpha_G \Psi^{(1,2)}_{G})\right]}{\sqrt{\mathbb E\left[\left[\alpha_{\emptyset}+ \sum_{G\in \mathcal G^{(1,2)}_{\leq D}} \alpha_G \Psi^{(1,2)}_{G}\right]^2\right]}}\enspace .
\end{align*}
\end{lemma}

The following result is the analogue of Theorem~\ref{thm:isorefo} for the basis $(1,(\Psi^{(1,2)}_G)_{G\in \mathcal{G}^{(1,2)}_{\leq D}})$.
\begin{theorem}\label{thm:isorefo2}
    There exist positive numerical constants $c_0$ and $c$, such that the following holds for any  $D \geq 2$ and all six models  \HidSubI, \StoBloI, \ToeSerI, \HidSubP, \StoBloP, and \ToeSerP. If we assume that 
\begin{equation}\left(\frac{k}{n} \right) \lor \left(\frac{\lambda k}{\sqrt{n q}} \right) \lor \left(\frac{\lambda}{q}\right)   \leq D^{-c_0}\enspace ,\end{equation}
then, for any $\alpha= (\alpha_\emptyset, (\alpha_G)_{G \in \mathcal G^{(1,2)}_{\leq D}})$, we have 
    \begin{equation}(1-cD^{-2})\|\alpha\|^2_2 \leq \mathbb E \left[\left(\alpha_\emptyset+\sum_{G \in \mathcal G^{(1,2)}_{\leq D}} \alpha_G \Psi^{(1,2)}_G\right)^2\right] \leq (1+cD^{-2})\|\alpha\|^2_2\enspace .\end{equation}
\end{theorem}
Again, this theorem is a straightforward consequence of the more general Theorem~\ref{thm:iso-estimation} and Proposition~\ref{prp:model:conditions} stated in the appendix.

\section{Main low-degree lower bounds}\label{s:mainres}

In this section, we  deduce LD lower bounds from the almost orthonormality of the basis. We start with estimation problems as those are more classical.

\subsection{Estimation problem}

 \begin{theorem}\label{thm:lowdeg2}
    There exist positive numerical constants $c$ and $c_0$ such that the following holds for any $D\geq 2$. Provided that 
\begin{equation}\label{eq:estim:condition}
  \left(\frac{k}{n} \right) \lor \left(\frac{\lambda k}{\sqrt{n q}} \right) \lor \left(\frac{\lambda}{q}\right) \leq D^{-c_0}\enspace , 
\end{equation}
all six models \HidSubI, \HidSubP, \StoBloI, \StoBloP, \ToeSerI, and \ToeSerP\ satisfy 
\begin{equation}\mathrm{Corr}_{\leq D} \leq \mathbb E[x](1+cD^{-1})\enspace .\end{equation}
\end{theorem}

Note that $\mathbb E[x]=\mathrm{Corr}_{\leq 0}$.  
This entails that, when~\eqref{eq:estim:condition} holds with $D$ of the order of $\log(n)$,  no polynomial of degree of order $\log(n)$ can perform significantly better than the constant prediction $\mathbb E[x]$ of degree $0$ polynomials. 

To discuss our results, let us focus on $D=\log(n)$ and regimes in $(q,k,n)$ such that $k \leq \frac{n}{\log^{c_0}(n)}$ and  $q\geq n/k^2$. 
Note that this forces $k$ to belong to $[\sqrt{n}; n\log^{-c_0}(n)]$.
Then, Theorem~\ref{thm:lowdeg2} states that recovery is impossible by low-degree polynomials as soon as
\begin{equation}\label{eq:condition:non_reconstruction}
\frac{\lambda k}{\sqrt{nq}} \leq \log^{-c_0}(n) \enspace . 
\end{equation}
For both hidden subclique (\HidSubI\ and \HidSubP) and stochastic block models (\StoBloI\ and \StoBloP), it is known that recovery is possible in polynomial time above this threshold~\cite{SchrammWein22,luo2023computational,SohnWein25}. In particular, we recover the impossibility results of~\cite{SchrammWein22,luo2023computational,SohnWein25} for \HidSubI\ and \StoBloI. In comparison to the tight bounds of~\cite{SohnWein25}, our results are less tight as we lose poly-logarithmic factors. Besides, our LD lower bounds are optimal when $q\geq n/k^2$ whereas~\cite{SohnWein25} deal with the case where $k\gg \sqrt{n}$. 
As already alluded to, we believe that the condition $q\geq n/k^2$, which arises because we require $\lambda\leq q$, is an artefact of the proof and can be weakened to $k\gg \sqrt{n}$ using arguments that are more tailored to the models \HidSubI\ and \StoBloI. To the best of our knowledge, the LD lower bounds for the permutation models \HidSubP\ and \StoBloP\ are novel. 

For the T\oe plitz seriation models (\ToeSerI\ and \ToeSerP), Condition~\eqref{eq:condition:non_reconstruction} matches that of~\cite{berenfeld2024seriation} for polynomial-time reconstruction in the dense regime ($q$ of the order of a constant). It is also similar to the LD lower bound in the Gaussian setting of~\cite{berenfeld2024seriation} for \ToeSerI\ and~\cite{Even25b} for \ToeSerP\ when $k\gg \sqrt{n}$.

\subsection{Complex testing problem}
Recall the altered distributions $\mathbb{P}_{H_1}$ introduced in Section~\ref{s:set}. Here, we deduce from Theorems~\ref{thm:iso} and~\ref{thm:iso-WR} and in particular from~\eqref{eq:Adv:ortho}  a bound for $\mathrm{Adv}_{\leq D}$.
 \begin{theorem}\label{thm:lowdeg}
There exist positive numerical constants $c$ and $c_0$ such that the following holds for any $D\geq 2$. Provided that 
\begin{equation}\label{eq:condition:LD:test}
\left(\frac{k}{n} \right) \lor \left(\frac{\lambda k}{\sqrt{n q}} \right) \lor \left(\frac{\lambda}{q}\right) \lor   \left(\epsilon \frac{\lambda k^2}{n\sqrt{q}}\right) \leq D^{-c_0}\enspace , 
\end{equation}
all six models \HidSubI, \HidSubP, \StoBloI, \StoBloP, \ToeSerI, and \ToeSerP\ satisfy 
\begin{equation}\mathrm{Adv}^2_{\leq D} \leq 1+ cD^{-1}\enspace .\end{equation}
\end{theorem}
Under the low-degree conjecture, Theorem~\ref{thm:lowdeg} provides a strong indication that when Condition~\eqref{eq:condition:LD:test} holds it is impossible to distinguish the distributions $\mathbb{P}$ and $\mathbb{P}_{H_1}$ in polynomial time. Given Theorem~\ref{thm:isorefo} and the bound~\eqref{eq:Adv:ortho}, to prove  this theorem, we only have to control the first moments $\mathbb{E}_{H_1}[\Psi_G]$ of the basis elements under the altered distribution $\mathbb{P}_{H_1}$, for $G\in \mathcal{G}_{\leq D}$.

Condition~\eqref{eq:condition:LD:test} is the conjunction of Condition~\eqref{eq:estim:condition} for reconstruction and the condition  $\epsilon \lambda k^2 \leq D^{-c_0} n\sqrt{q}$. In particular, the latter inequality is optimal for the alteration detection problem. Indeed, consider the statistic $T=\sum_{i<j} Y_{ij}$. Since $\lambda \leq q$, we have $T-\mathbb{E}[T]= O_{\mathbb{P}}(n\sqrt{q})$ and $T-\mathbb{E}_{H_1}[T]= O_{\mathbb{P}_{H_1}}(n\sqrt{q})$. As a consequence, $T$ is powerful as soon as $|\mathbb{E}[T] - \mathbb{E}_{H_1}[T]|\gg n\sqrt{q}$. Since $\mathbb{E}[T] - \mathbb{E}_{H_1}[T]$ is of the order of $\epsilon \lambda k^2$, the result follows.

\subsection{Insights for optimal polynomials}\label{subsec:insights}

Beyond LD lower bounds, the almost orthonormality of the basis provides an explicit description of nearly optimal LD polynomials. Consider, for the complex testing problem, the polynomial
\begin{equation*}
f^{*}=1+\sum_{G\in\mathcal{G}_{\leq D}}\mathbb{E}_{H_1}[\Psi_G]\,\Psi_G\enspace .
\end{equation*}
Under the conditions of Theorem~\ref{thm:isorefo}, we have $\mathbb{E}_{H_1}[f^*]=1+\sum_{G\in\mathcal{G}_{\leq D}}\mathbb{E}_{H_1}[\Psi_G]^2$ and $\mathbb{E}[(f^*)^2]\leq (1+cD^{-2})\,\mathbb{E}_{H_1}[f^*]$. Together with~\eqref{eq:Adv:ortho}, this entails
\begin{equation*}
\frac{\mathbb{E}_{H_1}[f^*]^2}{\mathbb{E}[(f^*)^2]}\geq \frac{1-cD^{-2}}{1+cD^{-2}}\,\mathrm{Adv}^2_{\leq D}\enspace ,
\end{equation*}
so that $f^*$ is a nearly optimal polynomial of degree at most $D$. Similarly, under the conditions of Theorem~\ref{thm:isorefo2}, the polynomial $\mathbb{E}[x]+\sum_{G\in\mathcal{G}^{(1,2)}_{\leq D}}\mathbb{E}[x\Psi^{(1,2)}_G]\,\Psi^{(1,2)}_G$ nearly achieves $\mathrm{Corr}_{\leq D}$.

This near-optimality only holds in the regime where the basis is almost orthonormal. Apart from the technical conditions $k/n\leq D^{-c_0}$ and $\lambda/q\leq D^{-c_0}$, this is the regime $\lambda k\leq D^{-c_0}\sqrt{nq}$, that is, below the phase transition for polynomial-time recovery, see~\eqref{eq:condition:non_reconstruction} and the subsequent discussion. Above this phase transition, the basis is not expected to remain almost orthonormal, and $f^*$ has no reason to be optimal. Hence, our basis does not directly provide optimal procedures in this regime. The insight is rather the following. Below the phase transition, the optimal LD polynomial is explicit in our basis, and the templates $G$ with the largest first moments $|\mathbb{E}_{H_1}[\Psi_G]|$ (resp.\ $|\mathbb{E}[x\Psi^{(1,2)}_G]|$) are those that drive $\mathrm{Adv}_{\leq D}$ (resp.\ $\mathrm{Corr}_{\leq D}$). As the parameters approach the phase transition, these first moments are the first ones to become large. This suggests that the counts of the corresponding motifs are the relevant statistics above the phase transition, although their performance then has to be established by a direct analysis. In the complex testing problems considered here, these templates are the single edge, which leads to the edge count statistic $T$ discussed above, and the paths, which lead to spectral-type statistics.

This principle has been instrumental in the subsequent works~\cite{carpentier2025phase,carpentier2025phase2} on SBMs with more than $\sqrt{n}$ groups. There, the dominant templates were identified below the phase transition, and counting the corresponding motifs (cliques, self-avoiding paths, or blow-up graphs, depending on the sparsity $q$) led to new polynomial-time procedures that match the LD lower bound, see Section~\ref{subsec:flexibility}.

\section{Graph definitions}\label{sec:graph:definition}

In order to show the almost orthonormality of the family $(\Psi_{G})_{G\in\mathcal{G}_{\leq D}}$, we have to work out cross products of the form $\mathbb{E}[P_{G^{(1)}}P_{G^{(2)}}]$ for two templates $G^{(1)}$ and $G^{(2)}$. In turn, this is done by working out quantities of the form $\mathbb{E}[P_{G^{(1)},\pi^{(1)}}P_{G^{(2)},\pi^{(2)}}]$ where $\pi^{(1)}$ and $\pi^{(2)}$ are two labelings of $G^{(1)}$ and $G^{(2)}$. This requires a systematic way to classify the combinatorial structures that arise when two template graphs are overlaid. The purpose of this section is to introduce all these concepts.

Recalling that we write $\pi^{(1)}[G^{(1)}]$ and $\pi^{(2)}[G^{(2)}]$ for the corresponding labeled graphs, we observe that $\mathbb{E}[P_{G^{(1)},\pi^{(1)}}P_{G^{(2)},\pi^{(2)}}]$ strongly depends on the nodes in common between these two graphs. Indeed, 
\begin{equation*}
P_{G^{(1)},\pi^{(1)}}P_{G^{(2)},\pi^{(2)}}= \prod_{(v_1,v_2)\in E^{(1)}}  Y_{\pi^{(1)}(v_1)\pi^{(1)}(v_2)} \prod_{(v_1,v_2)\in E^{(2)}}Y_{\pi^{(2)}(v_1)\pi^{(2)}(v_2)} \enspace , 
\end{equation*}
and the distribution of the product is a function of the edges that appear twice in $\pi^{(1)}[G^{(1)}]$ and $\pi^{(2)}[G^{(2)}]$ and of the edges that appear only once in these two graphs.

\subsection{Node matching and graph merging}

\subsubsection{Matching of nodes.} Consider two templates $G^{(1)}=(V^{(1)},E^{(1)})$ and $G^{(2)}=(V^{(2)},E^{(2)})$. Given labelings $\pi^{(1)}$ and $\pi^{(2)}$, we say  that two nodes $v^{(1)}$ and $v^{(2)}$ are matched if $\pi^{(1)}(v^{(1)})=  \pi^{(2)}(v^{(2)})$. More generally,  a matching $\mathbf M$ stands for a set of pairs of nodes $(v^{(1)},v^{(2)})\in V^{(1)}\times V^{(2)}$ where no node in $V^{(1)}$ or $V^{(2)}$ appears twice. We denote by $\mathcal M$ the collection of all possible node matchings. For $\mathbf M \in \mathcal M$, we define the collection $\Pi(\mathbf M)$ of labelings that are compatible with $\mathbf M$ as the collection of  $\pi^{(1)}\in \Pi_{V^{(1)}},\pi^{(2)} \in \Pi_{V^{(2)}}$  that satisfy the following property 
\begin{equation} \forall (v^{(1)},v^{(2)}) \in V^{(1)}\times V^{(2)}, \{\pi^{(1)}(v^{(1)}) = \pi^{(2)}(v^{(2)})\}\Longleftrightarrow \{(v^{(1)},v^{(2)}) \in \mathbf M\}\enspace . \end{equation}
Importantly, as $\mathbb{P}$ is permutation invariant, $\mathbb{E}[P_{G^{(1)},\pi^{(1)}}P_{G^{(2)},\pi^{(2)}}]$ is the same for all $(\pi^{(1)},\pi^{(2)})$ in $\Pi(\mathbf{M})$. 
Given a matching $\mathbf{M}$, we say that two edges $e\in E^{(1)}$ and $e'\in E^{(2)}$ are matched if the corresponding incident nodes are matched.

\subsubsection{Merged graph $G_\cup$, intersection graph $G_{\cap}$, and symmetric difference graph $G_{\Delta}$.} Consider two templates $G^{(1)}$ and $G^{(2)} \in \mathcal G_{\leq D}$ and two labelings $\pi^{(1)}$ and $\pi^{(2)}$. Then, the merged graph $G_{\cup}=(V_{\cup},E_{\cup})$ is defined as the union of $\pi^{(1)}[G^{(1)}]$ and $\pi^{(2)}[G^{(2)}]$, with the convention that two identical edges are merged into a single edge. Similarly,  we define the intersection graph
$G_{\cap}=(V_{\cap},E_{\cap})$ which is induced by the intersection of the edges of $\pi^{(1)}[G^{(1)}]$ and $\pi^{(2)}[G^{(2)}]$. Defining $E_{\Delta}=E_{\cup}\setminus E_{\cap}$, we define the symmetric difference graph $G_{\Delta}=(V_{\Delta},E_{\Delta})$ as the graph induced by $E_{\Delta}$, see Figure~\ref{fig:tildeG} for an example. We also have $|E_{\cup}| = | E^{(1)}| +|E^{(2)}| - |E_\cap|$ and $|V_{\cup}| = |V^{(1)}|+|V^{(2)}|-|\mathbf{M}|$ for $(\pi^{(1)},\pi^{(2)})\in \Pi(\mathbf M)$. Note that, for a fixed matching $\mathbf{M}$, all graphs $G_{\cup}$ (resp. $G_{\cap}$, $G_\Delta$) are isomorphic for $(\pi^{(1)},\pi^{(2)})\in \Pi(\mathbf M)$ and we shall refer to quantities such as $|E_{\Delta}|$, $|V_{\Delta}|$,\ldots associated to a matching $\mathbf{M}$.  

Finally, we write $\#\mathrm{CC}_{\Delta}$ for the number of connected components in $G_{\Delta}$, and $\#\mathrm{CC}_{\mathrm{pure}}$ for the number of connected components in $G_{\Delta}$ that are solely composed of unmatched nodes from $G^{(1)}$, or solely of unmatched nodes from $G^{(2)}$. This is the number of connected components that are ``untouched'' by the matching process. These two quantities only depend on $(\pi^{(1)},\pi^{(2)})$ through the matching $\mathbf{M}$. 

\begin{figure}
    \centering
\includegraphics{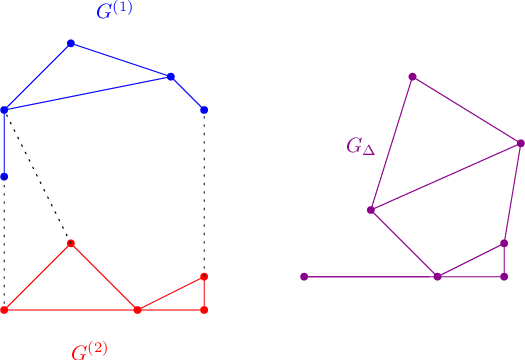}
    \caption{Illustration of two templates $G^{(1)}$ and $G^{(2)}$, a matching $\mathbf{M}$ and 
 the   symmetric difference graph $G_{\Delta}$. 
 }
    \label{fig:tildeG}
\end{figure}

\subsubsection{Sets of unmatched nodes and of semi-matched nodes.} 
 Write $U^{(1)}$, resp.~$U^{(2)}$ for the set of nodes in $\pi^{(1)}[G^{(1)}]$, resp.~$\pi^{(2)}[G^{(2)}]$ that are not matched, namely the {\bf unmatched nodes}, that is 
\begin{align*}
U^{(1)} =\pi^{(1)}(V^{(1)})\setminus \pi^{(2)}(V^{(2)})\ ; \quad \quad\quad 
U^{(2)} =\pi^{(2)}(V^{(2)})\setminus \pi^{(1)}(V^{(1)})\enspace . 
\end{align*}
Again, $|U^{(1)}|$ and $|U^{(2)}|$ only  depend on $(\pi^{(1)},\pi^{(2)})$ through the matching $\mathbf{M}$. 
We have, for $i\in \{1,2\}$,
\begin{align}\label{eq:unmatched}
    |V^{(i)}| = |\mathbf M| + |U^{(i)}|\enspace .
\end{align}
Write also $\mathbf M_{\mathrm{SM}} = \mathbf M_{\mathrm{SM}}(\mathbf M) \subset \mathbf M$, 
for the set of node matches of $(G^{(1)},G^{(2)})$ that are matched, and yet that are not pruned when creating the symmetric difference graph $G_{\Delta}$.
This is the set of {\bf semi-matched nodes}, i.e.\ at least one of their incident edges is not matched.  The remaining pairs of nodes $\mathbf M \setminus \mathbf M_{\mathrm{SM}}$ are said to be {\bf perfectly matched} as all the edges incident to them are matched. We write  $\mathbf M_{\mathrm{PM}} = \mathbf M_{\mathrm{PM}}(\mathbf M)$ for the set of perfect matches in $\mathbf M$. Note that 
\begin{equation}\label{eq:PMSM}
     |V^{(1)}| +|V^{(2)}| =   |V_{\Delta}| + |\mathbf M_{\mathrm{SM}}| + 2|\mathbf M_{\mathrm{PM}}|\enspace  .
\end{equation}

\subsubsection{Definition of some relevant sets of node matchings.} We define  $\mathcal M^\star\subset \mathcal{M}$ as the collection of matchings $\mathbf{M}$ such that all connected components of $G^{(1)}$ and of $G^{(2)}$ intersect\footnote{Here, we mean that, for each connected component, at least one of its vertices appears in a tuple of $\mathbf{M}$.} with $\mathbf{M}$. As a consequence, for any $\mathbf M \in \mathcal M^\star$, we have $\#\mathrm{CC}_{\mathrm{pure}} = 0$. Finally, we introduce  $\mathcal M_{\mathrm{PM}}\subset \mathcal{M}$ as
  the collection of perfect matchings, that is matchings $\mathbf M$ such that all the nodes in $V^{(1)}$ and $V^{(2)}$ are {\bf perfectly matched}. Note that, if $\mathbf M\in \mathcal M_{\mathrm{PM}}$, then $G_{\Delta}$ is the empty graph (with $E_{\Delta} = \emptyset$). Besides, $\mathcal{M}_{\mathrm{PM}} \neq \emptyset$ if and only if $G^{(1)}$ and $G^{(2)}$ are isomorphic, which is equivalent to $G^{(1)}= G^{(2)}$ when $G^{(1)}, G^{(2)} \in \mathcal G_{\leq D}$. 

\subsection{Further definitions}

This subsection gathers other concepts that will be useful for establishing the almost orthonormality of the basis. It can be skipped at first reading.

\subsubsection{Shadow matchings.} Write for two sets $\overline{U}_1\subset V^{(1)}, \overline{U}_2 \subset V^{(2)}$ and for a set of node matches $\underline{\mathbf M} \in \mathcal M$
 \begin{equation}\mathcal M_{\mathrm{shadow}}(\overline{U}_1, \overline{U}_2, \underline{\mathbf M} ) = \left\{\mathbf M' \in \mathcal M:~~U^{(1)}(\mathbf M') = \overline{U}_1,~U^{(2)}(\mathbf M') = \overline{U}_2,~\mathbf M_{\mathrm{SM}}(\mathbf M') = \underline{\mathbf M}\right\}\enspace ,\end{equation}
 namely the set of all matchings that lead to the set $\underline{\mathbf M} $ of semi-matched nodes and to the sets $\overline{U}_1, \overline{U}_2$ of unmatched nodes in resp.~$G^{(1)}, G^{(2)}$. We say that these matchings satisfy a given {\bf shadow} $(\overline{U}_1, \overline{U}_2, \underline{\mathbf M} )$ , see Figure~\ref{fig:shadow}. The only thing that can vary between two elements of $\mathcal M_{\mathrm{shadow}}(\overline{U}_1, \overline{U}_2, \underline{\mathbf M} )$ is the matching of the nodes that are not in $\overline{U}_1, \overline{U}_2$, or part of a pair of nodes in $\underline{\mathbf M}$. This matching must however ensure that all of these nodes are perfectly matched.

 \begin{figure}
    \centering
\includegraphics{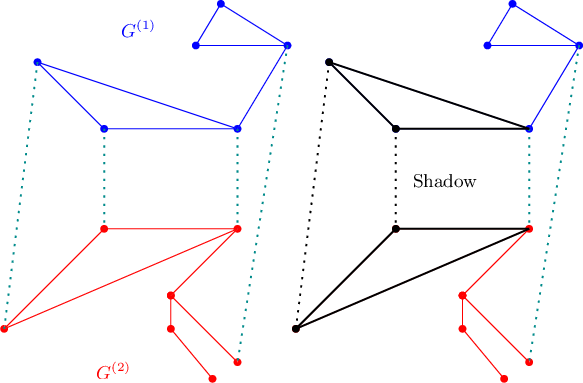}
    \caption{Illustration of the shadow of a graph. The information contained in the shadow consists of all labels of the nodes that are colored. The labels of the black part are not registered in the shadow, but we know the ``shape'' and that all nodes in the black part are perfectly matched.}
    \label{fig:shadow}
\end{figure}

\subsubsection{Edit distance between graphs.} For any two templates $G^{(1)}$ and $G^{(2)}$, we define  the so-called edit distance by
\begin{equation}\label{eq:definition:edit:distance}
d(G^{(1)}, G^{(2)}) := \min_{\mathbf M \in \mathcal M} |E_{\Delta}| \enspace .
\end{equation}
Note that $d(G^{(1)}, G^{(2)})=0$ if and only if $G^{(1)}$ and $G^{(2)}$ are isomorphic. As a consequence, if $G^{(1)}$ and $G^{(2)}$ are in $\mathcal{G}_{\leq D}$, the edit distance is equal to $0$ if and only if $G^{(1)}=G^{(2)}$.

\section{Core of the proof: the \HidSubI\ model when $q=1/2$}\label{s:core}

In what follows, our goal is to set aside the technicalities arising from the consideration of more complex models, and instead focus on the simple \HidSubI\ model in the case $q = 1/2$. This will allow us to clearly illustrate how our proof technique proceeds in order to prove almost orthonormality. We present a detailed and annotated proof for this specific case. Although the other models, as well as the case $q \neq 1/2$, involve certain important technical differences, the core ideas and methods of the proof remain the same.

Assume that $D\geq 2$ and that for some large enough universal constant $c_{0} \geq 5$:
\begin{align}\label{eq:signal1}
     \frac{\lambda k}{\sqrt{\overline{q} n}}\lor {\lambda\over \sqrt{\overline{q}}} \lor \frac{k}{n} \leq D^{-8c_{0}}\enspace ,
\end{align}
with $\overline{q}=q(1-q)$.
For $\mathbb P=\mathbb P_{H_{0}}$, our goal is to prove the almost $L^2(\mathbb P)$-orthonormality of the family of invariant polynomials $(\Psi_G)_{G\in \mathcal G_{\leq D}}$ with $\Psi_G=\frac{\overline{P}_G}{\sqrt{\mathbb V(G)}}$
defined in~\eqref{eq:cross} and~\eqref{eq:definition:P*G}.

Throughout this proof as well as other proofs, we shall often use that, for any template graph $G$, its numbers of vertices, edges, and connected components are respectively at most equal to $2D$, $D$, and $D$.  

\begin{proposition}\label{prop:ortho}
Let $\Gamma$ be the Gram matrix defined by  
\[
\left(\Gamma_{G^{(1)}, G^{(2)}}\right)_{G^{(1)}, G^{(2)}\in \mathcal G_{\leq D}} = \left(\mathbb E\left[\Psi_{G^{(1)}}\Psi_{G^{(2)}}\right]\right)_{G^{(1)}, G^{(2)}\in \mathcal G_{\leq D}}\enspace .
\]
Under Condition~\eqref{eq:signal1}, we have $\|\Gamma - I\|_{op}\leq  2 D^{-c_{0}}$. 
\end{proposition}
This proposition is essentially Theorem~\ref{thm:isorefo} in our specific model.
We emphasize that once this result is proven,  a bound on $\mathrm{Adv}_{\leq D}$ can be derived simply. Indeed, note first that 
 $\mathbb{E}\cro{\Psi_G}=0$ due to the centering~\eqref{eq:cross}, i.e.\ 1 is orthogonal to all $\Psi_G$. Hence 
 the previous proposition together with Lemma~\ref{lem:reduction:degree} implies that 
\begin{align}\label{eq:advat}
    \mathrm{Adv}_{\leq D}^2 &= \sup_{\alpha_{\emptyset},(\alpha_G)_{G\in \mathcal G_{\leq D}}} \frac{\mathbb E_{H_1}\left[ \alpha_{\emptyset}+\sum_{G \in \mathcal G_{\leq D}} \alpha_G \Psi_{G}\right]^2}{{\mathbb E\left[\left(\alpha_{\emptyset}+\sum_{G\in \mathcal G_{\leq D}} \alpha_G \Psi_{G}\right)^2\right]}} \leq \frac{1+\left\|(\mathbb E_{H_1}\Psi_G)_{G\in \mathcal G_{\leq D}}\right\|_2^2}{1- 2 D^{-c_{0}}}\enspace .
\end{align}
So it only remains to bound $\left\|(\mathbb E_{H_1}\Psi_G)_{G\in \mathcal G_{\leq D}}\right\|^2_2$ in order to get a bound on $\mathrm{Adv}_{\leq D}$.

In the remainder of this section, we focus on the proof of Proposition~\ref{prop:ortho}.
\medskip

\noindent
{\bf Step 0: Preliminary computations.} 
We observe that $\bbE\cro{Y_{ij}|z}=\Theta_{z_iz_j}=\lambda \mathbf{1}\{z_{i}\leq k\}\mathbf{1}\{z_{j}\leq k\}$ and 
$\bbE\cro{Y_{ij}^2|z}= \overline{q}+ \Theta_{z_iz_j}(1-2q)= \overline{q}$ for $q=1/2$. 
Consider two templates $G^{(1)}, G^{(2)}$, some node matching $\mathbf M \in \mathcal M$ and two injections $(\pi^{(1)}, \pi^{(2)})\in \Pi(\mathbf{M})$. 
Since the $Y_{ij}$ are conditionally independent given $z$, we have under $\mathbb P$
\begin{align}
   \mathbb{E}\left[P_{G^{(1)}, \pi^{(1)}}P_{G^{(2)}, \pi^{(2)}}\right] 
    & = \bbE\cro{\prod_{(i,j)\in E_{\Delta}}Y_{ij}\prod_{(i,j)\in E_{\cap}}Y_{ij}^2} \nonumber \\
   & =   \bbE\cro{\prod_{(i,j)\in E_{\Delta}} (\lambda  \mathbf{1}\{z_{i}\leq k\}\mathbf{1}\{z_{j}\leq k\}) \prod_{(i,j)\in E_{\cap}}\overline{q}} \nonumber \\
    & =  \lambda^{|E_{\Delta}|} \overline{q}^{|E_{\cap}|} \bbP\cro{z_{i}\leq k,\ \text{for}\ i\in V_{\Delta}} \nonumber \\
    &   = \lambda^{|E_{\Delta}|}\pa{\frac{k}{n}}^{|V_{\Delta}|} \overline{q}^{|E_{\cap}|} = \pa{\lambda\over \overline{q}}^{|E_{\Delta}|}\pa{\frac{k}{n}}^{|V_{\Delta}|} \overline{q}^{|E_{\cup}|}\enspace .  \label{eqn:old basis proof ingredients}
\end{align}
This implies in particular (with $G^{(2)}=\emptyset$)
\begin{align}
   \mathbb{E}\left[P_{G^{(1)}, \pi^{(1)}}\right] 
    &    = \lambda^{|E^{(1)}|}\pa{\frac{k}{n}}^{|V^{(1)}|}\enspace . \label{eq:comp}
\end{align}
Also, if $\mathbf M = \emptyset$, then for any functions $f^{(1)},f^{(2)}$ of the variables $Y_{ij}$ on the edges respectively in $\pi^{(1)}(E^{(1)}), \pi^{(2)}(E^{(2)})$, we have
\begin{equation}\label{eq:indep}
    \mathbb E\left[f^{(1)}f^{(2)}\right]=\mathbb E\left[f^{(1)}\right]\mathbb E\left[f^{(2)}\right]\enspace ,
\end{equation}
by independence of $(Y_{ij})_{(i,j)\in \pi^{(1)}(E^{(1)})}$ and $(Y_{ij})_{(i,j)\in \pi^{(2)}(E^{(2)})}$.

\medskip

\noindent
{\bf Step 1: From $P_{G,\pi}$ to $\overline{P}_{G,\pi}$.}  
A first key observation is that thanks to the centering~\eqref{eq:cross}, the Gram matrix \begin{equation}\left(\mathbb E\left[\overline{P}_{G^{(1)},\pi^{(1)}}\overline{P}_{G^{(2)},\pi^{(2)}}\right]\right)_{G^{(1)}, G^{(2)}\in \mathcal G_{\leq D}, \pi^{(1)} \in \Pi_{V^{(1)}}, \pi^{(2)} \in \Pi_{V^{(2)}}}\end{equation} associated to $(\overline{P}_{G,\pi})_{G\in \mathcal G_{\leq D}, \pi \in \Pi_V}$ is quite sparse, unlike the one associated to $(P_{G,\pi})_{G\in \mathcal G_{\leq D}, \pi \in \Pi_V}$. Furthermore,  on the non-zero  entries, it is quite close to the Gram matrix associated to $(P_{G,\pi})_{G\in \mathcal G_{\leq D}, \pi \in \Pi_V}$. 
\begin{proposition} \label{eq:boundbar}
   Let   $\mathcal M^\star$ be the collection of matchings $\mathbf{M}$ such that all connected components of $G^{(1)}$ and of $G^{(2)}$ intersect with $\mathbf{M}$.
    \begin{enumerate}
    \item If $\mathbf{M}\notin\mathcal{M}^{\star}$, we have $\mathbb{E}\left[\overline{P}_{G^{(1)}, \pi^{(1)}}\overline{P}_{G^{(2)}, \pi^{(2)}}\right] = 0$; 
    \item If $\mathbf{M}\in \mathcal{M}^\star$, we have
    \begin{equation*}
        \left|\dfrac{\mathbb{E}\left[\overline{P}_{G^{(1)}, \pi^{(1)}}\overline{P}_{G^{(2)}, \pi^{(2)}}\right] - \mathbb{E}\left[P_{G^{(1)}, \pi^{(1)}}P_{G^{(2)}, \pi^{(2)}}\right]}{\mathbb{E}\left[P_{G^{(1)}, \pi^{(1)}}P_{G^{(2)}, \pi^{(2)}}\right]}\right| \leq D^{-3c_{0}}\enspace.
    \end{equation*}
    \end{enumerate} 
\end{proposition}

\begin{proof}[Proof of Proposition~\ref{eq:boundbar}]

 Write $G^{(1)} = (G_1^{(1)}, \ldots, G^{(1)}_{\#\mathrm{CC}_{G^{(1)} }})$, and $G^{(2)} = (G_1^{(2)}, \ldots, G^{(2)}_{\#\mathrm{CC}_{G^{(2)} }})$ for the decomposition of $G^{(1)},G^{(2)}$ into their resp.~$\#\mathrm{CC}_{G^{(1)} },\#\mathrm{CC}_{G^{(2)} }$ connected components. Write also $\pi_1^{(1)}, \ldots, \pi^{(1)}_{\#\mathrm{CC}_{G^{(1)} }}$, and $\pi_1^{(2)}, \ldots, \pi^{(2)}_{\#\mathrm{CC}_{G^{(2)} }}$ for their respective labelings.
 \smallskip

\noindent
{\bf Proof of 1):} If $\mathbf M \not\in \mathcal M^\star$, there exists one connected component belonging to either $G^{(1)}$ or $G^{(2)}$, whose nodes are not matched in $\mathbf M$. Assume w.l.o.g.~that this connected component is $G_{1}^{(1)}$.
By Equation~\eqref{eq:indep}, we have
\begin{equation}\mathbb E\left[\overline{P}_{G^{(1)},\pi^{(1)}} \overline{P}_{G^{(2)},\pi^{(2)}}\right] = \mathbb E\left[\overline{P}_{G_1^{(1)},\pi^{(1)}}\right] \mathbb E\left[\prod_{l=2}^{\#\mathrm{CC}_{G^{(1)} }} \overline{P}_{G_l^{(1)},\pi^{(1)}} \times \overline{P}_{G^{(2)},\pi^{(2)}}\right]=0\enspace ,\end{equation}
since $\mathbb E[\overline{P}_{G^{(1)}_{1}, \pi^{(1)}}]=0$ according to the centering~\eqref{eq:cross}.
\smallskip

\noindent
{\bf Proof of 2):} From the identity 
\begin{equation}\prod_{\ell=1}^L (a_{\ell}-b_{\ell})=\sum_{S\subset [L]} (-1)^{|S|} \prod_{\ell \notin S} a_{\ell} \prod_{\ell \in S} b_{\ell}\enspace ,\end{equation}
we derive
\begin{align*}\lefteqn{    \mathbb{E}\left[\overline{P}_{G^{(1)}, \pi^{(1)}}\overline{P}_{G^{(2)}, \pi^{(2)}}\right] =} & \\ 
 &\sum_{\substack{S_1 \subset [\#\mathrm{CC}_{G^{(1)} }]\\ S_2 \subset [\#\mathrm{CC}_{G^{(2)} }]}} (-1)^{|S_1|+|S_2|}\mathbb E\left[\prod_{i\in [\#\mathrm{CC}_{G^{(1)} }]\setminus S_1} P_{G^{(1)}_i,\pi_i^{(1)}} \prod_{i\in [\#\mathrm{CC}_{G^{(2)} }]\setminus S_2} P_{G^{(2)}_i,\pi_i^{(2)}}\right] \\ & \hspace{6cm}\times\mathbb E\left[\prod_{i\in S_1} P_{G^{(1)}_i,\pi_i^{(1)}} \right]\mathbb E\left[\prod_{i\in S_2} P_{G^{(2)}_i,\pi_i^{(2)}}\right]\enspace .
\end{align*}
Then, the following lemma holds.
\begin{lemma}\label{lem:unioncomp}
For any $S_1 \subset [\#\mathrm{CC}_{G^{(1)} }]$ and any $S_2 \subset [\#\mathrm{CC}_{G^{(2)} }]$, we have 
\begin{align*}
0 \leq \mathbb E\left[\prod_{i\in [\#\mathrm{CC}_{G^{(1)} }]\setminus S_1} P_{G^{(1)}_i,\pi_i^{(1)}} \prod_{i\in [\#\mathrm{CC}_{G^{(2)} }]\setminus S_2} P_{G^{(2)}_i,\pi_i^{(2)}}\right]& 
\mathbb E\left[\prod_{i\in S_1} P_{G^{(1)}_i,\pi_i^{(1)}} \right]\mathbb E\left[\prod_{i\in S_2} P_{G^{(2)}_i,\pi_i^{(2)}}\right]\\
&\leq \left(\frac{k}{n}\right)^{(|S_1| + |S_2|)/2} \mathbb{E}\left[P_{G^{(1)}, \pi^{(1)}}P_{G^{(2)}, \pi^{(2)}}\right] . 
\end{align*}
\end{lemma}

This leads us to
\begin{align*}
\lefteqn{    \frac{\left|\mathbb{E}\left[\overline{P}_{G^{(1)}, \pi^{(1)}}\overline{P}_{G^{(2)}, \pi^{(2)}}\right] - \mathbb{E}\left[P_{G^{(1)}, \pi^{(1)}}P_{G^{(2)}, \pi^{(2)}}\right]\right|}{\mathbb{E}\left[P_{G^{(1)}, \pi^{(1)}}P_{G^{(2)}, \pi^{(2)}}\right]} }& \\
&\hspace{2cm}\leq  \sum_{S_1 \subset [\#\mathrm{CC}_{G^{(1)} }]\enspace ,  S_2 \subset [\#\mathrm{CC}_{G^{(2)} }]: |S_1|\lor |S_2| \geq 1} \left(\frac{k}{n}\right)^{(|S_1| + |S_2|)/2}\\
     &\hspace{2cm} \leq  \sum_{s_1 \leq D , s_2\leq D: s_1\lor s_2 \geq 1} D^{s_1+s_2}\left(\frac{k}{n}\right)^{(s_1 + s_2)/2}\leq D^{-3c_{0}}\enspace ,
\end{align*}
by Equation~\eqref{eq:signal1} with $c_{0} \geq 5$ and $D \geq 2$.
\end{proof}

\begin{proof}[Proof of Lemma~\ref{lem:unioncomp}]
Define the matching $\overline{\mathbf M}$ from $\mathbf{M}$  by removing all node pairs such that at least one node lies in the connected components indexed by $S_1$ or $S_2$. Then, we take  $(\overline{\pi}^{(1)}, \overline{\pi}^{(2)}) \in \Pi(\overline{ \mathbf M})$. Note that, without loss of generality, we can  take $\overline{\pi}^{(1)} = \pi^{(1)}$, take $\overline{\pi}^{(2)}$ equal to $\pi^{(2)}$ on the nodes of $V^{(2)}$ that do not belong to a match in $\mathbf M\setminus \overline{\mathbf{M}}$, and take $\overline{\pi}^{(2)}$ on the nodes of $V^{(2)}$ that belong to a match in $\mathbf M\setminus \overline{\mathbf{M}}$ with arbitrary distinct values in $[n]\setminus (\mathrm{Im}(\pi^{(1)})\cup \mathrm{Im}(\pi^{(2)}))$.
Equipped with this notation, we have 
\begin{align}\nonumber
     \mathbb E\left[\prod_{i\in [\#\mathrm{CC}_{G^{(1)} }]\setminus S_1} P_{G^{(1)}_i,\pi_i^{(1)}} \prod_{i\in [\#\mathrm{CC}_{G^{(2)} }]\setminus S_2} P_{G^{(2)}_i,\pi_i^{(2)}}\right]
\mathbb E\left[\prod_{i\in S_1} P_{G^{(1)}_i,\pi_i^{(1)}} \right]\mathbb E\left[\prod_{i\in S_2} P_{G^{(2)}_i,\pi_i^{(2)}}\right]\\ &\hspace{-4cm}= \mathbb E\left[P_{G^{(1)}, \overline{\pi}^{(1)}}P_{G^{(2)}, \overline{\pi}^{(2)}}\right]\enspace . \label{eq:coolio}
\end{align}
We write $ G_\Delta = ( V_\Delta,  E_\Delta),  G_\cap=( V_\cap,  E_\cap),   G_\cup=( V_\cup,  E_\cup)$ for the resp.~symmetric difference, intersection and union graphs corresponding to the labeled graphs $ \pi^{(1)}[G^{(1)}],  \pi^{(2)}[G^{(2)}]$ and also $\overline{G}_\Delta = (\overline{V}_\Delta, \overline{E}_\Delta), \overline{G}_\cap=(\overline{V}_\cap, \overline{E}_\cap),  \overline{G}_\cup=(\overline{V}_\cup, \overline{E}_\cup)$ for the resp.~symmetric difference, intersection and union graphs corresponding to the labeled graphs $\overline{\pi}^{(1)}[G^{(1)}], \overline{\pi}^{(2)}[G^{(2)}]$.

By Equation~\eqref{eqn:old basis proof ingredients}, we have
\begin{align*}
   \mathbb{E}\left[P_{G^{(1)}, \pi^{(1)}}P_{G^{(2)}, \pi^{(2)}}\right] 
      & = \left(\frac{\lambda}{\overline{q}}\right)^{|E_{\Delta}|}\pa{\frac{k}{n}}^{|V_{\Delta}|} \overline{q}^{|E_{\cup}|}\ ; \\ \mathbb{E}\left[P_{G^{(1)}, \overline{\pi}^{(1)}}P_{G^{(2)}, \overline{\pi}^{(2)}}\right] 
       &= \left(\frac{\lambda}{\overline{q}}\right)^{|\overline{E}_{\Delta}|}\pa{\frac{k}{n}}^{|\overline{V}_{\Delta}|} \overline{q}^{|\overline{E}_{\cup}|}\enspace.
\end{align*}
Since $\overline{\mathbf M} \subset \mathbf M$, we have
\begin{equation}|\overline{E}_{\Delta}| \geq |E_\Delta|~~~\mathrm{and}~~~~|\overline{E}_{\cup}| \geq |E_{\cup}|\enspace .\end{equation}
Hence, since $\lambda \leq \overline{q} =1/4$,
\begin{align*}
   \mathbb{E}\left[P_{G^{(1)}, \overline{\pi}^{(1)}}P_{G^{(2)}, \overline{\pi}^{(2)}}\right] 
       \leq \pa{\frac{k}{n}}^{|\overline{V}_{\Delta}| - | V_{\Delta}|} \mathbb{E}\left[P_{G^{(1)}, \pi^{(1)}}P_{G^{(2)}, \pi^{(2)}}\right]\enspace .
\end{align*}
In addition, again since  $\overline{\mathbf M} \subset \mathbf M$, we have 
\begin{equation}|\overline{V}_\Delta|\geq  |V_\Delta| + |\mathbf M \setminus \overline{\mathbf M}|\enspace .\end{equation}
On $\mathcal{M}^\star$, each connected component indexed by $S_1,S_2$ must contain at least one matched node in $\mathbf M$, which cannot be in $\overline{\mathbf M}$, so
we have
\begin{equation}2|\mathbf M \setminus \overline{\mathbf M}| \geq |S_1|+ |S_2|\enspace .\end{equation}
Combining the last three equations concludes the proof of this lemma.
\end{proof}
\medskip

\noindent
{\bf Step 2: Entrywise control of the Gram matrix.} 
A first step towards deriving a bound on the operator norm of $\Gamma-I$ is to derive a bound for each entry.
Building on the orthogonality between many $\overline{P}_{G^{(1)}, \pi^{(1)}}$ and $\overline{P}_{G^{(2)}, \pi^{(2)}}$, and on the proximity between $\overline{P}_{G,\pi}$ and $P_{G,\pi}$, we prove below that $\Gamma$
is entrywise  close to the identity, which is the core of the proof of Proposition~\ref{prop:ortho}.
\begin{proposition}\label{prop:scalprod}
    Consider two templates $G^{(1)},G^{(2)} \in \mathcal G_{\leq D}$. We have 
    \begin{equation*}
     |\Gamma_{G^{(1)}, G^{(2)}} -\mathbf 1\{G^{(1)} = G^{(2)}\}|=    \left|\mathbb E\left[\Psi_{G^{(1)}}\Psi_{G^{(2)}}\right] - \mathbf 1\{G^{(1)} = G^{(2)}\}\right|\leq 2D^{-3c_{0} (d (G^{(1)}, G^{(2)})\lor 1)}\enspace,
    \end{equation*}
    with $d$ the edit distance defined by~\eqref{eq:definition:edit:distance}. 
\end{proposition} 

\begin{proof}[Proof of Proposition~\ref{prop:scalprod}]
We have by definition:
\begin{align*}
     \mathbb E\left[\Psi_{G^{(1)}}\Psi_{G^{(2)}}\right]  &= \sum_{\mathbf{M}\in\mathcal{M}}\sum_{(\pi^{(1)}, \pi^{(2)})\in \Pi(\mathbf{M})}\frac{1}{\sqrt{\mathbb{V}(G^{(1)})\mathbb{V}(G^{(2)})}}\mathbb{E}\left[\overline{P}_{G^{(1)}, \pi^{(1)}}\overline{P}_{G^{(2)}, \pi^{(2)}}\right]\\
     &= \sum_{\mathbf{M}\in\mathcal{M}^{\star}}\sum_{(\pi^{(1)}, \pi^{(2)})\in \Pi(\mathbf{M})}\frac{1}{\sqrt{\mathbb{V}(G^{(1)})\mathbb{V}(G^{(2)})}}\mathbb{E}\left[\overline{P}_{G^{(1)}, \pi^{(1)}}\overline{P}_{G^{(2)}, \pi^{(2)}}\right]\enspace , 
    \end{align*}
where the second line follows from Proposition~\ref{eq:boundbar}. \smallskip

\noindent
{\bf Step 2a: Decomposition of the scalar product over $\mathcal M^\star \setminus \mathcal M_{\mathrm{PM}}$ and $\mathcal M_{\mathrm{PM}}$.} The set $\mathcal{M}_{\mathrm{PM}}$ is non-empty only if $G^{(1)} = G^{(2)}$. And if $G^{(1)} = G^{(2)}$, we have for any $\pi \in \Pi_{V^{(1)}}$
\begin{align*}\sum_{\mathbf{M}\in\mathcal{M}_{\mathrm{PM}}}\sum_{(\pi^{(1)}, \pi^{(2)})\in \Pi(\mathbf{M})}\frac{1}{\sqrt{\mathbb{V}(G^{(1)})\mathbb{V}(G^{(2)})}}\mathbb{E}\left[\overline{P}_{G^{(1)}, \pi^{(1)}}\overline{P}_{G^{(2)}, \pi^{(2)}}\right] & =  \frac{n!|\mathrm{Aut}(G^{(1)})|\mathbb{E}\left[\overline{P}_{G^{(1)}, \pi}^2\right]}{\left(n-|V^{(1)}|\right)!\ \mathbb{V}(G^{(1)})}  \\ &=  \frac{\mathbb E[\overline{P}_{G^{(1)}, \pi}^2]}{\overline{q}^{|E^{(1)}|}}\enspace ,\end{align*}
since  
\begin{align}\label{eq:boubou}
    |\mathcal{M}_{\mathrm{PM}}| =|\mathrm{Aut}(G^{(1)})|~~~~\mathrm{and}~~~~|\Pi(\mathbf{M})| = \frac{n!}{\left(n-(|V^{(1)}|+ |V^{(2)}| - |\mathbf M|)\right)!}\enspace ,
\end{align}
and by definition of $\mathbb V(G^{(1)})$. 
Equation~\eqref{eqn:old basis proof ingredients} ensures that $\mathbb E[P_{G^{(1)}, \pi^{(1)}}^2] = \overline{q}^{|E^{(1)}|}$, so
 by Proposition~\ref{eq:boundbar}  we have
\begin{equation}(1-D^{-3c_{0}})\overline{q}^{|E^{(1)}|}\leq \mathbb E[\overline{P}_{G^{(1)}, \pi^{(1)}}^2] \leq  \overline{q}^{|E^{(1)}|}(1+D^{-3c_{0}})\enspace .\end{equation}
Hence
\begin{align*}
     &\left|\mathbb E\left[\Psi_{G^{(1)}}\Psi_{G^{(2)}}\right] - \mathbf 1\{G^{(1)} = G^{(2)}\} \right|\\
     &\leq \left|\sum_{\mathbf{M}\in\mathcal{M}^{\star}\setminus \mathcal M_{\mathrm{PM}}}\sum_{(\pi^{(1)}, \pi^{(2)})\in \Pi(\mathbf{M})}\frac{1}{\sqrt{\mathbb{V}(G^{(1)})\mathbb{V}(G^{(2)})}}\mathbb{E}\left[\overline{P}_{G^{(1)}, \pi^{(1)}}\overline{P}_{G^{(2)}, \pi^{(2)}}\right]\right| + \1\{G^{(1)} = G^{(2)}\}D^{-3c_{0}}
     \\ & \quad\quad =: A  + \1\{G^{(1)} = G^{(2)}\}D^{-3c_{0}}\enspace.
    \end{align*}

\noindent
{\bf Step  2b: Making $A$ explicit as a sum of $A_{\mathbf{M}}$.} Observe that for any $\mathbf M \in \mathcal M$, we have that $\mathbb{E}\left[\overline{P}_{G^{(1)}, \pi^{(1)}}\overline{P}_{G^{(2)}, \pi^{(2)}}\right] = E_{\mathbf M}$ is constant for any $(\pi^{(1)}, \pi^{(2)})\in \Pi(\mathbf{M})$. Hence, by Equation~\eqref{eq:boubou},
\begin{align}
    A&= \left|\frac{1}{\sqrt{\mathbb{V}(G^{(1)})\mathbb{V}(G^{(2)})}}\sum_{\mathbf{M}\in\mathcal{M}^{\star}\setminus \mathcal{M}_{\mathrm{PM}}}\sum_{(\pi^{(1)}, \pi^{(2)})\in \Pi(\mathbf{M})}E_{\mathbf M}\right|\nonumber\\
    &= \left|\frac{1}{\sqrt{\mathbb{V}(G^{(1)})\mathbb{V}(G^{(2)})}}\sum_{\mathbf{M}\in\mathcal{M}^{\star}\setminus \mathcal{M}_{\mathrm{PM}}}\frac{n!}{\left(n-(|V^{(1)}|+ |V^{(2)}| - |\mathbf M|)\right)!}E_{\mathbf M}\right|\enspace .\nonumber
    \end{align}
    By Proposition~\ref{eq:boundbar} and Equation~\eqref{eqn:old basis proof ingredients}, we have
    \begin{equation}|E_{\mathbf M}| \leq \lambda^{|E_{\Delta}|}\left(\frac{k}{n}\right)^{|V_{\Delta}|}  \overline{q}^{|E_{\cap}|}
    \left(1+D^{-3c_{0}}\right)\enspace ,\end{equation}
where we recall that $|E_{\Delta}|$, $|V_{\Delta}|$, $|E_{\cap}|$ only depend on the matching $\mathbf{M}$ as all graphs $G_{\Delta}$ (resp.~$G_{\cap}$) are isomorphic for $(\pi^{(1)},\pi^{(2)})\in \Pi(\mathbf{M})$. ~\\
Since $\frac{n!}{\left(n-(|V^{(1)}|+ |V^{(2)}| - |\mathbf M|)\right)!} \frac{\sqrt{(n-|V^{(1)}|)!(n-|V^{(2)}|)!}}{n!}\leq n^{(|U^{(1)}| + |U^{(2)}|)/2}$ where we recall that $|U^{(a)}|$ is the number of unmatched nodes in $G^{(a)}$, and by definition of $\mathbb V(G)$, we have
    \begin{align}
    A&\leq  \frac{1}{\sqrt{\left|\mathrm{Aut}(G^{(1)})\right|\left|\mathrm{Aut}(G^{(2)})\right|}}\sum_{\mathbf{M}\in\mathcal{M}^{\star}\setminus \mathcal{M}_{\mathrm{PM}}}n^{(|U^{(1)}| + |U^{(2)}|)/2}\left(\frac{\lambda}{\sqrt{\overline{q}}}\right)^{|E_{\Delta}|}\left(\frac{k}{n}\right)^{|V_{\Delta}|} 
    \left(1+D^{-3c_{0}}\right)\nonumber\\
     &  \hspace{-0.3cm}  \leq  \frac{2}{\sqrt{\prod_{a=1,2}\left|\mathrm{Aut}(G^{(a)})\right|}}\sum_{\mathbf{M}\in\mathcal{M}^{\star}\setminus \mathcal{M}_{\mathrm{PM}}}\left(\frac{\lambda k}{\sqrt{\overline{q} n}}\right)^{|U^{(1)}| + |U^{(2)}|}\pa{\lambda\over \sqrt{\overline{q}}}^{|E_{\Delta}| - |U^{(1)}| - |U^{(2)}|}\left(\frac{k}{n}\right)^{ |V_{\Delta}| - |U^{(1)}| - |U^{(2)}|} \nonumber
\end{align}
using that $D \geq 2$ and $c_{0} \geq 4$, and rearranging terms in the last line. Write $A_{\mathbf{M}}$ for the summand in the last line. \smallskip

\noindent
{\bf Step 2c: Bounding  $A$ by summing over shadows.} Recall that we defined shadows and  $\mathcal M_{\mathrm{shadow}}$ in Section~\ref{sec:graph:definition}. We now regroup the sum inside $A$ by enumerating  all possible matchings  that are compatible with a  shadow. We get 
\begin{align*}
       A &\leq  \frac{2}{\sqrt{|\mathrm{Aut}(G^{(1)})| |\mathrm{Aut}(G^{(2)})|}}\sum_{\substack{U^{(1)} \subset V^{(1)}, U^{(2)} \subset V^{(2)}, \\\underline{\mathbf M} \in \mathcal M\setminus \mathcal M_{\mathrm{PM}}}}
       \quad \sum_{\mathbf M \in \mathcal M_{\mathrm{shadow}}(U^{(1)},U^{(2)}, \underline{\mathbf M})} A_{\mathbf{M}}\enspace .
\end{align*}
We have the following control for the cardinality of $\mathcal M_{\mathrm{shadow}}$:
    \begin{equation}|\mathcal M_{\mathrm{shadow}}(U^{(1)}, U^{(2)}, \underline{\mathbf M} )| \leq \min(|\mathrm{Aut}(G^{(1)})|, |\mathrm{Aut}(G^{(2)})|)\enspace ,\end{equation}
    see  Lemma~\ref{lem:shadow} and its proof. Observe that two matchings $\mathbf M$ and $\mathbf M'$ that belong to~\\ $\mathcal M_{\mathrm{shadow}}(U^{(1)},U^{(2)}, \underline{\mathbf M})$ have the same difference graph $G_{\Delta}$. Hence
\begin{align}
       A &\leq  2 
       \sum_{\substack{U^{(1)} \subset V^{(1)}, U^{(2)} \subset V^{(2)},\\ \underline{\mathbf M} \in \mathcal M\setminus \mathcal M_{\mathrm{PM}}}}  A_{\mathbf{M}}\enspace .\label{eqn:sum for proof sketch}
\end{align}

\noindent
{\bf Step 2d: Bounding $A_{\mathbf{M}}$.}  
A key observation is that for any graph $G_{\Delta} = (V_{\Delta},E_{\Delta})$ without isolated nodes, we have  
$|E_{\Delta}| \geq |V_{\Delta}| - \#\mathrm{CC}_{\Delta}$. 
 Since $|\mathbf{M}_{\mathrm{SM}}| + |U^{(1)}| + |U^{(2)}| = |V_{\Delta}|$, it follows
\begin{align*}
|E_{\Delta}| &\geq  |U^{(1)}| + |U^{(2)}| &\text{if }\mathbf M \in \mathcal M^\star,~~~~~ \mathrm{since~in~this~case}~~|\mathbf{M}_{\mathrm{SM}}| - \#\mathrm{CC}_{\Delta} \geq 0;\\
|E_{\Delta}| &\geq d(G^{(1)}, G^{(2)})\lor 1,~~~&\text{by definition of the edit distance and if } \mathbf M\not\in \mathcal M_{\mathrm{PM}}.
\end{align*}
Hence, the signal assumption~\eqref{eq:signal1} ensures that for $\mathbf M \in \mathcal M^\star\setminus \mathcal M_{\mathrm{PM}}$
\begin{align*}
    A_{\mathbf{M}} &\leq D^{-8c_{0}(|U^{(1)}| + |U^{(2)}|)} D^{-8c_{0} ( |E_\Delta| -(|U^{(1)}| + |U^{(2)}|))} D^{-8c_{0}|\mathbf M_{\mathrm{SM}}|}\leq D^{-8c_{0}[|E_\Delta| + |\mathbf M_{\mathrm{SM}}|]}\\ 
    &\leq D^{-4c_{0}\left[d(G^{(1)}, G^{(2)})\lor 1+|U^{(1)}|+ |U^{(2)}| + |\mathbf M_{\mathrm{SM}}|\right]}\enspace.
\end{align*}

\noindent
{\bf Step 2e: Final bound on $A$.} Plugging this bound on $A_{\mathbf M}$ back into Equation~\eqref{eqn:sum for proof sketch}, we get
\begin{align}
       A &\leq  2 
       \sum_{U^{(1)} \subset V^{(1)}, U^{(2)} \subset V^{(2)}, \underline{\mathbf M} \in \mathcal M\setminus \mathcal M_{\mathrm{PM}}}  D^{-4c_{0}\left[d(G^{(1)}, G^{(2)})\lor 1+|U^{(1)}|+ |U^{(2)}| + |\underline{\mathbf M}|\right]}\enspace .
\end{align}
So, when we enumerate over all possible sets $U^{(1)}, U^{(2)}, \underline{\mathbf M} $ that have respective cardinalities $u_1$, $u_2$, and  $m$, and since the numbers of such sets are bounded resp.~by $(2D)^{u_1}, (2D)^{u_2}$ and $(2D)^{2m}$, we obtain
\begin{align*}
       A &\leq  2
       \sum_{u_1, u_2, m \geq 0}(2D)^{u_1+u_2+2m}  D^{-4c_{0}\left[d(G^{(1)}, G^{(2)})\lor 1+u_1+u_2+m\right]} \leq D^{-3c_{0}  (d(G^{(1)},G^{(2)})\lor 1)}\enspace ,
\end{align*}
using again that $c_{0} \geq 5$ and $D \geq 2$. 
\end{proof}
\medskip

\noindent{\bf Step 3: From entrywise bound to operator norm bound.}
In Step 2, we proved an entrywise bound on $\Gamma-I$. To prove  Proposition~\ref{prop:ortho},
it remains to provide a bound in operator norm.
Since for symmetric matrices the $\ell^2\to\ell^2$ operator norm can be upper bounded by the $\ell^{\infty}\to\ell^{\infty}$ operator norm, which is the maximum of the $\ell^1$-norms of the rows, we can translate the entrywise bound of Proposition~\ref{prop:scalprod} to a bound in $\ell^2\to\ell^2$ operator norm
\begin{equation*}
\|\Gamma - I\|_{op}\leq \max_{G^{(1)}}\Big[\left|\Gamma_{G^{(1)}, G^{(1)}}-1\right|+  \sum_{G^{(2)} \in \mathcal G_{\leq D}, G^{(2)} \neq G^{(1)} }\left|\Gamma_{G^{(1)}, G^{(2)}}\right|\Big]\enspace .
\end{equation*} 
We have the following lemma. 
\begin{lemma}\label{lem:combinationarial:distance}
Fix a template $G^{(1)}$. For any positive integer $u$, we have
\begin{equation*}
\#\{G^{(2)}\in \mathcal{G}_{\leq D}: d( G^{(1)}, G^{(2)}) = u \} \leq (2u+2D)^{2u}\enspace . 
\end{equation*}
\end{lemma}

\begin{proof}[Proof of Lemma~\ref{lem:combinationarial:distance}]
If $d( G^{(1)}, G^{(2)}) = u$, this entails that there exist labelings $\pi^{(1)}$ and $\pi^{(2)}$ of these two templates such that the edit distance between the labeled graphs is equal to $u$. For a given graph with $v$ nodes, the number of graphs at edit distance equal to $u$ is at most $v^{2u}$. Since $G^{(1)}$ has at most $2D$ nodes and the number of additional nodes given by the labeling of $G^{(2)}$ is at most $2u$, the result follows. 
\end{proof}

We also use that, if $G^{(2)} \neq G^{(1)}$, then $d(G^{(1)}, G^{(2)})\geq 1$ as they are not isomorphic. It then follows from Proposition~\ref{prop:scalprod} and  Lemma~\ref{lem:combinationarial:distance} that
\begin{align*}
    \sum_{G^{(2)} \in \mathcal G_{\leq D}, G^{(2)} \neq G^{(1)} }\left|\Gamma_{G^{(1)}, G^{(2)}}\right|
    &\leq \sum_{G^{(2)} \in \mathcal G_{\leq D}, G^{(2)} \neq G^{(1)}} 2D^{-3c_{0}d(G^{(1)},G^{(2)})}\\
    &\leq \sum_{2D \geq  u \geq 1} |\{G^{(2)}: d( G^{(1)}, G^{(2)}) = u \}|  2D^{-3c_{0}u}\enspace\\
    &\leq \sum_{2D \geq u\geq 1} 2(2u+2D)^{2u}  D^{-3c_{0}u}\enspace
    \leq \sum_{2D \geq u\geq 1} 2D^{-(3c_{0}-8)u} \leq  D^{-c_{0}}\enspace ,
\end{align*}
since $2u+2D\leq 6D\leq D^{4}$ for $u\leq 2D$, $D \geq 2$, and $c_{0} \geq 5$.
Using Proposition~\ref{prop:scalprod} to bound the remaining term $\left|\Gamma_{G^{(1)}, G^{(1)}}-1\right|$, we conclude the proof of Proposition~\ref{prop:ortho}.

\section{Discussion}

\subsection{Flexibility of the almost orthonormal basis}\label{subsec:flexibility}

In this work, we introduced the polynomial basis $(\Psi_G)_{G\in\mathcal{G}_{\leq D}}$, which turns out to be almost orthonormal for a variety of permutation-invariant graph models. 
This almost orthonormality can be readily exploited to establish low-degree (LD) lower bounds for other testing and functional estimation problems, such as testing the value of $k$ or $\lambda$. 
Moreover, it enables a tighter connection between LD upper and lower bounds: below the recovery phase transition, it provides an explicit description of nearly optimal LD polynomials, which points to the motifs that should be counted, see Section~\ref{subsec:insights}.

To illustrate the flexibility of our approach, we mention some subsequent applications of it in~\cite{carpentier2025phase,carpentier2025phase2}
to SBMs with a large number $K=n/k$ of groups. In a striking paper, Chin et al.~\cite{pmlr-v291-chin25a} have recently shown that, at least in the regime where $q$ scales as $1/n$, it is possible to recover the groups when $K\geq \sqrt{n}$ below the Kesten--Stigum threshold that was long conjectured~\cite{Decelle2011} to be the computational barrier. Their procedure is based on the number of non-backtracking paths in the graph. However, they did not provide a matching LD lower bound. By considering an almost orthonormal basis similar to $\Psi_G$ but with a  different normalization,~\cite{carpentier2025phase} have established an LD lower bound for all regimes of $q$. They also introduced in~\cite{carpentier2025phase,carpentier2025phase2} new efficient procedures based on motif counting that match this LD lower bound.  The choice of the motifs ,\ cliques, self-avoiding paths, blow-up graphs~\cite{carpentier2025phase2}\ ,\ actually depends on the sparsity $q$. Importantly, in these works, the almost orthonormal basis provides strong insights that are instrumental for the construction of these new procedures.

Beyond graph data, this approach has been subsequently extended to Gaussian mixture models~\cite{carpentier2026low}. In the latter, the authors establish new LD lower bounds and matching upper bounds, thereby unveiling new regimes that are not captured by the cumulant techniques of~\cite{SchrammWein22} used in~\cite{Even25a}.

More generally, we expect this approach to extend naturally to other distributional models with transformation invariance. Compared with~\cite{SchrammWein22} and~\cite{SohnWein25}, our constructive method is more direct and transparent. 
In particular, it provides an alternative proof strategy to~\cite{SohnWein25} when inverting the overcomplete linear system therein becomes intractable, and an alternative to~\cite{SchrammWein22} when controlling the cumulants  proves difficult, or when the Jensen bound in~\cite{SchrammWein22} is not tight. We illustrated this by establishing LD lower bounds for \condinvarianceWR{} models.

\textbf{A generic recipe.} Our construction can be summarized in four steps, which we expect to apply to other latent variable models. (i) \emph{Reference basis.} Start from an orthonormal basis under a simple reference distribution: here, the monomials in the centered entries of $Y$ under $\Theta=0$ and, more generally, products of orthogonal polynomials of the noise distribution (e.g.\ Hermite polynomials for Gaussian noise). (ii) \emph{Symmetrization.} When $\mathbb{P}$ and $\mathbb{P}_{H_1}$ (or the joint distribution of $(x,Y)$ for estimation) are invariant under a group acting on the observations, here node relabelings, the LD criteria are achieved by invariant polynomials (Lemma~\ref{lem:reduction:permutation}). These are obtained by summing the basis elements over orbits, here indexed by templates. One should use as much invariance as the model offers: for Gaussian mixture models,~\cite{carpentier2026low} use the invariance of $Y\in\mathbb{R}^{n\times d}$ with respect to both row permutations and rotations of the columns. (iii) \emph{Centering along the independence structure.} Whenever a basis element factorizes into pieces involving disjoint sets of latent variables, each piece is centered separately. Here, these pieces are the connected components of the templates. Under \condinvariance, the covariance between two labeled templates then vanishes as soon as a connected component of one of them does not intersect the other, whereas under \condinvarianceWR, it is only small. (iv) \emph{Normalization.} Normalize each polynomial: here, by its standard deviation under the reference distribution $\Theta=0$, see~\eqref{eqn:variance of graph}. A different normalization was needed in~\cite{carpentier2025phase} to cover all regimes of $q$.

\subsection{Getting sharp results}

Our results can be improved in two main directions: 
\begin{itemize}
    \item First, we have only analyzed the regime $\lambda = o(q)$, see Theorem~\ref{thm:iso}. 
    We conjecture that this restriction is merely an artefact of the proof, and that (after a suitable renormalization) the basis $\Psi_G$ remains almost orthonormal in all regimes where recovery is computationally hard. 
    For \HidSubI\ and \StoBloI, however, certain adjustments are needed in both the variance proxy and the proof.
    In particular, for $\lambda \geq \sqrt{q}$, the variance proxies $\mathbb V(G), \mathbb V^{(1,2)}(G)$ can be significantly smaller than $\mathbb E[\overline{P}_{G}^2]$. Moreover, the contribution of perfectly matched nodes in $\mathbb E[\overline{P}_{G^{(1)},\pi^{(1)}}\overline{P}_{G^{(2)},\pi^{(2)}}]$ must be handled more carefully than in Section~\ref{s:core}. This extension was very recently carried out in~\cite{carpentier2025phase}.


    \item Second, compared with~\cite{SohnWein25}, our LD lower bounds are tight only up to a poly-$D$ factor. 
    Removing this extra factor is an interesting ,\ albeit likely delicate ,\ combinatorial problem, which we leave for future work.
\end{itemize}

\section*{Acknowledgements}
The work of A. Carpentier is partially supported by the Deutsche Forschungsgemeinschaft (DFG)- Project-ID 318763901 - SFB1294 "Data Assimilation", Project A03, by the DFG on the Forschungsgruppe FOR5381 "Mathematical Statistics in the Information Age - Statistical Efficiency and Computational Tractability", Project TP 02 (Project-ID 460867398), by the European Research council (ERC) through the ERC Consolidator Grant SOCE (Nr: 101229569). Funded by the European Union (ERC, project SOCE, 101229569). Views and opinions expressed are however those of the author(s) only and do not necessarily reflect those of the European Union or the European Research Council. Neither the European Union nor the granting authority can be held responsible for them. The work of the last three authors has also been fully or partially supported by ANR-21-CE23-0035 (ASCAI, ANR) and ANR-19-CHIA-0021-01 (BiSCottE, ANR). The work of A.~Carpentier and N.~Verzelen is also supported by the Universit\'e franco-allemande (UFA) through the Coll\`ege doctoral franco-allemand CDFA-02-25 ``Statistisches Lernen für komplexe stochastische Prozesse''.

\bibliographystyle{plainnat} 
\bibliography{biblio}

\begin{thebibliography}{45}
\providecommand{\natexlab}[1]{#1}
\providecommand{\url}[1]{\texttt{#1}}
\expandafter\ifx\csname urlstyle\endcsname\relax
  \providecommand{\doi}[1]{doi: #1}\else
  \providecommand{\doi}{doi: \begingroup \urlstyle{rm}\Url}\fi

\bibitem[Abbe and Sandon(2015)]{AbbeSandon2015a}
Emmanuel Abbe and Colin Sandon.
\newblock {Community Detection in General Stochastic Block Models: Fundamental
  Limits and Efficient Algorithms for Recovery}.
\newblock In \emph{{Proceedings of the 2015 IEEE 56th Annual Symposium on
  Foundations of Computer Science (FOCS)}}, {FOCS '15}, pages 670--688,
  Washington, DC, USA, 2015. IEEE Computer Society.
\newblock ISBN 978-1-4673-8191-8.
\newblock URL \url{http://dx.doi.org/10.1109/FOCS.2015.47}.

\bibitem[Bandeira et~al.(2021)Bandeira, Banks, Kunisky, Moore, and
  Wein]{bandeira2021spectral}
Afonso~S Bandeira, Jess Banks, Dmitriy Kunisky, Christopher Moore, and Alex
  Wein.
\newblock Spectral planting and the hardness of refuting cuts, colorability,
  and communities in random graphs.
\newblock In \emph{Conference on Learning Theory}, pages 410--473. PMLR, 2021.

\bibitem[Bandeira et~al.(2022)Bandeira, El~Alaoui, Hopkins, Schramm, Wein, and
  Zadik]{bandeira2022franz}
Afonso~S Bandeira, Ahmed El~Alaoui, Samuel Hopkins, Tselil Schramm, Alexander~S
  Wein, and Ilias Zadik.
\newblock The franz-parisi criterion and computational trade-offs in high
  dimensional statistics.
\newblock \emph{Advances in Neural Information Processing Systems},
  35:\penalty0 33831--33844, 2022.

\bibitem[Barak et~al.(2019)Barak, Hopkins, Kelner, Kothari, Moitra, and
  Potechin]{Barak19}
Boaz Barak, Samuel Hopkins, Jonathan Kelner, Pravesh~K. Kothari, Ankur Moitra,
  and Aaron Potechin.
\newblock A nearly tight sum-of-squares lower bound for the planted clique
  problem.
\newblock \emph{SIAM Journal on Computing}, 48\penalty0 (2):\penalty0 687--735,
  2019.
\newblock URL \url{https://doi.org/10.1137/17M1138236}.

\bibitem[Berenfeld et~al.(2024)Berenfeld, Carpentier, and
  Verzelen]{berenfeld2024seriation}
Cl{\'e}ment Berenfeld, Alexandra Carpentier, and Nicolas Verzelen.
\newblock Seriation of t{\oe}plitz and latent position matrices: optimal rates
  and computational trade-offs.
\newblock \emph{arXiv preprint arXiv:2408.10004}, 2024.

\bibitem[Berthet and Rigollet(2013)]{pmlr-v30-Berthet13}
Quentin Berthet and Philippe Rigollet.
\newblock {Complexity Theoretic Lower Bounds for Sparse Principal Component
  Detection}.
\newblock In Shai Shalev-Shwartz and Ingo Steinwart, editors,
  \emph{{Proceedings of the 26th Annual Conference on Learning Theory}},
  volume~30 of \emph{{Proceedings of Machine Learning Research}}, pages
  1046--1066, Princeton, NJ, USA, 12--14 Jun 2013. PMLR.
\newblock URL \url{http://proceedings.mlr.press/v30/Berthet13.html}.

\bibitem[Bordenave et~al.(2015)Bordenave, Lelarge, and
  Massouli{\'e}]{bordenave2015non}
Charles Bordenave, Marc Lelarge, and Laurent Massouli{\'e}.
\newblock Non-backtracking spectrum of random graphs: community detection and
  non-regular ramanujan graphs.
\newblock In \emph{2015 IEEE 56th Annual Symposium on Foundations of Computer
  Science}, pages 1347--1357. IEEE, 2015.

\bibitem[Brennan and Bresler(2020)]{brennan2020reducibility}
Matthew Brennan and Guy Bresler.
\newblock Reducibility and statistical-computational gaps from secret leakage.
\newblock In \emph{Conference on Learning Theory}, pages 648--847. PMLR, 2020.

\bibitem[Brennan et~al.(2018)Brennan, Bresler, and
  Huleihel]{brennan2018reducibility}
Matthew Brennan, Guy Bresler, and Wasim Huleihel.
\newblock Reducibility and computational lower bounds for problems with planted
  sparse structure.
\newblock In \emph{Conference On Learning Theory}, pages 48--166. PMLR, 2018.

\bibitem[Brennan et~al.(2021)Brennan, Bresler, Hopkins, Li, and
  Schramm]{brennan2020statistical}
Matthew Brennan, Guy Bresler, Samuel~B Hopkins, Jerry Li, and Tselil Schramm.
\newblock Statistical query algorithms and low-degree tests are almost
  equivalent.
\newblock In Mikhail Belkin and Samory Kpotufe, editors, \emph{Proceedings of
  Thirty Fourth Conference on Learning Theory}, volume 134 of \emph{Proceedings
  of Machine Learning Research}, pages 774--774. PMLR, 15--19 Aug 2021.

\bibitem[Cai and Ma(2023)]{cai2023matrix}
Tony Cai and Rong Ma.
\newblock Matrix reordering for noisy disordered matrices: Optimality and
  computationally efficient algorithms.
\newblock \emph{IEEE Transactions on Information Theory}, 2023.

\bibitem[Carpentier and Verzelen(2026)]{carpentier2026low}
Alexandra Carpentier and Nicolas Verzelen.
\newblock Low-degree lower bounds for clustering in moderate dimension.
\newblock \emph{arXiv preprint arXiv:2602.23023}, 2026.

\bibitem[Carpentier et~al.(2025{\natexlab{a}})Carpentier, Giraud, and
  Verzelen]{carpentier2025phase}
Alexandra Carpentier, Christophe Giraud, and Nicolas Verzelen.
\newblock Phase transition for stochastic block model with more than $\sqrt{n}$
  communities.
\newblock \emph{arXiv preprint arXiv:2509.15822}, 2025{\natexlab{a}}.

\bibitem[Carpentier et~al.(2025{\natexlab{b}})Carpentier, Giraud, and
  Verzelen]{carpentier2025phase2}
Alexandra Carpentier, Christophe Giraud, and Nicolas Verzelen.
\newblock Phase transition for stochastic block model with more than $\sqrt{n}$
  communities (ii).
\newblock \emph{arXiv preprint arXiv:2511.21526}, 2025{\natexlab{b}}.

\bibitem[Chin et~al.(2025)Chin, Mossel, Sohn, and Wein]{pmlr-v291-chin25a}
Byron Chin, Elchanan Mossel, Youngtak Sohn, and Alexander~S. Wein.
\newblock Stochastic block models with many communities and the
  {K}esten–{S}tigum bound - extended abstract.
\newblock In Nika Haghtalab and Ankur Moitra, editors, \emph{Proceedings of
  Thirty Eighth Conference on Learning Theory}, volume 291 of \emph{Proceedings
  of Machine Learning Research}, pages 1253--1258. PMLR, 30 Jun--04 Jul 2025.
\newblock URL \url{https://proceedings.mlr.press/v291/chin25a.html}.

\bibitem[Decelle et~al.(2011)Decelle, Krzakala, Moore, and
  Zdeborov\'a]{Decelle2011}
Aurelien Decelle, Florent Krzakala, Cristopher Moore, and Lenka Zdeborov\'a.
\newblock Asymptotic analysis of the stochastic block model for modular
  networks and its algorithmic applications.
\newblock \emph{Phys. Rev. E}, 84:\penalty0 066106, Dec 2011.
\newblock URL \url{https://link.aps.org/doi/10.1103/PhysRevE.84.066106}.

\bibitem[Deshpande and Montanari(2015)]{DM15}
Yash Deshpande and Andrea Montanari.
\newblock Finding hidden cliques of size $\sqrt{N/e}$ in nearly linear time.
\newblock \emph{Foundations of Computational Mathematics}, 15\penalty0
  (4):\penalty0 1069--1128, 2015.

\bibitem[Dhawan et~al.(2025)Dhawan, Mao, and Wein]{Dhawan25}
Abhishek Dhawan, Cheng Mao, and Alexander~S. Wein.
\newblock Detection of dense subhypergraphs by low-degree polynomials.
\newblock \emph{Random Structures \& Algorithms}, 66\penalty0 (1):\penalty0
  e21279, 2025.
\newblock \doi{https://doi.org/10.1002/rsa.21279}.
\newblock URL \url{https://onlinelibrary.wiley.com/doi/abs/10.1002/rsa.21279}.

\bibitem[Ding et~al.(2025)Ding, Hua, Slot, and Steurer]{ding2025low}
Jingqiu Ding, Yiding Hua, Lucas Slot, and David Steurer.
\newblock Low degree conjecture implies sharp computational thresholds in
  stochastic block model.
\newblock \emph{arXiv preprint arXiv:2502.15024}, 2025.

\bibitem[Ding et~al.(2024)Ding, Kunisky, Wein, and
  Bandeira]{ding2024subexponential}
Yunzi Ding, Dmitriy Kunisky, Alexander~S Wein, and Afonso~S Bandeira.
\newblock Subexponential-time algorithms for sparse pca.
\newblock \emph{Foundations of Computational Mathematics}, 24\penalty0
  (3):\penalty0 865--914, 2024.

\bibitem[Elimelech and Huleihel(2025)]{elimelech2025detecting}
Dor Elimelech and Wasim Huleihel.
\newblock Detecting arbitrary planted subgraphs in random graphs.
\newblock \emph{arXiv preprint arXiv:2503.19069}, 2025.

\bibitem[Even et~al.(2024)Even, Giraud, and Verzelen]{Even24}
Bertrand Even, Christophe Giraud, and Nicolas Verzelen.
\newblock Computation-information gap in high-dimensional clustering.
\newblock In Shipra Agrawal and Aaron Roth, editors, \emph{Proceedings of
  Thirty Seventh Conference on Learning Theory}, volume 247 of
  \emph{Proceedings of Machine Learning Research}, pages 1646--1712. PMLR, 30
  Jun--03 Jul 2024.
\newblock URL \url{https://proceedings.mlr.press/v247/even24a.html}.

\bibitem[Even et~al.(2025{\natexlab{a}})Even, Giraud, and Verzelen]{Even25a}
Bertrand Even, Christophe Giraud, and Nicolas Verzelen.
\newblock Computational lower bounds in latent models: clustering,
  sparse-clustering, biclustering.
\newblock \emph{arXiv preprint arXiv:2506.13647}, 2025{\natexlab{a}}.
\newblock URL \url{https://arxiv.org/abs/2506.13647}.

\bibitem[Even et~al.(2025{\natexlab{b}})Even, Giraud, and Verzelen]{Even25b}
Bertrand Even, Christophe Giraud, and Nicolas Verzelen.
\newblock Computational barriers for permutation-based problems, and cumulants
  of weakly dependent random variables.
\newblock \emph{arXiv preprint arXiv:2507.07946}, 2025{\natexlab{b}}.
\newblock URL \url{https://arxiv.org/abs/2507.07946}.

\bibitem[Flammarion et~al.(2019)Flammarion, Mao, and
  Rigollet]{flammarion2019optimal}
Nicolas Flammarion, Cheng Mao, and Philippe Rigollet.
\newblock Optimal rates of statistical seriation.
\newblock \emph{Bernoulli}, 25\penalty0 (1):\penalty0 623--653, 2019.

\bibitem[Gamarnik(2021)]{gamarnik2021overlap}
David Gamarnik.
\newblock The overlap gap property: A topological barrier to optimizing over
  random structures.
\newblock \emph{Proceedings of the National Academy of Sciences}, 118\penalty0
  (41):\penalty0 e2108492118, 2021.

\bibitem[Gamarnik et~al.(2024)Gamarnik, Jagannath, and
  Wein]{gamarnik2024hardness}
David Gamarnik, Aukosh Jagannath, and Alexander~S Wein.
\newblock Hardness of random optimization problems for boolean circuits,
  low-degree polynomials, and langevin dynamics.
\newblock \emph{SIAM Journal on Computing}, 53\penalty0 (1):\penalty0 1--46,
  2024.

\bibitem[Hopkins(2018)]{hopkins2018statistical}
Samuel Hopkins.
\newblock \emph{Statistical inference and the sum of squares method}.
\newblock PhD thesis, Cornell University, 2018.

\bibitem[Hopkins and Steurer(2017)]{Hopkins17}
Samuel~B. Hopkins and David Steurer.
\newblock Efficient bayesian estimation from few samples: Community detection
  and related problems.
\newblock In \emph{2017 IEEE 58th Annual Symposium on Foundations of Computer
  Science (FOCS)}, pages 379--390, 2017.
\newblock URL \url{https://ieeexplore.ieee.org/abstract/document/8104074}.

\bibitem[Hopkins et~al.(2017)Hopkins, K.~Kothari, A.~Potechin, Raghavendra,
  chramm, and Steurer]{HopkinsFOCS17}
Samuel~B. Hopkins, Pravesh K.~Kothari, Aaron A.~Potechin, Prasad Raghavendra,
  Tselil chramm, and David Steurer.
\newblock The power of sum-of-squares for detecting hidden structures.
\newblock In \emph{2017 IEEE 58th Annual Symposium on Foundations of Computer
  Science (FOCS)}, pages 720--731, Los Alamitos, CA, USA, oct 2017. IEEE
  Computer Society.
\newblock URL \url{https://doi.ieeecomputersociety.org/10.1109/FOCS.2017.72}.

\bibitem[Kearns(1998)]{kearns1998efficient}
Michael Kearns.
\newblock Efficient noise-tolerant learning from statistical queries.
\newblock \emph{Journal of the ACM (JACM)}, 45\penalty0 (6):\penalty0
  983--1006, 1998.

\bibitem[Kothari et~al.(2023)Kothari, Vempala, Wein, and
  Xu]{kothari2023planted}
Pravesh Kothari, Santosh~S Vempala, Alexander~S Wein, and Jeff Xu.
\newblock Is planted coloring easier than planted clique?
\newblock In \emph{The Thirty Sixth Annual Conference on Learning Theory},
  pages 5343--5372. PMLR, 2023.

\bibitem[Kunisky(2024)]{kunisky2024low}
Dmitriy Kunisky.
\newblock Low coordinate degree algorithms ii: Categorical signals and
  generalized stochastic block models.
\newblock \emph{arXiv preprint arXiv:2412.21155}, 2024.

\bibitem[Kunisky et~al.(2019)Kunisky, Wein, and Bandeira]{KuniskyWeinBandeira}
Dmitriy Kunisky, Alexander~S Wein, and Afonso~S Bandeira.
\newblock Notes on computational hardness of hypothesis testing: Predictions
  using the low-degree likelihood ratio.
\newblock In \emph{ISAAC Congress (International Society for Analysis, its
  Applications and Computation)}, pages 1--50. Springer, 2019.
\newblock URL
  \url{https://link.springer.com/chapter/10.1007/978-3-030-97127-4_1}.

\bibitem[Kunisky et~al.(2024)Kunisky, Moore, and Wein]{kunisky2024tensor}
Dmitriy Kunisky, Cristopher Moore, and Alexander~S Wein.
\newblock Tensor cumulants for statistical inference on invariant
  distributions.
\newblock In \emph{2024 IEEE 65th Annual Symposium on Foundations of Computer
  Science (FOCS)}, pages 1007--1026. IEEE, 2024.

\bibitem[Li(2025)]{li2025algorithmic}
Zhangsong Li.
\newblock Algorithmic contiguity from low-degree conjecture and applications in
  correlated random graphs.
\newblock \emph{arXiv preprint arXiv:2502.09832}, 2025.

\bibitem[Luo and Gao(2024)]{luo2023computational}
Yuetian Luo and Chao Gao.
\newblock Computational lower bounds for graphon estimation via low-degree
  polynomials.
\newblock \emph{The Annals of Statistics}, 52\penalty0 (5):\penalty0
  2318--2348, 2024.

\bibitem[Mao et~al.(2023)Mao, Wein, and Zhang]{mao2023detection}
Cheng Mao, Alexander~S Wein, and Shenduo Zhang.
\newblock Detection-recovery gap for planted dense cycles.
\newblock In \emph{The Thirty Sixth Annual Conference on Learning Theory},
  pages 2440--2481. PMLR, 2023.

\bibitem[Massouli{\'e}(2014)]{massoulie2014community}
Laurent Massouli{\'e}.
\newblock Community detection thresholds and the weak ramanujan property.
\newblock In \emph{Proceedings of the forty-sixth annual ACM symposium on
  Theory of computing}, pages 694--703, 2014.

\bibitem[Montanari and Wein(2025)]{montanari2025equivalence}
Andrea Montanari and Alexander~S Wein.
\newblock Equivalence of approximate message passing and low-degree polynomials
  in rank-one matrix estimation.
\newblock \emph{Probability Theory and Related Fields}, 191\penalty0
  (1):\penalty0 181--233, 2025.

\bibitem[Rush et~al.(2022)Rush, Skerman, Wein, and Yang]{rush2022easier}
Cynthia Rush, Fiona Skerman, Alexander~S Wein, and Dana Yang.
\newblock Is it easier to count communities than find them?
\newblock \emph{arXiv preprint arXiv:2212.10872}, 2022.

\bibitem[Schramm and Wein(2022)]{SchrammWein22}
Tselil Schramm and Alexander~S Wein.
\newblock Computational barriers to estimation from low-degree polynomials.
\newblock \emph{The Annals of Statistics}, 50\penalty0 (3):\penalty0
  1833--1858, 2022.
\newblock URL
  \url{https://projecteuclid.org/journals/annals-of-statistics/volume-50/issue-3/Computational-barriers-to-estimation-from-low-degree-polynomials/10.1214/22-AOS2179.short}.

\bibitem[Semerjian(2024)]{semerjian2024matrix}
Guilhem Semerjian.
\newblock Matrix denoising: Bayes-optimal estimators via low-degree
  polynomials.
\newblock \emph{Journal of Statistical Physics}, 191\penalty0 (10):\penalty0
  139, 2024.

\bibitem[Sohn and Wein(2025)]{SohnWein25}
Youngtak Sohn and Alexander~S. Wein.
\newblock Sharp phase transitions in estimation with low-degree polynomials.
\newblock \emph{Proceedings of the 57th Annual ACM Symposium on Theory of
  Computing}, pages 891--902, 2025.
\newblock URL \url{https://arxiv.org/abs/2502.14407}.

\bibitem[Wein(2025)]{SurveyWein2025}
Alexander~S. Wein.
\newblock Computational complexity of statistics: New insights from low-degree
  polynomials.
\newblock \emph{arXiv preprint arXiv:2506.10748}, 2025.
\newblock URL \url{https://arxiv.org/abs/2506.10748}.

\end{thebibliography}

\appendix

\section{Almost orthonormality of $(\Psi_{G})_{G\in \mathcal G_{\leq D}}$ under general conditions}\label{sec:main_orthonormalite}

In this section, we establish that the almost orthonormality of the family $(\Psi_{G})_{G\in \mathcal G_{\leq D}}$ ---\ resp.~$(\Psi^{(1,2)}_{G})_{G\in \mathcal G^{(1,2)}_{\leq D}}$ ---\ actually holds under some generic conditions, which are easy to check in our general model~\eqref{eq:definition:sample:graph}.  
To motivate these conditions, we explain where they are needed to extend the proof arguments of Section~\ref{s:core}.
We first state the following signal restriction on $\lambda, k, q$ that we will need in all models.
 \begin{condition}[\texttt{C-Signal}]\label{as:signal}
 We assume that $\Theta_{ij}\in [0,\lambda]$. For some  constant $c_{\texttt{s}}>1$, we have 
\begin{equation}\left(\frac{k}{n} \right) \lor \left(\frac{\lambda k}{\sqrt{n q}} \right) \lor \left(\frac{\lambda}{q}\right)  \leq D^{-8c_{\texttt{s}}}\enspace .\end{equation}
\end{condition}
This condition matches that in Theorems~\ref{thm:isorefo}~and~\ref{thm:isorefo2} and has been discussed just below Theorem~\ref{thm:isorefo}. In this general setting, $k$ plays the role of the sparsity and $\lambda$ that of the signal.  
In what follows, we distinguish between the \condinvariance scheme, which is simpler,  and the \condinvarianceWR scheme, where the independence property from Equation~\eqref{eq:indep} is lost, and we require an additional condition. The conditions below are stated for templates. Since $P_{G,\pi}$ only depends on the labeled graph $\pi[G]$ and since any graph without isolated nodes and with at most $D$ edges is isomorphic to a template, they also hold for such graphs. We use them in this form in the proofs of Lemmas~\ref{lem:gen-1} and~\ref{lem:recursion:covariance}, where they are applied to sub-configurations made of some connected components of two templates, endowed with the restriction of the matching.

\subsection{\condinvariance scheme}

The first part of Proposition~\ref{eq:boundbar} is still true in our more general models under \condinvariance, as, by independence of the labels $z_i$'s, Equation~\eqref{eq:indep} remains true. In the proof of the second part of Proposition~\ref{eq:boundbar}, the key lemma is Lemma~\ref{lem:unioncomp}, where we bound each term that appears in the decomposition of $\mathbb{E}\left[\overline{P}_{G^{(1)}, \pi^{(1)}}\overline{P}_{G^{(2)}, \pi^{(2)}}\right] $. For this purpose, we have to control the first and second moments of polynomials $P_{G, \pi}$.

\begin{condition}[\texttt{C-Moment}]\label{as:moment}
     For some non-negative constant $c_{\texttt{m}}$, the following holds. For all templates $G= (V,E)$ with at most $D$ edges, and for any labeling  $\pi\in \Pi_V$, we have  
    \begin{equation}\label{eq:as:moment}
        \left|\mathbb{E}\left[P_{G, \pi}\right]\right|\leq 
        \left(D^{c_{\texttt{m}}}\lambda\right)^{|E|} \left(D^{c_{\texttt{m}}}\frac{k}{n}\right)^{|V|-\#\mathrm{CC}_G}\enspace .
    \end{equation}
    Under \condinvarianceWR, we require in addition that~\eqref{eq:as:moment} still holds when $\mathbb{E}$ is replaced by $\widetilde{\mathbb{E}}$, where $\widetilde{\mathbb{E}}$ stands for the expectation in the model with the same $\Theta$ but where the latent assignments $z_i$'s are sampled under \condinvariance.
\end{condition}
For instance, in~\eqref{eq:comp}, we have proved that \HidSubI\ satisfies $\mathbb{E}\left[P_{G, \pi}\right]=
       \lambda^{|E|} \left(\frac{k}{n}\right)^{|V|}$ so that~\eqref{eq:as:moment} even holds with an additional factor $(k/n)^{\#\mathrm{CC}_G}$. It turns out that~\eqref{eq:as:moment} is sufficient for our purpose.
Under \condinvariance, we have $\widetilde{\mathbb{E}}=\mathbb{E}$, so that the additional requirement of \condmoment is only relevant under \condinvarianceWR. Note that $\widetilde{\mathbb{E}}$ involves the diagonal entries of $\Theta$, since two nodes may then share the same latent assignment. Under \condinvarianceWR, the bound on $\widetilde{\mathbb{E}}[P_{G,\pi}]$ does not follow from the bound on $\mathbb{E}[P_{G,\pi}]$, as the latter does not depend on the diagonal of $\Theta$. It is used in the proof of Lemma~\ref{lem:recursion:covariance}, where we control expectations under $\widetilde{\mathbb{E}}$.

In what follows, we introduce $p$ and $\overline{p}$ by 
\begin{equation}\label{eq:definition:bar:p}
 p=\lambda + q \ ; \quad \quad  \overline{p} := p(1-q)^2 + (1-p)q^2\enspace , 
\end{equation}
where we observe $\overline{p}= \overline{q} + \lambda(1-2q)\geq \overline{q}$ since we assume that $q\leq 1/2$. In our framework, $p$ corresponds to the maximum connection probability in the random graph.

In the same spirit as the previous condition, the following condition bounds the covariance and variance of monomials $P_{G, \pi}$ in terms of some characteristics of the graph $G$.
 \begin{condition}[\texttt{C-Variance}]\label{as:gen}
 For some non-negative constants $c_{\texttt{v},1}$, $c_{\texttt{v},2}$, and $c_{\texttt{v},3}$, the following holds. 
\begin{itemize}
    \item[1] Fix two templates $G^{(1)}=(V^{(1)},E^{(1)})$, $G^{(2)}=(V^{(2)},E^{(2)}) \in \mathcal G_{\leq D}$ and let $\mathbf M \in \mathcal M\setminus  \mathcal{M}_{\mathrm{PM}}$ be a matching. For any $(\pi^{(1)},\pi^{(2)}) \in \Pi(\mathbf M)$ we have 
 \begin{equation*}
        \Big| \mathbb{E}\left[P_{G^{(1)}, \pi^{(1)}} P_{G^{(2)}, \pi^{(2)}}\right] \Big| \leq  c_{\texttt{v},2} (D^{c_{\texttt{v},1}}\lambda)^{|E_{\Delta}|} \overline{p}^{|E_{\cap}|} \left(D^{c_{\texttt{v},1}}\frac{k}{n}\right)^{|V_{\Delta}| - \#\mathrm{CC}_{\Delta} }\enspace .
    \end{equation*}
    \item[2] For any template $G=(V,E) \in \mathcal G_{\leq D}$ and for any $\pi \in \Pi_V$, we have 
\begin{equation*}
        \left|\mathbb{E}\left[P_{G, \pi}^2\right] -   \overline{q}^{|E|}\right|\leq 
        c_{\texttt{v},3}D^{-c_{\texttt{s}}}\overline{q}^{|E|} \enspace , 
    \end{equation*}
    where $c_{\texttt{s}}$ is the constant arising in \condsignal.
\end{itemize}
\end{condition}
In~\eqref{eqn:old basis proof ingredients}, we have proved that, for \HidSubI\ with $q=1/2$  ---\ so that $\overline{p}=\overline{q}$ ---\ we have 
  $\mathbb{E}\left[P_{G^{(1)}, \pi^{(1)}} P_{G^{(2)}, \pi^{(2)}}\right] = \lambda^{|E_{\Delta}|}\overline{q}^{|E_{\cap}|}\pa{\tfrac{k}{n}}^{|V_{\Delta}|}$. In the first part of the above condition, we only require a bound up to a polynomial factor in $D$  and up to a factor $(k/n)^{\#\mathrm{CC}_{\Delta}}$. 
Similarly, \eqref{eqn:old basis proof ingredients} implies that for \HidSubI\ with $q=1/2$ we have $\mathbb{E}\left[P_{G, \pi}^2\right] =   \overline{q}^{|E|}$. The second part of \condvariance only requires that this holds approximately, up to a relative error of order $D^{-c_{\texttt{s}}}$. Since \condsignal\ enforces $\overline{p}\leq \overline{q}(1+D^{-8c_{\texttt{s}}})$, this second part holds in all our models with a numerical constant $c_{\texttt{v},3}$ ---\ see Proposition~\ref{prp:model:conditions}.

Under the above conditions,  we can adapt the proof arguments of Section~\ref{s:core} to establish the almost orthonormality of the $\Psi_G$'s.

\begin{theorem}\label{thm:iso}
    Consider the  \condinvariance scheme and fix $D \geq 2$. Assume that Conditions \condsignal, \condmoment, and \condvariance are fulfilled  with $c_{\texttt{s}}>1$  large enough in comparison to the other constants.
    Then, for all  $(\alpha_{\emptyset}, (\alpha_G)_{G \in \mathcal G_{\leq D}})$, we have
    \begin{equation}\left(1-c D^{-c_{\texttt{s}}/2}\right)\|\alpha\|^2_2 \leq \mathbb E \left[\left(\alpha_{\emptyset }+ \sum_{G \in \mathcal G_{\leq D}} \alpha_G\Psi_G\right)^2\right]\leq \left(1+ c D^{-c_{\texttt{s}}/2}\right)\|\alpha\|^2_2  \enspace , \end{equation}
where the positive constant $c$ depends on the constants $c_{\texttt{m}},c_{\texttt{v},1},c_{\texttt{v},2}, c_{\texttt{v},3}$.
\end{theorem}

\subsection{\condinvarianceWR scheme}

Under the \condinvarianceWR scheme, polynomials $P_{G^{(1)},\pi^{(1)}}$ and $P_{G^{(2)},\pi^{(2)}}$ with disjoint nodes ---\ that is $\pi^{(1)}(V^{(1)})\cap \pi^{(2)}(V^{(2)})=\emptyset$ ---\ are not independent anymore, although this dependency is arguably quite weak. Therefore, the first part of Proposition~\ref{eq:boundbar} no longer holds in these models. The purpose of the next condition is to establish that $\mathbb{E}[\overline{P}_{G^{(1)},\pi^{(1)}}\overline{P}_{G^{(2)},\pi^{(2)}}]$ is small enough for matchings $\mathbf{M}$ that do not belong to $\mathcal{M}^\star$. 
Although this condition is arguably quite ad hoc and technical to define, it turns out to be relatively simple to check in all our models.
\begin{condition}[\texttt{C-Variance-Permutation}]\label{as:moment:without:replacement}
Let  $G^{(1)}=(V^{(1)},E^{(1)})$, $G^{(2)}=(V^{(2)},E^{(2)}) $ be two graphs that are either templates in $\mathcal G_{\leq D}$ or the empty graph, with the convention $P_{\emptyset,\pi}=1$; if one of them is empty, we assume that the other one has at least two connected components, and let $\mathbf M \in \mathcal M\setminus  \mathcal{M}_{\mathrm{PM}}$ be a matching. Consider any $(\pi^{(1)},\pi^{(2)}) \in \Pi(\mathbf M)$.
 As in \condmoment, $\widetilde{\mathbb{E}}$ stands for the expectation in the model with the same $\Theta$ but where the latent assignments $z_i$'s are sampled under \condinvariance.
Define the graph $\mathcal{N}[z;G_{\cup}]$ with vertices $\omega_0$, $\omega_1, \ldots ,\omega_{\#\mathrm{CC}_{\mathrm{pure}}}$ where, for $i>0$,  $\omega_i$ corresponds to the $i$-th pure connected component of $G_{\Delta}$ and $\omega_0$ corresponds to the collection of the remaining nodes of $G_{\cup}$ ---\ if there is no such node, we do not define $\omega_0$.  We set an edge between $\omega_i$  and $\omega_j$ if and only if at least one  vertex $a$ in the node set $\omega_j$ of $G_{\cup}$ shares the same  latent assignment as one vertex $b$ in the node set $\omega_i$ of $G_{\cup}$, that is $z_a=z_b$. 

  Define $\cA$ as the event that the graph $\mathcal{N}[z;G_{\cup}]$ is connected. Then, we have 
 \begin{equation*}
        \Big| \widetilde{\mathbb{E}}\left[\1\{\mathcal{A}\}P_{G^{(1)}, \pi^{(1)}} P_{G^{(2)}, \pi^{(2)}}\right] \Big| \leq  c_{\texttt{vd},2}D^{c_{\texttt{vd},1}} (D^{c_{\texttt{vd},1}}\lambda)^{|E_{\Delta}|} \overline{p}^{|E_{\cap}|} \left(D^{c_{\texttt{vd},1}}\frac{k}{n}\right)^{|V_{\Delta}| - \#\mathrm{CC}_{\Delta} }\left(c_{\texttt{vd},2}\frac{D^{c_{\texttt{vd},1}}}{\sqrt{n}}\right)^{\#\mathrm{CC}_{\mathrm{pure}}}\enspace ,
    \end{equation*}
for some non-negative constants $c_{\texttt{vd},1}$, $c_{\texttt{vd},2}$.
\end{condition}
Allowing one of the two graphs to be empty is needed in the proof of Lemma~\ref{lem:recursion:covariance}, where this condition is applied to sub-configurations of $(G^{(1)},\pi^{(1)})$ and $(G^{(2)},\pi^{(2)})$, endowed with the restriction of the matching, for instance to a group of pure connected components that all come from $G^{(2)}$. We exclude the case where the configuration is reduced to a single pure connected component: then $\cA$ is of probability one and the factor $n^{-1/2}$ cannot hold in general. For instance, for \HidSubI, we have $\widetilde{\mathbb{E}}[P_{G,\pi}]=\lambda^{|E|}(k/n)^{|V|}$ for a connected template $G$.

The following theorem holds under the above assumptions.
\begin{theorem}\label{thm:iso-WR}
    Consider the \condinvarianceWR scheme  and fix $D \geq 2$. Assume that Conditions \condsignal, \condmoment, \condvariance, and \condmomentWR are fulfilled  with $c_{\texttt{s}}>1$  large enough in comparison to the other constants.
    Then, for all  $(\alpha_{\emptyset}, (\alpha_G)_{G \in \mathcal G_{\leq D}})$, we have
    \begin{equation}\left(1-c D^{-c_{\texttt{s}}/2}\right)\|\alpha\|^2_2 \leq \mathbb E \left[\left(\alpha_{\emptyset }+ \sum_{G \in \mathcal G_{\leq D}} \alpha_G\Psi_G\right)^2\right]\leq \left(1+ c D^{-c_{\texttt{s}}/2}\right)\|\alpha\|^2_2  \enspace , \end{equation}
    where $c$ depends on the constants $c_{\texttt{m}},c_{\texttt{v},1},c_{\texttt{v},2},c_{\texttt{v},3},c_{\texttt{vd},1}, c_{\texttt{vd},2}$.
\end{theorem}

\subsection{Almost orthonormality for estimation}

For estimation, 
the following theorem holds. It is a generic version of Theorem~\ref{thm:isorefo2}.
\begin{theorem}\label{thm:iso-estimation}
Fix any $D\geq 2$.  Under either the conditions of Theorem~\ref{thm:iso} or those of Theorem~\ref{thm:iso-WR}, we have 
    \begin{equation}\left(1-c D^{-c_{\texttt{s}}/2}\right)\|\alpha\|^2_2 \leq \mathbb E \left[\left(\alpha_{\emptyset }+ \sum_{G \in \mathcal G^{(1,2)}_{\leq D}} \alpha_G\Psi^{(1,2)}_G\right)^2\right]\leq \left(1+ c D^{-c_{\texttt{s}}/2}\right)\|\alpha\|^2_2  \enspace , \end{equation}
    for all  $(\alpha_{\emptyset}, (\alpha_G)_{G \in \mathcal G^{(1,2)}_{\leq D}})$. Here, the positive constant $c$ depends on the other constants in the conditions.  
\end{theorem}

\subsection{All conditions are satisfied in our models}

The next proposition states that all six models satisfy the desired conditions. The explicit values for $c_{\texttt{m}}$, $c_{\texttt{v},1}$, $c_{\texttt{v},2}$, $c_{\texttt{v},3}$, $c_{\texttt{vd},1}$, and $c_{\texttt{vd},2}$ are given in the proofs.

\begin{proposition}\label{prp:model:conditions}
  Fix $D\geq 2$ and assume that the parameters $(k,n,\lambda,q)$ satisfy~\condsignal for some constant $c_{\texttt{s}}$. Then, \HidSubI, \StoBloI, \ToeSerI\ satisfy  Conditions \condmoment and \condvariance. Also, \HidSubP, \StoBloP, \ToeSerP\ satisfy
\condmoment, \condvariance, and \condmomentWR. All the constants $c_{\texttt{m}}$, $c_{\texttt{v},1}$, $c_{\texttt{v},2}$, $c_{\texttt{v},3}$, $c_{\texttt{vd},1}$, and $c_{\texttt{vd},2}$ in these conditions are numerical constants. In particular, they do not depend on $c_{\texttt{s}}$.
\end{proposition}
Then, Theorem~\ref{thm:isorefo} is a straightforward consequence of Theorems~\ref{thm:iso} and~\ref{thm:iso-WR} and Proposition~\ref{prp:model:conditions}, whereas Theorem~\ref{thm:isorefo2} is  a consequence of Theorem~\ref{thm:iso-estimation} and Proposition~\ref{prp:model:conditions}. Indeed, since the constants given by Proposition~\ref{prp:model:conditions} do not depend on $c_{\texttt{s}}$, it suffices to take $c_{\texttt{s}}\geq 4$ large enough in these theorems, that is $c_0=8c_{\texttt{s}}$ large enough in Theorems~\ref{thm:isorefo} and~\ref{thm:isorefo2}, and to note that $D^{-c_{\texttt{s}}/2}\leq D^{-2}$.

\section{Proof of the almost orthonormality results (Theorems~\ref{thm:iso}~and~\ref{thm:iso-WR})}

We prove both theorems simultaneously. Define  the symmetric Gram matrix $\Gamma$ of size $|\mathcal G_{\leq D}|+1$ associated to the basis $(1, (\Psi_{G})_{G\in \mathcal{G}_{\leq D}})$ by 
 $\Gamma_{G^{(1)}, G^{(2)}} := \mathbb E[\Psi_{G^{(1)}}\Psi_{G^{(2)}}]$ for any $G^{(1)}, G^{(2)} \in \mathcal G_{\leq D}$, $\Gamma_{1,1}=1$, and $\Gamma_{1,G}:=\mathbb{E}[\Psi_{G}]$. Under \condinvariance, we have $\Gamma_{1,G}=0$, since the connected components of $G$ involve disjoint sets of nodes and are therefore independent, while $\overline{P}_{G,\pi}$ is the product of the centered polynomials associated to them. This is no longer true under \condinvarianceWR\ when $G$ has several connected components, and we rely instead on the following lemma, proved in Section~\ref{sec:proof:technical:lemmas}.
\begin{lemma}\label{lem:mean:Psi}
Let $G=(V,E)\in \mathcal{G}_{\leq D}$. If $G$ is connected, or under \condinvariance, we have $\mathbb{E}[\Psi_G]=0$. Under \condinvarianceWR, assume that Conditions \condmoment\ and \condmomentWR\ are fulfilled and that \condsignal\ is fulfilled with a constant $c_{\texttt{s}}$ large enough. Then, we have 
\begin{equation*}
\left|\mathbb{E}[\Psi_G]\right|\leq D^{-3c_{\texttt{s}}|E|}\enspace . 
\end{equation*}
The same results hold for $\Psi^{(1,2)}_G$ with $G\in \mathcal{G}^{(1,2)}_{\leq D}$, the condition that $G$ is connected being replaced by the condition that $G$ has only one non-trivial connected component.
\end{lemma}
Write $I$ for the identity matrix. In order to establish  the result, it suffices to  bound the operator norm 
 $\|\Gamma -I\|_{op}$. The key step of the proof is to control each entry of the matrix $\Gamma$. Recall the distance $d(\cdot,\cdot)$ defined in~\eqref{eq:definition:edit:distance}.

\begin{proposition}\label{prop:scal}
Consider any $D\geq 2$. Under \condinvariance, assume that Conditions \condmoment, \condvariance and \condsignal are fulfilled and that the constant  $c_{\texttt{s}}>4$ is large compared to the other ones. Under \condinvarianceWR, assume that Conditions \condmoment, \condvariance, \condmomentWR, and \condsignal are fulfilled with a constant   $c_{\texttt{s}}>4$ that is large enough.

There exist two positive constants $c$ and $c'$  depending on those arising in Conditions \condvariance and \condmoment (and \condmomentWR in the second case) such that the following holds for any templates $G^{(1)}, G^{(2)}\in \mathcal G_{\leq D}$.
\begin{itemize}
    \item[1] if $G^{(1)} \neq G^{(2)}$:
    \begin{equation}\left|\mathbb E[\Psi_{G^{(1)}}\Psi_{G^{(2)}}]  \right|  \leq c D^{-c_{\texttt{s}} d(G^{(1)},G^{(2)})}\enspace ,\end{equation}
    \item[2] and if $G^{(1)} = G^{(2)}$:
    \begin{equation} \left|\mathbb E[(\Psi_{G^{(1)}})^2] -1 \right|  \leq c'D^{-c_{\texttt{s}}}\enspace .\end{equation}
\end{itemize}
\end{proposition}

It is easy to conclude from this last proposition. 
Since the operator norm of a symmetric matrix is bounded by the maximum $\ell_1$ norm of its rows, we have 
\begin{equation*}
\|\Gamma - I\|_{op}\leq \max\left\{\sum_{G\in \mathcal{G}_{\leq D}}\left|\Gamma_{1,G}\right|\ ;\ \max_{G^{(1)}}\left[\left|\Gamma_{1,G^{(1)}}\right|+\left|\Gamma_{G^{(1)}, G^{(1)}}-1\right|+  \sum_{G^{(2)} \in \mathcal G_{\leq D}, G^{(2)} \neq G^{(1)} }\left|\Gamma_{G^{(1)}, G^{(2)}}\right|\right]\right\} \enspace . 
\end{equation*} 
Under \condinvariance, the terms $\Gamma_{1,G}$ are equal to zero. Under \condinvarianceWR, we use that a template with $e$ edges has at most $2e$ nodes, so that there are at most $(2e)^{2e}\leq D^{4e}$ templates with $e$ edges (recall that $e\leq D$ and $D\geq 2$). By Lemma~\ref{lem:mean:Psi}, it follows that 
\begin{equation*}
\max_{G^{(1)}}\left|\Gamma_{1,G^{(1)}}\right|\leq \sum_{G\in \mathcal{G}_{\leq D}}\left|\Gamma_{1,G}\right|\leq \sum_{e\geq 1} D^{4e}D^{-3c_{\texttt{s}}e}\leq 2D^{-(3c_{\texttt{s}}-4)}\leq D^{-c_{\texttt{s}}}\enspace , 
\end{equation*}
since $c_{\texttt{s}}\geq 16$ and $D\geq 2$.

To bound the latter sum, we use that, for a fixed template $G^{(1)}\in \mathcal{G}_{\leq D}$, the number of templates $G^{(2)}\in \mathcal{G}_{\leq D}$ such that $d( G^{(1)}, G^{(2)}) = u$ is bounded by  $(2u+2D)^{2u}$.
When $G^{(2)}\in \mathcal G_{\leq D}$ differs from $G^{(1)}$, it is, by definition, not isomorphic to $G^{(1)}$ and $d(G^{(1)}, G^{(2)})\geq 1$. It then follows from Proposition~\ref{prop:scal} that
\begin{align*}
    \sum_{G^{(2)} \in \mathcal G_{\leq D}, G^{(2)} \neq G^{(1)} }\left|\Gamma_{G^{(1)}, G^{(2)}}\right|
    &\leq \sum_{G^{(2)} \in \mathcal G_{\leq D}, G^{(2)} \neq G^{(1)}} cD^{-c_{\texttt{s}}d(G^{(1)},G^{(2)})}\\
    &\leq \sum_{2D \geq  u \geq 1} |\{G^{(2)}: d( G^{(1)}, G^{(2)}) = u \}| c D^{-c_{\texttt{s}}u}\enspace\\
    &\leq \sum_{2D \geq u\geq 1} (2u+2D)^{2u} c D^{-c_{\texttt{s}}u}\enspace\\
    &\leq \sum_{2D \geq u\geq 1} c'D^{-(c_{\texttt{s}}-8)u} \leq c D^{-c_{\texttt{s}}/2}\enspace ,
\end{align*}
since $D \geq 2$ provided we have $c_{\texttt{s}} \geq 16$. 
Applying the second part of Proposition~\ref{prop:scal}, we conclude that 
\begin{equation*}
\|\Gamma - I\|_{op}\leq c D^{-c_{\texttt{s}}/2} + (c'+1)D^{-c_{\texttt{s}}}\enspace . 
\end{equation*}

\subsection{Proof of Proposition~\ref{prop:scal}}

  We first state the following lemmas, whose proofs are postponed to the subsequent subsections. Given two templates $G^{(1)}$, $G^{(2)}$ and labelings $\pi^{(1)}$ and $\pi^{(2)}$, recall that $G_{\Delta}$ stands for the labeled graph corresponding to the symmetric difference between $\pi^{(1)}[G^{(1)}]$ and  $\pi^{(2)}[G^{(2)}]$. Also, recall the collection $\mathcal{M}^{\star}$  of matchings  of two templates $G^{(1)}$ and $G^{(2)}$ that do not lead to any pure connected component.

 \begin{lemma}\label{lem:gen-1} Consider both  \condinvariance and \condinvarianceWR. 
    Suppose that Conditions~\condmoment and~\condvariance are fulfilled and that \condsignal is fulfilled with a constant $c_{\texttt{s}}$ large enough compared to the constants arising in the other conditions. 
\begin{itemize}
    \item[1] Let $G^{(1)}, G^{(2)} \in \mathcal G_{\leq D}$ be two templates and let $\mathbf M \in \mathcal M^\star\setminus  \mathcal{M}_{\mathrm{PM}}$ be a matching. For any $(\pi^{(1)},\pi^{(2)}) \in \Pi(\mathbf M)$, we have $\Big| \mathbb{E}\left[\overline{P}_{G^{(1)}, \pi^{(1)}} \overline{P}_{G^{(2)}, \pi^{(2)}}\right] \Big|\leq  \psi[G_{\Delta}]$ where 
 \begin{align}\label{eq:definition:psi:Gdelta}
   \psi[G_{\Delta}]:=  (1+c_{\texttt{v},2}\vee c_{\texttt{v},3}) (D+1)^{2}  \overline{p}^{|E_{\cap}|} (D^{c_{\texttt{v},1}\vee c_{\texttt{m}}}\lambda)^{|E_{\Delta}|} \left(\frac{D^{1+c_{\texttt{v},1}\vee c_{\texttt{m}} }k}{n}\right)^{|V_{\Delta}| - \#\mathrm{CC}_{\Delta}}\enspace . 
    \end{align} 
\item[2] Also, for any template $G=(V,E) \in \mathcal G_{\leq D}$ and any $\pi \in \Pi_V$, we have 
\begin{equation*}
        \left|\mathbb{E}\left[\overline{P}_{G, \pi}^2\right]- \overline{q}^{|E|}\right| \leq  \left[2(1+c_{\texttt{v},2}\vee c_{\texttt{v},3}) D^{4+ c_{\texttt{v},1}\vee c_{\texttt{m}}}\frac{k}{n}+c_{\texttt{v},3}D^{-c_{\texttt{s}}}\right]\overline{q}^{|E|}\enspace .
    \end{equation*}
\end{itemize}
\end{lemma}

 \begin{lemma}\label{lem:gen-2}
Let $G^{(1)}, G^{(2)} \in \mathcal G_{\leq D}$ be two templates and let $\mathbf M \in \mathcal M\setminus  \mathcal{M}^\star$ be a matching with at least one pure connected component, and let $(\pi^{(1)},\pi^{(2)}) \in \Pi(\mathbf M)$. We consider two cases: 
\begin{itemize}
    \item[1] Under \condinvariance, we have  $\mathbb{E}\left[\overline{P}_{G^{(1)}, \pi^{(1)}} \overline{P}_{G^{(2)}, \pi^{(2)}}\right]=0$. By convention, we define $\psi[G_\Delta]=0$. 
    \item[2] Under \condinvarianceWR, we assume that Conditions \condmoment, \condvariance, and \condmomentWR  are fulfilled and that \condsignal is fulfilled with a constant $c_{\texttt{s}}$ large enough.  Then, we have 
\begin{equation*}
\left|\mathbb{E}\left[\overline{P}_{G^{(1)}, \pi^{(1)}} \overline{P}_{G^{(2)}, \pi^{(2)}}\right]\right|\leq \psi[G_{\Delta}]\enspace, 
\end{equation*}
where we define in this case 
\begin{align}\label{eq:definition:psi:Gdelta:WR}
   \psi[G_{\Delta}]:=  c_0 D^{c_1}  \overline{p}^{|E_{\cap}|} (D^{c_1}\lambda)^{|E_{\Delta}|} \left(\frac{D^{c_1 }k}{n}\right)^{|V_{\Delta}| - \#\mathrm{CC}_{\Delta}}
   \left[c_0\frac{D^{c_1}}{\sqrt{n}}\right]^{\#\mathrm{CC}_{\mathrm{pure}}}
   \enspace , 
    \end{align}
for some constants $c_0$ and $c_1$ that only depend on those in~\condmoment,~\condvariance, and~\condmomentWR. 
\end{itemize} 
\end{lemma}

Provided that the constant $c_{\texttt{s}}$ arising  in Condition~\condsignal is large enough, the conditions of the above lemmas are fulfilled.  Fix any templates $G^{(1)}$ and $G^{(2)}$ in $\mathcal G_{\leq D}$.

\subsubsection{Step 1: Sum of covariances.}  We distinguish perfect matchings and non-perfect matchings ---\ see Section~\ref{sec:graph:definition} for definitions.  We start from the decomposition
\begin{align*}
    \mathbb E[\Psi_{G^{(1)}}\Psi_{G^{(2)}}] &= \frac{1}{\sqrt{\mathbb V(G^{(1)})\mathbb V(G^{(2)})}}\sum_{\pi^{(1)}\in \Pi_{V^{(1)}}, \pi^{(2)}\in \Pi_{V^{(2)}}}\mathbb E[\overline{P}_{G^{(1)},\pi^{(1)}} \overline{P}_{G^{(2)}, \pi^{(2)}}] \\
    &= \frac{1}{\sqrt{\mathbb V(G^{(1)})\mathbb V(G^{(2)})}} \sum_{\mathbf M \in \mathcal M}\sum_{(\pi^{(1)}, \pi^{(2)}) \in \Pi(\mathbf M)}\mathbb E[\overline{P}_{G^{(1)},\pi^{(1)}} \overline{P}_{G^{(2)}, \pi^{(2)}}] \\
    &= \frac{1}{\sqrt{\mathbb V(G^{(1)})\mathbb V(G^{(2)})}} \Bigg[\sum_{\mathbf M \in  \mathcal M_{\mathrm{PM}}}\sum_{(\pi^{(1)}, \pi^{(2)}) \in \Pi(\mathbf M)}\mathbb E[\overline{P}_{G^{(1)},\pi^{(1)}} \overline{P}_{G^{(2)}, \pi^{(2)}}] \\
    &+ \sum_{\mathbf M \in \mathcal M\setminus \mathcal M_{\mathrm{PM}}}\sum_{(\pi^{(1)}, \pi^{(2)}) \in \Pi(\mathbf M)}\mathbb E[\overline{P}_{G^{(1)},\pi^{(1)}} \overline{P}_{G^{(2)}, \pi^{(2)}}] \Bigg]\enspace . 
\end{align*}
If $G^{(1)}\neq G^{(2)}$, there does not exist any perfect matching ---\ see Section~\ref{sec:graph:definition}. Hence, we have 
\begin{equation}\label{eq:prop7:different_skel}
     \Big|\mathbb E[\Psi_{G^{(1)}}\Psi_{G^{(2)}}]\Big| \leq  \Bigg| \frac{1}{\sqrt{\mathbb V(G^{(1)})\mathbb V(G^{(2)})}}\sum_{\mathbf M \in \mathcal M\setminus \mathcal M_{\mathrm{PM}}}\sum_{(\pi^{(1)}, \pi^{(2)}) \in \Pi(\mathbf M)}\mathbb E[\overline{P}_{G^{(1)},\pi^{(1)}} \overline{P}_{G^{(2)}, \pi^{(2)}}]\Bigg|\enspace .
\end{equation}
Conversely, if $G^{(1)}=G^{(2)}$,  the number of perfect matchings is, by definition, the size of the automorphism group, that is $|\mathcal{M}_{\mathrm{PM}}| = |\mathrm{Aut}(G^{(1)})|$. Besides, for such a perfect matching $\mathbf M$, the number of possible labelings is simply $|\Pi(\mathbf{M})| = \tfrac{n!}{(n - |V^{(1)}|)!}$.
 By Lemma~\ref{lem:gen-1}, and the definition~\eqref{eqn:variance of graph} of $\mathbb V(G^{(1)})$, we get  
\begin{align}\nonumber
\left|\mathbb E[\Psi_{G^{(1)}}\Psi_{G^{(2)}}] - 1\right|& \leq c \left[\frac{kD^{4+ c_{\texttt{v},1}\vee c_{\texttt{m}}}}{n} + D^{-c_{\texttt{s}}}\right]   +  \Bigg| \frac{1}{\sqrt{\mathbb V(G^{(1)})\mathbb V(G^{(2)})}}\sum_{\mathbf M \in \mathcal M\setminus \mathcal M_{\mathrm{PM}}}\sum_{(\pi^{(1)}, \pi^{(2)}) \in \Pi(\mathbf M)}\mathbb E\left[\overline{P}_{G^{(1)},\pi^{(1)}} \overline{P}_{G^{(2)}, \pi^{(2)}}\right]\Bigg|\\ 
&\leq  c_0 D^{-c_{\texttt{s}}}   +  \Bigg| \frac{1}{\sqrt{\mathbb V(G^{(1)})\mathbb V(G^{(2)})}}\sum_{\mathbf M \in \mathcal M\setminus \mathcal M_{\mathrm{PM}}}\sum_{(\pi^{(1)}, \pi^{(2)}) \in \Pi(\mathbf M)}\mathbb E\left[\overline{P}_{G^{(1)},\pi^{(1)}} \overline{P}_{G^{(2)}, \pi^{(2)}}\right]\Bigg|
\enspace , \label{eq:prop7:same_skel}
\end{align}
where we used Condition~\condsignal\ and  that $c_{\texttt{s}}$ is large enough in the last line and where $c$ and $c_0$ are positive constants that depend on the constants arising in the conditions. 
Hence, we just need to bound
\begin{equation}\label{eq:definition_A}
A := \Bigg| \frac{1}{\sqrt{\mathbb V(G^{(1)})\mathbb V(G^{(2)})}}\sum_{\mathbf M \in \mathcal M\setminus \mathcal M_{\mathrm{PM}}}\sum_{(\pi^{(1)}, \pi^{(2)}) \in \Pi(\mathbf M)}\mathbb E[\overline{P}_{G^{(1)},\pi^{(1)}} \overline{P}_{G^{(2)}, \pi^{(2)}}]\Bigg|\enspace .
\end{equation}
In light of~\eqref{eq:prop7:different_skel} and~\eqref{eq:prop7:same_skel}, it suffices to establish that 
\begin{equation}\label{eq:objective_A}
       A \leq   c''D^{-c_{\texttt{s}}  (d(G^{(1)},G^{(2)})\lor 1)}\enspace. 
\end{equation}
We start from Lemmas~\ref{lem:gen-1}~and~\ref{lem:gen-2}. 
\begin{align*}
    A
    &\leq \frac{1}{\sqrt{\mathbb V(G^{(1)})\mathbb V(G^{(2)})}}\sum_{\mathbf M \in \mathcal M\setminus \mathcal M_{\mathrm{PM}}}\sum_{(\pi^{(1)}, \pi^{(2)}) \in \Pi(\mathbf M)} \psi[G_{\Delta}]\enspace . 
\end{align*}
For fixed $(G^{(1)}, G^{(2)}, \mathbf{M})$, the number $|\Pi(\mathbf{M})|$ of possible labelings that are compatible with $\mathbf{M}$ is~\\  $\frac{n!}{(n - (|V^{(1)}|+ |V^{(2)}| - |\mathbf M|))!}$\enspace . It then follows from the definition~\eqref{eqn:variance of graph}  of $\mathbb V(G)$ that 
\begin{align*}
       A &\leq \frac{1}{\overline{q}^{(|E^{(1)}|+ |E^{(2)}|)/2}\sqrt{|\mathrm{Aut}(G^{(1)})| |\mathrm{Aut}(G^{(2)})|}}  \sum_{\mathbf M \in \mathcal M\setminus \mathcal M_{\mathrm{PM}}}\frac{\sqrt{(n-|V^{(1)}|)!(n-|V^{(2)}|)!}}{(n - (|V^{(1)}|+ |V^{(2)}| - |\mathbf M|))!} \psi[G_{\Delta}]\enspace .
\end{align*}

Since $|\mathbf{M}|\leq |V^{(1)}|\wedge |V^{(2)}|$, it follows that $(n-|V^{(1)}|)![(n - (|V^{(1)}|+ |V^{(2)}| - |\mathbf M|))!]^{-1}\leq n^{|V^{(2)}|-|\mathbf{M}|}$ and $(n-|V^{(2)}|)![(n - (|V^{(1)}|+ |V^{(2)}| - |\mathbf M|))!]^{-1}\leq n^{|V^{(1)}|-|\mathbf{M}|}$. We arrive at
\begin{align*}
       A &\leq \frac{1}{\overline{q}^{(|E^{(1)}|+ |E^{(2)}|)/2}\sqrt{|\mathrm{Aut}(G^{(1)})| |\mathrm{Aut}(G^{(2)})|}}  \sum_{\mathbf M \in \mathcal M\setminus \mathcal M_{\mathrm{PM}}}n^{ (|V^{(1)}|+|V^{(2)}|)/2 - |\mathbf M|}  \psi[G_{\Delta}] \\
       &\leq \frac{1}{\overline{q}^{(|E^{(1)}|+ |E^{(2)}|)/2}\sqrt{|\mathrm{Aut}(G^{(1)})| |\mathrm{Aut}(G^{(2)})|}}  \sum_{\mathbf M \in \mathcal M\setminus \mathcal M_{\mathrm{PM}}}n^{ (|U^{(1)}|+|U^{(2)}|)/2}  \psi[G_{\Delta}] \enspace , 
\end{align*}
where we used Equation~\eqref{eq:unmatched} in the last line and we recall that $U^{(1)}$ and $U^{(2)}$ are the sets of nodes in $G^{(1)}$ and $G^{(2)}$ that are not matched ---\ see again Section~\ref{sec:graph:definition} for definitions. 

\subsubsection{Step 2: Building on Lemmas~\ref{lem:gen-1} and~\ref{lem:gen-2}.}
Define $A_0 = A\sqrt{|\mathrm{Aut}(G^{(1)})| |\mathrm{Aut}(G^{(2)})|}$.  Also, we write $U = |U^{(1)}|+|U^{(2)}|$. 
Recall the definitions of $\psi[G_{\Delta}]$ from Lemmas~\ref{lem:gen-1} and~\ref{lem:gen-2}.  In all cases (that is for both lemmas and under both sampling schemes, recalling that $\psi[G_{\Delta}]=0$ in the first case of Lemma~\ref{lem:gen-2}), we have  
\begin{equation}\label{eq:psi:unified}
   \psi[G_{\Delta}]\leq  c_0 D^{c_1}  \overline{p}^{|E_{\cap}|} (D^{c_{1}}\lambda)^{|E_{\Delta}|} \left(\frac{D^{c_1 }k}{n}\right)^{|V_{\Delta}| - \#\mathrm{CC}_{\Delta}}\left[c_0\frac{D^{c_1}}{\sqrt{n}} \right]^{\#\mathrm{CC}_{\mathrm{pure}}}\enspace ,
\end{equation}
 for some constants $c_0$ and $c_1$ that depend on $c_{\texttt{m}}$, $c_{\texttt{v},1}, c_{\texttt{v},2},  c_{\texttt{v},3}$, and possibly $c_{\texttt{vd},1}$, $c_{\texttt{vd},2}$.

\begin{align*}
       A_0 &\leq  c_0D^{c_1}
       \sum_{\mathbf M \in \mathcal M\setminus \mathcal M_{\mathrm{PM}}}n^{U/2} 
        \left(\frac{\overline{p}}{\overline{q}}\right)^{(|E^{(1)}|+|E^{(2)}|)/2} \left(D^{c_{1}}\frac{\lambda}{\sqrt{\overline{p}}}\right)^{|E_{\Delta}|} \left(\frac{D^{c_1 }k}{n}\right)^{|V_{\Delta}| - \#\mathrm{CC}_{\Delta}}\left[c_{0} \frac{D^{c_{1}}}{\sqrt{n}} \right]^{\#\mathrm{CC}_{\mathrm{pure}}}\enspace .
\end{align*}
By Condition~\condsignal and definition~\eqref{eq:definition:bar:p}, we have 
\begin{equation}\label{eq:upper_bar_r}
    \frac{\overline{p}}{\overline{q}} = 1 + \frac{\lambda(1 - 2q)}{q(1 - q)}\leq 1 + \frac{\lambda}{q(1 - q)}\leq 1 + \frac{D^{-8c_{\texttt{s}}}}{1 - q}\leq 1 + 2D^{-8c_{\texttt{s}}}\leq 1+ D^{-4} \enspace .  
\end{equation}
 In the last inequality, we used that $c_{\texttt{s}}>1$ and that $D\geq 2$. As a consequence $(1+D^{-4})^{2D}\leq 2$ and we deduce that 
\begin{align}
       A_0 &\leq
       2c_0D^{c_1}\sum_{\mathbf M \in \mathcal M\setminus \mathcal M_{\mathrm{PM}}} n^{U/2} \nonumber
       \left(D^{c_{1}}\frac{\lambda}{\sqrt{\overline{q}}}\right)^{|E_{\Delta}|} \left(\frac{D^{c_1 }k}{n}\right)^{|V_{\Delta}| - \#\mathrm{CC}_{\Delta}}\left[c_{0} \frac{D^{c_{1}}}{\sqrt{n}} \right]^{\#\mathrm{CC}_{\mathrm{pure}}}\\ & \mathrel{=:} 2 c_0D^{c_1}\sum_{\mathbf M \in \mathcal M\setminus \mathcal M_{\mathrm{PM}}} A_{\mathbf{M}}\enspace .     \label{eq:definition_A_0}
\end{align}

\subsubsection{Step 3: Relying on the graph properties of $G_{\Delta}$ for $A_{\mathbf{M}}$.} 
Let us decompose $(G^{(1)},G^{(2)},\mathbf{M})$ into $(G'^{(1)}, G'^{(2)},\mathbf{M}, G^{(3)})$ where, in $G'^{(1)}$ (resp. $G'^{(2)}$),  we have removed all the pure connected components  of $G^{(1)}$ (resp. $G^{(2)}$) and we gather all these connected components in $G^{(3)}$. For $(G'^{(1)}, G'^{(2)},\mathbf{M})$, we can then define  the number $U'$ of unmatched nodes and the symmetric difference graph $G'_{\Delta}$ with $\#\mathrm{CC}'_{\Delta}$ connected components. Equipped with this notation, we have $\#\mathrm{CC}_{\mathrm{pure}}= \#\mathrm{CC}_{G^{(3)}}$, $|V_{\Delta}|= |V'_{\Delta}|+ |V^{(3)}|$, $\#\mathrm{CC}_{\Delta}=\#\mathrm{CC}'_{\Delta}+ \#\mathrm{CC}_{G^{(3)}}$ and $U= U' + |V^{(3)}|$. Then, we reorganize $A_{\mathbf{M}}$ as follows
\begin{align*}
A_{\mathbf{M}} & =  
       \left(D^{c_{1}}\frac{\lambda}{\sqrt{\overline{q}}}\right)^{|E'_{\Delta}|}\left(\frac{D^{c_1 }k}{\sqrt{n}}\right)^{|U'|} \left(\frac{D^{c_1 }k}{n}\right)^{|V'_{\Delta}|-|U'| - \#\mathrm{CC}'_{\Delta}} \\ & \quad \quad  \cdot 
       \left(D^{c_{1}}\frac{\lambda}{\sqrt{\overline{q}}}\right)^{|E^{(3)}|}  \left(\frac{D^{c_1 }k}{\sqrt{n}}\right)^{|V^{(3)}| -  \#\mathrm{CC}_{G^{(3)}}} (c_{0} D^{c_{1}})^{\#\mathrm{CC}_{G^{(3)}}}\\
       &\stackrel{(a)}{\leq} \left(D^{c_{1}}\frac{\lambda}{\sqrt{\overline{q}}}\right)^{|E'_{\Delta}|}\left(\frac{D^{c_1 }k}{\sqrt{n}}\right)^{|U'|} \left(\frac{D^{c_1 }k}{n}\right)^{|\mathbf{M}_{\mathrm{SM}}| - \#\mathrm{CC}'_{\Delta}} 
       \left(D^{c_{1}}\frac{\lambda}{\sqrt{\overline{q}}}\right)^{|E^{(3)}|}  \left(\frac{c_0 D^{2c_1 }k}{\sqrt{n}}\right)^{|V^{(3)}| -  \#\mathrm{CC}_{G^{(3)}}} \\ 
              &\leq \left(D^{c_{1}}\frac{\lambda}{\sqrt{\overline{q}}}\right)^{|E'_{\Delta}|- |U'|}\left(\frac{D^{2c_1 }k\lambda}{\sqrt{n\overline{q}}}\right)^{|U'|} \left(\frac{D^{c_1 }k}{n}\right)^{|\mathbf{M}_{\mathrm{SM}}| - \#\mathrm{CC}'_{\Delta}} \\ & \quad \quad \cdot 
       \left(D^{c_{1}}\frac{\lambda}{\sqrt{\overline{q}}}\right)^{|E^{(3)}|-|V^{(3)}| +  \#\mathrm{CC}_{G^{(3)}}}  \left(\frac{c_0 D^{3c_1 }k\lambda }{\sqrt{n\overline{q}}}\right)^{|V^{(3)}| -  \#\mathrm{CC}_{G^{(3)}}}  \enspace ,
\end{align*} 
where we used in $(a)$ that $|V'_{\Delta}|= |U'|+ |\mathbf{M}_{\mathrm{SM}}|$ as the nodes in $G'_{\Delta}$ are either unmatched or semi-matched and that $|V^{(3)}|\geq 2\#\mathrm{CC}_{G^{(3)}}$ as all connected components have at least two nodes. Let us show that all the exponents in the above bound on $A_{\mathbf{M}}$ are non-negative. 
 For any graph $G=(V,E)$, we have $|E|-|V| +  \#\mathrm{CC}_{G}\geq 0$ and $|V|\geq \#\mathrm{CC}_{G}$. 
\begin{lemma}\label{lem:countgra}
We have 
\begin{eqnarray*}
 |\mathbf M_{\mathrm{SM}}|\geq  \#\mathrm{CC}'_{\Delta}\;, \qquad |E'_{\Delta}| \geq  U' \enspace . 
\end{eqnarray*}
\end{lemma}

Hence, relying on Condition~\condsignal, we obtain
\begin{align}\label{eq:upper:AM}
A_{\mathbf{M}}\leq   D^{-6c_{\texttt{s}}[|E'_{\Delta}|+ |E^{(3)}|+ |\mathbf{M}_{\mathrm{SM}}| - \#\mathrm{CC}'_{\Delta} ] } \enspace . 
\end{align}
Since $|\mathbf{M}_{\mathrm{SM}}| \geq  \#\mathrm{CC}'_{\Delta}$ and $|E'_{\Delta}|+ |E^{(3)}|= |E_{\Delta}|\geq d(G^{(1)},G^{(2)})\vee 1$ by definition~\eqref{eq:definition:edit:distance} of $d(\cdot,\cdot)$, it follows that 
\begin{equation*}
|E'_{\Delta}|+ |E^{(3)}|+ |\mathbf{M}_{\mathrm{SM}}| - \#\mathrm{CC}'_{\Delta} \geq d(G^{(1)},G^{(2)}) \vee 1\enspace . 
\end{equation*}
Also, since each connected component of $G^{(3)}$ has at least two nodes, we deduce that $|E^{(3)}|\geq |V^{(3)}|/2$. Since $|E'_{\Delta}|\geq U'$ and $U'+|\mathbf{M}_{\mathrm{SM}}|= |V'_{\Delta}|\geq 2 \#\mathrm{CC}'_{\Delta}$ as each connected component of $G'_{\Delta}$ has at least two nodes, we conclude that  
\begin{equation*}
|E'_{\Delta}|+ |E^{(3)}|+ |\mathbf{M}_{\mathrm{SM}}| - \#\mathrm{CC}'_{\Delta}\geq \frac{|V^{(3)}|}{2} + \frac{U'+ |\mathbf{M}_{\mathrm{SM}}|}{2} =   \frac{U +  |\mathbf{M}_{\mathrm{SM}}|}{2} \enspace  . 
\end{equation*}
Gathering the two previous bounds in~\eqref{eq:upper:AM}, we get 
\begin{equation*}
A_{\mathbf{M}}\leq    D^{-3c_{\texttt{s}}[[U+ |\mathbf M_{\mathrm{SM}}|]  \lor d(G^{(1)},G^{(2)})\lor 1]} \enspace . 
\end{equation*}
Coming back to~\eqref{eq:definition_A_0} and using again that $c_{\texttt{s}}$ is large enough, we get 

\begin{align*}
       A_0 &\leq 2 c_0
       \sum_{\mathbf M \in \mathcal M\setminus \mathcal M_{\mathrm{PM}}}\left(D^{-2c_{\texttt{s}} }\right)^{ [U+ |\mathbf M_{\mathrm{SM}}|]  \lor d(G^{(1)},G^{(2)})\lor 1}\enspace . 
\end{align*}

\subsubsection{Step 4: Summing over shadows.} Recall the definition of shadows and of $\mathcal{M}_{\mathrm{shadow}}$ in Section~\ref{sec:graph:definition}. We now regroup the sum inside $A$ by enumerating  all possible matchings that are compatible with a specific shadow. Recall also the definition of $A_0$. We get 
\begin{align*}
       A &\leq  \frac{2c_0 }{\sqrt{|\mathrm{Aut}(G^{(1)})| |\mathrm{Aut}(G^{(2)})|}}\\ 
       &\sum_{\substack{U^{(1)} \subseteq V^{(1)},\\ U^{(2)} \subseteq V^{(2)},\\ \underline{\mathbf M} \in \mathcal M\setminus \mathcal M_{\mathrm{PM}}}} \quad 
       \sum_{\mathbf M \in \mathcal M_{\mathrm{shadow}}(U^{(1)},U^{(2)}, \underline{\mathbf M})}\left(D^{-2c_{\texttt{s}} }\right)^{ [U+ |\mathbf M_{\mathrm{SM}}|]  \lor d(G^{(1)},G^{(2)})\lor 1} \enspace .
\end{align*}
We have the following control for $\mathcal M_{\mathrm{shadow}}$.
\begin{lemma}\label{lem:shadow}
For any $\overline{U}_1\subset V^{(1)}$, $\overline{U}_2\subset V^{(2)}$, and $\underline{\mathbf M}$, we have
    \begin{equation}|\mathcal M_{\mathrm{shadow}}(\overline{U}_1, \overline{U}_2, \underline{\mathbf M} )| \leq \min(|\mathrm{Aut}(G^{(1)})|, |\mathrm{Aut}(G^{(2)})|)\enspace .\end{equation}
\end{lemma}
Observe that two matchings $\mathbf M$ and $\mathbf M'$ that  belong to $\mathcal M_{\mathrm{shadow}}(U^{(1)},U^{(2)}, \underline{\mathbf M})$ have the same symmetric difference graph $G_{\Delta}$ and have a common value of $|\mathbf M_{\mathrm{SM}}|$. Hence, it follows from Lemma~\ref{lem:shadow} that 
\begin{align*}
       A &\leq  2 c_{0}  
       \sum_{\substack{U^{(1)} \subset V^{(1)}, U^{(2)} \subset V^{(2)}, \\ \underline{\mathbf M} \in \mathcal M\setminus \mathcal M_{\mathrm{PM}}}} \left(D^{-2c_{\texttt{s}} }\right)^{[U+ |\mathbf M_{\mathrm{SM}}|]  \lor d(G^{(1)},G^{(2)})\lor 1 }\enspace .
\end{align*}
Now enumerating over all possible sets $U^{(1)}, U^{(2)}, \underline{\mathbf M} $ that have respective cardinalities $u_1$, $u_2$, and  $m$ and noting that the sizes of these collections are respectively  bounded by $(2D)^{u_1}, (2D)^{u_2}$ and $(2D)^{2m}$, we conclude that 
\begin{align*}
       A &\leq  2 c_{0}  
       \sum_{u_1, u_2, m \geq 0}(2D)^{u_1+u_2+2m} \left(D^{-2c_{\texttt{s}} }\right)^{ [u_1+u_2+m]  \lor d(G^{(1)},G^{(2)}) \lor 1 } \leq c''D^{-c_{\texttt{s}}  (d(G^{(1)},G^{(2)})\lor 1)}\enspace ,
\end{align*}
assuming that $c_{\texttt{s}} \geq 5$ and $D \geq 2$. We have established~\eqref{eq:objective_A} and this concludes the proof.

\subsection{Proof of Lemma~\ref{lem:gen-1}}

First, we consider some $\mathbf{M}\in \mathcal{M}^\star\setminus \mathcal{M}_{\mathrm{PM}}$. Since $\mathbf{M}$ belongs to $\mathcal{M}^{\star}$, the matching $\mathbf{M}$ does not let any connected component of $\pi^{(1)}[G^{(1)}]$ and $\pi^{(2)}[G^{(2)}]$ be unmatched. Let us decompose $G^{(1)} = (G^{(1)}_1,\ldots, G^{(1)}_{cc_1})$ and  $G^{(2)} = (G^{(2)}_1,\ldots, G^{(2)}_{cc_2})$ into their $cc_1$ and $cc_2$ connected components. 
Given a set $S_1\subset [cc_1]$, define $G^{(1)}_{-S_1}$ as the subgraph of $G^{(1)}$ obtained by removing the connected components corresponding to $S_1$. 
Write $A:= \big|\mathbb{E}\left[\overline{P}_{G^{(1)}, \pi^{(1)}} \overline{P}_{G^{(2)}, \pi^{(2)}}\right] \big|$.
By definition, we have 

\begin{align*}
\big|A \big|&= \Bigg|\sum_{S_1 \subseteq [cc_1]} \sum_{S_2 \subseteq [cc_2]} \mathbb E \left[P_{G^{(1)}_{- S_1}, \pi^{(1)}} P_{G^{(2)}_{-S_2}, \pi^{(2)}}\right] \prod_{i\in S_1} \mathbb E[P_{G^{(1)}_{i}, \pi^{(1)}}] \prod_{i\in S_2} \mathbb E[P_{G^{(2)}_{i}, \pi^{(2)}}] (-1)^{|S_1|+|S_2|}\Bigg|\\
&\leq \sum_{s_1=0}^{cc_1}\sum_{s_2=0}^{cc_2} D^{s_1+s_2} \max_{S_1: |S_1|=s_1 , \ S_2: |S_2|=s_2} \left|\mathbb E \left[P_{G^{(1)}_{- S_1}, \pi^{(1)}} P_{G^{(2)}_{-S_2}, \pi^{(2)}}\right] \prod_{i\in S_1} \mathbb E[P_{G^{(1)}_{i}, \pi^{(1)}}] \prod_{i\in S_2} \mathbb E[P_{G^{(2)}_{i}, \pi^{(2)}}]\right|\enspace . 
\end{align*}
 
Then, we apply  \condmoment as well as the first part of \condvariance. When the configuration made of $(G^{(1)}_{-S_1},\pi^{(1)})$ and $(G^{(2)}_{-S_2},\pi^{(2)})$ is perfectly matched, the first part of \condvariance\ does not apply, and we use its second part instead: $\mathbb{E}[P_{G^{(1)}_{-S_1},\pi^{(1)}}P_{G^{(2)}_{-S_2},\pi^{(2)}}]\leq (1+c_{\texttt{v},3})\overline{q}^{|E_{\cap;-S_1,-S_2}|}\leq (1+c_{\texttt{v},3})\overline{p}^{|E_{\cap;-S_1,-S_2}|}$, since $\overline{q}\leq \overline{p}$ and since $E_{\Delta,-S_1,-S_2}$ and $V_{\Delta,-S_1,-S_2}$ are then empty. When both graphs are empty, this expectation is equal to $1$. When exactly one of the two graphs is empty, say $G^{(2)}_{-S_2}$, \condvariance\ does not apply either, and we use \condmoment\ instead: $|\mathbb{E}[P_{G^{(1)}_{-S_1},\pi^{(1)}}]|\leq (D^{c_{\texttt{m}}}\lambda)^{|E_{\Delta,-S_1,-S_2}|}(D^{c_{\texttt{m}}}k/n)^{|V_{\Delta,-S_1,-S_2}| - \#\mathrm{CC}_{\Delta, -S_1, -S_2}}$, since the symmetric difference graph is then $\pi^{(1)}[G^{(1)}_{-S_1}]$ and $E_{\cap;-S_1,-S_2}$ is empty. In all these cases, the multiplicative constant is either $c_{\texttt{v},2}$, $1+c_{\texttt{v},3}$, or $1$, and is therefore at most $1+c_{\texttt{v},2}\vee c_{\texttt{v},3}$, whereas $D^{c_{\texttt{v},1}}$ and $D^{c_{\texttt{m}}}$ are both at most $D^{c_{\texttt{v},1}\vee c_{\texttt{m}}}$. We write $G_{\Delta; -S_1;-S_2}$ for the symmetric difference graph associated to $G^{(1)}_{- S_1}$ and $G^{(2)}_{- S_2}$. We get

\begin{align*}
\big|A\big|&\leq \sum_{s_1=0}^{cc_1}\sum_{s_2=0}^{cc_2} D^{s_1+s_2} \max_{S_1: |S_1|=s_1 , \ S_2: |S_2|=s_2} (1+c_{\texttt{v},2}\vee c_{\texttt{v},3}) (D^{c_{\texttt{v},1}\vee c_{\texttt{m}}}\lambda)^{|E_{\Delta,-S_1,-S_2}|} \overline{p}^{|E_{\cap;-S_1,-S_2}|} \\ &\hspace{3cm}  \left(D^{c_{\texttt{v},1}\vee c_{\texttt{m}}}\frac{k}{n}\right)^{|V_{\Delta,-S_1,-S_2}| - \#\mathrm{CC}_{\Delta, -S_1, -S_2}}  \prod_{a=1}^2\prod_{i\in S_a}\left(D^{c_{\texttt{m}}}\lambda\right)^{|E_{i}^{(a)}|} \left(D^{c_{\texttt{m}}}\frac{k}{n}\right)^{|V_{i}^{(a)}|-1}\\ 
&\leq \sum_{s_1=0}^{cc_1}\sum_{s_2=0}^{cc_2} D^{s_1+s_2} \max_{S_1: |S_1|=s_1 , \ S_2: |S_2|=s_2} (1+c_{\texttt{v},2}\vee c_{\texttt{v},3}) (D^{c_{\texttt{v},1}\vee c_{\texttt{m}}}\lambda)^{|E_{\Delta}|} \overline{p}^{|E_{\cap}|} \left(D^{c_{\texttt{v},1}\vee c_{\texttt{m}}}\frac{k}{n}\right)^{R}\enspace , 
\end{align*}
where we used in the last line that $D^{c_{\texttt{v},1}\vee c_{\texttt{m}}}\lambda\leq \overline{p}\leq 1$ by Equation~\eqref{eq:definition:bar:p} and Condition~\condsignal with  $c_{\texttt{s}}>0$ sufficiently large and where 
\begin{equation}\label{eq:definition:R}
R:= |V_{\Delta,-S_1,-S_2}| - \#\mathrm{CC}_{\Delta, -S_1, -S_2} + \sum_{i\in S_1} (|V_i^{(1)}|-1) +   \sum_{i\in S_2} (|V_i^{(2)}|-1)\enspace . 
\end{equation}
Since the number of nodes $|V_i^{(1)}|$ of each connected component is at least $2$, we easily check that 
\begin{equation*}
R\geq \sum_{i\in S_1} (|V_i^{(1)}|-1) +   \sum_{i\in S_2} (|V_i^{(2)}|-1) \geq |S_1|+|S_2|\enspace .   
\end{equation*}
Coming back to $A$, we use $D^{s_1+s_2}\leq D^{R}$ and that the double sum over $(s_1,s_2)$ has $(cc_1+1)(cc_2+1)\leq (D+1)^2$ terms. We arrive at the bound
\begin{align*}
\big|A\big|\leq (1+c_{\texttt{v},2}\vee c_{\texttt{v},3})(D+1)^{2}\max_{S_1,S_2}
 (D^{c_{\texttt{v},1}\vee c_{\texttt{m}}}\lambda)^{|E_{\Delta}|} \overline{p}^{|E_{\cap}|} \left(D^{1+ c_{\texttt{v},1}\vee c_{\texttt{m}}}\frac{k}{n}\right)^{R}\enspace . 
\end{align*}
 Again, by Condition~\condsignal with  $c_{\texttt{s}}>0$ sufficiently large, we have $D^{1+ c_{\texttt{v},1}\vee c_{\texttt{m}}}k\leq n$. Hence, it remains to lower bound $R$.

\begin{lemma}\label{lem:lower:bound:R}
The quantity $R$ defined in~\eqref{eq:definition:R} satisfies $R \geq |V_{\Delta}| - \#\mathrm{CC}_{\Delta}$. 
\end{lemma}
Combining this lemma with our previous bound, we arrive at 
\begin{align}\label{eq:upper:bound:psi:without:pure}
\big|\mathbb{E}\left[\overline{P}_{G^{(1)}, \pi^{(1)}} \overline{P}_{G^{(2)}, \pi^{(2)}}\right] \big| \leq \psi[G_{\Delta}]\enspace , 
\end{align}
in the specific case where $\mathbf{M}\in\mathcal{M}^\star\setminus \mathcal{M}_{\mathrm{PM}}$, that is when $\#\mathrm{CC}_{\mathrm{pure}}=0$. We have proved the first part of the lemma.

\subsubsection{Let us turn to $\mathbb{E}\big[\overline{P}^2_{G,\pi}\big]$.} We start as in the first step of this proof. We decompose $G= (G_1,\ldots, G_{cc})$ into its $cc$ connected components. First, we bound  $ A:= \mathbb{E}\big[\overline{P}^2_{G, \pi}\big] - \mathbb{E}\left[P^2_{G, \pi}\right]$. 
Given $S\subset [cc]$, we write $P_{G_{-S},\pi}$ for the polynomial associated to the graph where we have removed the connected components in $S$. Expanding the expression of $\overline{P}$, we derive that 

\begin{align*}
|A|& \leq \Bigg|\sum_{S_1, S_2 \subseteq [cc]: \, S_1\cup S_2 \neq \emptyset}\mathbb E \left[P_{G_{- S_1}, \pi} P_{G_{-S_2}, \pi}\right] \prod_{i\in S_1} \mathbb E[P_{G_{i}, \pi}] \prod_{i\in S_2} \mathbb E[P_{G_{i}, \pi}] (-1)^{|S_1|+|S_2|}\Bigg|\\
&\leq \sum_{s_1, s_2=0:\ s_1+s_2>0}^{cc} D^{s_1+s_2} \max_{S_1: |S_1|=s_1 , \ S_2: |S_2|=s_2}\Big| \mathbb E \left[P_{G_{- S_1}, \pi} P_{G_{-S_2}, \pi}\right] \prod_{i\in S_1} \mathbb E[P_{G_{i}, \pi}] \prod_{i\in S_2} \mathbb E[P_{G_{i}, \pi}]\Big|\enspace . 
\end{align*}
Then, we apply \condmoment as well as the first part of \condvariance, or its second part when $S_1=S_2$, in which case $G_{-S_1}$ and $G_{-S_2}$ are perfectly matched, as explained in the first part of the proof. As in the first part of the proof, when exactly one of $G_{-S_1}$ and $G_{-S_2}$ is empty (that is $S_1=[cc]$ or $S_2=[cc]$, but not both), we use \condmoment\ instead of \condvariance. We write $G_{\Delta; -S_1;-S_2}$ (resp. $G_{\cap; -S_1; -S_2}$) for the symmetric difference graph (resp. intersection graph) associated to $G_{- S_1}$ and $G_{- S_2}$. We get 
\begin{align*}
\big|A\big|&\leq \sum_{s_1,s_2: s_1+s_2>0} D^{s_1+s_2} \max_{S_1: |S_1|=s_1 , \ S_2: |S_2|=s_2} (1+c_{\texttt{v},2}\vee c_{\texttt{v},3}) (D^{c_{\texttt{v},1}\vee c_{\texttt{m}}}\lambda)^{|E_{\Delta,-S_1,-S_2}|} \overline{p}^{|E_{\cap;-S_1,-S_2}|} \\ &\hspace{3cm} 
\left(D^{c_{\texttt{v},1}\vee c_{\texttt{m}}}\frac{k}{n}\right)^{|V_{\Delta,-S_1,-S_2}| - \#\mathrm{CC}_{\Delta, -S_1, -S_2}}   \prod_{a=1}^2\prod_{i\in S_a}\left(D^{c_{\texttt{m}}}\lambda\right)^{|E_{i}|} \left(D^{c_{\texttt{m}}}\frac{k}{n}\right)^{|V_{i}|-1}\\ 
&\leq \sum_{s_1,s_2: s_1+s_2>0} D^{s_1+s_2} \max_{S_1: |S_1|=s_1 , \ S_2: |S_2|=s_2}  (1+c_{\texttt{v},2}\vee c_{\texttt{v},3}) \overline{p}^{|E|} \left(D^{c_{\texttt{v},1}\vee c_{\texttt{m}}}\frac{k}{n}\right)^{s_1+s_2}\enspace , 
\end{align*}
where we used Condition \condsignal with $c_{\texttt{s}}$ large enough  to ensure that $ D^{1+c_{\texttt{v},1}\vee c_{\texttt{m}}}\frac{k}{n}\leq 1$ and  $ D^{c_{\texttt{v},1}\vee c_{\texttt{m}}}\lambda \leq \overline{p}$ and we used that $|E_{\Delta,-S_1,-S_2}|+ |E_{\cap;-S_1,-S_2}| + \sum_{a,i}|E_{i}|\geq |E|$. We have proved in~\eqref{eq:upper_bar_r} that $\overline{p}\leq \overline{q}(1+D^{-4})$. Since $|E|\leq D$, we arrive at 
\begin{equation*}
\left|\mathbb{E}\left[\overline{P}^2_{G, \pi}\right] - \mathbb{E}\left[P^2_{G, \pi}\right]\right|\leq 2(1+c_{\texttt{v},2}\vee c_{\texttt{v},3}) D^{4+ c_{\texttt{v},1}\vee c_{\texttt{m}}}\frac{k}{n}\overline{q}^{|E|}\enspace .
\end{equation*}
Combining this inequality with Condition \condvariance, we conclude that 
\begin{equation*}
\left|\mathbb{E}\left[\overline{P}^2_{G, \pi}\right] - \overline{q}^{|E|}\right|\leq \left[2(1+c_{\texttt{v},2}\vee c_{\texttt{v},3}) D^{4+ c_{\texttt{v},1}\vee c_{\texttt{m}}}\frac{k}{n}+c_{\texttt{v},3}D^{-c_{\texttt{s}}}\right]\overline{q}^{|E|}\enspace .
\end{equation*}

\subsection{Proof of Lemma~\ref{lem:gen-2}}

The first statement of the lemma is straightforward. Without loss of generality, we may assume that there exists a connected component $H$ of $G^{(1)}$ such that $\pi^{(1)}(V_H)$ does not intersect $\pi^{(2)}(V^{(2)})$. Writing $\tau$ for the restriction of $\pi^{(1)}$ to $V_H$, $G^{(0)}= (V^{(0)},E^{(0)})$ for the remainder of $G^{(1)}$ after we have removed $H$, and $\pi^{(0)}$ for the restriction of $\pi^{(1)}$ to $V^{(0)}$, we get 
\begin{equation*}
\overline{P}_{G^{(1)},\pi^{(1)}}\overline{P}_{G^{(2)},\pi^{(2)}}= \overline{P}_{H,\tau} \overline{P}_{G^{(0)},\pi^{(0)}}\overline{P}_{G^{(2)},\pi^{(2)}} =\left[P_{H,\tau}- \mathbb{E}\left(P_{H,\tau}\right) \right] \overline{P}_{G^{(0)},\pi^{(0)}}\overline{P}_{G^{(2)},\pi^{(2)}}\enspace . 
\end{equation*}
Since the  latent assignments $z_i$ are sampled with replacement, $P_{H,\tau}$ is independent of $\overline{P}_{G^{(0)},\pi^{(0)}}\overline{P}_{G^{(2)},\pi^{(2)}}$ and it follows that $\mathbb{E}\left[\overline{P}_{G^{(1)},\pi^{(1)}}\overline{P}_{G^{(2)},\pi^{(2)}}\right]=0$.

The main challenge in the proof is to consider the setting where we sample  latent assignments without replacement. Indeed, in this case, polynomials associated to indices are not independent and we have to quantify  this  dependence. 
Let $G^{(1)}=(V^{(1)},E^{(1)})$ and $G^{(2)}=(V^{(2)},E^{(2)})$ be two templates with at most $D$ edges, $\mathbf{M}\in \mathcal{M}\setminus \mathcal{M}^\star$ be a matching that leads to at least one pure connected component, and let $(\pi^{(1)},\pi^{(2)})\in \Pi(\mathbf{M})$. For short, we write $r=\#\mathrm{CC}_{\mathrm{pure}}\geq 1$ for  the number of such pure connected components. Then, we let $(H_1,\tau_1),\ldots, (H_r,\tau_r)$ denote these pure connected components and their corresponding labelings, and we write $H_a=(V_{H_a},E_{H_a})$. Besides, we write $(G^{(0)},\pi^{(0)})$ and $(G^{'(0)},\pi^{'(0)})$ for the remainders of $(G^{(1)},\pi^{(1)})$ and of $(G^{(2)},\pi^{(2)})$ after we have removed them. Equipped with this notation, we have the following decomposition
\begin{align*}
\overline{P}_{G^{(1)},\pi^{(1)}}\overline{P}_{G^{(2)},\pi^{(2)}}= \overline{P}_{G^{(0)},\pi^{(0)}}\overline{P}_{G^{'(0)},\pi^{'(0)}}\prod_{a=1}^r \overline{P}_{H_a,\tau_a}\enspace . 
\end{align*}
To ease the reading, we write, for $a=1,\ldots,r$, $\overline{P}_a$ for $\overline{P}_{H_a,\tau_a}$ and we define
$\overline{P}_0:= \overline{P}_{G^{(0)},\pi^{(0)}}\overline{P}_{G^{'(0)},\pi^{'(0)}} - \mathbb{E}[\overline{P}_{G^{(0)},\pi^{(0)}}\overline{P}_{G^{'(0)},\pi^{'(0)}}]$ as the centered version of $\overline{P}_{G^{(0)},\pi^{(0)}}\overline{P}_{G^{'(0)},\pi^{'(0)}}$. Then, it follows that 
\begin{align}\label{eq:decomposition:produit:polynome}
\mathbb{E}\left[\overline{P}_{G^{(1)},\pi^{(1)}}\overline{P}_{G^{(2)},\pi^{(2)}}\right]= \mathbb{E}\left[\prod_{a=0}^r \overline{P}_{a}\right]
+\mathbb{E}\left[\overline{P}_{G^{(0)},\pi^{(0)}}\overline{P}_{G^{'(0)},\pi^{'(0)}}\right] \mathbb{E}\left[\prod_{a=1}^r \overline{P}_{a}\right]\enspace . 
\end{align}
The quantity $\mathbb{E}\left[\overline{P}_{G^{(0)},\pi^{(0)}}\overline{P}_{G^{'(0)},\pi^{'(0)}}\right]$ has been controlled in Lemma~\ref{lem:gen-1} as the graph $G^{(0)}_\Delta$ that arises from $(G^{(0)},\pi^{(0)})$ and $(G^{'(0)},\pi^{'(0)})$ does not contain any pure connected component: if $(G^{(0)},\pi^{(0)})$ and $(G^{'(0)},\pi^{'(0)})$ are perfectly matched, we use the second part of Lemma~\ref{lem:gen-1}, which gives $\mathbb{E}[\overline{P}^2_{G^{(0)},\pi^{(0)}}]\leq 2\overline{q}^{|E^{(0)}|}\leq 2\overline{p}^{|E^{(0)}_{\cap}|}$, and if both graphs are empty, this quantity is equal to $1$. 
Since $\sum_{a=1}^{r}|V_{H_a}|+ |V^{(0)}_{\Delta}|= |V_{\Delta}|$ and $\sum_{a=1}^r |E_{H_a}| + |E^{(0)}_{\Delta}|=|E_{\Delta}|$, we only have to prove that 
\begin{align}\label{eq:obj_A}
\Bigg|\mathbb{E}\left[\prod_{a=1}^r \overline{P}_{a}\right]\Bigg| & \leq c_{0}D^{c_{1}}\left(D^{c_1}\lambda\right)^{  \sum_{a=1}^r |E_{H_a}|}\left(D^{c_{1}}\frac{k}{n}\right)^{\sum_{a=1}^r(|V_{H_a}|-1)}\left[c_0\frac{D^{c_1}}{\sqrt{n}}\right]^{r} \ ; \\
\Bigg|\mathbb{E}\left[\prod_{a=0}^r \overline{P}_{a}\right]\Bigg|& \leq c_0\overline{p}^{|E_{\cap}|}D^{c_1}\left(D^{c_1}\lambda\right)^{|E_{\Delta}| }\left(D^{c_{1}}\frac{k}{n}\right)^{|V_{\Delta}|-\#\mathrm{CC}_{\Delta}}\left[c_0\frac{D^{c_1}}{\sqrt{n}}\right]^{r} \enspace ,  \label{eq:obj_B}
\end{align}
for some  positive quantities $c_0$ and $c_1$ large enough that depend on the constants in Conditions \condmoment, \condvariance, and \condmomentWR.

In this proof,  $c_0$ and $c_1$ may change from line to line. 
Importantly, each random variable $\overline{P}_i$ is centered and involves different  latent assignments because we sample without replacement. To emphasize these  latent assignments, we write $Z_a= \{z_{\tau_a(v)}: v\in V_{H_a}\}$ for $a=1,\ldots, r$ and we also define $Z_0=  \{z_{\pi^{(0)}(v)}: v\in V^{(0)}\}\cup  \{z_{\pi^{'(0)}(v)}: v\in V^{'(0)}\}$ for the  latent assignments involved in $\overline{P}_0$.

We introduce the probability distribution $\mathbb{P}_R$ where, as in $\mathbb{P}$, each $Z_a$ is sampled without replacement but, contrary to $\mathbb{P}$, the $Z_a$'s are sampled independently. Define the event $\mathcal{E}$ (resp. $\mathcal{E}'$) such that all the $Z_a$'s with $a=1,\ldots, r$ (resp. $a=0,\ldots, r$) are pairwise disjoint. Then, by definition of $\mathbb{E}_R$, we have 
\begin{align}\label{eq:definition:A}
\mathbb{E}\left[\prod_{a=1}^{r} \overline{P}_a\right]&= \frac{1}{\mathbb{P}_R(\mathcal{E})}\mathbb{E}_R\left[\mathbf{1}\{\mathcal{E}\}\prod_{a=1}^{r}\overline{P}_a\right]\ ;\\
\mathbb{E}\left[\prod_{a=0}^{r} \overline{P}_a\right]&= \frac{1}{\mathbb{P}_R(\mathcal{E}')}\mathbb{E}_R\left[\mathbf{1}\{\mathcal{E}'\}\prod_{a=0}^{r}\overline{P}_a\right]\enspace .\label{eq:definition:A2}
\end{align}
We have 
\begin{equation}
    \mathbb{P}_{R}(\mathcal{E})\geq \mathbb{P}_R(\mathcal{E}') \geq \left(1 - \frac{4D}{n}\right)^{4D}\geq \frac{1}{2}\enspace .
\end{equation}

Our purpose is therefore to upper bound the expectations in the right-hand sides of~\eqref{eq:definition:A} and~\eqref{eq:definition:A2}. We shall prove this in a recursive manner.

In the sequel, we write $[0;r]$ for the set $\{0,1,\ldots, r\}$. 
Given a partition $\cB=(B_1,\ldots, B_t)$ of $[r]$ or of $[0;r]$, we define the event  $\mathcal{E}_{\cB}$ such that, for any two groups $B\neq B'$ in the partition and any $(i,j)\in B\times B'$, we have $Z_i\cap Z_j=\emptyset$. If $\cB$  is the trivial partition (i.e. $|\cB|=1$), we take the convention that the event $\mathcal{E}_{\cB}$ is an event of probability one. 

Given a subset $B\subset [0;r]$, we define the event $\cA_B$ such that, for any $i$ and $i'$ in $B$, there exists a sequence $i_0=i$, $i_1$, \ldots, $i_s= i'$ in $B$ such that $Z_{i_t}\cap Z_{i_{t+1}}\neq \emptyset$ for all $t=0,\ldots, s-1$. In other words, if we draw edges between $i$ and $j$ whenever $Z_i\cap Z_j\neq \emptyset$, then, under the event $\cA_B$, $B$ is connected. Henceforth, under $\cA_B$, we say that the {\it polynomials indexed by $B$ are connected through their  latent assignments}. 
When $|B|=1$, we take the convention that $\cA_B$ is a probability-one event.
More generally, given $B\subset [0;r]$, we say that an event $\cF$ is a \emph{connectivity event on $B$} if $\cF= \bigcap_{C\in \mathcal{C}}\cA_C$ for some collection $\mathcal{C}$ of subsets of $B$ such that $B\in \mathcal{C}$. In other words, under $\cF$, the polynomials indexed by $B$ are connected through their latent assignments, and so are the polynomials indexed by $C$, for each $C\in \mathcal{C}$. For instance, $\cA_B$ is a connectivity event on $B$ (take $\mathcal{C}=\{B\}$). Any connectivity event $\cF$ on $B$ satisfies $\cF\subset \cA_B$, only depends on $(Z_a)_{a\in B}$, and is of probability one when $|B|=1$. Such events arise in the recursive argument of the proof of Lemma~\ref{lem:recursion:CP} below.

The following lemma is a  consequence of Condition~\condmomentWR. In this lemma,  $E'_{\Delta}$, $V'_{\Delta}$, and $E'_{\cap}$ refer to graphs associated to the intersection and the symmetric difference of  $(G^{(0)},\pi^{(0)})$ and $(G^{'(0)},\pi^{'(0)})$. 
\begin{lemma}\label{lem:recursion:covariance}
Let $B$ be a subset of $[0;r]$. If $0\notin B$, we define 
\begin{align*}
\varphi(B):= c_0D^{c_{1}}\left(D^{c_1}\lambda\right)^{\sum_{a\in B} |E_{H_a}|}\left(D^{c_{1}}\frac{k}{n}\right)^{\sum_{a\in B} (|V_{H_a}|-1)}\left[c_{0}\frac{D^{c_1}}{\sqrt{n}}\right]^{|B|} \enspace . 
\end{align*}
If $0\in B$, we define  
\begin{align*}
\varphi(B):=   c_0\overline{p}^{|E'_{\cap}|}D^{c_{1}}\left(D^{c_1}\lambda\right)^{|E'_{\Delta}| + \sum_{a\in B\setminus \{0\}} |E_{H_a}|}\left(D^{c_{1}}\frac{k}{n}\right)^{|V'_{\Delta}|-\#\mathrm{CC}'_{\Delta}+\sum_{a\in B\setminus\{0\}} (|V_{H_a}|-1)}\left[c_{0}\frac{D^{c_{1}}}{\sqrt{n}}\right]^{|B|-1}\enspace . 
\end{align*}
Then, for any connectivity event $\cF$ on $B$, we have
\begin{equation*}
\left|\mathbb{E}_R\left[\mathbf{1}\{\cF\}\prod_{a\in B }\overline{P}_a\right]\right|\leq \varphi(B)\enspace .
\end{equation*}
In particular, $\left|\mathbb{E}_R\left[\mathbf{1}\{\mathcal{A}_{B}\}\prod_{a\in B }\overline{P}_a\right]\right|\leq \varphi(B)$.
\end{lemma}

The following lemma is proved by induction. 
\begin{lemma}\label{lem:recursion:CP}
For any partition $\cB$ of $[r]$, we have 
\begin{align*}
\left| \mathbb{E}_R\left[\1\{\mathcal{E}_{\cB}\}\prod_{B\in \cB}\1\{\cA_B\} \prod_{a\in B} \overline{P}_a  \right] \right|\leq  r^{2r} \prod_{i=1}^r \varphi(\{i\})\enspace . 
\end{align*}
For any partition $\cB$ of $[0;r]$,  we have 
\begin{align*}
\left| \mathbb{E}_R\left[\1\{\mathcal{E}_{\cB}\}\prod_{B\in \cB}\1\{\cA_B\} \prod_{a\in B} \overline{P}_a  \right] \right|\leq (r+1)^{2(r+1)} \prod_{i=0}^r \varphi(\{i\}) \enspace . 
\end{align*}

\end{lemma}
Before establishing these lemmas, let us finish the proof.  Coming back to~\eqref{eq:definition:A} and~\eqref{eq:definition:A2}, we straightforwardly derive from Lemma~\ref{lem:recursion:CP} applied to the finest partitions $\{\{1\},\ldots,\{r\}\}$ and $\{\{0\},\ldots,\{r\}\}$ (for which $\cE_{\cB}=\cE$, resp. $\cE_{\cB}=\cE'$, and all the events $\cA_B$ are of probability one), and from $\mathbb{P}_R(\mathcal{E})\geq \mathbb{P}_R(\mathcal{E}')\geq 1/2$, that
\begin{align*}
\left|\mathbb{E}\left[\prod_{a=1}^{r} \overline{P}_a\right]\right|\leq 2r^{2r} \prod_{i=1}^r \varphi(\{i\})  \ ; \quad \quad
\left|\mathbb{E}\left[\prod_{a=0}^{r} \overline{P}_a\right]\right|\leq 2(r+1)^{2(r+1)} \prod_{i=0}^r \varphi(\{i\})\enspace .
\end{align*}
This yields~\eqref{eq:obj_A}~and~\eqref{eq:obj_B}. The result follows.

\begin{proof}[Proof of Lemma~\ref{lem:recursion:CP}]
We only prove the result for a partition $\cB$ of $[r]$, the  arguments being the same for partitions $\cB$ of $[0;r]$.  
For two partitions $\cB$ and $\cB'$, we write $\cB\prec \cB'$ if $\cB'$ is (strictly) finer than $\cB$, that is, if $\cB\neq \cB'$ and $\cB$ is obtained by merging some groups of $\cB'$.
We rely on the connectivity events introduced before Lemma~\ref{lem:recursion:covariance}: they are needed because the recursion below does not preserve events of the form $\cA_B$ (see~\eqref{eq:recursion:events}).

\medskip

\noindent
{\bf Step 1: A property of $\varphi$.} It follows from the definition of $\varphi$ that $\varphi(B)= (c_0D^{c_1})^{1-|B|}\prod_{i\in B}\varphi(\{i\})$ for any non-empty $B$. Since we can take $c_0\geq 1$, this implies
\begin{equation}\label{eq:varphi:sub}
\prod_{B\in \cB}\varphi(B)\leq \prod_{i=1}^r \varphi(\{i\}) \quad \text{ for any partition $\cB$ of $[r]$}\enspace .
\end{equation}

\medskip

\noindent
{\bf Step 2: Induction.} Given a partition $\cB$ of $[r]$, we write $N(\cB)$ for the number of sequences of partitions $(\cB^{(0)},\ldots, \cB^{(l)})$, with $l\geq 0$, such that $\cB^{(0)}= \cB$ and $\cB^{(i-1)}\succ \cB^{(i)}$ for all $i=1,\ldots, l$. Considering the second element of such a sequence (when $l\geq 1$), we get
\begin{equation}\label{eq:recursion:N}
N(\cB)= 1+ \sum_{\cB'\prec \cB}N(\cB')\enspace .
\end{equation}
We prove by induction on $|\cB|$ that, for any partition $\cB$ of $[r]$ and any collection $(\cF_B)_{B\in \cB}$ where each $\cF_B$ is a connectivity event on $B$, we have
\begin{equation}\label{eq:induction:CP}
\left|\mathbb{E}_R\left[\mathbf{1}\{\cE_\mathcal{B}\}\prod_{B\in \cB} W_{B}\right]\right|\leq N(\cB) \prod_{i=1}^r \varphi(\{i\})\enspace , \quad \text{ where }W_B:= \mathbf{1}\{\cF_B\}\prod_{a\in B}\overline{P}_a\enspace .
\end{equation}
Taking $\cF_B=\cA_B$ for all $B\in \cB$, the lemma then follows from the bound on $N(\cB)$ established in Step 3.
If $|\cB|=1$, then $\cB=\{[r]\}$, $N(\cB)=1$, and $\cE_{\cB}$ is of probability one, so that~\eqref{eq:induction:CP} follows from Lemma~\ref{lem:recursion:covariance} (applied with $\cF=\cF_{[r]}$) and~\eqref{eq:varphi:sub}.
Consider now a partition $\cB$  of $[r]$  with more than one group, and assume that~\eqref{eq:induction:CP} holds for all partitions with fewer groups than $\cB$.
Note that, for $B\in \cB$, the
$W_{B}$'s are independent under $\mathbb{P}_R$ (recall that $\cF_B$ only depends on $(Z_a)_{a\in B}$). This  entails that
\begin{equation*}
\mathbb{E}_R\big[\prod_{B\in \cB} W_{B}\big] = \prod_{B\in \cB}\mathbb{E}_R[W_{B}]\enspace. 
\end{equation*}
As a consequence, we get 
\begin{align*}
 \mathbb{E}_R[\mathbf{1}\{\cE_\mathcal{B}\}\prod_{B\in \cB} W_{B}] = \prod_{B\in \cB}\mathbb{E}_R[W_{B}] - \mathbb{E}_R\left[\1\{\mathcal{E}^c_{\cB}\}\prod_{B\in \cB} W_{B}\right] \enspace. 
\end{align*}
Let us decompose the event $\cE_{\cB}^c\cap \bigcap_{B\in \cB}\cF_B$. Given the latent assignments, consider the graph on $[r]$ where $i$ and $j$ are adjacent if $Z_i\cap Z_j\neq \emptyset$, and let $\Gamma$ be the partition of $[r]$ into the connected components of this graph. For any partition $\cB'$ of $[r]$, the event $\cE_{\cB'}\cap \bigcap_{B'\in \cB'}\cA_{B'}$ is exactly the event $\{\Gamma= \cB'\}$. Under $\bigcap_{B\in \cB}\cF_B\subset \bigcap_{B\in \cB}\cA_B$, each group of $\cB$ is connected, so that $\Gamma$ is either equal to $\cB$ or coarser than $\cB$; besides, $\Gamma=\cB$ if and only if $\cE_{\cB}$ holds. As a consequence, under $\bigcap_{B\in \cB}\cF_B$, the event $\cE^c_{\cB}$ is the disjoint union of the events $\{\Gamma=\cB'\}$ over all $\cB'\prec \cB$. This leads to
\begin{equation}\label{eq:recursion:events}
\mathbf{1}\{\mathcal{E}^c_{\cB}\}\prod_{B\in \cB} \mathbf{1}\{\cF_{B}\} = \sum_{\cB' \prec \cB}\mathbf{1}\{\mathcal{E}_{\cB'}\}\prod_{B'\in \cB'} \mathbf{1}\{\cA_{B'}\}\prod_{B\in \cB} \mathbf{1}\{\cF_{B}\}= \sum_{\cB' \prec \cB}\mathbf{1}\{\mathcal{E}_{\cB'}\}\prod_{B'\in \cB'} \mathbf{1}\{\cF'_{B'}\}\enspace ,
\end{equation}
where, for $B'\in \cB'$, $\cF'_{B'}:= \cA_{B'}\cap \bigcap_{B\in \cB:\, B\subset B'}\cF_B$ is a connectivity event on $B'$ (each group of $\cB$ is included in exactly one group of $\cB'$). Note that we cannot replace $\cF'_{B'}$ by $\cA_{B'}$ in~\eqref{eq:recursion:events}: under $\cA_{B'}$ alone, the polynomials indexed by a group $B\subset B'$ of $\cB$ may be connected only through polynomials indexed by $B'\setminus B$. Writing $W'_{B'}:= \mathbf{1}\{\cF'_{B'}\}\prod_{a\in B'}\overline{P}_a$, we therefore obtain
\begin{equation*}
\left|\mathbb{E}_R\left[\mathbf{1}\{\cE_\mathcal{B}\}\prod_{B\in \cB} W_{B}\right]\right| \leq  \prod_{B\in \cB}\left|\mathbb{E}_R[W_{B}]\right| +   \sum_{\cB' \prec \cB}\left|\mathbb{E}_R\left[\mathbf{1}\{\mathcal{E}_{\cB'}\}\prod_{B'\in \cB'} W'_{B'}\right]\right|\enspace .
\end{equation*}
By Lemma~\ref{lem:recursion:covariance} (applied with $\cF=\cF_B$ for each $B\in \cB$) and~\eqref{eq:varphi:sub}, the first term is at most $\prod_{B\in \cB}\varphi(B)\leq \prod_{i=1}^r \varphi(\{i\})$. Since $|\cB'|<|\cB|$ for $\cB'\prec \cB$, the induction hypothesis applies to each term of the sum, which is therefore at most $N(\cB')\prod_{i=1}^r \varphi(\{i\})$. By~\eqref{eq:recursion:N}, we conclude that~\eqref{eq:induction:CP} holds for $\cB$.

\medskip

\noindent
{\bf Step 3: Bound on $N(\cB)$.} Let $m=|\cB|$. Any sequence $(\cB^{(0)},\ldots, \cB^{(l)})$ as in the definition of $N(\cB)$ can be completed into a sequence $\cB=\mathcal{C}^{(m)}\succ \mathcal{C}^{(m-1)}\succ \cdots \succ \mathcal{C}^{(1)}= \{[r]\}$ where each step merges exactly two groups (so that $|\mathcal{C}^{(j)}|=j$), by merging groups two at a time between consecutive $\cB^{(i)}$'s and after $\cB^{(l)}$. The original sequence is recovered from $(\mathcal{C}^{(m)},\ldots, \mathcal{C}^{(1)})$ and from the set $\{|\cB^{(0)}|,\ldots, |\cB^{(l)}|\}\subset [m]$. Since there are $\prod_{j=2}^{m}\binom{j}{2}= m!(m-1)!\,2^{1-m}$ such completed sequences and at most $2^{m-1}$ such sets (they all contain $m$), we get $N(\cB)\leq m!(m-1)!\leq r^{2r}$. Together with~\eqref{eq:induction:CP}, this concludes the proof.
\end{proof}

\begin{proof}[Proof of Lemma~\ref{lem:recursion:covariance}]
Without loss of generality, we only have to consider the cases where $B=[0;r]$ or $B=[r]$. We start by considering $B=[r]$ and we write $\cA$ for $\cA_B$.
We fix a connectivity event $\cF$ on $B$. Recall that $\overline{P}_a= P_{H_a,\tau_a}-\mathbb{E}[P_{H_a,\tau_a}]$ for $a\in [r]$.
Here, if $B$ is a singleton ($r=1$), then $\mathbb{E}_R[\mathbf{1}\{\cF\}\prod_{a\in B}\overline{P}_a]$ is zero as the polynomials are centered (recall that $\cF$ is of probability one in that case, and that $\overline{P}_a$ is also centered under $\mathbb{P}_R$, since $Z_a$ has the same distribution under $\mathbb{P}$ and under $\mathbb{P}_R$) and we assume henceforth that $|B|\geq 2$.
First, we expand the product
\begin{equation*}
\mathbb{E}_R\left[\mathbf{1}\{\cF\}\prod_{a\in [r]}\overline{P}_a\right]= \sum_{T\subset [r]}(-1)^{r-|T|}\prod_{a\in [r]\setminus T}\mathbb{E}\left[P_{H_a,\tau_a}\right] \mathbb{E}_R\left[\mathbf{1}\{\cF\}\prod_{a\in T} P_{H_a,\tau_a} \right]\enspace.
\end{equation*}
Given $z\in [n]^n$, we consider the restriction $z_{\mathrm{supp}}$ of $z$ to $\cup_{a=1}^r \tau_a(V_{H_a})$.
Importantly, for any configuration $z_{\mathrm{supp}}$, we have  $\mathbb{E}[ P_{H_a,\tau_a}|z_{\mathrm{supp}}]\geq 0$ since the probability of each edge in $Y^\star$ is at least $q$.
Since the polynomials $P_{H_a,\tau_a}$, $a\in [r]$, involve disjoint sets of nodes, they are independent conditionally on $z_{\mathrm{supp}}$, so that $\mathbb{E}[\prod_{a\in T} P_{H_a,\tau_a}|z_{\mathrm{supp}}]\geq 0$ for any $T\subset [r]$. Note that the conditional distribution of $Y$ given $z$ is the same under $\mathbb{P}$ and $\mathbb{P}_R$. Since $\cF\subset \cA$ and since both events only depend on $z_{\mathrm{supp}}$, we deduce that, for any $T\subset [r]$,
\begin{equation}\label{eq:reduction:F:A}
0\leq \mathbb{E}_R\left[\mathbf{1}\{\cF\}\prod_{a\in T} P_{H_a,\tau_a}\right]= \mathbb{E}_R\left[\mathbf{1}\{\cF\}\,\mathbb{E}\Big[\prod_{a\in T} P_{H_a,\tau_a}\Big|z_{\mathrm{supp}}\Big]\right]\leq \mathbb{E}_R\left[\mathbf{1}\{\cA\}\prod_{a\in T} P_{H_a,\tau_a}\right]\enspace .
\end{equation}
This is the only place where the event $\cF$ enters the proof: from now on, we only work with $\cA$.
Recall that $\widetilde{\mathbb{P}}$ refers to the distribution where all the $z_i$s are sampled independently.
Any configuration $z'_{\mathrm{supp}}$ that satisfies $\cA$ and arises with positive probability under $\mathbb{P}_R$ satisfies

\begin{equation*}
\mathbb{P}_{R}[z_{\mathrm{supp}}=z'_{\mathrm{supp}}]\leq \widetilde{\mathbb{P}}[z_{\mathrm{supp}}=z'_{\mathrm{supp}}]\left(1- \frac{4D}{n}\right)^{-4D}\leq 2\widetilde{\mathbb{P}}[z_{\mathrm{supp}}=z'_{\mathrm{supp}}]\enspace ,
\end{equation*}
where we used that at most $4D$ nodes are involved in  $\cup_{a=1}^r \tau_a(V_{H_a})$ (each of the two templates has at most $2D$ nodes) and used that $D^2/n$ is small enough by Condition \condsignal.
Combining the expansion above with~\eqref{eq:reduction:F:A}, and using  that, for any configuration $z$, $\widetilde{\mathbb{E}}\left[\mathbf{1}\{\mathcal{A}\}\prod_{a\in T} P_{H_a,\tau_a}|z \right]\geq 0$,   we obtain
\begin{align}\label{eq:upper:bound:E_R}
\left|\mathbb{E}_R\left[\mathbf{1}\{\cF\}\prod_{a\in [r]}\overline{P}_a\right]\right|\leq 2 \sum_{T\subset [r]}\prod_{a\in [r]\setminus  T}\mathbb{E}\left[P_{H_a,\tau_a}\right] \widetilde{\mathbb{E}}\left[\mathbf{1}\{\mathcal{A}\}\prod_{a\in T} P_{H_a,\tau_a} \right]\enspace . 
\end{align}
To control this term, we decompose according to the partition $\cT=(T_1,\ldots, T_{|\cT|})$ of $T$ into groups of polynomials (indexed by $T$) that are connected through their latent assignments. Writing $\cE_{\cT}$ for the event that $Z_i\cap Z_j=\emptyset$ for all $i$ and $j$ in different groups of $\cT$, and $\cA_{\cT}= \cap_{i=1}^{|\cT|} \cA_{T_i}$, we have $1=\sum_{\cT}\mathbf{1}\{\cE_{\cT}\}\mathbf{1}\{\cA_{\cT}\}$, where the sum runs over all partitions $\cT$ of $T$ (when $T=\emptyset$, there is only one term, with $|\cT|=0$). Fix such a partition $\cT$ and write $z_T$ for the latent assignments of the nodes involved in $\prod_{a\in T}P_{H_a,\tau_a}$. We claim that, on the event $\cE_{\cT}$,
\begin{equation}\label{eq:connectivity:outside:T}
\widetilde{\mathbb{P}}\left[\cA\,|\, z_T\right]\leq r^{r}\left(\frac{16D^2}{n}\right)^{|\cT|+r-|T|-1}\enspace .
\end{equation}
Indeed, consider the $m:=|\cT|+r-|T|$ super-nodes formed by the node sets of the groups $T_1,\ldots, T_{|\cT|}$ and by the node sets of the polynomials $P_{H_b,\tau_b}$, $b\in [r]\setminus T$. Link two super-nodes when they contain two nodes with the same latent assignment. Under $\cA$, the resulting graph is connected, so that it contains one of the at most $m^{m-2}\leq r^r$ spanning trees on $m$ vertices. Fix such a spanning tree and, for each of its $m-1$ edges, a pair of nodes, one in each of the two super-nodes, that should share the same latent assignment; there are at most $(4D)^2$ such pairs per edge, as at most $4D$ nodes are involved. These $m-1$ pairs of nodes form a forest, and each tree of this forest contains at most one node of each super-node (otherwise, the spanning tree would contain a cycle). On $\cE_{\cT}$, two nodes in two different groups of $\cT$ have different latent assignments, whereas the latent assignments of the nodes of the $P_{H_b,\tau_b}$, $b\in [r]\setminus T$, are independent of $z_T$ and uniformly distributed under $\widetilde{\mathbb{P}}$. Hence, conditionally on $z_T$, the probability that the $m-1$ pairs of nodes share the same latent assignments is at most $n^{-(m-1)}$. A union bound over the spanning trees and the pairs of nodes leads to~\eqref{eq:connectivity:outside:T}.
Since $\mathbf{1}\{\cE_{\cT}\}$, $\mathbf{1}\{\cA_{\cT}\}$, and $\mathbb{E}[\prod_{a\in T} P_{H_a,\tau_a}|z]$ only depend on $z_T$, and since $\mathbb{E}[\prod_{a\in T} P_{H_a,\tau_a}|z]\geq 0$, we deduce from~\eqref{eq:connectivity:outside:T} that
\begin{align*}
\widetilde{\mathbb{E}}\left[\mathbf{1}\{\mathcal{A}\}\prod_{a\in T} P_{H_a,\tau_a}\right]& = \sum_{\mathcal{T}} \widetilde{\mathbb{E}}\left[\mathbf{1}\{\cE_{\cT}\}\mathbf{1}\{\mathcal{A}_{\mathcal{T}}\}\,\mathbb{E}\Big[\prod_{a\in T} P_{H_a,\tau_a}\Big|z\Big]\,\widetilde{\mathbb{P}}\left[\cA\,|\, z_T\right]\right]\\
&\leq r^r\sum_{\mathcal{T}}  \left(\frac{16D^2}{n}\right)^{|\cT|+r-|T|-1} \widetilde{\mathbb{E}}\left[\mathbf{1}\{\mathcal{A}_{\mathcal{T}}\}\prod_{a\in T} P_{H_a,\tau_a}\right]\\
&=   r^r\sum_{\mathcal{T}}  \left(\frac{16D^2}{n}\right)^{|\cT|+r-|T|-1} \prod_{i=1}^{|\cT|} \widetilde{\mathbb{E}}\left[\1\{\mathcal{A}_{T_i}\}\prod_{a\in T_i} P_{H_a,\tau_a}\right]\enspace ,
\end{align*}
where we used the independence of the sampling design in the last line. Note that the exponent $|\cT|+r-|T|-1$ accounts for the fact that, under $\cA$, the polynomials indexed by $[r]\setminus T$ are also connected to the other ones.
Coming back to~\eqref{eq:upper:bound:E_R}, we arrive at
\begin{align}
\left|\mathbb{E}_R\left[\mathbf{1}\{\cF\}\prod_{a\in [r]}\overline{P}_a\right]\right|
& \leq 2^{r+1}r^{2r} \max_{T\subset [r]}\max_{\cT:\text{partition of }T}\left(\frac{16D^2}{n}\right)^{|\cT|+r-|T|-1} \prod_{a\in [r]\setminus  T}\mathbb{E}\left[P_{H_a,\tau_a}\right] \nonumber \\
&\quad \quad \quad \cdot \prod_{i=1}^{|\cT|} \widetilde{\mathbb{E}}\left[\1\{\mathcal{A}_{T_i}\}\prod_{a\in T_i} P_{H_a,\tau_a}\right]\enspace . \label{eq:upper_bound_E_R2}
\end{align}
To conclude, we rely on Conditions \condmoment and \condmomentWR. This leads us to the desired bound.
More precisely, fix $T$ and $\cT$ as in~\eqref{eq:upper_bound_E_R2}, and let $s$ (resp. $g$) be the number of groups of $\cT$ of size one (resp. of size at least two). The factors $\mathbb{E}[P_{H_a,\tau_a}]$ with $a\notin T$ and the factors $\widetilde{\mathbb{E}}[P_{H_a,\tau_a}]$ associated to the groups $\{a\}$ of size one are controlled by \condmoment (under $\mathbb{E}$ and under $\widetilde{\mathbb{E}}$, respectively), and do not bring any power of $n^{-1/2}$. By \condmomentWR, each group $T_i$ with $|T_i|\geq 2$ brings a factor $[c_0D^{c_1}/\sqrt{n}]^{|T_i|}$. Together with the factor $(16D^2/n)^{|\cT|+r-|T|-1}$, the resulting power of $n^{-1/2}$ is
$(|T|-s)+2(s+g+r-|T|-1)\geq r$. Indeed, if $g\geq 1$, this follows from $|T|\leq r$; if $g=0$, then $|T|=s$ and this power equals $2r-2\geq r$, since $r\geq 2$. Since $r\leq 2D$, the combinatorial factor $2^{r+1}r^{2r}$ is absorbed in $[c_0D^{c_1}/\sqrt{n}]^{r}$ by increasing the constants, and we obtain the bound $\varphi([r])$.

Let us turn to the case where $B=[0;r]$. The only difference is that the polynomial $\overline{P}_0$ is now involved. If $r=0$, that is $B=\{0\}$, then $\cF$ is of probability one and $\mathbb{E}_R[\overline{P}_0]=0$, since $Z_0$ has the same distribution under $\mathbb{P}$ and under $\mathbb{P}_R$. We assume henceforth that $r\geq 1$. Recall that
\begin{equation*}
\overline{P}_0: =  \overline{P}_{G^{(0)},\pi^{(0)}}\overline{P}_{G^{'(0)},\pi^{'(0)}} -  \mathbb{E}\left[\overline{P}_{G^{(0)},\pi^{(0)}}\overline{P}_{G^{'(0)},\pi^{'(0)}}\right] \enspace . 
\end{equation*}
Denote by $cc_0$ and $cc'_0$ the numbers of connected components of $G^{(0)}$ and $G'^{(0)}$ and write $(G^{(0;i)},\pi^{(0;i)})$ and  $(G^{'(0;i)},\pi^{'(0;i)})$ for the corresponding labeled connected components. Then, arguing as for~\eqref{eq:upper_bound_E_R2}, and using, as in~\eqref{eq:reduction:F:A}, that $0\leq \mathbb{E}_R[\mathbf{1}\{\cF\}Q]\leq \mathbb{E}_R[\mathbf{1}\{\cA\}Q]$ for any product $Q$ of polynomials among the $P_{H_a,\tau_a}$, $P_{G^{(0;i)},\pi^{(0;i)}}$, and $P_{G^{'(0;i)},\pi^{'(0;i)}}$ (in such a product, each edge appears at most twice, so that $\mathbb{E}[Q|z]\geq 0$), we get
\begin{align} \label{eq:upper_bound_E_R3}
& \left|\mathbb{E}_R\left[\mathbf{1}\{\cF\}\prod_{a\in [0;r]}\overline{P}_a\right]\right|\leq 2^{r+2}(r+1)^{2(r+1)}\left[ S_1 + S_2 \right]  \  ; \\  \nonumber
S_1&: =  \max_{T\subset [r]}\left|\mathbb{E}\left[\overline{P}_{G^{(0)},\pi^{(0)}}\overline{P}_{G^{'(0)},\pi^{'(0)}}\right]\right|\max_{\cT}\left(\frac{16D^2}{n}\right)^{|\cT|+r-|T|} \prod_{a\in [r]\setminus  T}\mathbb{E}\left[P_{H_a,\tau_a}\right] \\ \nonumber &\quad \quad \quad \cdot \prod_{i=1}^{|\cT|} \widetilde{\mathbb{E}}\left[\1\{\mathcal{A}_{T_i}\}\prod_{a\in T_i} P_{H_a,\tau_a}\right]  \  ;  \\ \nonumber
S_2&:=  \sum_{s=0}^{cc_0}\sum_{s'=0}^{cc'_0}D^{s+s'}\max_{\substack{L\subset [cc_0],\, |[cc_0]\setminus L|=s\\ L'\subset [cc'_0],\, |[cc'_0]\setminus L'|=s'}} \quad \max_{0\in T\subset [0;r]}\max_{\cT}\left(\frac{16D^2}{n}\right)^{|\cT|+r-|T|}\\  \nonumber &\quad \quad \quad \cdot \prod_{a\in [cc_0]\setminus L} \mathbb{E}\left[P_{G^{(0;a)},\pi^{(0;a)}}\right] \prod_{a\in [cc_0']\setminus L'} \mathbb{E}\left[P_{G^{'(0;a)},\pi^{'(0;a)}}\right] \prod_{a\in [r]\setminus  T}\mathbb{E}\left[P_{H_a,\tau_a}\right] \\  \nonumber
&\quad \quad \quad \cdot  \prod_{i: 0\notin T_i} \widetilde{\mathbb{E}}\left[\1\{\mathcal{A}_{T_i}\}\prod_{a\in T_i} P_{H_a,\tau_a} \right] 
\\ \nonumber &\quad \quad \quad \cdot \prod_{i: 0\in T_i} \widetilde{\mathbb{E}}\left[\1\{\mathcal{A}_{T_i}\}\prod_{a\in T_i\setminus \{0\}} P_{H_a,\tau_a}\prod_{a\in L} P_{G^{(0;a)},\pi^{(0;a)}} \prod_{a\in L'}P_{G^{'(0;a)},\pi^{'(0;a)}}   \right] \enspace .
\end{align}
Here, $\cT$ is a partition of $T$, and the factor $(16D^2/n)^{|\cT|+r-|T|}$ is obtained as in~\eqref{eq:connectivity:outside:T}, the super-nodes being now formed by the groups of $\cT$ and by the indices in $[0;r]\setminus T$, so that there are $|\cT|+r+1-|T|$ of them. In $S_1$, $\overline{P}_0$ has been replaced by a constant and the latent assignments $Z_0$ form one of these super-nodes. In $S_2$, the super-node containing $0$ gathers all the nodes of $G^{(0)}$ and $G^{'(0)}$ and those of the $P_{H_a,\tau_a}$ with $a$ in the same group as $0$; we then condition on all the latent assignments of this super-node.
All the terms in $S_1$ and $S_2$ are straightforwardly controlled using Conditions \condmoment, \condvariance, and \condmomentWR, except for the last expression
\begin{equation*}
\prod_{i: 0\in T_i} \widetilde{\mathbb{E}}\left[\1\{\mathcal{A}_{T_i}\}\prod_{a\in T_i\setminus\{0\}} P_{H_a,\tau_a}\prod_{a\in L} P_{G^{(0;a)},\pi^{(0;a)}} \prod_{a\in L'}P_{G^{'(0;a)},\pi^{'(0;a)}}  \right]\enspace . 
\end{equation*}
Let $T_0$ be the group of $\cT$ that contains $0$, let $W:=T_0\setminus\{0\}$, and let 
\begin{equation*}
Q:= \prod_{a\in W} P_{H_a,\tau_a}\prod_{i\in L} P_{G^{(0;i)},\pi^{(0;i)}} \prod_{i\in L'}P_{G^{'(0;i)},\pi^{'(0;i)}}
\end{equation*}
be the product of polynomials in the last expression. Write $N$ for the set of nodes of $G^{(0)}$ and $G^{'(0)}$ that are not involved in $Q$, that is the nodes of the left-out components $G^{(0;i)}$, $i\notin L$, and $G^{'(0;i)}$, $i\notin L'$, that do not belong to any $G^{(0;i)}$, $i\in L$, or $G^{'(0;i)}$, $i\in L'$. The event $\cA_{T_0}$ also involves the latent assignments of the nodes in $N$, and is therefore not of the form required in~\condmomentWR. 

\smallskip

\noindent
\emph{Decomposition of $Q$.} Consider the configuration formed by the labeled graphs $H_a$, $a\in W$, $G^{(0;i)}$, $i\in L$, and $G^{'(0;i)}$, $i\in L'$, the matching being the restriction of $\mathbf{M}$, and write $\tilde{G}_{\Delta}$ for the corresponding symmetric difference graph. Besides the $H_a$, $a\in W$, which we call \emph{original} pure components, this configuration may have \emph{new} pure connected components: a component $G^{(0;i)}$, $i\in L$ (resp. $G^{'(0;i)}$, $i\in L'$), is pure when all the nodes it shares with $G^{'(0)}$ (resp. $G^{(0)}$) belong to left-out components. Let $\tilde{r}$ be the number of pure connected components of this configuration, let $\tilde{P}_1,\ldots,\tilde{P}_{\tilde{r}}$ be the corresponding polynomials, and let $\tilde{P}_0$ be the product of the polynomials of the other components (with $\tilde{P}_0=1$ if there is none), so that $Q= \tilde{P}_0 \tilde{P}_1\ldots \tilde{P}_{\tilde{r}}$, these polynomials involving disjoint sets of nodes. Let $S$ be the set of nodes involved in $Q$ and $z_S$ their latent assignments. For $B\subset [0;\tilde{r}]$, let $\widetilde{\cA}_{B}$ be the event that the polynomials $\tilde{P}_l$, $l\in B$, are connected through their latent assignments, and let $\tilde{\cB}_{z_S}$ be the partition of $[0;\tilde{r}]$ into the corresponding connected components. For a partition $\tilde{\cB}$ of $[0;\tilde{r}]$, we write $u(\tilde{\cB})$ for the number of groups of $\tilde{\cB}$ that only contain original pure components.

\smallskip

\noindent
\emph{Connection to the left-out nodes.} We claim that 
\begin{equation}\label{eq:connection:left:out}
\widetilde{\mathbb{P}}\left[\cA_{T_0}\,|\, z_S\right]\leq \left(\frac{16D^2}{n}\right)^{u(\tilde{\cB}_{z_S})}\enspace . 
\end{equation}
Indeed, let $B$ be a group of $\tilde{\cB}_{z_S}$ that only contains original pure components. By definition of $\tilde{\cB}_{z_S}$, the nodes of the polynomials indexed by $B$ do not share any latent assignment with the other nodes of $S$. Hence, under $\cA_{T_0}$, one of these nodes shares its latent assignment with a node of $N$. Besides, two different groups of $\tilde{\cB}_{z_S}$ have disjoint sets of latent assignments, so that they are connected to different nodes of $N$. Since, conditionally on $z_S$, the latent assignments of the nodes in $N$ are independent and uniformly distributed under $\widetilde{\mathbb{P}}$, a union bound over the choice of these $u(\tilde{\cB}_{z_S})$ nodes of $N$ and of the corresponding nodes in the groups (at most $4D$ choices for each of them) leads to~\eqref{eq:connection:left:out}. Since $\mathbb{E}[Q|z]$ is non-negative and only depends on $z_S$, and since $\mathbf{1}\{\tilde{\cB}_{z_S}=\tilde{\cB}\}\leq \prod_{B\in \tilde{\cB}}\mathbf{1}\{\widetilde{\cA}_{B}\}$, we deduce from~\eqref{eq:connection:left:out} and from the independence of the sampling design that 
\begin{align}\nonumber
\widetilde{\mathbb{E}}\left[\1\{\mathcal{A}_{T_0}\} Q \right] & = \sum_{\tilde{\cB}} \widetilde{\mathbb{E}}\left[\mathbf{1}\{\tilde{\cB}_{z_S}=\tilde{\cB}\}\, \mathbb{E}[Q|z]\, \widetilde{\mathbb{P}}\left[\cA_{T_0}\,|\, z_S\right]\right]\\
&\leq \sum_{\tilde{\cB}\text{ partition of }[0;\tilde{r}] }\left(\frac{16D^2}{n}\right)^{u(\tilde{\cB})}\prod_{B\in \tilde{\cB}} \widetilde{\mathbb{E}}\left[\1\{\widetilde{\mathcal{A}}_{B}\}  \prod_{l\in B}\tilde{P}_{l} \right]\enspace .\label{eq:decomposition:group:0}
\end{align}

\smallskip

\noindent
\emph{Control of each factor.} In~\eqref{eq:decomposition:group:0}, the factors with $B=\{l\}$, $l\geq 1$, are controlled by~\condmoment\ under $\widetilde{\mathbb{E}}$. If $B=\{0\}$ and the components involved in $\tilde{P}_0$ are perfectly matched, then $\tilde{P}_0=P_{H}^2$ for some labeled graph $H=(V_H,E_H)$, and $\widetilde{\mathbb{E}}[\tilde{P}_0]\leq \overline{p}^{|E_H|}$ since $\mathbb{E}[Y_{ij}^2|z]= \overline{q}+\Theta_{z_iz_j}(1-2q)\leq \overline{p}$. All the other factors are controlled by~\condmomentWR, applied to the configuration formed by the components indexed by $B$. Here, as in the case $B=[r]$, \condmomentWR\ is applied to sub-configurations of $(G^{(1)},\pi^{(1)})$ and $(G^{(2)},\pi^{(2)})$, endowed with the restriction of the matching $\mathbf{M}$. In particular, each pure connected component of $\tilde{G}_{\Delta}$ that belongs to a group $B$ with $|B|\geq 2$, or to the group containing $0$, brings a factor $c_0D^{c_1}/\sqrt{n}$. 

\smallskip

\noindent
\emph{Bookkeeping.} We gather these bounds for a fixed term of $S_2$, that is for fixed $L$, $L'$, $T$, $\cT$, and $\tilde{\cB}$.
\begin{itemize}
\item[(a)] Each original pure component $a\in W$ brings at least one factor $n^{-1/2}$: by~\condmomentWR, when its group in $\tilde{\cB}$ has at least two elements or contains $0$, and through the factor $16D^2/n$ in~\eqref{eq:decomposition:group:0} otherwise. 
\item[(b)] For the part associated to $G^{(0)}$ and $G^{'(0)}$, we argue as in the proof of Lemma~\ref{lem:gen-1}: the left-out components are controlled by~\condmoment, and the other ones by the above bounds. Since $D^{c_1}\lambda\leq \overline{p}\leq 1$, the edges bring at most $(D^{c_1}\lambda)^{|E'_{\Delta}|}\overline{p}^{|E'_{\cap}|}$, whereas the nodes bring $(D^{c_1}k/n)^{R}$, where $R$ is defined as in~\eqref{eq:definition:R} for $(G^{(0)},G^{'(0)})$ with $S_1=[cc_0]\setminus L$ and $S_2=[cc'_0]\setminus L'$. As in the proof of Lemma~\ref{lem:gen-1}, we have $R\geq s+s'$, which allows us to absorb the factor $D^{s+s'}$ in the definition of $S_2$, and $R\geq |V'_{\Delta}|-\#\mathrm{CC}'_{\Delta}$ by Lemma~\ref{lem:lower:bound:R}.
\item[(c)] The number of partitions $\tilde{\cB}$ and the constants arising in~\condmomentWR\ for the groups of $\tilde{\cB}$ are at most $(cD^{c})^{\tilde{r}+1}$, for some constant $c$. The part of this factor associated to the original pure components is absorbed in the factors $c_0D^{c_1}$ associated to them in $\varphi([0;r])$. Each new pure component is a component $G^{(0;i)}$ or $G^{'(0;i)}$ whose edges either belong to $E'_{\Delta}$, or are shared with a left-out component. In the first case, the edge brings $D^{c}\lambda$ instead of $D^{c_1}\lambda$ in~(b); in the second case, together with the left-out component, it brings $(D^c\lambda)^2\leq 2D^{2c-16c_{\texttt{s}}}\overline{p}$ by~\condsignal\ and since $\overline{p}\geq q/2$. Hence, each edge of a new pure component carries a spare factor, which absorbs the corresponding constant provided that $c_1$ and $c_{\texttt{s}}$ are large enough. 
\end{itemize}
 As in the case $B=[r]$, one then checks that, in $S_1$ and in $S_2$, the resulting power of $n^{-1/2}$ is at least $r=|B|-1$: with $t$ the number of groups of $\cT$ of the form $\{a\}$ with $a\neq 0$, this power is at least $(|T|-t)+2(|\cT|+r-|T|)\geq r$ in $S_1$ (where $|T|\leq r$), and at least $(|T|-1-t)+2(|\cT|+r-|T|)\geq r$ in $S_2$ (where $|T|\leq r+1$ and $|\cT|\geq t+1$), since, by~(a), the group containing $0$ brings a factor $n^{-1/2}$ for each original pure component it contains. Putting everything together and coming back to~\eqref{eq:upper_bound_E_R3} concludes the proof.


\end{proof}

\subsection{Proofs of technical lemmas}\label{sec:proof:technical:lemmas}

\begin{proof}[Proof of Lemma~\ref{lem:mean:Psi}]
Let $G=(V,E)\in\mathcal{G}_{\leq D}$ with connected components $G_1,\ldots,G_r$, and let $\pi\in\Pi_V$. By definition, $\overline{P}_{G,\pi}=\prod_{l=1}^r\left[P_{G_l,\pi}-\mathbb{E}[P_{G_l,\pi}]\right]$. If $r=1$, we have $\mathbb{E}[\overline{P}_{G,\pi}]=0$. Under \condinvariance, the factors of this product are independent, as they involve disjoint sets of nodes, so that $\mathbb{E}[\overline{P}_{G,\pi}]=0$. 

Let us turn to \condinvarianceWR\ with $r\geq 2$. Then, $\mathbb{E}[\overline{P}_{G,\pi}]$ is a quantity of the form $\mathbb{E}[\prod_{a=1}^r\overline{P}_a]$ considered in the proof of Lemma~\ref{lem:gen-2}, for the configuration made of $(G,\pi)$ and of the empty graph, all the connected components of $G$ being pure. The proof of~\eqref{eq:obj_A} only relies on Lemmas~\ref{lem:recursion:covariance} and~\ref{lem:recursion:CP} for subsets and partitions of $[r]$, and therefore only on \condmoment, on \condsignal, and on \condmomentWR\ applied to groups of at least two connected components of $G$, which is allowed when the other graph is empty. Hence, \eqref{eq:obj_A} holds, that is 
\begin{equation*}
\left|\mathbb{E}[\overline{P}_{G,\pi}]\right|\leq c_{0}D^{c_{1}}\left(D^{c_1}\lambda\right)^{|E|}\left(D^{c_{1}}\frac{k}{n}\right)^{|V|-r}\left[c_0\frac{D^{c_1}}{\sqrt{n}}\right]^{r}\enspace . 
\end{equation*}
Since $\mathbb{E}[\Psi_G]=\mathbb{V}(G)^{-1/2}\sum_{\pi\in\Pi_V}\mathbb{E}[\overline{P}_{G,\pi}]$ and $|\Pi_V|= n!/(n-|V|)!\leq n^{|V|}$, the definition~\eqref{eqn:variance of graph} of $\mathbb{V}(G)$ leads to 
\begin{equation*}
\left|\mathbb{E}[\Psi_G]\right|\leq n^{|V|/2}\,\overline{q}^{-|E|/2}\max_{\pi\in\Pi_V}\left|\mathbb{E}[\overline{P}_{G,\pi}]\right|\leq (c_0D^{c_1})^{r+1}\left(\frac{D^{c_1}\lambda}{\sqrt{\overline{q}}}\right)^{|E|}\left(\frac{D^{c_1}k}{\sqrt{n}}\right)^{m}\enspace , \quad \text{where }m:=|V|-r\enspace . 
\end{equation*}
Since each connected component of $G$ has at least two nodes and is connected, we have $r\leq m\leq |E|$, so that $r+1\leq 2|E|$. Since $\overline{q}\geq q/2$ (recall that $q\leq 1/2$) and $\lambda\leq 1$, Condition~\condsignal\ implies that $\lambda/\sqrt{\overline{q}}\leq \sqrt{2\lambda}\sqrt{\lambda/q}\leq \sqrt{2}D^{-4c_{\texttt{s}}}$ and $\lambda k/\sqrt{n\overline{q}}\leq \sqrt{2}D^{-8c_{\texttt{s}}}$. As a consequence, 
\begin{equation*}
\left|\mathbb{E}[\Psi_G]\right|\leq (c_0D^{c_1})^{2|E|}\left(\frac{D^{2c_1}\lambda k}{\sqrt{n\overline{q}}}\right)^{m}\left(\frac{D^{c_1}\lambda}{\sqrt{\overline{q}}}\right)^{|E|-m}\leq \left(\sqrt{2}c_0^2D^{4c_1-4c_{\texttt{s}}}\right)^{|E|}\leq D^{-3c_{\texttt{s}}|E|}\enspace ,
\end{equation*}
provided that $c_{\texttt{s}}$ is large enough compared to $c_0$ and $c_1$. 

For $G\in\mathcal{G}^{(1,2)}_{\leq D}$, the same argument applies to the $r$ non-trivial connected components of $G$, with $|\Pi^{(1,2)}_V|=(n-2)!/(n-|V|)!\leq n^{|V|-2}$ and the definition of $\mathbb{V}^{(1,2)}(G)$. Writing $V_{\mathrm{nt}}$ for the set of nodes of the non-trivial connected components, we have $|V|-2\leq |V_{\mathrm{nt}}|$, so that the power of $n$ is at most $(|V|-2)/2-(|V_{\mathrm{nt}}|-r)-r/2\leq -(|V_{\mathrm{nt}}|-r)/2$, and we conclude as above with $m=|V_{\mathrm{nt}}|-r$. 
\end{proof}

\begin{proof}[Proof of Lemma~\ref{lem:countgra}]
Each connected component of $G'_{\Delta}$ contains at least one matched node. This node cannot be perfectly matched, otherwise it would not arise in $G'_{\Delta}$. As a consequence, we have $|\mathbf{M}_{\mathrm{SM}}|\geq \#\mathrm{CC}'_{\Delta}$. Besides, $G'_\Delta$ satisfies $|E'_{\Delta}|\geq |V'_{\Delta}|-  \#\mathrm{CC}'_{\Delta}$. By the previous inequality this enforces that $|E'_{\Delta}|\geq |V'_{\Delta}| - |\mathbf{M}_{\mathrm{SM}}| = U'$. 
\end{proof}

\begin{proof}[Proof of Lemma~\ref{lem:lower:bound:R}]
Recall that, for any graph $H=(V_H,E_H)$ without isolated nodes, any forest made of edges of $H$ has at most $|V_H|-\#\mathrm{CC}_H$ edges, with equality for a spanning forest of $H$. Let $F$ be a spanning forest of $G_{\Delta}$, so that $|F|=|V_{\Delta}|-\#\mathrm{CC}_{\Delta}$. Each edge of $F$ is an unmatched edge of $\pi^{(1)}[G^{(1)}]$ or of $\pi^{(2)}[G^{(2)}]$. Either it belongs to a removed component $\pi^{(a)}[G^{(a)}_i]$ with $i\in S_a$ and $a\in\{1,2\}$, or it remains unmatched once the components indexed by $S_1$ and $S_2$ have been removed, so that it belongs to $E_{\Delta,-S_1,-S_2}$. Since any subset of $F$ is a forest, $F$ contains at most $|V^{(a)}_i|-1$ edges of each removed component $\pi^{(a)}[G^{(a)}_i]$, and at most $|V_{\Delta,-S_1,-S_2}| - \#\mathrm{CC}_{\Delta, -S_1, -S_2}$ edges of $E_{\Delta,-S_1,-S_2}$. Summing these bounds, we obtain $|V_{\Delta}|-\#\mathrm{CC}_{\Delta}=|F|\leq R$. Note that this argument does not require $\mathbf{M}$ to belong to $\mathcal{M}^{\star}$.
\end{proof}

\begin{proof}[Proof of Lemma~\ref{lem:shadow}]
Fix $\overline{U}_1$, $\overline{U}_2$, and $\underline{\mathbf{M}}$. Then, we define $V^{(1)}_{\mathrm{PM}}$ (resp. $V^{(2)}_{\mathrm{PM}}$) as the set of perfectly matched nodes. By construction, $V^{(1)}_{\mathrm{PM}}$ is the set of nodes of $V^{(1)}$ that are neither  in $\overline{U}_1$ nor in $\underline{\mathbf{M}}$. 
Fix a matching  $\mathbf{M}^{(0)}$ in $\mathcal{M}_{\text{shadow}}(\overline{U}_1, \overline{U}_2, \underline{\mathbf{M}})$ and consider the corresponding bijection $\phi^{(0)}: V^{(1)}_{\mathrm{PM}}\to V^{(2)}_{\mathrm{PM}}$ defined by $\mathbf{M}^{(0)}_{\mathrm{PM}}=\{(v,\phi^{(0)}(v)): v\in V^{(1)}_{\mathrm{PM}}\}$. Now, consider any other matching $\mathbf{M}'$ in $\mathcal{M}_{\text{shadow}}(\overline{U}_1, \overline{U}_2, \underline{\mathbf{M}})$ and build similarly the bijection $\phi':V^{(1)}_{\mathrm{PM}}\to V^{(2)}_{\mathrm{PM}}$ such that  $\mathbf{M}'_{\mathrm{PM}}=\{(v,\phi'(v)): v\in V^{(1)}_{\mathrm{PM}}\}$.  Then, we can define the bijection $\varphi': V^{(1)}\to V^{(1)}$ such that $\varphi'(v)=v$ if $v\notin V^{(1)}_{\mathrm{PM}}$ and $\varphi'(v)= (\phi^{(0)})^{-1}(\phi'(v))$ otherwise. 
 We claim that $\varphi'$ is an automorphism of $G^{(1)}$. Let us conclude the proof before establishing the claim. 
Any two distinct $\mathbf{M}'$ and $\mathbf{M}''$  in $\mathcal{M}_{\text{shadow}}(\overline{U}_1, \overline{U}_2, \underline{\mathbf{M}})$ lead to distinct automorphisms $\varphi'$ and $\varphi''$. Thus, we get
\begin{equation*}
\left|\mathcal{M}_{\text{shadow}}(\overline{U}_1, \overline{U}_2, \underline{\mathbf{M}})\right|\leq |\mathrm{Aut}(G^{(1)})| \enspace . 
\end{equation*}
By symmetry, we also conclude that the cardinality is at most $|\mathrm{Aut}(G^{(2)})|$.

Let us prove the claim. Since $\varphi'$ is a bijection of $V^{(1)}$, it suffices to show that it maps edges of $G^{(1)}$ to edges of $G^{(1)}$. Consider any edge $(v_1,v_2)$ in $G^{(1)}$. If neither $v_1$ nor $v_2$ belongs to $V^{(1)}_{\mathrm{PM}}$, then $(\varphi'(v_1),\varphi'(v_2))= (v_1,v_2)$ are connected in $G^{(1)}$.  If $v_1$ belongs to  $V^{(1)}_{\mathrm{PM}}$ and $v_2$ does not belong to $V^{(1)}_{\mathrm{PM}}$, then it follows that $v_2$ is semi-matched (it cannot be unmatched, as the edge $(v_1,v_2)$ is matched since $v_1$ is perfectly matched), i.e.\ there exists $w\in V^{(2)}$ such that $(v_2,w)\in \mathbf{M}_{\mathrm{SM}}= \underline{\mathbf{M}}$. Since $v_1$ is perfectly matched in $\mathbf{M}'$, it follows that $(\phi'(v_1),w)$ are connected in  $G^{(2)}$. By the same argument applied to $\mathbf{M}^{(0)}$, in which $(\phi^{(0)})^{-1}(\phi'(v_1))$ is perfectly matched to $\phi'(v_1)$ and $v_2$ is matched to $w$, we deduce that $((\phi^{(0)})^{-1}(\phi'(v_1)),v_2)=(\varphi'(v_1),\varphi'(v_2))$ are connected in  $G^{(1)}$. Finally, we consider the case where both $v_1$ and $v_2$ belong to $V^{(1)}_{\mathrm{PM}}$. Since both are perfectly matched in $\mathbf{M}'$, $(\phi'(v_1),\phi'(v_2))$ are connected in $G^{(2)}$ and they belong to $V^{(2)}_{\mathrm{PM}}$. Repeating again the argument with $\mathbf{M}^{(0)}$, we conclude that $(\varphi'(v_1),\varphi'(v_2))=((\phi^{(0)})^{-1}(\phi'(v_1)),(\phi^{(0)})^{-1}(\phi'(v_2)))$ are connected in $G^{(1)}$.

\end{proof}

\section{Proof of Proposition~\ref{prp:model:conditions}}\label{sec:proof:model:conditions}

\subsection{\texttt{Independent sampling} (Condition~\ref{as:perinv} in the main text)}
We start with some notation and general computations. Consider any two templates $G^{(1)}=(V^{(1)},E^{(1)})$, $G^{(2)}=(V^{(2)},E^{(2)})$ in $\mathcal{G}_{\leq D}$ and any two labelings $\pi^{(1)}$ and $\pi^{(2)}$. Write $G_{\cup}=(V_{\cup},E_{\cup})$ for the merged labeled graph of $\pi^{(1)}[G^{(1)}]$ and $\pi^{(2)}[G^{(2)}]$ and $G_{\Delta}=(V_{\Delta},E_{\Delta})$ for the associated labeled symmetric difference graph ---\ see Section~\ref{sec:graph:definition} for definitions. We may decompose the product of polynomials
\begin{align} \nonumber
P_{G^{(1)}, \pi^{(1)}} P_{G^{(2)}, \pi^{(2)}}&= \prod_{(i,j)\in E^{(1)}} Y_{\pi^{(1)}(i)\pi^{(1)}(j)} \prod_{(i,j)\in E^{(2)}} Y_{\pi^{(2)}(i)\pi^{(2)}(j)}\\
& =  \prod_{(i,j)\in E_{\Delta}} Y_{ij} \prod_{(i,j)\in  E_{\cap} } Y_{ij}^2 \enspace . \label{eq:product}
\end{align}
Recall that, given $z$, the $(Y_{ij})_{1\leq i< j\leq n}$ are independent with $\mathbb{P}[Y_{ij}= (1-q)|z]= q+\Theta_{z_iz_j}$ and $\mathbb{P}[Y_{ij}= -q|z]= 1-q-\Theta_{z_iz_j}$ ---\ see~\eqref{eq:definition:sample:graph} and~\eqref{eq:definition:Y}. In particular, we have 
\begin{align}\label{eq:conditional:expectation}
\mathbb{E}\left[P_{G^{(1)}, \pi^{(1)}} P_{G^{(2)}, \pi^{(2)}}|z\right]= \prod_{(i,j)\in E_{\Delta}} \Theta_{z_iz_j}\prod_{(i,j)\in  E_{\cap}} \left[(1-q)^2(q+\Theta_{z_iz_j})+ q^2(1-q-\Theta_{z_iz_j})\right]\enspace .  
\end{align}
In all the problems  that we consider in this subsection, we have that $\Theta$ only takes two values: $0$ or $\lambda=p-q>0$. Recall the definitions of $\overline{p} = p(1-q)^2 + (1-p)q^2$
and $\overline{q} = q(1 - q)$. Since $\overline{p}=\overline{q}+ (p-q)(1-2q)$ and since we assume that $q\leq 1/2$, it follows that $\overline{p}\geq \overline{q}$. In this specific case, the identity~\eqref{eq:conditional:expectation} simplifies to 
\begin{align}\label{eq:conditional:expectation_pq}
\mathbb{E}\left[P_{G^{(1)}, \pi^{(1)}} P_{G^{(2)}, \pi^{(2)}}|z\right]=  \lambda^{|E_{\Delta}|}\prod_{(i,j)\in E_{\Delta}} \mathbf{1}_{\Theta_{z_iz_j}\neq 0} \prod_{(i,j)\in  E_{\cap}}\left[\overline{q}\1_{\Theta_{z_iz_j}=0} + \overline{p}\1_{\Theta_{z_iz_j}\neq 0} \right] \enspace .  
\end{align}
We shall build upon this identity to establish the different conditions for each model.
Similarly, we have the following formula.
\begin{equation}\label{eq:conditional:expectation_pq:first_moment}
\mathbb{E}\left[P_{G^{(1)}, \pi^{(1)}}|z\right]= \lambda^{|E^{(1)}|}\prod_{(i,j)\in E^{(1)}} \mathbf{1}\{\Theta_{z_{\pi^{(1)}(i)}z_{\pi^{(1)}(j)}}\neq 0\}\enspace .
\end{equation} 
Let us turn to checking our assumptions for the three models.

\subsubsection{Hidden subclique model \HidSubI}

 Consider a template $G=(V,E)$ and a labeling $\pi$. In light of~\eqref{eq:conditional:expectation_pq:first_moment}, the conditional expectation of $P_{G,\pi}$ given $z$ is non-zero if and only if $z_{\pi(i)}\leq k$ for all $i\in V$. This leads us to
\begin{equation}\label{eq:first_moment_HS}
 \mathbb{E}\left[P_{G, \pi}\right] =  \lambda^{|E|} \mathbb{E}\left[\prod_{i\in V} \1\{z_{\pi(i)}\leq k\}\right]\enspace , 
\end{equation}
where $\mathbb{E}[\prod_{i\in V} \1\{z_{\pi(i)}\leq k\}]= \left(\frac{k}{n}\right)^{|V|}$. Hence, Condition~\condmoment holds  with $c_{\texttt{m}}=0$.

Consider any two templates $G^{(1)}$, $G^{(2)}$ and any two labelings $\pi^{(1)}$, $\pi^{(2)}$. 
Recall that $G_{\cup}$ (resp. $G_{\cap}$, $G_{\Delta}$) stands for the merged (resp. intersection, resp. symmetric difference) graph of $\pi^{(1)}[G^{(1)}]$ and $\pi^{(2)}[G^{(2)}]$. We integrate~\eqref{eq:conditional:expectation_pq} over all  latent assignments $z$ that lead to a non-zero conditional expectation in~\eqref{eq:conditional:expectation_pq}. First, observe that this requires $z_i\leq k$ for all $i$  in $V_{\Delta}$. Write $V_{\mathrm{PM}}$ for the set of perfectly matched nodes
and $V_{\mathrm{SM}}$ for the set of semi-matched nodes, that is nodes that are matched but not perfectly matched ---\ see Section~\ref{sec:graph:definition}. Since templates have no isolated nodes, a node of $G_{\cup}$ is perfectly matched if and only if all its incident edges belong to $E_{\cap}$. Hence, $(V_{\mathrm{PM}},V_{\Delta})$ form a partition of $V_{\cup}$ and $V_{\mathrm{SM}}\subset V_{\Delta}$. Besides, $V_{\cap}\subset V_{\mathrm{PM}}\cup V_{\mathrm{SM}}$, but a semi-matched node does not necessarily belong to $V_{\cap}$. As a consequence, to compute $\mathbb{E}\left[P_{G^{(1)}, \pi^{(1)}} P_{G^{(2)}, \pi^{(2)}}\right]$, we have to sum over all possible configurations for $(z_i)_{i\in V_{\mathrm{PM}}}$. For any subset $T\subseteq V_{\cup}$ of nodes,  we write $E_\cap [T]$ for the set of edges of $E_{\cap}$ with both endpoints in $T$. 
\begin{equation}\label{eq:second_moment_HS_R}
            \mathbb{E}\left[P_{G^{(1)}, \pi^{(1)}} P_{G^{(2)}, \pi^{(2)}}\right] =   \lambda^{|E_{\Delta}|}   \mathbb{E}[\prod_{i\in V_{\Delta}} \1\{z_i\leq k\}] \sum_{S \subseteq  V_{\mathrm{PM}}} \overline{q}^{| E_{\cap}| - |E_{\cap}[S \cup V_{\mathrm{SM}}]|} \overline{p}^{| E_{\cap}[ S\cup V_{\mathrm{SM}}]|} \mathbb{P}[\{i: z_i\leq k\} \cap V_{\mathrm{PM}} = S ]  \enspace ,  
\end{equation}
where, in the random size model \HidSubI, we have  $\mathbb{P}[\{i: z_i\leq k\} \cap V_{\mathrm{PM}} = S ]= \left(\frac{k}{n}\right)^{|S|}\left(1 - \frac{k}{n}\right)^{|V_{\mathrm{PM}}| - |S|}$.  Noting that $\overline{p} \geq \overline{q}$, we get the following bound
\begin{equation}\label{eq:second_moment_HS_R_bound}
\overline{q}^{|E_{\cap}|}  \lambda^{|E_{\Delta}|}  \mathbb{E}[\prod_{i\in V_{\Delta}} \1\{z_i\leq k\}] \leq \mathbb{E}\left[P_{G^{(1)}, \pi^{(1)}} P_{G^{(2)}, \pi^{(2)}}\right] \leq  \overline{p}^{|E_{\cap}|}  \lambda^{|E_{\Delta}|}  \mathbb{E}[\prod_{i\in V_{\Delta}} \1\{z_i\leq k\}]\enspace , 
\end{equation}
where $\mathbb{E}[\prod_{i\in V_{\Delta}} \1\{z_i\leq k\}]= \left(k/n\right)^{|V_{\Delta}| }$.
As a consequence, \HidSubI\ satisfies the first part of Condition~\condvariance with $c_{\texttt{v},1}=0$ and $c_{\texttt{v},2}=1$. Next, we consider the case where $(G^{(1)},\pi^{(1)})= (G^{(2)},\pi^{(2)})$ so that $E_{\Delta}=\emptyset$. By Condition~\condsignal , we have 
$\lambda\leq q D^{-8c_{\texttt{s}}}$, so that $\overline{p}= \overline{q}+ \lambda(1-2q)\leq \overline{q}(1+D^{-8c_{\texttt{s}}})$, since $\lambda(1-2q)\leq \lambda(1-q)$ and $\lambda(1-q)\leq q(1-q)D^{-8c_{\texttt{s}}}= \overline{q}D^{-8c_{\texttt{s}}}$. It follows from~\eqref{eq:second_moment_HS_R_bound}  that
\[
\left|\mathbb{E}\left[P^2_{G,\pi}\right]- \overline{q}^{|E|}\right|\leq \overline{q}^{|E|}\left[\left(1+ D^{-8c_{\texttt{s}}} \right)^{D}-1\right]\leq  2 D^{1-8c_{\texttt{s}}}\overline{q}^{|E|}\enspace ,
\]
since $(1+x)^D-1\leq e^{Dx}-1\leq 2Dx$ for $0\leq Dx\leq 1$.
The second part of Condition~\condvariance is therefore satisfied with $c_{\texttt{v},3}=2$, since $1-8c_{\texttt{s}}\leq -c_{\texttt{s}}$.

\subsubsection{Stochastic Block Model \StoBloI}
As before, we first work out the first moment of the polynomials. In order to have $\mathbb{E}\left[P_{G, \pi}|z\right]\neq 0$, it is necessary that $\Theta_{z_{\pi(i)}z_{\pi(j)}}$ is non-zero for all edges $(i,j)$ of $G$. As a consequence, all nodes in a connected component of $\pi[G]$ should belong to the same group of the SBM, so that 
\begin{equation}\label{eq:first_moment_SBM}
 \mathbb{E}\left[P_{G, \pi}\right] =  \lambda^{|E|}\mathbb{P}[\{z \text{ in the same block over each CC of $G$}\}]\enspace , 
\end{equation}
where, for \StoBloI, we have $\mathbb{P}[\{z \text{ in the same block over each CC}\}]= \left(\frac{k}{n}\right)^{|V|- \#\mathrm{CC}_{G}}$.  Hence, Condition~\condmoment holds with $c_{\texttt{m}}=0$.

Let us turn to the second moment. Coming back to~\eqref{eq:conditional:expectation_pq}, we see that the conditional expectation of $P_{G^{(1)}, \pi^{(1)}}P_{G^{(2)}, \pi^{(2)}}$ is non-zero only if, inside each connected component of $G_{\Delta}$, all the nodes belong to the same group of the SBM. Write $R(z)$ for the partition of $V_{\cup}$ associated to groups of the SBM and write $R_0$ for the finest partition of $V_{\cup}$ such that each connected component of $G_{\Delta}$ is included in a group of $R_0$. For any two partitions $R_1$ and $R_2$, we write $R_1\preceq R_2$ if $R_2$ is finer than or equal to $R_1$. Finally, we write $i\stackrel{R}{\sim} j$ when $i$ and $j$ are in the same group according to the partition $R$. Then, we have 
\begin{equation}\label{eq:second_moment_SBM}
 \mathbb{E}\left[P_{G^{(1)}, \pi^{(1)}}P_{G^{(2)}, \pi^{(2)}}\right] =  \lambda^{|E_\Delta|}\sum_{R\preceq R_0}\mathbb{P}[R(z)= R]  \overline{q}^{|E_{\cap}|} \prod_{(i,j)\in E_{\cap}} \left(\frac{\overline{p}}{\overline{q}}\right)^{\1\{i\stackrel{R}{\sim} j\}}\enspace .
\end{equation}
Since $\overline{p}\geq \overline{q}$,  we conclude that 
\begin{equation}\label{eq:second_moment_SBM_bound}
 \overline{q}^{|E_{\cap}|}  \lambda^{|E_{\Delta}|} \mathbb{P}[R(z)\preceq R_0] \leq \mathbb{E}\left[P_{G^{(1)}, \pi^{(1)}}P_{G^{(2)}, \pi^{(2)}}\right] \leq  \overline{p}^{|E_{\cap}|}  \lambda^{|E_{\Delta}|} \mathbb{P}[R(z)\preceq R_0] \enspace  , 
\end{equation}
where, in \StoBloI, we have  $\mathbb{P}[R(z)\preceq R_0]= \left(k/n\right)^{|V_{\Delta}| - \#\mathrm{CC}_{\Delta}}$. Hence, \StoBloI\ satisfies the first part of Condition~\condvariance with $c_{\texttt{v},1}=0$ and $c_{\texttt{v},2}=1$.  Next, we consider the case where $(G^{(1)},\pi^{(1)})= (G^{(2)},\pi^{(2)})$ so that $E_{\Delta}=\emptyset$. By Condition~\condsignal , we have 
$\lambda\leq q D^{-8c_{\texttt{s}}}$, so that $\overline{p}= \overline{q}+ \lambda(1-2q)\leq \overline{q}(1+D^{-8c_{\texttt{s}}})$ as for \HidSubI.
Thus, as for \HidSubI, we conclude that the second part of Condition~\condvariance is satisfied with $c_{\texttt{v},3}=2$.

\subsubsection{T\oe plitz Seriation \ToeSerI} 
By~\eqref{eq:conditional:expectation_pq:first_moment}, we derive that
\begin{align} \label{eq:first_moment_TS_first} 
 \mathbb{E}\left[P_{G, \pi}\right] & =  \lambda^{|E|}\mathbb{E}\left[\prod_{(i,j)\in E} \1\{|z_{\pi(i)}-z_{\pi(j)}|\leq \frac{k}{2}\} \right] \\
 &\leq \lambda^{|E|}\left(\frac{k+1}{n}\right)^{|V|-\#\mathrm{CC}_{G}}\enspace .
 \label{eq:first_moment_TS}
\end{align}
To establish this upper bound, we considered a subset of $E$ corresponding to a spanning forest of $G$ and we used that the probability that $|z_i-z_j|\leq k/2$ is at most $(k+1)/n$. Hence, Condition~\condmoment holds with $c_{\texttt{m}}=1$ since $D\geq 2$.

Let us turn to the second moment. Since $\overline{p}\geq \overline{q}$, arguing as for the SBM case, we get 
\begin{equation}\label{eq:second_moment_TS_bound1}
  \overline{q}^{|E_{\cap}|}  \lambda^{|E_{\Delta}|}  
 \leq   \frac{\mathbb{E}\left[P_{G^{(1)}, \pi^{(1)}}P_{G^{(2)}, \pi^{(2)}}\right]}{ \mathbb{E}\left[\prod_{(i,j)\in E_{\Delta}} \1\{|z_i-z_j|\leq k/2\} \right] } \leq  \overline{p}^{|E_{\cap}|}  \lambda^{|E_{\Delta}|}\enspace . 
\end{equation}
Arguing as for~\eqref{eq:first_moment_TS}, with a spanning forest of $G_{\Delta}$, we get
\begin{equation}\label{eq:upper_configuration}
\mathbb{E}[\prod_{(i,j)\in E_{\Delta}} \1\{|z_i-z_j|\leq k/2\}]\leq \left(\frac{k+1}{n}\right)^{|V_{\Delta}| - \#\mathrm{CC}_{\Delta}}\enspace . 
\end{equation}
As a consequence of~\eqref{eq:second_moment_TS_bound1} and~\eqref{eq:upper_configuration}, we readily deduce that the first part of Condition~\condvariance holds with $c_{\texttt{v},1}=c_{\texttt{v},2}=1$. Now, we focus on $\mathbb{E}[P^2_{G, \pi}]$. Since $E_{\Delta}=\emptyset$, we get 
$\overline{q}^{|E|}\leq \mathbb{E}[P^2_{G, \pi}]  \leq \overline{q}^{|E|} (\tfrac{\overline{p}}{\overline{q}})^{|E|}$. We argue as in the previous case that $(\tfrac{\overline{p}}{\overline{q}})^{|E|}\leq (1+D^{-8c_{\texttt{s}}})^{D}\leq 1+2D^{1-8c_{\texttt{s}}}$ to conclude that the second part of Condition~\condvariance holds with $c_{\texttt{v},3}= 2$.

\subsection{\texttt{Permutation sampling} (Condition~\ref{as:perinv:without:replacement} in the main text)}
For \HidSubP, \StoBloP, and \ToeSerP, the additional requirement of \condmoment, that is~\eqref{eq:as:moment} with $\widetilde{\mathbb{E}}$ in place of $\mathbb{E}$, is exactly \condmoment for \HidSubI, \StoBloI, and \ToeSerI, respectively, since these models share the same matrix $\Theta$. It has therefore been established in the previous subsection.
Throughout this subsection, we use that \condsignal entails $k/n\leq D^{-8c_{\texttt{s}}}\leq D^{-8}$. Since $k\geq 1$ and $D\geq 2$, it follows that $n\geq D^{8}\geq 32D^2$ and that $4Dk\leq n$.

\subsubsection{Hidden subclique model \HidSubP} Both bounds~\eqref{eq:first_moment_HS} and~\eqref{eq:second_moment_HS_R_bound} are still valid in this model: \eqref{eq:first_moment_HS} holds for any distribution of $z$, and~\eqref{eq:second_moment_HS_R_bound} follows directly from~\eqref{eq:conditional:expectation_pq}, since each factor of the product over $E_{\cap}$ belongs to $[\overline{q},\overline{p}]$ (unlike~\eqref{eq:second_moment_HS_R}, this does not rely on the independence of the $z_i$'s). However, as the sample of the $z_i$'s is now without replacement, this changes the probabilities of the form $ \mathbb{E}[\prod_{i\in V} \1\{z_i\leq k \}]$. In particular, for any fixed $V$, we have 
\begin{equation}\label{eq:control:first:moment:HSF}
\mathbb{E}\left[\prod_{i\in V} \1\{z_i\leq k\}\right]= \prod_{j=0}^{|V|-1}\frac{(k-j)_+}{n-j}\leq \left(\frac{k}{n}\right)^{|V|}\enspace ,
\end{equation}
where $x_+=\max(x,0)$, since $(k-j)/(n-j)\leq k/n$ for $0\leq j<n$. Hence, all the bounds on moments for \HidSubI\ remain valid. This includes the lower bound $\overline{q}^{|E|}\leq \mathbb{E}[P^2_{G,\pi}]$ used for the second part of \condvariance, which is~\eqref{eq:second_moment_HS_R_bound} with $V_{\Delta}=\emptyset$. In particular, Condition \condvariance holds with the same constants as for \HidSubI, namely $c_{\texttt{v},1}=0$, $c_{\texttt{v},2}=1$, and $c_{\texttt{v},3}=2$, and Condition \condmoment holds with $c_{\texttt{m}}=0$.

It remains to establish Condition \condmomentWR. For short, we write $cc= \#\mathrm{CC}_{\mathrm{pure}}$. We only consider the case where the vertex $\omega_0$ exists, that is where some node of $G_{\cup}$ does not belong to a pure connected component of $G_{\Delta}$ (this happens when $G_{\Delta}$ has a non-pure connected component, but also when some nodes are perfectly matched), the other case being handled similarly. 
Recall the graph $\mathcal{N}[z;G_{\cup}]$ in the definition of \condmomentWR.
Let $\boldsymbol{\mathcal{T}}$ denote the collection of all trees over the vertices $\{\omega_0,\ldots, \omega_{cc}\}$. By~\eqref{eq:conditional:expectation}, we have $\mathbb{E}[P_{G^{(1)}, \pi^{(1)}} P_{G^{(2)}, \pi^{(2)}}|z]\geq 0$ for any $z$, since $\Theta$ has non-negative entries. Hence, inside the expectations below, we may upper bound indicator functions by sums of indicator functions.
As $\cA$ corresponds to the event where $\mathcal{N}[z;G_{\cup}]$ is connected, we can upper bound $\mathbf{1}\{\mathcal{A}\}$ by $\sum_{\mathcal{T}\in \boldsymbol{\mathcal{T}}} \mathbf{1}\{\mathcal{N}[z;G_{\cup}] \succeq  \mathcal{T} \}$, where $\mathcal{N}[z;G_{\cup}] \succeq  \mathcal{T}$ means $ \mathcal{T}$ is a subgraph of $\mathcal{N}[z;G_{\cup}]$.

\begin{equation*}
            \widetilde{\mathbb{E}}\left[\mathbf{1}\{\cA\}P_{G^{(1)}, \pi^{(1)}} P_{G^{(2)}, \pi^{(2)}}\right]\leq \sum_{\mathcal{T}\in \boldsymbol{\mathcal{T}}} \widetilde{\mathbb{E}}
            \left[\mathbf{1}\{\mathcal{N}[z;G_{\cup}] \succeq  \mathcal{T} \}P_{G^{(1)}, \pi^{(1)}} P_{G^{(2)}, \pi^{(2)}}\right]\enspace .  
\end{equation*}
In turn, the existence of a given edge in $\mathcal{N}[z;G_{\cup}]$ between $\omega_i$ and $\omega_j$ corresponds to the equality of two  latent assignments in a node corresponding to $\omega_i$ and a node corresponding to $\omega_j$. Let $W$ be a set of couples of nodes in $[n]$. We introduce the event $\mathcal{C}_W$ that, for each couple in $W$, the two nodes share the same latent assignment. Given $\mathcal{T}\in \boldsymbol{\mathcal{T}}$, let $\mathbf{W}_{\mathcal{T}}$ be the collection of all sets $W$ obtained by choosing, for each edge $(\omega_i,\omega_j)$ of $\mathcal{T}$, one couple $(a,b)$ of nodes with $a\in \omega_i$ and $b\in \omega_j$. By definition of $\mathcal{N}[z;G_{\cup}]$, the event $\{\mathcal{N}[z;G_{\cup}] \succeq  \mathcal{T}\}$ is included in the union of the events $\mathcal{C}_W$ over $W\in \mathbf{W}_{\mathcal{T}}$.  Since $|\boldsymbol{\mathcal{T}}|= (cc+1)^{cc-1}$ and since $|\mathbf{W}_{\mathcal{T}}|\leq (8D^2)^{cc}$ (each edge of $\mathcal{T}$ involves at least one pure connected component, which has at most $2D$ nodes, whereas $\omega_0$ has at most $4D$ nodes), we arrive at 
\begin{align}\nonumber
            \widetilde{\mathbb{E}}\left[\mathbf{1}\{\cA\}P_{G^{(1)}, \pi^{(1)}} P_{G^{(2)}, \pi^{(2)}}\right] &\leq \sum_{\mathcal{T}\in \boldsymbol{\mathcal{T}}}\sum_{W\in \mathbf{W}_{\mathcal{T}}} \widetilde{\mathbb{E}}
            \left[\mathbf{1}\{\mathcal{C}_W\}P_{G^{(1)}, \pi^{(1)}} P_{G^{(2)}, \pi^{(2)}}\right] \\
 &\leq (cc+1)^{cc-1}\left(8D^2\right)^{cc}  \max_{\mathcal{T}\in \boldsymbol{\mathcal{T}}, W\in \mathbf{W}_{\mathcal{T}}}\widetilde{\mathbb{E}}\left[\mathbf{1}\{\mathcal{C}_W\}P_{G^{(1)}, \pi^{(1)}} P_{G^{(2)}, \pi^{(2)}} \right] \enspace .
            \label{eq:second_moment_HS_F:2}
\end{align}
We point out that the upper bound~\eqref{eq:second_moment_HS_F:2} is also valid for \StoBloP\ and for \ToeSerP\ and we shall use it again.
It remains to bound $\widetilde{\mathbb{E}}[\mathbf{1}\{\mathcal{C}_W\}P_{G^{(1)}, \pi^{(1)}} P_{G^{(2)}, \pi^{(2)}}]$ for $W\in \mathbf{W}_{\mathcal{T}}$. We do it simultaneously for the three models. Write $\rho:=\max_{x\in [n]}|\{y\in [n]:\ \Theta_{xy}\neq 0\}|/n$ for the maximal proportion of non-zero entries in a row of $\Theta$, so that $\rho=k/n$ for \HidSub\ and \StoBlo, and $\rho\leq (k+1)/n$ for \ToeSer. We claim that
\begin{equation}\label{eq:bound:C_W}
\widetilde{\mathbb{E}}\left[\mathbf{1}\{\mathcal{C}_W\}P_{G^{(1)}, \pi^{(1)}} P_{G^{(2)}, \pi^{(2)}}\right]\leq n^{-|W|}\, \overline{p}^{|E_{\cap}|} \lambda^{|E_{\Delta}|}\rho^{|V_{\Delta}|-\#\mathrm{CC}_{\Delta}}\enspace .
\end{equation}
Let $F_{\Delta}$ be a spanning forest of $G_{\Delta}$, that is the union of a spanning tree of each connected component of $G_{\Delta}$, so that $F_{\Delta}$ has $|V_{\Delta}|-\#\mathrm{CC}_{\Delta}$ edges. By~\eqref{eq:conditional:expectation_pq}, we have
\begin{equation*}
\mathbb{E}\left[P_{G^{(1)}, \pi^{(1)}} P_{G^{(2)}, \pi^{(2)}}|z\right]\leq \overline{p}^{|E_{\cap}|} \lambda^{|E_{\Delta}|}\prod_{(a,b)\in F_{\Delta}}\mathbf{1}\{\Theta_{z_az_b}\neq 0\}\enspace .
\end{equation*}
Let $F$ be the graph over $V_{\cup}$ whose edges are those of $F_{\Delta}$ together with the couples of $W$. We first show that $F$ is a forest. Each connected component of $G_{\Delta}$ is included in one of the sets $\omega_0,\ldots, \omega_{cc}$, so that the edges of $F_{\Delta}$ stay inside these sets, whereas each couple of $W$ connects the two sets $\omega_i$ and $\omega_j$ of an edge of $\mathcal{T}$, each edge of $\mathcal{T}$ corresponding to exactly one couple. Since $F_{\Delta}$ is a forest, a cycle in $F$ would contain at least one couple of $W$. Following this cycle and recording the edges of $\mathcal{T}$ associated to its couples, we would obtain a closed walk in $\mathcal{T}$ that contains at least one edge and uses each edge at most once. This is impossible since $\mathcal{T}$ is a tree: removing an edge disconnects $\mathcal{T}$, so that a closed walk crosses each edge an even number of times. Hence, $F$ is a forest. Now, root each tree of $F$ and reveal the latent assignments from the roots to the leaves. Under $\widetilde{\mathbb{P}}$, the latent assignments are independent and uniformly distributed on $[n]$. Hence, for a node $a$ with parent $b$, conditionally on the latent assignments already revealed, the event $z_a=z_b$ has probability $1/n$ when $(a,b)$ is a couple of $W$, and the event $\Theta_{z_az_b}\neq 0$ has probability at most $\rho$ when $(a,b)$ is an edge of $F_{\Delta}$, by definition of $\rho$ and since $\Theta$ is symmetric. Multiplying these conditional probabilities along the edges of $F$, we get
\begin{equation*}
\widetilde{\mathbb{P}}\left[\mathcal{C}_W\cap \left\{\Theta_{z_az_b}\neq 0 \text{ for all } (a,b)\in F_{\Delta}\right\}\right]\leq n^{-|W|}\rho^{|V_{\Delta}|-\#\mathrm{CC}_{\Delta}}\enspace .
\end{equation*}
Together with the previous display, this proves~\eqref{eq:bound:C_W}.

For \HidSubP, we have $\rho=k/n$ and $|W|=cc$. Combining~\eqref{eq:second_moment_HS_F:2} and~\eqref{eq:bound:C_W}, we conclude that
\begin{align*}
            \widetilde{\mathbb{E}}\left[\mathbf{1}\{\cA\}P_{G^{(1)}, \pi^{(1)}} P_{G^{(2)}, \pi^{(2)}}\right]&\leq (cc+1)^{cc-1}\left(\frac{8D^2}{n}\right)^{cc}  \overline{p}^{|E_{\cap}|} \lambda^{|E_{\Delta}|} \left(\frac{k}{n}\right)^{|V_{\Delta}|-\#\mathrm{CC}_{\Delta}} \nonumber \\ 
            &\leq \overline{p}^{|E_{\cap}|} \lambda^{|E_{\Delta}|} \left(\frac{k}{n}\right)^{|V_{\Delta}|-\#\mathrm{CC}_{\Delta}}\left(\frac{16D^2}{\sqrt{n}}\right)^{cc}\enspace . 
\end{align*}
In the last line, we used that $(cc+1)^{cc-1}\leq (3D)^{cc}$, since $cc\leq 2D$, and that $3D\cdot 8D^2/n\leq 16D^2/\sqrt{n}$, since $n\geq 32D^2$ by \condsignal. In the case where there is no vertex $\omega_0$, the same bound follows from $(cc+1)^{cc-1}(8D^2/n)^{cc/2}\leq (3\sqrt{8}D^2/\sqrt{n})^{cc}$. The same arguments apply to \StoBloP\ and \ToeSerP\ below.
Hence, \condmomentWR holds with $c_{\texttt{vd},1}=2$ and $c_{\texttt{vd},2}=16$.

\subsubsection{Stochastic Block Model \StoBloP}
The first moment expression~\eqref{eq:first_moment_SBM} and the second moment bounds~\eqref{eq:second_moment_SBM_bound} still hold, the only difference being the controls of the probabilities $\mathbb{P}[\{z \text{ in the same block over each CC}\}]$ and $\mathbb{P}[R(z)\preceq R_0]$.
Note that $R(z)\preceq R_0$ is the event that $z$ is in the same block over each connected component of $G_{\Delta}$. Hence, it suffices to prove that, for any labeled graph $H$ with node set $V_H$ such that $|V_H|\leq 4D$, and with $\#\mathrm{CC}_H$ connected components, we have
\begin{equation}\label{eq:control:first:model:label:sbm}
\mathbb{P}[\{z \text{ in the same block over each CC of $H$}\}]\leq \left(\frac{k}{n}\right)^{|V_H|-\#\mathrm{CC}_H}\enspace .
\end{equation}
For this purpose, we reveal the latent assignments of the nodes of $H$ one at a time, connected component by connected component. Consider a node of $H$ that is not the first revealed node of its connected component, and assume that $t\leq 4D$ nodes have already been revealed, among which $m\geq 1$ nodes of its connected component, all lying in the same block. At most $k-m$ values of this block remain available among the $n-t$ remaining values, so that the conditional probability that this node falls in the same block is at most $(k-m)/(n-t)\leq k/n$. Here, we used that $kt\leq 4Dk\leq n\leq nm$. Multiplying these conditional probabilities over the $|V_H|-\#\mathrm{CC}_H$ nodes that are not the first ones of their connected components leads to~\eqref{eq:control:first:model:label:sbm}. As a consequence, all the bounds on moments for \StoBloI\ remain valid. This includes the lower bound $\overline{q}^{|E|}\leq \mathbb{E}[P^2_{G,\pi}]$ used for the second part of \condvariance, which is~\eqref{eq:second_moment_SBM_bound} with $V_{\Delta}=\emptyset$, in which case $\mathbb{P}[R(z)\preceq R_0]=1$. In particular, Condition \condvariance holds with the same constants as for \StoBloI, namely $c_{\texttt{v},1}=0$, $c_{\texttt{v},2}=1$, and $c_{\texttt{v},3}=2$, and Condition \condmoment holds with $c_{\texttt{m}}=0$.

It remains to establish Condition~\condmomentWR. As for \HidSubP, we only consider the case where the vertex $\omega_0$ exists. Since $\rho=k/n$ for \StoBloP, combining~\eqref{eq:second_moment_HS_F:2} and~\eqref{eq:bound:C_W} with $|W|=cc$, we conclude that
\begin{align*}
            \widetilde{\mathbb{E}}\left[\mathbf{1}\{\cA\}P_{G^{(1)}, \pi^{(1)}} P_{G^{(2)}, \pi^{(2)}}\right]&\leq (cc+1)^{cc-1}\left(\frac{8D^2}{n}\right)^{cc}  \overline{p}^{|E_{\cap}|} \lambda^{|E_{\Delta}|} \left(\frac{k}{n}\right)^{|V_{\Delta}|-\#\mathrm{CC}_{\Delta}} \nonumber \\ 
            &\leq \overline{p}^{|E_{\cap}|} \lambda^{|E_{\Delta}|} \left(\frac{k}{n}\right)^{|V_{\Delta}|-\#\mathrm{CC}_{\Delta}}\left(\frac{16D^2}{\sqrt{n}}\right)^{cc}\enspace . 
\end{align*}
Hence, \condmomentWR holds with $c_{\texttt{vd},1}=2$ and $c_{\texttt{vd},2}=16$.

\subsubsection{T\oe plitz Seriation \ToeSerP} Again, we mainly reduce the analysis to \ToeSerI. Both~\eqref{eq:first_moment_TS_first} and~\eqref{eq:second_moment_TS_bound1} are still valid and we only have to bound quantities of the form  $\mathbb{E}\left[\prod_{(i,j)\in E} \1\{|z_i-z_j|\leq k/2\}\right]$ for some graph $G=(V,E)$ 
with $|V|\leq 4D$. Arguing as for \ToeSerI\ but using the sampling without replacement, we reveal the latent assignments along a spanning forest of $G$. Given the values already revealed, a node takes one of the at least $n-4D$ remaining values uniformly at random, and at most $k$ of them are at distance at most $k/2$ from the latent assignment of its parent, since the latent assignments are pairwise distinct. Hence, we get
\begin{align*}
\mathbb{E}\left[\prod_{(i,j)\in E} \1\{|z_i-z_j|\leq k/2\}\right]\leq \left[\frac{k}{n-4D}\right]^{|V|-\#\mathrm{CC}_G}\leq 2 \left(\frac{k}{n}\right)^{|V|-\#\mathrm{CC}_G}\enspace , 
\end{align*}
since $(n/(n-4D))^{4D}\leq (1-16D^2/n)^{-1}\leq 2$, as $32D^2\leq n$.   Then, all the upper bounds on moments for \ToeSerI\ remain valid for \ToeSerP up to an additional multiplicative factor $2$. In particular, Condition \condvariance is still valid but with constants 
$c_{\texttt{v},1}=1$, $c_{\texttt{v},2}=2$, and $c_{\texttt{v},3}=2$. Finally,  \condmoment holds with $c_{\texttt{m}}=1$ since $2\leq D^{|E|}$.

It remains to establish \condmomentWR. As for \HidSubP, we only consider the case where the vertex $\omega_0$ exists. Since $\rho\leq (k+1)/n$ for \ToeSerP, combining~\eqref{eq:second_moment_HS_F:2} and~\eqref{eq:bound:C_W} with $|W|=cc$, we conclude that
\begin{align*}
            \widetilde{\mathbb{E}}\left[\mathbf{1}\{\cA\}P_{G^{(1)}, \pi^{(1)}} P_{G^{(2)}, \pi^{(2)}}\right]&\leq (cc+1)^{cc-1}\left(\frac{8D^2}{n}\right)^{cc}  \overline{p}^{|E_{\cap}|} \lambda^{|E_{\Delta}|} \left(\frac{k+1}{n}\right)^{|V_{\Delta}|-\#\mathrm{CC}_{\Delta}} \nonumber \\ 
            &\leq \overline{p}^{|E_{\cap}|} \lambda^{|E_{\Delta}|} \left(\frac{k+1}{n}\right)^{|V_{\Delta}|-\#\mathrm{CC}_{\Delta}}\left(\frac{16D^2}{\sqrt{n}}\right)^{cc}\enspace . 
\end{align*}
Hence, \condmomentWR holds with $c_{\texttt{vd},1}=2$ and $c_{\texttt{vd},2}=16$.

\section{Proof of Theorem~\ref{thm:lowdeg} }

We start from Lemma~\ref{lem:reduction:degree} and then we use the almost orthonormality of the basis ---\ see Theorem~\ref{thm:isorefo}.
\begin{align}\nonumber
    \mathrm{Adv}^2_{\leq D} &= \sup_{\alpha_{\emptyset}, (\alpha_G)_{G\in \mathcal G_{\leq D}}} \frac{\mathbb E^2_{H_1}\left[ \alpha_{\emptyset} + \sum_{G \in \mathcal G_{\leq D}} \alpha_G \Psi_{G}\right]}{\mathbb E\left[\left[\alpha_{\emptyset} +\sum_{G\in \mathcal G_{\leq D}} \alpha_G \Psi_{G}\right]^2\right]} \leq \sup_{\alpha_{\emptyset}, (\alpha_G)_{G\in \mathcal G_{\leq D}}}  \frac{\left[\alpha_{\emptyset}+ \sum_{G\in \mathcal G_{\leq D}} \alpha_G \mathbb E_{H_1}[\Psi_{G}]\right]^2}{(1-cD^{-2}) \|\alpha\|^2_2}\\
    & \leq  (1-cD^{-2})^{-1} \left[1 + \sum_{G\in \mathcal G_{\leq D}} \mathbb E^2_{H_1}[\Psi_{G}]\right] \enspace .\label{eq:upper:adv:lowdeg}
\end{align}
In the last line, we used the Cauchy-Schwarz inequality. Besides, we assume that $c_0$ is large enough so that the constant $c$ of Theorem~\ref{thm:isorefo} satisfies $cD^{-2}\leq 1/2$. This is possible by taking $c_{\texttt{s}}$ large enough in Theorems~\ref{thm:iso} and~\ref{thm:iso-WR}, since $D^{-c_{\texttt{s}}/2}\leq 2^{2-c_{\texttt{s}}/2}D^{-2}$.
As a consequence, we only have to bound the first moment of the polynomials $\Psi_G$ under the alternative for all our  six models. We simultaneously consider all six models.

\paragraph{Coupling of $\mathbb{P}$ and $\mathbb{P}_{H_1}$.} We define $\mathbb{P}$ and $\mathbb{P}_{H_1}$ on a common probability space. We sample independently the latent assignments $z$ as in the original model, $\hat{l}$ as in the alteration, and independent Bernoulli variables $\xi=(\xi_u)_{u\in [n]}$ with parameter $\epsilon$ (for \HidSub, there is no $\hat{l}$). We say that a node $u$ is \emph{altered} if $\xi_u=1$ and $z_u$ belongs to the \emph{altered region}, that is $z_u\leq k$ for \HidSub, $\lceil z_u/k\rceil = \hat{l}$ for \StoBlo, and $|z_u-\hat{l}|\leq k/2$ for \ToeSer. Under $\mathbb{P}_{H_1}$, conditionally on $(z,\hat{l},\xi)$, the entries $Y_{uv}$ are independent with mean $\Theta_{z_uz_v}\mathbf{1}\{\text{neither $u$ nor $v$ is altered}\}$. Under $\mathbb{P}$, conditionally on $(z,\hat{l},\xi)$, they are independent with mean $\Theta_{z_uz_v}$, so that $(\hat{l},\xi)$ is independent of $(z,Y)$.

\subsection{Step 1: Moment of $\overline{P}_{G,\pi}$ for a template $G$.}

Let us denote by $r$ the number of connected components of $G$ and write $(G^{(1)},\pi^{(1)}),\ldots, (G^{(r)},\pi^{(r)})$ for the corresponding decomposition, so that, in this section, $G^{(l)}$ stands for the $l$-th connected component of $G$.
For $l=1,\ldots, r$, define the event $\zeta_{l}$ that no node in $\pi^{(l)}(V^{(l)})$ is altered. Since the edges of $\pi^{(l)}[G^{(l)}]$ only involve nodes of $\pi^{(l)}(V^{(l)})$, it follows from the coupling that, for any $B\subset [r]$,
\begin{equation}\label{eq:coupling:zeta}
\mathbb{E}_{H_1}\left[\prod_{l\in B}\mathbf{1}\{\zeta_l\}P_{G^{(l)},\pi^{(l)}}\right]= \mathbb{E}\left[\prod_{l\in B}\mathbf{1}\{\zeta_l\}P_{G^{(l)},\pi^{(l)}}\right]\enspace .
\end{equation}
Besides, on the event $\zeta_l^c$, some node of $\pi^{(l)}(V^{(l)})$ is altered. Since $G^{(l)}$ has no isolated node, this node has an incident edge in $\pi^{(l)}[G^{(l)}]$, and, under $\mathbb{P}_{H_1}$, the corresponding entry has zero mean conditionally on $(z,\hat{l},\xi)$. As the graphs $\pi^{(l)}[G^{(l)}]$, $l\in [r]$, have pairwise disjoint sets of nodes, the entries involved in $\prod_{l\in B}P_{G^{(l)},\pi^{(l)}}$ are distinct, and therefore conditionally independent. Since the events $\zeta_l$ only depend on $(z,\hat{l},\xi)$, we deduce that, for any $B\subset [r]$, any $l_0\in B$, and any bounded random variable $W$ that only depends on $(z,\hat{l},\xi)$,
\begin{equation}\label{eq:coupling:zeta:c}
\mathbb{E}_{H_1}\left[W\mathbf{1}\{\zeta_{l_0}^c\}\prod_{l\in B}P_{G^{(l)},\pi^{(l)}}\right]= 0\enspace .
\end{equation}
 Then, expanding the polynomial $\overline{P}_{G,\pi}$ and introducing the events $\zeta_l$, we obtain
\begin{align}\nonumber
\mathbb{E}_{H_1}\left[\prod_{l=1}^r  \overline{P}_{G^{(l)},\pi^{(l)}}\right]& = \sum_{T\subset [r]} (-1)^{|T|} \mathbb{E}_{H_1}\left[\prod_{l\in [r]\setminus T} (\mathbf{1}\{\zeta_l^c\} + \mathbf{1}\{\zeta_l\})  P_{G^{(l)},\pi^{(l)}}  \right]\prod_{l\in T}\mathbb{E}[P_{G^{(l)},\pi^{(l)}} ]
\\ \nonumber
&=
\sum_{T\subset [r]} (-1)^{|T|} \mathbb{E}_{H_1}\left[\prod_{l\in [r]\setminus T}  \mathbf{1}\{\zeta_l\}  P_{G^{(l)},\pi^{(l)}}  \right]\prod_{l\in T}\mathbb{E}[P_{G^{(l)},\pi^{(l)}} ]
\\ \nonumber
&=
\sum_{T\subset [r]} (-1)^{|T|} \mathbb{E}\left[\prod_{l\in [r]\setminus T}  \mathbf{1}\{\zeta_l\}  P_{G^{(l)},\pi^{(l)}}  \right]\prod_{l\in T}\mathbb{E}[P_{G^{(l)},\pi^{(l)}} ]
\\ \nonumber
& = \mathbb{E}\left[\prod_{l\in [r]} \left(\overline{P}_{G^{(l)},\pi^{(l)}} -  \mathbf{1}\{\zeta_l^c\}  P_{G^{(l)},\pi^{(l)}}\right)  \right]\\
& =
\sum_{T\subset [r]} \mathbb{E}\left[(-1)^{r- |T|}\prod_{l\in [r]\setminus T} \mathbf{1}\{\zeta_l^c\}  P_{G^{(l)},\pi^{(l)}} \prod_{l\in T} \overline{P}_{G^{(l)},\pi^{(l)}} \right]\enspace ,\label{eq:decomposition:H1}
\end{align}
where we used~\eqref{eq:coupling:zeta:c} in the second line and~\eqref{eq:coupling:zeta} in the third line.

We shall rely on the following bound, which holds in all six models: for any $B\subset [r]$,
\begin{equation}\label{eq:upper_bound:expectation:perturbation}
\mathbb{E}\left[\prod_{l\in B} \mathbf{1}\{\zeta_l^c\}  P_{G^{(l)},\pi^{(l)}}\right] \leq  \lambda^{\sum_{l\in B}|E^{(l)}|}   \left(\frac{2k}{n}\right)^{\sum_{l\in B}|V^{(l)}|}  (2D\epsilon   )^{|B|} \enspace .
\end{equation}
Let us prove it. Define $\rho:=k/n$ for \HidSub\ and \StoBlo, and $\rho:=(k+1)/n$ for \ToeSer. Then, each row of $\Theta$ contains at most $\rho n$ non-zero entries, and the altered region contains at most $\rho n$ values. Fix $\hat{l}$. Conditionally on $z$, the expectation of $\prod_{l\in B}P_{G^{(l)},\pi^{(l)}}$ is equal to $\lambda^{\sum_{l\in B}|E^{(l)}|}$ if $\Theta_{z_uz_v}\neq 0$ for all the edges $(u,v)$ of the graphs $\pi^{(l)}[G^{(l)}]$, $l\in B$, and is equal to $0$ otherwise. Under $\cap_{l\in B}\zeta_l^c$, for each $l\in B$, there exists a node $u_l\in \pi^{(l)}(V^{(l)})$ such that $\xi_{u_l}=1$ and $z_{u_l}$ belongs to the altered region. We take a union bound over the choice of $(u_l)_{l\in B}$, which leads to at most $(2D)^{|B|}$ terms, since $|V^{(l)}|\leq 2D$. As $\xi$ is independent of $z$, each of these terms brings a factor $\epsilon^{|B|}$. Then, for each $l\in B$, we reveal the latent assignments of the nodes of $\pi^{(l)}(V^{(l)})$ along a spanning tree of $\pi^{(l)}[G^{(l)}]$ rooted at $u_l$. Given the latent assignments already revealed, the latent assignment of a node is uniformly distributed over $[n]$ under \condinvariance, and over a set of at least $n-2D$ values under \condinvarianceWR, since at most $|V|\leq 2D$ nodes are involved. Among these values, at most $\rho n$ belong to the altered region (for the root $u_l$), or lead to a non-zero entry of $\Theta$ with the latent assignment of the parent node (for the other nodes). Since $\rho n/(n-2D)\leq 2k/n$ (recall that $k\geq 2$ for \ToeSer\ and that $n\geq 8D$ by~\eqref{eq:condition:LD:test}), we obtain~\eqref{eq:upper_bound:expectation:perturbation}.

For the models \HidSubI, \StoBloI, and \ToeSerI, the pairs $(z_u,\xi_u)$, $u\in [n]$, are independent and independent of $\hat{l}$. Hence, conditionally on $\hat{l}$, the pairs $(\overline{P}_{G^{(l)},\pi^{(l)}},\mathbf{1}\{\zeta_l^c\}P_{G^{(l)},\pi^{(l)}})$, $l=1,\ldots,r$, are independent, since they involve disjoint sets of nodes. As $\overline{P}_{G^{(l)},\pi^{(l)}}$ is centered and independent of $\hat{l}$, all the terms with $T\neq \emptyset$ in~\eqref{eq:decomposition:H1} are equal to zero. Hence, this simplifies to
\begin{equation*}
\mathbb{E}_{H_1}\left[\prod_{l=1}^r  \overline{P}_{G^{(l)},\pi^{(l)}}\right]= (-1)^{r} \mathbb{E}\left[\prod_{l\in [r]} \mathbf{1}\{\zeta_l^c\}  P_{G^{(l)},\pi^{(l)}} \right]\enspace .
\end{equation*}
Applying~\eqref{eq:upper_bound:expectation:perturbation} with $B=[r]$, we conclude that
\begin{equation}\label{eq:upper:polynomial:independent}
\left|\mathbb{E}_{H_1}\left[\prod_{l=1}^r  \overline{P}_{G^{(l)},\pi^{(l)}}\right]\right|\leq \lambda^{|E|}\left(\frac{2k}{n}\right)^{|V|-\#\mathrm{CC}_G} \left(4D\epsilon\frac{k}{n}\right)^{\#\mathrm{CC}_G}\enspace .
\end{equation}

For the models \HidSubP, \StoBloP, and \ToeSerP, additional work is required to account for the dependencies.
\begin{lemma}\label{lem:upper:polynomial:dependent}
Consider the models \HidSubP, \StoBloP, and \ToeSerP. Under the assumptions of Theorem~\ref{thm:lowdeg}, we have
\begin{equation}\label{eq:upper:polynomial:dependent}
\left|\mathbb{E}_{H_1}\left[\prod_{l=1}^r  \overline{P}_{G^{(l)},\pi^{(l)}}\right]\right|\leq c_1(D^c\lambda)^{|E|} \left(D^c \frac{k}{n}\right)^{|V|-\#\mathrm{CC}_G}\left(c_1D^{c}\left(\frac{k\epsilon}{n}\vee \frac{1}{\sqrt{n}} \right) \right)^{\#\mathrm{CC}_G}\enspace ,
\end{equation}
where $c$ and $c_1$ are numerical constants.
\end{lemma}
        \subsection{Step 2: Bounding $|\mathbb E_{H_1}[\Psi_{G}]|$.}
     By the definition~\eqref{eq:definition:P*G} of $\Psi_G$, and since $|\Pi_V|\leq n^{|V|}$ and $\mathbb{V}(G)\geq |\Pi_V|\overline{q}^{|E|}$, we have
        \begin{equation*}
        |\mathbb E_{H_1}[\Psi_{G}]|\leq n^{|V|/2}\overline{q}^{-|E|/2}\max_{\pi\in \Pi_V}|\mathbb E_{H_1}[\overline{P}_{G,\pi}]|\enspace .
        \end{equation*}
        Hence, we derive from~\eqref{eq:upper:polynomial:independent}~and~\eqref{eq:upper:polynomial:dependent} that, in all six models, we have
        \begin{equation*}
        |\mathbb E_{H_1}[\Psi_{G}]|  \leq c_1
        \left(D^{c}\frac{\lambda}{\sqrt{\overline{q}}}\right)^{|E|}\left(D^{c}\frac{k}{\sqrt{n}}\right)^{|V|-\#\mathrm{CC}_G} \left[c_1 D^{c} \left(1 + \epsilon\frac{k}{\sqrt{n}}\right)\right]^{\#\mathrm{CC}_G}\enspace .
        \end{equation*}
        Reorganizing the products, we get
        \begin{equation*}
        |\mathbb E_{H_1}[\Psi_{G}]|  \leq     c_1\left[c_1 D^{3c} \left(\frac{\lambda k}{\sqrt{\overline{q} n}} + \frac{\lambda k^2\epsilon}{n\sqrt{\overline{q}}}\right)\right]^{\#\mathrm{CC}_G}
        \left(D^{c}\frac{\lambda}{\sqrt{\overline{q}}}\right)^{|E|+\#\mathrm{CC}_G- |V|}\left(D^{2c}\frac{k\lambda}{\sqrt{n\overline{q}}}\right)^{|V|-2\#\mathrm{CC}_G}\enspace .
        \end{equation*}
        As $G$ does not contain any isolated node, each connected component contains at least two nodes and therefore $\#\mathrm{CC}_G\leq |V|/2$. Besides, any graph satisfies $|E|\geq |V|-\#\mathrm{CC}_G$. Hence, all the exponents in the above bound are non-negative, and their sum is equal to $|E|$. Since $q\leq 1/2$, we have $\overline{q}\geq q/2$, so that, by~\eqref{eq:condition:LD:test} and since $q\leq 1$,
        \begin{equation*}
        \frac{\lambda}{\sqrt{\overline{q}}}\leq \frac{\sqrt{2}\lambda}{q}\leq \sqrt{2}D^{-c_0}\ ;\quad  \frac{\lambda k}{\sqrt{n\overline{q}}}\leq \sqrt{2}D^{-c_0}\ ; \quad \frac{\lambda k^2\epsilon}{n\sqrt{\overline{q}}}\leq \sqrt{2}D^{-c_0}\enspace .
        \end{equation*}
        Assuming $c_1\geq 1$, each of the $|E|$ factors in the above bound is therefore at most $3c_1D^{3c-c_0}$. Since $|E|\geq 1$ and $D\geq 2$, we obtain, provided that $c_0$ is large enough,
        \begin{equation*}
        |\mathbb E_{H_1}[\Psi_{G}]|  \leq   D^{-c_0|E|/2}\enspace .
        \end{equation*}

\subsection{Step 3: Bounding the advantage.} Coming back to~\eqref{eq:upper:adv:lowdeg},  we simply need to enumerate all $G\in \mathcal{G}_{\leq D}$. For this purpose, we use the crude bound that there are no more than $v^{2e}$ templates with $v$ nodes and $e$ edges. This allows us to conclude that
\begin{align*}
    \mathrm{Adv}^2_{\leq D}
    & \leq  (1-cD^{-2})^{-1} \left[1 + \sum_{G\in \mathcal G_{\leq D}} D^{-c_0|E|}\right] \\
    & \leq  (1-cD^{-2})^{-1} \left[1 + \sum_{v=1}^{2D} \sum_{e=1}^{D} v^{2e} D^{-c_0e}  \right]\\
    &\leq 1 + \frac{c}{D}\enspace ,
\end{align*}
provided that $c_0$ is large enough and where the numerical constant $c$ may change from line to line.  This concludes the proof.

\begin{proof}[Proof of Lemma~\ref{lem:upper:polynomial:dependent}]
In this proof, the positive numerical constants $c$ and $c_1$ may change from line to line. Starting from~\eqref{eq:decomposition:H1}, we get
\begin{align*}
\left|\mathbb{E}_{H_1}\left[\prod_{l=1}^r  \overline{P}_{G^{(l)},\pi^{(l)}}\right]\right|&\leq  \sum_{T\subset [r]} \left|\mathbb{E}\left[\prod_{l\in [r]\setminus T} \mathbf{1}\{\zeta_l^c\} P_{G^{(l)},\pi^{(l)}} \prod_{l\in T} \overline{P}_{G^{(l)},\pi^{(l)}} \right]\right|\\
&\leq  2^{r}\max_{T\subset [r]}\mathbb{E}\left[\prod_{l\in [r]\setminus T} \mathbf{1}\{\zeta_l^c\}  P_{G^{(l)},\pi^{(l)}}\right]\left|\mathbb{E}\left[\prod_{l\in T} \overline{P}_{G^{(l)},\pi^{(l)}}  \right] \right| \\
& \quad \quad +  2^{r} \max_{T\subset [r]} \left|\mathbb{E}\left[\left(\prod_{l\in [r]\setminus T} \mathbf{1}\{\zeta_l^c\}  P_{G^{(l)},\pi^{(l)}}-\mathbb{E}\left[\prod_{l\in [r]\setminus T} \mathbf{1}\{\zeta_l^c\}  P_{G^{(l)},\pi^{(l)}}\right]\right) \prod_{l\in T} \overline{P}_{G^{(l)},\pi^{(l)}}  \right] \right| \\
& \quad \quad := 2^{r}\left(\max_{T\subset [r]} A_{1;T}+ \max_{T\subset [r]}A_{2;T}\right)\enspace .
\end{align*}
It is easy to control $A_{1;T}$ from our previous computations. Write $B = [r]\setminus T$. The term $\mathbb{E}[\prod_{l\in B} \mathbf{1}\{\zeta_l^c\}  P_{G^{(l)},\pi^{(l)}}]$ is controlled by~\eqref{eq:upper_bound:expectation:perturbation}, whereas the term $|\mathbb{E}[\prod_{l\in T} \overline{P}_{G^{(l)},\pi^{(l)}}]|$ is equal to $1$ if $T=\emptyset$, to $0$ if $T$ is a singleton and, if $|T|\geq 2$, is controlled by~\eqref{eq:obj_A}, applied as in the proof of Lemma~\ref{lem:mean:Psi} to the template made of the connected components indexed by $T$ and to the empty graph (all three models satisfy \condmoment\ and \condmomentWR):
\begin{equation*}
\left|\mathbb{E}\left[\prod_{l\in T} \overline{P}_{G^{(l)},\pi^{(l)}}  \right] \right|\leq c_1D^{c}\left(D^{c}\lambda\right)^{\sum_{l\in T}|E^{(l)}|}\left(D^{c}\frac{k}{n}\right)^{\sum_{l\in T}(|V^{(l)}|-1)}\left[c_1\frac{D^{c}}{\sqrt{n}}\right]^{|T|}\enspace .
\end{equation*}
In the product of these two bounds, each connected component indexed by $l\in B$ brings a factor $4Dk\epsilon/n$ times $\lambda^{|E^{(l)}|}(2k/n)^{|V^{(l)}|-1}$, and each connected component indexed by $l\in T$ brings a factor $c_1D^c/\sqrt{n}$ times $(D^{c}\lambda)^{|E^{(l)}|}(D^ck/n)^{|V^{(l)}|-1}$. Since $r\geq 1$, the leading factor $c_1D^c$ is absorbed by increasing $c$, and we derive that
\begin{equation*}
\max_{T\subset [r]}A_{1;T} \leq  c_1(D^c\lambda)^{|E|} \left(D^c \frac{k}{n}\right)^{|V|-\#\mathrm{CC}_G}\left(c_1D^{c}\left(\frac{k\epsilon}{n}\vee \frac{1}{\sqrt{n}} \right) \right)^{\#\mathrm{CC}_G}\enspace .
\end{equation*}
The control of $A_{2;T}$ requires additional work.

\begin{lemma}\label{lem:controle:A2T}
All three models \HidSubP, \StoBloP, and \ToeSerP\ satisfy
\begin{equation*}
\max_{T\subset [r]}A_{2;T} \leq  c_1(D^c\lambda)^{|E|} \left(D^c \frac{k}{n}\right)^{|V|-\#\mathrm{CC}_G}\left(c_1D^{c}\left(\frac{k\epsilon}{n}\vee \frac{1}{\sqrt{n}} \right) \right)^{\#\mathrm{CC}_G}\enspace .
\end{equation*}
\end{lemma}

Gathering these two bounds, we conclude that
\begin{equation*}
\left|\mathbb{E}_{H_1}\left[\prod_{l=1}^r  \overline{P}_{G^{(l)},\pi^{(l)}}\right]\right|\leq c_1(D^c\lambda)^{|E|} \left(D^c \frac{k}{n}\right)^{|V|-\#\mathrm{CC}_G}\left(c_1D^{c}\left(\frac{k\epsilon}{n}\vee \frac{1}{\sqrt{n}} \right) \right)^{\#\mathrm{CC}_G}\enspace .
\end{equation*}

\end{proof}

\begin{proof}[Proof of Lemma~\ref{lem:controle:A2T}]
Fix $T\subset [r]$. If $T=\emptyset$ or $T=[r]$, then $A_{2;T}=0$. Otherwise, we use a similar approach to the proof of Lemma~\ref{lem:gen-2}. Let us slightly change the notation in order to adapt the arguments. We write $(G^{'(1)},\pi^{'(1)})$, \ldots, $(G^{'(s)},\pi^{'(s)})$ for the connected components indexed by $[r]\setminus T$, and $(G^{(1)},\pi^{(1)})$, \ldots, $(G^{(r)},\pi^{(r)})$ for those indexed by $T$, so that, in this proof, $r$ stands for $|T|$. We have $r\geq 1$ and $s\geq 1$, and all these labeled graphs are connected and have pairwise disjoint sets of nodes. For $i=1,\ldots, s$, we write $\zeta'_i$ for the event that no node of $\pi^{'(i)}(V^{'(i)})$ is altered. Then, we write
\begin{equation*}
P_0 := \prod_{i=1}^s \mathbf{1}\{\zeta^{'c}_i\} P_{G^{'(i)},\pi^{'(i)}} \enspace .
\end{equation*}
Also,  $\overline{P}_0 := P_0 - \mathbb{E}[P_0]$. Furthermore, for $i=1,\ldots, r$,  we write $P_i:= P_{G^{(i)},\pi^{(i)}}$ and
$\overline{P}_i:= \overline{P}_{G^{(i)},\pi^{(i)}}$. Define $e:= \sum_{i=1}^r |E^{(i)}|$, $e':= \sum_{i=1}^s |E^{'(i)}|$, $v:= \sum_{i=1}^r |V^{(i)}|$, and $v':= \sum_{i=1}^s |V^{'(i)}|$.
To establish the lemma, we need to show that
\begin{equation*}
\left|\mathbb{E}\left[\prod_{i=0}^{r} \overline{P}_i\right]\right|\leq c_1(D^c\lambda)^{e+e'} \left(D^c \frac{k}{n}\right)^{v+v'-r-s}\left(c_1D^{c}\left(\frac{k\epsilon}{n}\vee \frac{1}{\sqrt{n}} \right) \right)^{r+ s}\enspace .
\end{equation*}
Let $Z_0$ be the set of latent assignments of the nodes in $\cup_{i=1}^s\pi^{'(i)}(V^{'(i)})$ and, for $a=1,\ldots, r$, let $Z_a$ be the set of latent assignments of the nodes in $\pi^{(a)}(V^{(a)})$. As in the proof of Lemma~\ref{lem:gen-2}, we introduce the distribution $\mathbb{P}_R$ under which each $Z_a$ is sampled without replacement, the $Z_a$'s being independent, the event $\mathcal{E}'$ that $Z_0,\ldots, Z_r$ are pairwise disjoint, and, for $B\subset [0;r]$ and for partitions $\cB$ of $[0;r]$, the events $\cA_B$ and $\cE_{\cB}$, as well as the connectivity events on $B$. Under $\mathbb{P}_R$, and under $\widetilde{\mathbb{P}}$ below, $(\hat{l},\xi)$ is sampled as under $\mathbb{P}$, independently of $z$, and the conditional distribution of $Y$ given $z$ is unchanged. Since the distribution under $\mathbb{P}$ of the variables involved in $\overline{P}_0,\ldots,\overline{P}_r$ is their distribution under $\mathbb{P}_R$ conditionally on $\mathcal{E}'$, we have, as in~\eqref{eq:definition:A2},
\begin{equation*}
\left|\mathbb{E}\left[\prod_{i=0}^{r} \overline{P}_i\right]\right|= \frac{1}{\mathbb{P}_R(\mathcal{E}')}\left|\mathbb{E}_R\left[\mathbf{1}\{\mathcal{E}'\}\prod_{i=0}^{r} \overline{P}_i\right]\right|\leq 2 \left|\mathbb{E}_R\left[\mathbf{1}\{\mathcal{E}'\}\prod_{i=0}^{r} \overline{P}_i\right]\right|\enspace ,
\end{equation*}
since $\mathbb{P}_R(\mathcal{E}')\geq (1-2D/n)^{2D}\geq 1/2$. Note that the indicator function cannot be removed: under $\mathbb{P}_R$, $\overline{P}_0,\ldots,\overline{P}_r$ are independent and centered, so that $\mathbb{E}_R[\prod_{i=0}^{r} \overline{P}_i]=0$.

For a non-empty subset $B\subset [0;r]$, we define
\begin{equation*}
\varphi(B):= c_1D^{c}\left(D^{c}\lambda\right)^{e_B}\left(D^{c}\frac{k}{n}\right)^{v_B}\left[c_1\frac{D^{c}}{\sqrt{n}}\right]^{|B\setminus\{0\}|}\left(c_1 D^{c}\frac{k\epsilon}{n}\right)^{s\mathbf{1}\{0\in B\}}\enspace ,
\end{equation*}
where $e_B:= \sum_{i\in B\setminus\{0\}}|E^{(i)}| + e'\mathbf{1}\{0\in B\}$ and $v_B:= \sum_{i\in B\setminus\{0\}}(|V^{(i)}|-1) + (v'-s)\mathbf{1}\{0\in B\}$. As in~\eqref{eq:varphi:sub}, we have $\varphi(B)= (c_1D^c)^{1-|B|}\prod_{i\in B}\varphi(\{i\})$, so that $\prod_{B\in \cB}\varphi(B)\leq \prod_{i=0}^r\varphi(\{i\})$ for any partition $\cB$ of $[0;r]$.

Consider first $B\subset [0;r]$ with $0\notin B$. The polynomials $\overline{P}_i$, $i\in B$, are the centered polynomials associated to the connected components of a template, as in the proof of Lemma~\ref{lem:mean:Psi}. Hence, Lemma~\ref{lem:recursion:covariance} ensures, up to a renaming of the constants, that $|\mathbb{E}_R[\mathbf{1}\{\cF\}\prod_{i\in B}\overline{P}_i]|\leq \varphi(B)$ for any connectivity event $\cF$ on $B$. When $0\in B$, we adapt the proof of Lemma~\ref{lem:recursion:covariance} to establish the following bound.
\begin{lemma}\label{lem:recursion:covariance:2}
    For any subset $B\subset[0;r]$ such that $0\in B$ and any connectivity event $\cF$ on $B$, we have
    \begin{equation*}
\left|\mathbb{E}_R\left[\mathbf{1}\{\cF\}\prod_{i\in B}\overline{P}_i\right]\right|\leq \varphi(B)\enspace .
    \end{equation*}
\end{lemma}
The proof of Lemma~\ref{lem:recursion:CP} only relies on the bounds $|\mathbb{E}_R[\mathbf{1}\{\cF\}\prod_{i\in B}\overline{P}_i]|\leq \varphi(B)$ for all connectivity events $\cF$ on $B$, on~\eqref{eq:varphi:sub}, and on the independence under $\mathbb{P}_R$ of the polynomials indexed by different groups of a partition, which still holds here as $(\hat{l},\xi)$ only enters $\overline{P}_0$. Hence, arguing exactly as in this proof with the finest partition of $[0;r]$, we get
\begin{equation*}
\left|\mathbb{E}_R\left[\mathbf{1}\{\mathcal{E}'\}\prod_{i=0}^{r} \overline{P}_i\right]\right|\leq (r+1)^{2(r+1)}\prod_{i=0}^{r}\varphi(\{i\})\enspace .
\end{equation*}
By definition of $\varphi$, we have $\prod_{i=0}^{r}\varphi(\{i\})= (c_1D^c)^{r+1}(D^c\lambda)^{e+e'}(D^ck/n)^{v+v'-r-s}[c_1D^c/\sqrt{n}]^{r}(c_1D^ck\epsilon/n)^{s}$. Since $r+1\leq r+s$ and $r+1\leq 2D$, the factor $2(r+1)^{2(r+1)}(c_1D^c)^{r+1}$ is absorbed by increasing the constants, and the lemma follows.
\end{proof}

\begin{proof}[Proof of Lemma~\ref{lem:recursion:covariance:2}]
Without loss of generality, we assume that $B=[0;r]$; otherwise, we apply the result to the configuration made of $P_0$ and of the polynomials $P_i$, $i\in B\setminus\{0\}$. We write $\cA$ for $\cA_{[0;r]}$ and we fix a connectivity event $\cF$ on $[0;r]$. If $r=0$, then $\cF$ is of probability one, and $\mathbb{E}_R[\overline{P}_0]=0$, since $Z_0$ has the same distribution under $\mathbb{P}$ and under $\mathbb{P}_R$. Hence, we assume that $r\geq 1$. The approach closely follows that of the proof of Lemma~\ref{lem:recursion:covariance}. First, we expand the product
\begin{equation*}
\mathbb{E}_R\left[\mathbf{1}\{\cF\}\prod_{i=0}^r\overline{P}_i\right]= \sum_{T\subset [0;r]}(-1)^{r+1-|T|}\prod_{i\in [0;r]\setminus T}\mathbb{E}\left[P_i\right] \mathbb{E}_R\left[\mathbf{1}\{\cF\}\prod_{i\in T} P_i \right]\enspace.
\end{equation*}
For any $T\subset [0;r]$, we have $\mathbb{E}[\prod_{i\in T}P_i|z,\hat{l},\xi]\geq 0$, since these polynomials involve distinct entries of $Y$, whose conditional means are entries of $\Theta$. As in~\eqref{eq:reduction:F:A}, this entails $0\leq \mathbb{E}_R[\mathbf{1}\{\cF\}\prod_{i\in T}P_i]\leq \mathbb{E}_R[\mathbf{1}\{\cA\}\prod_{i\in T}P_i]$. As in the proof of Lemma~\ref{lem:recursion:covariance}, the latter is at most $2\widetilde{\mathbb{E}}[\mathbf{1}\{\cA\}\prod_{i\in T}P_i]$, since at most $2D$ nodes are involved, where $\widetilde{\mathbb{E}}$ stands for the expectation with respect to the distribution where we sample the latent assignments $z_i$'s with replacement.
Then, we decompose according to the partition $\cT$ of $T$ into groups of polynomials (indexed by $T$) that are connected through their latent assignments. Arguing as for~\eqref{eq:connectivity:outside:T}, the super-nodes being now formed by the groups of $\cT$ and by the node sets of the polynomials $P_i$, $i\in [0;r]\setminus T$, so that there are $|\cT|+r+1-|T|$ of them, we get, on the event $\cE_{\cT}$, $\widetilde{\mathbb{P}}[\cA|z_T,\hat{l},\xi]\leq (r+1)^{r+1}(16D^2/n)^{|\cT|+r-|T|}$, where $z_T$ stands for the latent assignments of the nodes involved in $\prod_{i\in T}P_i$. Since $(\hat{l},\xi)$ only enters $P_0$, the polynomials indexed by different groups of $\cT$ are independent under $\widetilde{\mathbb{P}}$, and we arrive at
\begin{align}\label{eq:upper_bound_E_R:bis}
\left|\mathbb{E}_R\left[\mathbf{1}\{\cF\}\prod_{i=0}^r \overline{P}_i\right]\right| &\leq 2^{r+2}(r+1)^{2(r+1)} \max_{T\subset [0;r]}\max_{\cT:\text{ partition of }T}\left(\frac{16D^2}{n}\right)^{|\cT|+r-|T|} \nonumber \\
&\quad \quad \cdot \prod_{i\in [0;r]\setminus  T}\mathbb{E}\left[P_{i}\right]\prod_{T'\in \cT} \widetilde{\mathbb{E}}\left[\1\{\mathcal{A}_{T'}\}\prod_{i\in T'} P_{i} \right]\enspace .
\end{align}
We now bound the factors in~\eqref{eq:upper_bound_E_R:bis}. For $i\in [r]\setminus T$, $\mathbb{E}[P_i]$ is controlled by \condmoment, which is fulfilled for all three models and, if $0\notin T$, $\mathbb{E}[P_0]\leq \lambda^{e'}(2k/n)^{v'-s}(4Dk\epsilon/n)^{s}$ by~\eqref{eq:upper_bound:expectation:perturbation}. Consider a group $T'\in \cT$ such that $0\notin T'$. If $|T'|=1$, the term $\widetilde{\mathbb{E}}[\1\{\mathcal{A}_{T'}\}\prod_{i\in T'} P_i]$ is controlled by \condmoment\ under $\widetilde{\mathbb{E}}$. If $|T'|\geq 2$, it is controlled by \condmomentWR, applied to the template made of the connected components indexed by $T'$ and to the empty graph, for which the event $\cA$ of \condmomentWR\ is $\cA_{T'}$; this brings an additional factor $[c_1D^c/\sqrt{n}]^{|T'|}$. It remains to control this term for the group $T_0\in \cT$ that contains $0$, when $0\in T$. This is the purpose of the following lemma, whose proof is postponed to the end of this section.
\begin{lemma}\label{lem:upper:E:R:2}
Consider the models \HidSubP, \StoBloP, and \ToeSerP. For any subset $T_0\subset [0;r]$ such that $0\in T_0$, we have
\begin{equation}\label{eq:upper:E:R:2}
\widetilde{\mathbb{E}}\left[\1\{\mathcal{A}_{T_0}\}\prod_{i\in T_0} P_i \right]\leq  \lambda^{e_{T_0}}\left(\frac{2k}{n}\right)^{v_{T_0}}\left(\frac{8D^3}{n}\right)^{|T_0|-1}\left(\frac{4Dk\epsilon}{n}\right)^{s}\enspace ,
\end{equation}
where $e_{T_0}$ and $v_{T_0}$ are as in the definition of $\varphi$.
\end{lemma}
Gathering all these bounds in~\eqref{eq:upper_bound_E_R:bis}, we see that the factors associated to the edges and to the nodes are those arising in $\varphi([0;r])$, and it remains to count the powers of $n^{-1/2}$. Fix $T$ and $\cT$, and let $t$ be the number of groups of $\cT$ of the form $\{i\}$ with $i\neq 0$. These groups do not bring any power of $n^{-1/2}$; a group $T'\in \cT$ with $0\notin T'$ and $|T'|\geq 2$ brings $|T'|$ of them and, by Lemma~\ref{lem:upper:E:R:2}, the group $T_0$ brings at least $|T_0|-1$ of them. Together with the factor $(16D^2/n)^{|\cT|+r-|T|}$, the resulting power of $n^{-1/2}$ is at least
\begin{itemize}
\item $(|T|-t)+2(|\cT|+r-|T|)\geq 2r-|T|+t\geq r$ if $0\notin T$, since $|\cT|\geq t$ and $|T|\leq r$;
\item $(|T|-1-t)+2(|\cT|+r-|T|)\geq 2r+1-|T|+t\geq r$ if $0\in T$, since $|\cT|\geq t+1$ and $|T|\leq r+1$.
\end{itemize}
The additional powers of $n^{-1/2}$ come with factors that are smaller than one by~\eqref{eq:condition:LD:test}, provided that $c_0$ is large enough. Finally, since $r+1\leq 2D$, the combinatorial factor $2^{r+2}(r+1)^{2(r+1)}$ and the constants arising in \condmomentWR\ for the at most $r+1$ groups of $\cT$ are absorbed in the factor $(c_1D^c)^{r+1}$ of $\varphi([0;r])$ by increasing the constants. This concludes the proof.
\end{proof}

\begin{proof}[Proof of Lemma~\ref{lem:upper:E:R:2}]
Let $m:=|T_0|$. Under $\cA_{T_0}$, consider the graph whose vertices are the $m$ super-nodes formed by the node set of $P_0$ and by the node sets of the $P_i$, $i\in T_0\setminus\{0\}$, two super-nodes being linked when they contain two nodes with the same latent assignment. This graph is connected, so that it contains one of the at most $m^{m-2}\leq (2D)^{m-1}$ spanning trees on $m$ vertices. Fix such a spanning tree and, for each of its $m-1$ edges, a pair of nodes, one in each of the two super-nodes, that share the same latent assignment (at most $4D^2$ choices per edge). We write $W$ for the corresponding set of $m-1$ pairs. Besides, as in the proof of~\eqref{eq:upper_bound:expectation:perturbation}, we take a union bound over nodes $u_1,\ldots, u_s$, with $u_j\in \pi^{'(j)}(V^{'(j)})$, such that $\xi_{u_j}=1$ and $z_{u_j}$ belongs to the altered region (at most $(2D)^s$ choices, each bringing a factor $\epsilon^s$). Consider the graph $F$ whose vertices are the nodes involved in $\prod_{i\in T_0}P_i$ together with an additional vertex $\star$, and whose edges are a spanning tree of each of the graphs $\pi^{'(j)}[G^{'(j)}]$, $j=1,\ldots,s$, and $\pi^{(i)}[G^{(i)}]$, $i\in T_0\setminus\{0\}$, the pairs of $W$, and the edges $(\star,u_j)$, $j=1,\ldots, s$. The graph $F$ is a tree. 
We assign $\hat{l}$ to $\star$ and reveal the latent assignments from $\star$ to the leaves of $F$. Under $\widetilde{\mathbb{P}}$, the latent assignments are independent, uniformly distributed on $[n]$, and independent of $\hat{l}$. Hence, given the values already revealed, a node $u$ with parent $w$ in $F$ satisfies $z_u=z_w$ with probability $1/n$ if $(u,w)\in W$, $z_u$ belongs to the altered region with probability at most $\rho$ if $w=\star$, and $\Theta_{z_uz_w}\neq 0$ with probability at most $\rho$ otherwise, where $\rho$ is defined as in the proof of~\eqref{eq:upper_bound:expectation:perturbation}. Since $\mathbb{E}[\prod_{i\in T_0}P_i|z,\hat{l},\xi]$ is at most $\lambda^{e_{T_0}}$, and vanishes unless $\Theta_{z_uz_w}\neq 0$ for all the edges $(u,w)$ of the graphs, and since $F$ contains $m-1$ pairs of $W$, $s$ edges adjacent to $\star$, and $v_{T_0}$ other edges, we obtain
\begin{equation*}
\widetilde{\mathbb{E}}\left[\1\{\mathcal{A}_{T_0}\}\prod_{i\in T_0} P_i \right]\leq (2D)^{m-1}(4D^2)^{m-1}(2D\epsilon)^{s}\lambda^{e_{T_0}}n^{-(m-1)}\rho^{s+v_{T_0}}\enspace .
\end{equation*}
Since $\rho\leq 2k/n$, this leads to~\eqref{eq:upper:E:R:2}.
\end{proof}

\section{Proofs for LD estimation problems}

\subsection{Proof of Theorem~\ref{thm:iso-estimation}}

This proof closely follows that of Theorem~\ref{thm:iso} and we only emphasize the few differences. In particular, we define the Gram matrix $\Gamma$ of size $|\mathcal{G}_{\leq D}^{(1,2)}|+1$ associated to the basis $(1, (\Psi^{(1,2)}_G)_{G\in \mathcal{G}_{\leq D}^{(1,2)}})$ by 
 $\Gamma_{G^{(1)}, G^{(2)}} = \mathbb E[\Psi^{(1,2)}_{G^{(1)}}\Psi^{(1,2)}_{G^{(2)}}]$ for any $G^{(1)}, G^{(2)} \in \mathcal G^{(1,2)}_{\leq D}$, $\Gamma_{1,1}=1$, and $\Gamma_{1,G}=\mathbb{E}[\Psi^{(1,2)}_{G}]$ for $G\in \mathcal G^{(1,2)}_{\leq D}$, which is controlled by Lemma~\ref{lem:mean:Psi}. First, we bound the individual terms of $\Gamma$ by stating a counterpart of Proposition~\ref{prop:scal}.

 For this purpose, we need to define a variant of $d(G^{(1)},G^{(2)})$. Let 
 \begin{equation}\label{eq:definition:edit:distance:estimation}
d^{(1,2)}(G^{(1)}, G^{(2)}) := \min_{\mathbf M \in \mathcal M^{(1,2)}} |E_{\Delta}| \enspace .
\end{equation}
Note that $d^{(1,2)}(G^{(1)}, G^{(2)})=0$ if and only if $G^{(1)}$ and $G^{(2)}$ are equivalent.

 \begin{proposition}\label{prop:scal2}
  Fix $D\geq 2$.  Under \condinvariance, we assume that Conditions \condvariance, \condmoment and \condsignal are fulfilled and that the constant  $c_{\texttt{s}}>4$ is large compared to the other ones. Under \condinvarianceWR, we assume that Conditions \condvariance, \condmoment, \condmomentWR, and \condsignal are fulfilled and that the constant  $c_{\texttt{s}}>4$ is large compared to the other ones. There exist two positive constants $c$ and $c'$  depending on those arising in Conditions \condvariance, \condmoment, and possibly \condmomentWR such that the following holds for any templates $G^{(1)}, G^{(2)}\in \mathcal G_{\leq D}^{(1, 2)}$.
\begin{itemize}
    \item[1] if $G^{(1)} \neq G^{(2)}$:
    \begin{equation}\left|\mathbb E[\Psi^{(1,2)}_{G^{(1)}}\Psi^{(1,2)}_{G^{(2)}}]  \right|  \leq c D^{-c_{\texttt{s}} d^{(1,2)}(G^{(1)},G^{(2)})}\enspace ,\end{equation}
    \item[2] and if $G^{(1)} = G^{(2)}$:
    \begin{equation} \left|\mathbb E[(\Psi^{(1,2)}_{G^{(1)}})^2] -1 \right|  \leq c'D^{-c_{\texttt{s}}}\enspace .\end{equation}
\end{itemize}
 \end{proposition}

 Then, we establish that $\Gamma$ is diagonally dominant as in the proof of Theorem~\ref{thm:iso} (including the row and the column of $\Gamma$ associated to $1$, which are handled by Lemma~\ref{lem:mean:Psi}, the number of templates in $\mathcal{G}^{(1,2)}_{\leq D}$ with $e$ edges being at most $(2e+2)^{2e}\leq D^{6e}$), the only small difference being that we sum over templates in $\mathcal{G}^{(1,2)}_{\leq D}$ and that we consider the distance $d^{(1,2)}(G^{(1)},G^{(2)})$. To handle this, we first observe that, as long as $G^{(1)}\neq G^{(2)}$ in $\mathcal{G}^{(1,2)}_{\leq D}$, we have $d^{(1,2)}(G^{(1)},G^{(2)})\geq 1$. Also, given a positive integer $u$ and a template $G^{(1)}$, the number of templates $G^{(2)}$ in $\mathcal{G}^{(1,2)}_{\leq D}$ such that  $d^{(1,2)}( G^{(1)}, G^{(2)}) = u$ is bounded by $(2u+2D+2)^{2u}$, since $G^{(2)}$ is obtained from $G^{(1)}$, which has at most $2D+2$ nodes, by removing or adding $u$ edges between at most $2D+2+2u$ nodes. As $u\leq 2D$, we have $(2u+2D+2)^{2u}\leq (6D+2)^{2u}\leq D^{8u}$, as for templates in $\mathcal{G}_{\leq D}$. The rest of the proof is unchanged.

\begin{proof}[Proof of Proposition~\ref{prop:scal2}]

This proof closely follows that of Proposition~\ref{prop:scal} up to a few changes. First, we claim that the analogues of Lemmas~\ref{lem:gen-1} and~\ref{lem:gen-2}  still hold. The proof is postponed to the end of the subsection. 

 \begin{lemma}\label{lem:gen:est}
    Consider the same assumptions as in Lemma~\ref{lem:gen-1} or~\ref{lem:gen-2}.  
\begin{itemize}
    \item[1] Let $G^{(1)}, G^{(2)} \in \mathcal G^{(1,2)}_{\leq D}$ be two templates and let $\mathbf M \in \mathcal M^{(1,2)}\setminus  \mathcal{M}^{(1,2)}_{\mathrm{PM}}$ be a matching. For any $(\pi^{(1)},\pi^{(2)}) \in \Pi^{(1,2)}(\mathbf M)$, we have $\Big| \mathbb{E}\left[\overline{P}^{(1,2)}_{G^{(1)}, \pi^{(1)}} \overline{P}^{(1,2)}_{G^{(2)}, \pi^{(2)}}\right] \Big|\leq  \psi[G_{\Delta}]$ where, in this subsection, $\psi[G_{\Delta}]$ stands for the right-hand side of~\eqref{eq:psi:unified}, which upper bounds the quantities $\psi[G_{\Delta}]$ defined in Lemmas~\ref{lem:gen-1}~and~\ref{lem:gen-2}.
\item[2] Also, for any template $G=(V,E) \in \mathcal G^{(1,2)}_{\leq D}$ and any $\pi \in \Pi^{(1,2)}_V$, we have 
\begin{equation*}
        \left|\mathbb{E}\left[(\overline{P}^{(1,2)}_{G, \pi })^2\right]- \overline{q}^{|E|}\right| \leq  \left[2(1+c_{\texttt{v},2}\vee c_{\texttt{v},3}) D^{4+ c_{\texttt{v},1}\vee c_{\texttt{m}}}\frac{k}{n}+c_{\texttt{v},3}D^{-c_{\texttt{s}}}\right]\overline{q}^{|E|}\enspace .
    \end{equation*}
\end{itemize}
\end{lemma}

As in Step 1 from the proof of Proposition~\ref{prop:scal}, we start from the identity
\begin{align*}
    \mathbb E[\Psi^{(1,2)}_{G^{(1)}}\Psi^{(1,2)}_{G^{(2)}}]&= \frac{1}{\sqrt{\mathbb V^{(1,2)}(G^{(1)})\mathbb V^{(1,2)}(G^{(2)})}}\sum_{\pi^{(1)}\in \Pi^{(1,2)}_{V^{(1)}}, \pi^{(2)}\in \Pi^{(1,2)}_{V^{(2)}}}\mathbb E[\overline{P}^{(1,2)}_{G^{(1)},\pi^{(1)}} \overline{P}^{(1,2)}_{G^{(2)}, \pi^{(2)}}]\\ 
    &= \frac{1}{\sqrt{\mathbb V^{(1,2)}(G^{(1)})\mathbb V^{(1,2)}(G^{(2)})}} \sum_{\mathbf M \in \mathcal M^{(1,2)}}\sum_{(\pi^{(1)}, \pi^{(2)}) \in \Pi^{(1,2)}(\mathbf M)}\mathbb E[\overline{P}^{(1,2)}_{G^{(1)},\pi^{(1)}} \overline{P}^{(1,2)}_{G^{(2)}, \pi^{(2)}}]\enspace , 
\end{align*}
where $\mathcal M^{(1,2)}$ is the set of all matchings in $\mathcal M$ that contain $(v^{(1)}_1,v^{(2)}_1)$ and $(v^{(1)}_2,v^{(2)}_2)$ and where $\Pi^{(1,2)}(\mathbf M)$ is the set of all pairs $(\pi^{(1)}, \pi^{(2)}) \in \Pi(\mathbf M)$ such that $\pi^{(a)}(v_1^{(a)})=1, \pi^{(a)}(v_2^{(a)})=2$ for $a=1,2$.

Observe that  $|\Pi^{(1,2)}(\mathbf M)| = \frac{(n-2)!}{(n - |V^{(1)}| )!}$ for a perfect matching.  
We proceed similarly to the proof of Proposition~\ref{prop:scal}. For $G^{(1)} \neq G^{(2)}$, we have 
\begin{equation}\label{eq:covariance:estimation}
\left| \mathbb E[\Psi^{(1,2)}_{G^{(1)}}\Psi^{(1,2)}_{G^{(2)}}]\right| \leq A^{(1,2)}\enspace  ,
\end{equation}
whereas, for  $G^{(1)} = G^{(2)}$, we have 
\begin{equation}\label{eq:variance:estimation}
| \mathbb E[\Psi^{(1,2)}_{G^{(1)}}\Psi^{(1,2)}_{G^{(2)}}]-1| \leq A^{(1,2)} + B^{(1,2)}\enspace ,
\end{equation}
where
\begin{equation}A^{(1,2)} := \Bigg| \frac{1}{\sqrt{\mathbb V^{(1,2)}(G^{(1)})\mathbb V^{(1,2)}(G^{(2)})}}\sum_{\mathbf M \in \mathcal M^{(1,2)}\setminus \mathcal M_{\mathrm{PM}}^{(1,2)}}\sum_{(\pi^{(1)}, \pi^{(2)}) \in \Pi^{(1,2)}(\mathbf M)}\mathbb E[\overline{P}_{G^{(1)},\pi^{(1)}}^{(1,2)} \overline{P}_{G^{(2)}, \pi^{(2)}}^{(1,2)}]\Bigg|\enspace ;\end{equation}
\begin{equation}B^{(1,2)} := \mathbf 1\{G^{(1)} = G^{(2)}\}\Bigg| \frac{1}{\overline{q}^{|E^{(1)}|}}\mathbb E[(\overline{P}_{G^{(1)},\pi^{(1)}}^{(1,2)})^2] - 1\Bigg|\enspace ,\end{equation}
where the last quantity does not depend on the choice of  $\pi^{(1)} \in \Pi^{(1,2)}_{V^{(1)}}$. By the second part of Lemma~\ref{lem:gen:est} and Condition~\condsignal, arguing as for~\eqref{eq:prop7:same_skel}, we have
\begin{equation}\label{eq:upper:B_12}
B^{(1,2)}\leq c_0 D^{-c_{\texttt{s}}}\enspace . 
\end{equation}
Recall that  $|\Pi^{(1,2)}(\mathbf{M})| = \frac{(n-2)!}{(n - (|V^{(1)}|+ |V^{(2)}| - |\mathbf M| ))!}$. By definition~\eqref{eqn:variance of graph:estimation}  of $\mathbb V^{(1,2)}(G)$ and Lemma~\ref{lem:gen:est}, we get  
\begin{align*}
       A^{(1,2)} &\leq \frac{1}{\overline{q}^{(|E^{(1)}|+ |E^{(2)}|)/2}\sqrt{|\mathrm{Aut}^{(1,2)}(G^{(1)})| |\mathrm{Aut}^{(1,2)}(G^{(2)})|}}  \sum_{\mathbf M \in \mathcal M^{(1,2)}\setminus \mathcal M^{(1,2)}_{\mathrm{PM}}}\frac{\sqrt{(n-|V^{(1)}|)!(n-|V^{(2)}|)!}}{(n - (|V^{(1)}|+ |V^{(2)}| - |\mathbf M|))!} \psi[G_{\Delta}]\enspace .
\end{align*}
Then, arguing as at the end of Step 1 in the proof of Proposition~\ref{prop:scal}, 
we get 
\begin{align*}
       A^{(1,2)} &\leq \frac{1}{\overline{q}^{(|E^{(1)}|+ |E^{(2)}|)/2}\sqrt{|\mathrm{Aut}^{(1,2)}(G^{(1)})| |\mathrm{Aut}^{(1,2)}(G^{(2)})|}}  \sum_{\mathbf M \in \mathcal M^{(1,2)}\setminus \mathcal M^{(1,2)}_{\mathrm{PM}}}n^{ (|U^{(1)}|+|U^{(2)}|)/2}  \psi[G_{\Delta}]\enspace . 
\end{align*}
Then, Step 3 of the proof of Proposition~\ref{prop:scal} is still valid upon replacing $d(G^{(1)},G^{(2)})$ by $d^{(1,2)}(G^{(1)},G^{(2)})$, since $|E_{\Delta}|\geq d^{(1,2)}(G^{(1)},G^{(2)})\vee 1$ for any $\mathbf{M}\in \mathcal M^{(1,2)}\setminus \mathcal M^{(1,2)}_{\mathrm{PM}}$ (if $E_{\Delta}=\emptyset$, all the nodes are perfectly matched). Lemma~\ref{lem:countgra} and the identities of Step 3 remain valid, as $v_1$ and $v_2$ are always matched, an isolated node $v_a$ being perfectly matched if it is isolated in both templates, and semi-matched otherwise.  We then proceed as in Steps 2 and 3 of the proof of Proposition~\ref{prop:scal}.
We obtain
\begin{align*}
       A^{(1,2)} &\leq   \frac{2 c_{0}}{\sqrt{|\mathrm{Aut}^{(1,2)}(G^{(1)})| |\mathrm{Aut}^{(1,2)}(G^{(2)})|}} 
       \sum_{\mathbf M \in \mathcal M^{(1,2)}\setminus \mathcal M^{(1,2)}_{\mathrm{PM}}}\left(D^{-2c_{\texttt{s}} }\right)^{ [U+ |\mathbf M_{\mathrm{SM}}|]  \lor d^{(1,2)}(G^{(1)},G^{(2)})\lor 1} \enspace , 
\end{align*}
for some constant $c_0$. 
We conclude the proof by a slight modification of Step 4 of the proof of Proposition~\ref{prop:scal}. As in the latter, we enumerate all possible matchings corresponding to any possible shadow:
\begin{align*}
       A^{(1,2)} &\leq  \frac{2c_0}{\sqrt{|\mathrm{Aut}^{(1,2)}(G^{(1)})| |\mathrm{Aut}^{(1,2)}(G^{(2)})|}}\\
       &\sum_{\substack{U^{(1)} \subset V^{(1)}\setminus \{v_1^{(1)},v_2^{(1)}\},\\ U^{(2)} \subset V^{(2)}\setminus \{v_1^{(2)},v_2^{(2)}\},\\ \underline{\mathbf M} \in \mathcal M\setminus \mathcal M_{\mathrm{PM}}}}\quad
       \sum_{\mathbf M \in \mathcal M_{\mathrm{shadow}}^{(1,2)}(U^{(1)},U^{(2)}, \underline{\mathbf M})}\left(D^{-2c_{\texttt{s}} }\right)^{ [U+ |\mathbf M_{\mathrm{SM}}|]  \lor d^{(1,2)}(G^{(1)},G^{(2)})\lor 1}\enspace ,
\end{align*}
where $\mathcal M_{\mathrm{shadow}}^{(1,2)}(U^{(1)},U^{(2)}, \underline{\mathbf M}) = \mathcal M_{\mathrm{shadow}}(U^{(1)},U^{(2)}, \underline{\mathbf M}) \cap  \mathcal M^{(1,2)}$.
Similarly to Lemma~\ref{lem:shadow}, we have the following bound. Indeed, since all the matchings in $\mathcal M^{(1,2)}$ contain $(v^{(1)}_1,v^{(2)}_1)$ and $(v^{(1)}_2,v^{(2)}_2)$, the automorphisms of $G^{(1)}$ built in the proof of Lemma~\ref{lem:shadow} fix $v^{(1)}_1$ and $v^{(1)}_2$, and similarly for $G^{(2)}$.
\begin{equation}|\mathcal M^{(1,2)}_{\mathrm{shadow}}(U^{(1)},U^{(2)}, \underline{\mathbf M})| \leq \min(|\mathrm{Aut}^{(1,2)}(G^{(1)})|, |\mathrm{Aut}^{(1,2)}(G^{(2)})|) \enspace . \end{equation}
Then, as at the end of the proof of Proposition~\ref{prop:scal}, we conclude that 
\begin{equation*}
  A^{(1,2)}\leq c D^{-c_{\texttt{s}} [d^{(1,2)}(G^{(1)},G^{(2)})\vee 1]}\enspace .  
\end{equation*}
Coming back to~\eqref{eq:covariance:estimation}, we have established the first part of the proposition. The second part of the proposition follows from the latter inequality together with~\eqref{eq:variance:estimation}
and~\eqref{eq:upper:B_12}.

\end{proof}

\begin{proof}[Proof of Lemma~\ref{lem:gen:est}]
For $a=1,2$, write $G^{'(a)}$ for the graph obtained from $G^{(a)}$ by removing $v^{(a)}_1$ and/or $v^{(a)}_2$ when they are isolated in $G^{(a)}$, and $\pi^{'(a)}$ for the restriction of $\pi^{(a)}$ to the nodes of $G^{'(a)}$. By definition, $\overline{P}^{(1,2)}_{G^{(a)},\pi^{(a)}}= \overline{P}_{G^{'(a)},\pi^{'(a)}}$, and $G^{'(a)}$ has at least one edge and no isolated node, so that it is isomorphic to a template in $\mathcal{G}_{\leq D}$ with the same number of edges. In particular, the second part of the lemma follows from the second part of Lemma~\ref{lem:gen-1} applied to $(G',\pi')$.

Let us turn to the first part. We have $(\pi^{'(1)},\pi^{'(2)})\in \Pi(\mathbf{M}')$, where $\mathbf{M}'$ is obtained from $\mathbf{M}$ by removing the pairs that involve a removed node, and we write $G'_{\Delta}$ for the corresponding symmetric difference graph. Since the removed nodes have no incident edge, the labeled graphs $\pi^{'(a)}[G^{'(a)}]$ and $\pi^{(a)}[G^{(a)}]$ have the same edges. Hence, $G'_{\Delta}=G_{\Delta}$, and the corresponding intersection graphs have the same edges. In particular, $E'_{\Delta}=E_{\Delta}$ is non-empty, as $\mathbf{M}\notin \mathcal{M}^{(1,2)}_{\mathrm{PM}}$, so that $\mathbf{M}'\notin \mathcal{M}_{\mathrm{PM}}$. The only difference between the two configurations is that the nodes $\pi^{(a)}(v^{(a)}_b)$ that were matched to a removed node are unmatched in $\mathbf{M}'$. As a consequence, a pure connected component of $G_{\Delta}$ (with respect to $\mathbf{M}$) is still pure with respect to $\mathbf{M}'$, and the number $\#\mathrm{CC}'_{\mathrm{pure}}$ of pure connected components with respect to $\mathbf{M}'$ is at least $\#\mathrm{CC}_{\mathrm{pure}}$. Applying Lemmas~\ref{lem:gen-1}~and~\ref{lem:gen-2} to $(G^{'(1)},\pi^{'(1)})$ and $(G^{'(2)},\pi^{'(2)})$, together with~\eqref{eq:psi:unified}, we obtain the bound~\eqref{eq:psi:unified} with $\#\mathrm{CC}'_{\mathrm{pure}}$ in place of $\#\mathrm{CC}_{\mathrm{pure}}$. Since $c_0D^{c_1}\leq \sqrt{n}$ by Condition~\condsignal\ with $c_{\texttt{s}}$ large enough, and since $\#\mathrm{CC}'_{\mathrm{pure}}\geq \#\mathrm{CC}_{\mathrm{pure}}$, this bound is at most $\psi[G_{\Delta}]$.
\end{proof}

\subsection{Proof of Theorem~\ref{thm:lowdeg2}}

It follows from Lemma~\ref{lem:reduction:degree2}  and Theorem~\ref{thm:isorefo2} that 
\begin{align*}
    \mathrm{Corr}^2_{\leq D} &\leq [1-cD^{-2}]^{-1} \sup_{\alpha}\frac{\left(\alpha_{\emptyset}\mathbb{E}[x] + \sum_{G\in \mathcal{G}^{(1,2)}_{\leq D}}\alpha_G \mathbb{E}[x\Psi^{(1,2)}_{G}]  \right)^2}{\|\alpha\|_2^2} =[1-cD^{-2}]^{-1}\left[\mathbb{E}[x]^2 +  \sum_{G\in \mathcal{G}^{(1,2)}_{\leq D}} \mathbb{E}[x\Psi^{(1,2)}_{G}]^2\right]
\enspace . 
\end{align*}

As for~\eqref{eq:upper:adv:lowdeg}, we used the Cauchy-Schwarz inequality and we assume that $c_0$ is large enough so that the constant $c$ of Theorem~\ref{thm:isorefo2} satisfies $cD^{-2}\leq 1/2$, which is possible by taking $c_{\texttt{s}}$ large enough in Theorem~\ref{thm:iso-estimation}. We readily have $\mathbb{E}[x]\leq (k+1)/n\leq 2k/n$ for all six models and even $\mathbb{E}[x]\leq k^2/n^2$ for \HidSubI\ and \HidSubP.  Hence, we mainly need to bound the first moments squared $\mathbb{E}[x\Psi^{(1,2)}_{G}]^2$, which is done in the following lemma. 
\begin{lemma}\label{lem:adv}
    Under the assumptions of Theorem~\ref{thm:lowdeg2} and for $c_{0}>0$ a large enough universal constant, all six models satisfy  
\begin{align*}
|\mathbb E[x \Psi^{(1,2)}_G]|
&\leq \frac{k}{n}D^{-c_{0}|E|/2}\enspace .
\end{align*}
\end{lemma}

For \HidSubI\ and \HidSubP, $\mathbb{E}[x]$ is of the order of $k^2/n^2$, so that we need the following sharper bound.

\begin{lemma}\label{lem:adv2}
    Under the assumptions of Theorem~\ref{thm:lowdeg2} and for $c_{0}>0$ a large enough universal constant, \HidSubI\ and \HidSubP\ satisfy
\begin{align*}
|\mathbb E[x \Psi^{(1,2)}_G]|
&\leq \frac{k^2}{n^2}D^{-c_{0}|E|/2}\enspace .
\end{align*}
\end{lemma}

Let us finish the proof before showing the lemmas. If $x=0$ almost surely, which happens for \HidSubP\ and \StoBloP\ when $k=1$, then $\mathrm{Corr}_{\leq D}=0=\mathbb{E}[x]$ and there is nothing to prove. Otherwise, we have
\begin{equation}\label{eq:lower:Ex}
\mathbb{E}[x]\geq \frac{k}{2n}\ \text{ for \StoBloI, \StoBloP, \ToeSerI, \ToeSerP}\ ;\quad \quad \mathbb{E}[x]\geq \frac{k^2}{2n^2}\ \text{ for \HidSubI, \HidSubP}\enspace .
\end{equation}
Indeed, $\mathbb{E}[x]=k/n$ for \StoBloI, $\mathbb{E}[x]=(k-1)/(n-1)\geq k/(2n)$ for \StoBloP\ with $k\geq 2$, $\mathbb{E}[x]=k^2/n^2$ for \HidSubI, and $\mathbb{E}[x]=k(k-1)/[n(n-1)]\geq k^2/(2n^2)$ for \HidSubP\ with $k\geq 2$. For \ToeSerI\ (resp. \ToeSerP), given $z_1$, at least $k/2+1$ of the $n$ (resp. $k/2$ of the $n-1$) possible values of $z_2$ satisfy $|z_1-z_2|\leq k/2$, so that $\mathbb{E}[x]\geq k/(2n)$.

A template in $\mathcal{G}^{(1,2)}_{\leq D}$ with $e$ edges has at most $2e+2$ nodes, so that there are at most $(2e+2)^{2e}\leq D^{6e}$ such templates. Hence, provided that $c_0\geq 9$, Lemma~\ref{lem:adv} leads to
\begin{equation*}
\sum_{G\in \mathcal{G}^{(1,2)}_{\leq D}}\mathbb{E}[x\Psi^{(1,2)}_{G}]^2\leq \frac{k^2}{n^2}\sum_{e\geq 1}D^{(6-c_0)e}\leq \frac{k^2}{n^2}D^{-2}\enspace ,
\end{equation*}
and, similarly, Lemma~\ref{lem:adv2} leads to $\sum_{G}\mathbb{E}[x\Psi^{(1,2)}_{G}]^2\leq k^4n^{-4}D^{-2}$ for \HidSubI\ and \HidSubP. By~\eqref{eq:lower:Ex}, we have in all six models $\sum_{G}\mathbb{E}[x\Psi^{(1,2)}_{G}]^2\leq 4D^{-2}\mathbb{E}[x]^2$. Since $(1-cD^{-2})^{-1}\leq 1+2cD^{-2}$, we conclude that
\begin{equation*}
\mathrm{Corr}^2_{\leq D}\leq (1+2cD^{-2})(1+4D^{-2})\mathbb{E}[x]^2\leq (1+c'D^{-2})\mathbb{E}[x]^2\enspace ,
\end{equation*}
so that $\mathrm{Corr}_{\leq D}\leq (1+c'D^{-2})\mathbb{E}[x]\leq (1+c'D^{-1})\mathbb{E}[x]$.

\begin{proof}[Proof of Lemma~\ref{lem:adv2}]
For both \HidSubI\ and \HidSubP,  we argue as in Section~\ref{sec:proof:model:conditions} to bound the first and second moments of polynomials. By~\eqref{eq:first_moment_HS}, \eqref{eq:second_moment_HS_R_bound}, and~\eqref{eq:control:first:moment:HSF}, we have, for both models, under $\mathbb{E}$ and under $\widetilde{\mathbb{E}}$,
  \begin{align*}
        \left|\mathbb{E}\left[P_{G, \pi}\right]\right|& \leq
        \lambda^{|E|} \left(\frac{k}{n}\right)^{|V|}\enspace ;\\
        \Big| \mathbb{E}\left[P_{G^{(1)}, \pi^{(1)}} P_{G^{(2)}, \pi^{(2)}}\right] \Big|& \leq  \lambda^{|E_{\Delta}|} \overline{p}^{|E_{\cap}|} \left(\frac{k}{n}\right)^{|V_{\Delta}|}\enspace .
\end{align*}
Also, for \HidSubP, arguing as for~\eqref{eq:bound:C_W}, but rooting one of the trees of the forest $F$ at a node of $V_{\Delta}$ (which is non-empty since $\mathbf{M}\notin \mathcal{M}_{\mathrm{PM}}$), whose latent assignment is at most $k$ with probability $k/n$, we get
 \begin{equation*}
        \Big| \widetilde{\mathbb{E}}\left[\1\{\mathcal{A}\}P_{G^{(1)}, \pi^{(1)}} P_{G^{(2)}, \pi^{(2)}}\right] \Big| \leq c_{\texttt{vd},2}D^{c_{\texttt{vd},1}}\lambda^{|E_{\Delta}|} \overline{p}^{|E_{\cap}|} \left(D^{c_{\texttt{vd},1}}\frac{k}{n}\right)^{1+ |V_{\Delta}| - \#\mathrm{CC}_{\Delta} }\left(c_{\texttt{vd},2}\frac{D^{c_{\texttt{vd},1}}}{\sqrt{n}}\right)^{\#\mathrm{CC}_{\mathrm{pure}}}\enspace .
    \end{equation*}
Since $\#\mathrm{CC}_G\geq 1$ and since $\#\mathrm{CC}_{\Delta}\geq 1$ for $\mathbf{M}\notin \mathcal{M}_{\mathrm{PM}}$, these three bounds are smaller by a factor $k/n$ than their counterparts in \condmoment, in the first part of \condvariance, and in \condmomentWR. In the proofs of Lemmas~\ref{lem:gen-1}~and~\ref{lem:gen-2}, each term of the expansions involves at least one factor controlled by one of these three bounds: either a removed connected component, or a pure connected component $H_a$, controlled by \condmoment\ or \condmomentWR, or a configuration that is not perfectly matched, controlled by \condvariance. Hence, for \HidSubI\ and \HidSubP, the bounds of Lemmas~\ref{lem:gen-1}~and~\ref{lem:gen-2}, and therefore the bound~\eqref{eq:psi:unified}, hold with an additional factor $k/n$. Besides, $\mathbb{E}[x]\leq k^2/n^2$ for both models. Plugging these bounds in the proof of Lemma~\ref{lem:adv} below, each bound in this proof gets an additional factor $k/n$, and the result follows.
\end{proof}

\begin{proof}[Proof of Lemma~\ref{lem:adv}]
As a warmup, consider the case where $E = \{(v_1,v_2)\}$ so that $G$ only contains two nodes. We have 
\begin{equation}|\mathbb{E}[x\Psi_G^{(1, 2)}]| = \frac{\mathbb{E}[x\lambda] - \mathbb{E}[x]\mathbb{E}[x\lambda]}{\sqrt{\overline{q}}}\leq \frac{\lambda}{\sqrt{\overline{q}}}\mathbb{E}[x]\end{equation} 
in all six models. Since $q\leq 1/2$, we have $\lambda/\sqrt{\overline{q}}= (\lambda/q)\sqrt{q/(1-q)}\leq \lambda/q\leq D^{-c_0}$ by the signal condition in Theorem~\ref{thm:lowdeg2}. Since $\mathbb{E}[x]\leq 2k/n$, we deduce that $|\mathbb E[x \Psi^{(1,2)}_G]|\leq 2D^{-c_0}k/n\leq D^{-c_{0}/2} k/n $ (and even $|\mathbb E[x \Psi^{(1,2)}_G]|\leq D^{-c_{0}/2} k^2/n^2$ for \HidSubI\ and \HidSubP, since then $\mathbb{E}[x]\leq k^2/n^2$).

We now turn to templates $G\in \mathcal{G}_{\leq D}^{(1,2)}$ with $|V|\geq 3$ nodes. We consider two cases depending on the connections between $v_1$ and $v_2$.

\subsubsection{Case 1: $(v_1,v_2) \not\in E$.} 
 Let $\pi$ be any labeling in $\Pi^{(1,2)}_V$. If $v_1$ or $v_2$ is isolated in $G$, we prune $(G,\pi)$ into $(G',\pi')$ by removing the isolated node(s); otherwise, we take $(G',\pi')=(G,\pi)$. In this way, $\overline{P}^{(1,2)}_{G,\pi}= \overline{P}_{G',\pi'}$. Besides, we define $G^{(0)}$ as the template that only contains the edge $(v_1,v_2)$ and $\pi^{(0)}$ such that $\pi^{(0)}(v_1)=1$ and $\pi^{(0)}(v_2)=2$. We have for all models that 
 \begin{equation*}
 \mathbb{E}[x\overline{P}^{(1,2)}_{G,\pi}] = \frac{1}{\lambda}\mathbb{E}[P_{G^{(0)},\pi^{(0)}}\overline{P}_{G',\pi'}]\enspace  . 
 \end{equation*}
Coming back to the definition of $\Psi^{(1,2)}_{G}$, this leads us to 
\begin{align}
\left|\mathbb{E}[x\Psi^{(1,2)}_{G}]\right|\leq \frac{n^{|V|/2-1}}{\overline{q}^{|E|/2}}\left(\frac{1}{\lambda}\left|\mathbb{E}[\overline{P}_{G^{(0)},\pi^{(0)}}\overline{P}_{G',\pi'}]\right| + \mathbb{E}[x]\left|\mathbb{E}[\overline{P}_{G',\pi'}]\right| \right) \enspace . \label{eq:upper:global}
\end{align}
By Proposition~\ref{prp:model:conditions}, provided that we choose $c_0$ large enough in the  statement of Theorem~\ref{thm:lowdeg2}, all our six models satisfy Conditions~\condmoment and \condvariance, as well as \condmomentWR for \HidSubP, \StoBloP, \ToeSerP\ for some numerical constants and we are therefore in a position to apply Lemmas~\ref{lem:gen-1}~and~\ref{lem:gen-2} to all six models. To bound $\mathbb{E}[\overline{P}_{G',\pi'}]$, which is zero when $G'$ is connected or under \condinvariance, we rely, as in the proof of Lemma~\ref{lem:mean:Psi}, on the bound~\eqref{eq:obj_A}, which corresponds to the configuration made of $(G',\pi')$ and of the empty graph. Besides, we use that $\mathbb{E}[x]\leq 2k/n$.
\begin{align} \nonumber
 \frac{n^{|V|/2-1}}{\overline{q}^{|E|/2}}\mathbb{E}[x]\left|\mathbb{E}[\overline{P}_{G',\pi'}]\right|&\leq  c  D^{2} \frac{k}{n}     \left(D^{c'}\frac{\lambda}{\sqrt{\overline{q}}}\right)^{|E|} n^{|V|/2-1} \left(\frac{D^{c' }k}{n}\right)^{|V'| - \#\mathrm{CC}_{G'}}\left[c \frac{D^{c'}}{\sqrt{n}} \right]^{\#\mathrm{CC}_{G'}}\\  \nonumber
 &\leq  c  D^{2} \frac{k}{n}     \left(D^{c'}\frac{\lambda}{\sqrt{\overline{q}}}\right)^{|E|} \left[ \frac{D^{c'}k}{\sqrt{n}} \right]^{|V'| - \#\mathrm{CC}_{G'}}\left(c D^{c'}\right)^{\#\mathrm{CC}_{G'}}\\  \nonumber
 &\leq  c  D^{2} \frac{k}{n}     \left(D^{c'}\frac{\lambda}{\sqrt{\overline{q}}}\right)^{|E|-|V'|+ \#\mathrm{CC}_{G'}} \left[ \frac{D^{2c'}k\lambda}{\sqrt{n \overline{q}}} \right]^{|V'| - \#\mathrm{CC}_{G'}}\left(c D^{c'}\right)^{\#\mathrm{CC}_{G'}}\\  \nonumber
 &\leq  c  D^{2} \frac{k}{n}     \left(D^{c'}\frac{\lambda}{\sqrt{\overline{q}}}\right)^{|E|-|V'|+ \#\mathrm{CC}_{G'}} \left[ c\frac{D^{3c'}k\lambda}{\sqrt{n \overline{q}}} \right]^{|V'| - \#\mathrm{CC}_{G'}}\\  
 &\leq  \frac{1}{2}\frac{k}{n} D^{-c_0|E|/2}\enspace , \label{eq:upper:global:2}  
\end{align}
where we used in the second  line that $|V'|\geq |V|-2$ and that, for $c_0$ large enough, we have  $c D^{c'}k/n\leq 1$ and we used in the penultimate line that 
$|V'|\geq 2\#\mathrm{CC}_{G'}$. In the last line, we used the conditions of Theorem~\ref{thm:lowdeg2} as well as the fact that $|E|\geq |V'|- \#\mathrm{CC}_{G'}$ and $c_0$ is large enough.

Let us turn to the first term in~\eqref{eq:upper:global}. We again apply Lemma~\ref{lem:gen-2}. 
\begin{align}\label{eq:upper:global:3}
    \frac{n^{|V|/2-1}}{\lambda \overline{q}^{|E|/2}}\left|\mathbb{E}[\overline{P}_{G^{(0)},\pi^{(0)}}\overline{P}_{G',\pi'}]\right|
    &\leq  c  D^{2+c'}  \left(D^{c'}\frac{\lambda}{\sqrt{\overline{q}}}\right)^{|E|} n^{|V|/2-1} \left(\frac{D^{c' }k}{n}\right)^{|V|- a}\left[c \frac{D^{c'}}{\sqrt{n}} \right]^{b}\enspace , 
\end{align}    
where $a$ corresponds to the number of connected components in the concatenation of $\pi'[G']$ and $\pi^{(0)}[G^{(0)}]$ and $b$ corresponds to the number of pure connected components in the same graph. Note that $a$ and $b$ depend on the connection of $v_1$ and $v_2$ in $G$. We consider four subcases.

\medskip 

\noindent 
{\bf Case 1-a:} both $v_1$ and $v_2$ are isolated in $G$. In this case, $|V'|= |V| -2$, $b= \#\mathrm{CC}(G')+1$ and $a= \#\mathrm{CC}(G')+1$. We deduce from~\eqref{eq:upper:global:3} 
that 
\begin{align} \nonumber
    \frac{n^{|V|/2-1}}{\lambda \overline{q}^{|E|/2}}\left|\mathbb{E}[\overline{P}_{G^{(0)},\pi^{(0)}}\overline{P}_{G',\pi'}]\right|
    &\leq  c  D^{2+c'}      \left(D^{c'}\frac{\lambda}{\sqrt{\overline{q}}}\right)^{|E|} n^{|V'|/2} \left(\frac{D^{c' }k}{n}\right)^{|V'|+1 - \#\mathrm{CC}(G') }\left[c \frac{D^{c'}}{\sqrt{n}} \right]^{\#\mathrm{CC}(G')+1} \\  \nonumber
    &\leq  c^2  D^{2+3c'} \frac{k}{n} \left(D^{c'}\frac{\lambda}{\sqrt{\overline{q}}}\right)^{|E|}\left(\frac{D^{c' }k}{\sqrt{n}}\right)^{|V'|- \#\mathrm{CC}(G') }\left[c D^{c'} \right]^{\#\mathrm{CC}(G')}  \frac{1}{\sqrt{n}} \\ \nonumber
    &\leq  c^2  D^{2+3c'} \frac{k}{n} \left(D^{c'}\frac{\lambda}{\sqrt{\overline{q}}}\right)^{|E|-|V'| +  \#\mathrm{CC}(G')}\left(c\frac{D^{2c' }k\lambda }{\sqrt{n\overline{q}}}\right)^{|V'|- \#\mathrm{CC}(G') }\left[c D^{c'}\right]^{\#\mathrm{CC}(G')}\frac{1}{\sqrt{n}}\\ \nonumber
    &\leq  c^2 D^{2+3c'} \frac{k}{n} \left(D^{c'}\frac{\lambda}{\sqrt{\overline{q}}}\right)^{|E|-|V'| +  \#\mathrm{CC}(G')}\left(c\frac{D^{3c' }k\lambda }{\sqrt{n\overline{q}}}\right)^{|V'|- \#\mathrm{CC}(G') }\frac{1}{\sqrt{n}}\\
    &\leq \frac{1}{2}\frac{k}{n} D^{-c_{0}|E|/2}\enspace , \label{eq:upper:global:3-1}  
\end{align}
where we argued as for~\eqref{eq:upper:global:2}. 
\medskip 

\noindent 
{\bf Case 1-b:}  $v_1$ or $v_2$ is isolated in $G$, but not both of them. In this case, $|V'|= |V|-1$, $a=\#\mathrm{CC}(G')$, and $b= \#\mathrm{CC}(G')-1$. Arguing as previously, we deduce  from~\eqref{eq:upper:global:3} 
that 
\begin{align} \nonumber
    \frac{n^{|V|/2-1}}{\lambda \overline{q}^{|E|/2}}\left|\mathbb{E}[\overline{P}_{G^{(0)},\pi^{(0)}}\overline{P}_{G',\pi'}]\right|
    &\leq  c  D^{2+c'}      \left(D^{c'}\frac{\lambda}{\sqrt{\overline{q}}}\right)^{|E|} n^{|V'|/2-1/2} \left(\frac{D^{c' }k}{n}\right)^{|V'|+1 - \#\mathrm{CC}(G') }\left[c \frac{D^{c'}}{\sqrt{n}} \right]^{\#\mathrm{CC}(G')-1} \\  \nonumber
    &\leq  c  D^{2+2c'} \frac{k}{n} \left(D^{c'}\frac{\lambda}{\sqrt{\overline{q}}}\right)^{|E|}\left(c\frac{D^{2c' }k}{\sqrt{n}}\right)^{|V'|- \#\mathrm{CC}(G') }  \\ \nonumber
    &\leq  c  D^{2+2c'} \frac{k}{n} \left(D^{c'}\frac{\lambda}{\sqrt{\overline{q}}}\right)^{|E|-|V'| +  \#\mathrm{CC}(G')}\left(c\frac{D^{3c' }k\lambda }{\sqrt{n\overline{q}}}\right)^{|V'|- \#\mathrm{CC}(G') }\\
    &\leq \frac{1}{2}\frac{k}{n} D^{-c_{0}|E|/2}\label{eq:upper:global:3-2} \enspace . 
\end{align}
\medskip 

\noindent 
{\bf Case 1-c:}  $v_1$ and $v_2$ are not isolated in $G$, but they do not belong to the same connected component. In this case, $|V'|= |V|$, $a= \#\mathrm{CC}(G')-1$, and $b= \#\mathrm{CC}(G')-2$. Arguing as previously, we deduce  from~\eqref{eq:upper:global:3} 
that 
\begin{align} \nonumber
    \frac{n^{|V|/2-1}}{\lambda \overline{q}^{|E|/2}}\left|\mathbb{E}[\overline{P}_{G^{(0)},\pi^{(0)}}\overline{P}_{G',\pi'}]\right|
    &\leq  c  D^{2+c'}      \left(D^{c'}\frac{\lambda}{\sqrt{\overline{q}}}\right)^{|E|} n^{|V'|/2-1} \left(\frac{D^{c' }k}{n}\right)^{|V'|+1 - \#\mathrm{CC}(G') }\left[c \frac{D^{c'}}{\sqrt{n}} \right]^{\#\mathrm{CC}(G')-2} \\  \nonumber
    &\leq  c  D^{2+2c'} \frac{k}{n} \left(D^{c'}\frac{\lambda}{\sqrt{\overline{q}}}\right)^{|E|}\left(\frac{cD^{2c' }k}{\sqrt{n}}\right)^{|V'|- \#\mathrm{CC}(G') } \\ \nonumber
    &\leq  c  D^{2+2c'} \frac{k}{n} \left(D^{c'}\frac{\lambda}{\sqrt{\overline{q}}}\right)^{|E|-|V'| +  \#\mathrm{CC}(G')}\left(c\frac{D^{3c' }k\lambda }{\sqrt{n\overline{q}}}\right)^{|V'|- \#\mathrm{CC}(G') }\\
    &\leq \frac{1}{2}\frac{k}{n} D^{-c_{0}|E|/2}\enspace . \label{eq:upper:global:3-3} 
\end{align}
\medskip 

\noindent 
{\bf Case 1-d:}  $v_1$ and $v_2$ belong to the same connected component in $G$. In this case, $|V'|= |V|$, $a= \#\mathrm{CC}(G')$, and $b= \#\mathrm{CC}(G')-1$. Arguing as previously, we deduce  from~\eqref{eq:upper:global:3} 
that 
\begin{align} \nonumber
    \frac{n^{|V|/2-1}}{\lambda \overline{q}^{|E|/2}}\left|\mathbb{E}[\overline{P}_{G^{(0)},\pi^{(0)}}\overline{P}_{G',\pi'}]\right|
    &\leq  c  D^{2+c'}    \left(D^{c'}\frac{\lambda}{\sqrt{\overline{q}}}\right)^{|E|} n^{|V'|/2-1} \left(\frac{D^{c' }k}{n}\right)^{|V'| - \#\mathrm{CC}(G') }\left[c \frac{D^{c'}}{\sqrt{n}} \right]^{\#\mathrm{CC}(G')-1} \\  \nonumber
    &\leq  c  D^{2+c'} \frac{1}{\sqrt{n}}  \left(D^{c'}\frac{\lambda}{\sqrt{\overline{q}}}\right)^{|E|}\left(c\frac{D^{2c' }k}{\sqrt{n}}\right)^{|V'|- \#\mathrm{CC}(G') }\\ 
&\leq 
c^2  D^{2+3c'}\frac{k}{n}  \left(D^{c'}\frac{\lambda}{\sqrt{\overline{q}}}\right)^{|E|+1 - |V'|+  \#\mathrm{CC}(G')}\left(\frac{cD^{3c' }\lambda k}{\sqrt{n\overline{q} }}\right)^{|V'|-1- \#\mathrm{CC}(G')} \nonumber \\ 
    &\leq \frac{k}{2n}  D^{-c_{0}|E|/2}\enspace , \label{eq:upper:global:3-4} 
\end{align}
where we used again in the last line that all connected components have at least two nodes and we used the conditions on $\lambda$ from the statement of Theorem~\ref{thm:lowdeg2}.
Then, gathering~\eqref{eq:upper:global}--\eqref{eq:upper:global:3-4} concludes the proof in Case 1.

\subsubsection{Case 2: $(v_1,v_2) \in E$.} We decompose $G$ into $G^{(1)}$ and $G'$ where $G^{(1)}$  corresponds to the connected component of $G$ that contains both $v_1$ and $v_2$, whereas $G'$ contains all the other connected components.  We only consider the case where $G'$ is non-empty, the case where $G$ has only one connected component being similar. 
Fix any $\pi\in \Pi^{(1,2)}_V$ and write $\pi^{(1)}$ and $\pi'$ for the corresponding restrictions of the labeling to $V^{(1)}$ and $V'$. By definition of the polynomials $\Psi^{(1,2)}_G$ and  
$\overline{P}^{(1,2)}_G$, we have 
\begin{align}\nonumber
|\mathbb E[x \Psi^{(1,2)}_G]| &= \sqrt{\frac{(n-2)!}{(n-|V|)! |\mathrm{Aut}^{(1,2)}(G)| \overline{q}^{|E|}}}\left|\mathbb E[xP_{G^{(1)},\pi^{(1)}} \overline{P}_{G',\pi'}] - \mathbb E[P_{G^{(1)},\pi^{(1)}}]\mathbb{E}[x\overline{P}_{G',\pi'}] \right|\\ \nonumber
&\leq  \frac{n^{|V|/2-1}}{\overline{q}^{|E|/2}}\left|\mathbb E[P_{G^{(1)},\pi^{(1)}} \overline{P}_{G',\pi'}] - \mathbb E[P_{G^{(1)},\pi^{(1)}}]\mathbb{E}[x\overline{P}_{G',\pi'}] \right| \\
&\leq   \frac{n^{|V|/2-1}}{\overline{q}^{|E|/2}}\left[\left|\mathbb E[\overline{P}_{G^{(1)},\pi^{(1)}} \overline{P}_{G',\pi'}]\right|+ \left|\mathbb{E}[P_{G^{(1)},\pi^{(1)}}]\mathbb{E}[\overline{P}_{G',\pi'}] \right|+ \left|\mathbb{E}[P_{G^{(1)},\pi^{(1)}}]\mathbb{E}[x\overline{P}_{G',\pi'}] \right|\right] \enspace ,
\label{eq:upper:correlation}
\end{align}
where we used in the second line that conditionally on $z$ the expectation of $P_{G^{(1)},\pi^{(1)}} \overline{P}_{G',\pi'}$ is zero whenever $x=0$. 

We first bound the first term $\mathbb E[\overline{P}_{G^{(1)},\pi^{(1)}} \overline{P}_{G',\pi'}]$.  Since $\pi^{(1)}[G^{(1)}]$ and $\pi'[G']$ do not intersect, it follows from Lemma~\ref{lem:gen-2} that 
\begin{align}\nonumber    
 \frac{n^{|V|/2-1}}{\overline{q}^{|E|/2}} \left|\mathbb E[\overline{P}_{G^{(1)},\pi^{(1)}} \overline{P}_{G',\pi'}]\right|&\leq
c  D^{2}   \left(D^{c'}\frac{\lambda}{\sqrt{\overline{q}}}\right)^{|E|} n^{|V|/2-1} \left(\frac{D^{c' }k}{n}\right)^{|V| - \#\mathrm{CC}_{G}}\left(c \frac{D^{c'}}{\sqrt{n}} \right)^{\#\mathrm{CC}_{G}}\\ 
&\leq \frac{1}{2}\frac{k}{n} D^{-c_{0}|E|/2} \label{eq:upper:correlation-2} \enspace , 
\end{align}
where we used that  $c_0$ is large enough  and we argued as in Case 1. 
Let us turn to the second term in~\eqref{eq:upper:correlation}. By Condition~\condmoment, we have 
\begin{equation*}
\mathbb{E}[P_{G^{(1)},\pi^{(1)}}]\leq (D^{c'}\lambda)^{|E^{(1)}|}\left(D^{c'}\frac{k}{n}\right)^{|V^{(1)}|-1} \enspace . 
\end{equation*}
Also, if $G'$ has a single connected component, we have $\mathbb{E}[\overline{P}_{G',\pi'}]=0$. If $G'$ has at least two connected components, then $\mathbb{E}[\overline{P}_{G',\pi'}]$ is controlled by Lemma~\ref{lem:gen-2}. In particular, we have 
\begin{equation*}
\left|\mathbb{E}[\overline{P}_{G',\pi'}]\right|\leq  cD^{c'} \left(D^{c'}\lambda \right)^{|E'|}\left(\frac{D^{c'}k}{n}\right)^{|V'|-\#\mathrm{CC}_{G'}}\left(c\frac{D^{c'}}{\sqrt{n}}\right)^{\#\mathrm{CC}_{G'}} \enspace . 
\end{equation*}
We deduce from the two previous bounds that 
\begin{align}\nonumber
 \frac{n^{|V|/2-1}}{\overline{q}^{|E|/2}} \left|\mathbb{E}[P_{G^{(1)},\pi^{(1)}}]\mathbb{E}[\overline{P}_{G',\pi'}] \right|&\leq  c \frac{1}{\sqrt{n}}  D^{c'} \left(D^{c'}\frac{\lambda}{\sqrt{\overline{q}}} \right)^{|E|}\left(\frac{D^{c'}k}{\sqrt{n}}\right)^{|V|-\#\mathrm{CC}_{G}}\left(cD^{c'}\right)^{\#\mathrm{CC}_{G}-1}\\ \nonumber
 & \leq  \frac{k}{n}c^2 D^{3c'} \left(D^{c'}\frac{\lambda}{\sqrt{\overline{q}}} \right)^{|E|}\left(c\frac{D^{2c'}k}{\sqrt{n}}\right)^{|V|-1 -\#\mathrm{CC}_{G}} \\ \nonumber
 &\leq  \frac{k}{n}c^2 D^{3c'} \left(D^{c'}\frac{\lambda}{\sqrt{\overline{q}}} \right)^{|E|- |V|+1 +\#\mathrm{CC}_{G} }\left(c\frac{D^{3c'}k \lambda  }{\sqrt{n\overline{q}}}\right)^{|V|-1 -\#\mathrm{CC}_{G}} \\
 & \leq \frac{k}{4n}D^{-c_0 |E|/2}\enspace ,  \label{eq:upper:correlation-10}
\end{align}
where we used that $c_0$ is large enough in the statement of the theorem. 
To handle the last term $\left|\mathbb{E}[P_{G^{(1)},\pi^{(1)}}]\mathbb{E}[x\overline{P}_{G',\pi'}] \right|$, we control $\mathbb{E}[x\overline{P}_{G',\pi'}]$ by arguing as in Case 1-a. Since $G'$ contains neither $v_1$ nor $v_2$, we have, as in Case 1, 
\begin{equation*}
\mathbb{E}[x\overline{P}_{G',\pi'}]=\lambda^{-1}\mathbb{E}[\overline{P}_{G^{(0)},\pi^{(0)}}\overline{P}_{G',\pi'}]+\mathbb{E}[x]\mathbb{E}[\overline{P}_{G',\pi'}]\enspace ,
\end{equation*}
where $(G^{(0)},\pi^{(0)})$ is the single edge $(v_1,v_2)$ labeled by $(1,2)$. The first term is controlled as in Case 1-a, and the second one by $\mathbb{E}[x]\leq 2k/n$ and the above bound on $\mathbb{E}[\overline{P}_{G',\pi'}]$. Together with the above bound on $\mathbb{E}[P_{G^{(1)},\pi^{(1)}}]$, and arguing as for~\eqref{eq:upper:correlation-10}, this allows us to prove that  $\frac{n^{|V|/2-1}}{\overline{q}^{|E|/2}}\left|\mathbb{E}[P_{G^{(1)},\pi^{(1)}}]\mathbb{E}[x\overline{P}_{G',\pi'}] \right|\leq \frac{k}{4n}D^{-c_0 |E|/2}$. 
Then, gathering~\eqref{eq:upper:correlation}, \eqref{eq:upper:correlation-2}, \eqref{eq:upper:correlation-10}, and the bound on the last term, we conclude that 
\begin{equation*}
|\mathbb E[x \Psi^{(1,2)}_G]| \leq \frac{k}{n}D^{-c_0 |E|/2}\enspace . 
\end{equation*}
The result follows.

\end{proof}

\section{Proofs of the invariance properties}\label{sec:proof_invariance_lemmas}
\begin{proof}[Proof of Lemma~\ref{lem:reduction:permutation}]

Let $\mathrm{adv}_{\leq D}(f) = \frac{\mathbb{E}_{H_1}[f] }{\sqrt{\mathbb{E}[f^2]}}$ be the objective function optimized in the advantage. Suppose $f^{\star} \in  \arg\max \mathrm{adv}_{\leq D}$ is a polynomial of degree at most $D$ attaining the maximal advantage ---\ it is not hard to see that the maximum exists. Since $\mathrm{adv}_{\leq D}(-f)=-\mathrm{adv}_{\leq D}(f)$, we have $\mathbb{E}_{H_1}[f^{\star}]\geq 0$. Let $\sigma$ be any permutation of $[n]$. By the permutation invariance properties of the distributions under the null and the alternative hypotheses, it turns out  that  the polynomial  $Y\mapsto f^{\star}(Y_{\sigma})$ also maximizes the advantage. Defining the permutation invariant polynomial $f_{\mathrm{inv}}^{\star}$ by   $f_{\mathrm{inv}}^{\star}(Y)= \frac{1}{n!}\sum_{\sigma} f^{\star}(Y_{\sigma})$, we get that  
\begin{equation*}
    \mathrm{adv}_{\leq D}\left(\frac{1}{n!}\sum_{\sigma} f^{\star}(Y_{\sigma})\right) = \dfrac{\mathbb{E}_{H_1}\left[\frac{1}{n!}\sum_{\sigma} f^{\star}(Y_{\sigma})\right]}{\sqrt{\mathbb{E}\left[\left(\frac{1}{n!}\sum_{\sigma} f^{\star}(Y_{\sigma})\right)^2\right]}}\enspace . 
\end{equation*}
By invariance of $H_1$ with respect to permutations, the numerator is equal to $\mathbb{E}_{H_1}[f^{\star}(Y)]\geq 0$. For the denominator, by convexity of the square function and invariance with respect to permutations, we get 
\begin{equation*}
    \mathbb{E}\left[\left(\frac{1}{n!}\sum_{\sigma} f^{\star}(Y_{\sigma})\right)^2\right]\leq \mathbb{E}\left[(f^{\star}(Y))^2\right]\enspace. 
\end{equation*}
Therefore, $\mathrm{adv}_{\leq D}(f_{\mathrm{inv}}^{\star})\geq \mathrm{adv}_{\leq D}(f^{\star})$ and the advantage is maximized  by a permutation invariant function. 
\end{proof}

\begin{proof}[Proof of Lemma~\ref{lem:invariant:graph}]
Since each $Y_{ij}$ only takes the two values $-q$ and $1-q$, we have $Y_{ij}^2=(1-2q)Y_{ij}+q(1-q)$. Hence, as a function of $Y$, any polynomial of degree at most $D$ coincides with a multilinear polynomial of degree at most $D$. For a non-empty set $S$ of pairs $\{i,j\}\subset [n]$ with $|S|\leq D$, write $\phi_S(Y)=\prod_{\{i,j\}\in S}Y_{ij}$. Since the multilinear monomials are linearly independent functions on $\{-q,1-q\}^{n(n-1)/2}$, any $f\in \mathcal{P}^{\mathrm{inv}}_{\leq D}$ can be written in a unique way as 
\begin{equation}\label{eq:decomposition_f:basis}
f(Y)= \alpha_{\emptyset} + \sum_{S:\, 1\leq |S|\leq D} \alpha_{S} \phi_{S}(Y)\enspace . 
\end{equation}
Each such $S$ can be seen as a labeled graph without isolated nodes, and we write $G_S$ for the unique template of $\mathcal{G}_{\leq D}$ that is isomorphic to $S$. Given a permutation $\sigma$ of $[n]$, we have $\phi_S(Y_{\sigma})=\phi_{\sigma(S)}(Y)$, where $\sigma(S)=\{\{\sigma(i),\sigma(j)\}:\ \{i,j\}\in S\}$, so that $f(Y_{\sigma})= \alpha_{\emptyset} + \sum_{S}\alpha_{\sigma^{-1}(S)}\phi_{S}(Y)$. Since $f(Y_{\sigma})=f(Y)$, the uniqueness of~\eqref{eq:decomposition_f:basis} entails that $\alpha_{\sigma^{-1}(S)}=\alpha_S$ for all $S$ and all $\sigma$. As two sets $S$ and $S'$ satisfy $G_S=G_{S'}$ if and only if $S'=\sigma(S)$ for some permutation $\sigma$, $\alpha_S$ only depends on $G_S$, and we write $\alpha_S=\beta_{G_S}$. 

Fix a template $G=(V,E)\in \mathcal{G}_{\leq D}$. Since $|V|\leq 2D\leq n$, the map $\pi\mapsto \pi[G]$ from $\Pi_V$ to the collection of sets $S$ such that $G_S=G$ is onto, and each such $S$ has exactly $|\mathrm{Aut}(G)|$ preimages. Hence, $P_G=|\mathrm{Aut}(G)|\sum_{S:\, G_S=G}\phi_S$, and we obtain
\begin{equation}\label{eq:decomposition_f:basis2}
f(Y)= \alpha_{\emptyset} + \sum_{G \in \mathcal G_{\leq D}} \frac{\beta_{G}}{|\mathrm{Aut}(G)|} P_{G}(Y)\enspace . 
\end{equation}
Finally, the polynomials $P_G$, $G\in \mathcal{G}_{\leq D}$, are non-zero linear combinations of pairwise disjoint sets of monomials, so that $(1,(P_G)_{G\in \mathcal{G}_{\leq D}})$ is linearly independent and the decomposition~\eqref{eq:decomposition_f:basis2} is unique.
\end{proof}

\begin{proof}[Proof of Lemma~\ref{lem:reduction:degree2}]
Relying on the permutation invariance of the distribution $\mathbb{P}$, we can argue as in the proof of Lemma~\ref{lem:reduction:permutation} that there exists a polynomial $f$ with $\mathrm{deg}(f)\leq D$ that maximizes $\frac{\mathbb{E}[fx]}{\sqrt{\mathbb{E}[f^2]}}$ over all polynomials of degree at most $D$ and that is invariant under permutations of $\{3,\ldots, n\}$; in other words, for all permutations $\sigma:[n]\to[n]$ such that $\sigma(1)=1$ and $\sigma(2)=2$, we have $f(Y)=f(Y_{\sigma})$ where $Y_{\sigma}= (Y_{\sigma(i),\sigma(j)})$. Then, similarly to the proof of Lemma~\ref{lem:invariant:graph}, we check that $(1,(P^{(1,2)}_{G})_{G\in \mathcal{G}^{(1,2)}_{\leq D}} )$ spans the space of polynomials that are invariant under permutations of $\{3,\ldots, n\}$. Since $(1,(\Psi^{(1,2)}_{G})_{G\in \mathcal{G}^{(1,2)}_{\leq D}} )$ spans the same space, this allows us to conclude. 

\end{proof}

\end{document}